%% file: main.tex
\documentclass[a4paper, 11pt]{article}
\usepackage[margin=2.8cm]{geometry}
\usepackage{authblk}

\input{CommandsForACM.sty}
\input{exampleCommands.sty}
\usepackage{authblk}

\newcommand{\customfootnotetext}[2]{{% Group to localize change to footnote
		\renewcommand{\thefootnote}{#1}% Update footnote counter representation
		\footnotetext[0]{#2}}}% Print footnote text
\begin{document}

\title{Representational Systems Theory: A Unified Approach to Encoding, Analysing and Transforming Representations}

\author[1]{Daniel Raggi\textsuperscript{*}}%\thanks{daniel.raggi@cl.cam.ac.uk}}
%\email{daniel.raggi@cl.cam.ac.uk}
%\orcid{0000-0002-9207-6621}
\author[1]{Gem Stapleton\textsuperscript{*}}%\thanks{ges55@cam.ac.uk}}
%\email{ges55@cam.ac.uk}
%\orcid{0000-0002-6567-6752}
\author[1]{Mateja Jamnik}%\thanks{mateja.jamnik@cl.cam.ac.uk}}
%\email{mateja.jamnik@cl.cam.ac.uk}
%\orcid{0000-0003-2772-2532}
%\affiliation{%
%  \institution{University of Cambridge}
%  \streetaddress{Computer Laboratory, William Gates Building, JJ Thomson Avenue}
%  \city{Cambridge}
%  \country{UK}
%  \postcode{CB3 0FD}
%}
\author[2]{Aaron Stockdill}%\thanks{a.a.stockdill@sussex.ac.uk}}
%\email{a.a.stockdill@sussex.ac.uk}
%\orcid{0000-0003-3312-526}
\author[2]{Grecia Garcia Garcia}%\thanks{g.garcia-garcia@sussex.ac.uk}}
%\email{g.garcia-garcia@sussex.ac.uk}
%\orcid{0000-0002-7327-7225}
\author[2]{Peter C-H. Cheng}%\thanks{p.c.h.cheng@sussex.ac.uk}}
%\email{p.c.h.cheng@sussex.ac.uk}
%\orcid{0000-0002-0355-5955}
%\affiliation{%
%  \institution{University of Sussex}
%%\streetaddress{}
%  \city{Brighton}
%  \country{UK}
%}
\affil[1]{University of Cambridge, UK}
\affil[2]{University of Sussex, UK}
\maketitle
\customfootnotetext{*}{Authors Raggi and Stapleton contributed equally to this research.}

\begin{abstract}
The study of representations is of fundamental importance to any form of communication, and our ability to exploit them effectively is paramount. This article presents a novel theory -- \textit{Representational Systems Theory} --  that is designed to abstractly encode a wide variety of representations from three core perspectives: \textit{syntax}, \textit{entailment}, and their \textit{properties}. By introducing the concept of a \textit{construction space}, we are able to encode each of these core components under a single, unifying paradigm. Using our Representational Systems Theory, it becomes possible to structurally transform representations in one system into representations in another. An intrinsic facet of our structural transformation technique is \textit{representation selection} based on properties that representations possess, such as their relative cognitive effectiveness or structural complexity. A major theoretical barrier to providing \textit{general} structural transformation techniques is a lack of terminating algorithms. Representational Systems Theory permits the derivation of partial transformations when no terminating algorithm can produce a full transformation. Since Representational Systems Theory provides a universal approach to encoding representational systems, a further key barrier is eliminated: the need to devise system-specific structural transformation algorithms, that are necessary when different systems adopt different formalisation approaches. Consequently, Representational Systems Theory is the first general framework that provides a unified approach to encoding representations, supports representation selection via structural transformations, and has the potential for widespread practical application.
\end{abstract}

\section{Introduction}

\input{introduction}

\section{The Major Contributions}\label{sec:example}

\input{majorContributions}

%\pagebreak

%
\section{Preliminaries}\label{sec:prelims}
%
\input{preliminaries}

\section{Representational Systems}\label{sec:RSs}

\input{representationalSystems}

\section{Constructions}\label{sec:constructions}
%%%%

\input{constructions}
%%%
%%%

%\section{Propagation Functions}\label{sec:propagation} %% for another paper
%%
%\input{propagation} %% for another paper

%

\section{Describing Representational Systems Using Patterns}\label{sec:patterns}

\input{patterns}

\section{Structural Transformations}\label{sec:structuralTransformations}

\input{dotAndFOAcommands}

\input{structuralTransformations}

%
%\pagebreak
\section{Conclusion}\label{sec:conclusion}

\input{conclusion}

\section*{Acknowledgments}
 We thank Jean Flower, Alan Bundy and Peter Chapman for many helpful comments on earlier drafts of this paper. The research is funded by the UK EPSRC, grant numbers EP/T019603/1 and EP/T019034/1.
%\pagebreak

%\input{acmNotes}

\bibliographystyle{plain}
\bibliography{vmgbibliography}

\appendix
\section{Notation}\label{sec:defnsAndNotation}
\input{appendixNotation}

\section{Terminology}\label{sec:defnsTerms}
\input{appendixTerminology}
\section{Representational Systems: Appendix}\label{sec:app:representationalSystems}
\input{appendixRepresentationalSystems}
\section{Constructions: Appendix}\label{sec:app:constructions}

\input{appendixGeneralisedSplits}

\section{Patterns: Appendix}\label{sec:app:patterns}
\input{appendixPatterns}

\section{Structural Transformations: Appendix}\label{sec:app:structure}

\input{appendixStructuralTransformations}

\end{document}

%% file: introduction.tex
Representations, which we view as syntactic entities that have the capacity to represent objects when assigned semantics, are at the heart of how we communicate. They come in many forms, such as diagrams, natural language sentences, formal languages, programming languages, sketches, and many others; the variety is almost endless. The choice of representation is key to its efficacy, which of course is inherently linked to the purpose of creating the representation itself. For instance, the creator of the representation could simply want its reader to understand some fact. Perhaps the representation is intended to aid its reader with solving a particular problem. Whatever the purpose the representation is intended to serve, some choices of representations will be more suitable than others~\cite{COX1999343}. With an abundance of information available to us in this digital age, it is perhaps more important than ever to enable effective representation choices. This is a major motivation for the contribution of this paper, which presents \textit{Representational Systems Theory}. Our novel theory is designed to provide a fundamental, mathematical framework that facilitates the study and use of representations and representational systems.

Strictly, one could argue that a representation \textit{necessarily} represents some object, since that is what it means to \textit{represent}. In common usage, the word representation is often used to refer to syntactic entities without specifying the objects being represented, even though the intended objects need not be unambiguously identifiable. As this implies, representations can take many different semantics, which are determined by the context of use or the intention of the utterer (e.g. homonyms). Given this, the distinct concepts of \textit{representation} and \textit{token}, as seen in the major works of Peirce~\cite{peirce:cp}, are helpful to us. It is more precise to say that \textit{tokens} are syntactic entities that have the capacity to represent objects when assigned semantics. We use the terms \textit{token} and, following common usage, \textit{representation} interchangeably in our informal discussions but when speaking formally we use the technically precise term \textit{token} because we do not require any meaning to be assigned.

\begin{paragraph}{What is a Representational System?}
It is important to articulate the difference between a \textit{representation} (strictly, token) and a \textit{representational system}, since they are primary terms of reference throughout this paper. As just suggested, a representation is a syntactic entity that has the capacity to represent objects~\cite{peirce:cp}. By contrast, a \textit{representational system} is a collection of representations. Licato and Zhang take an object-oriented view of a representational system\footnote{Licato and Zhang's definition of a representational system is designed to provide a consistent language with which to describe research contributions in AI~\cite{licato:ersiai}. Thus, their motivation is different to ours; in part, our goal is to devise a formal theory that is capable of being exploited by AI engines to facilitate changes of representation.}, defining it to be a class comprising two parts: a finite set of members (i.e. representations), each of which is assigned a type, and a finite set of methods, which corresponds to what the members can do~\cite{licato:ersiai}. This, of course, is devoid of any notion of semantics, which is key to what it means for a representation \textit{to represent}, for which Licato and Zhang assume a function can be assigned that maps members to the represented objects.

Like Licato and Zhang, our formalisation of a representational system views it as comprising (potentially infinitely many) members -- the tokens -- that are assigned types. However, our conceptualisation of a representational system is at a finer level of granularity than Licato and Zhang's, by explicitly capturing: how tokens are syntactically built from other tokens in a \textit{grammatical} sense, \textit{entailment} relations between tokens (e.g. embodying deductive~\cite{sep-logic-classical}, inductive~\cite{sep-logic-inductive} or defeasible~\cite{koons:dr} reasoning), and \textit{meta-level properties} of representations\footnote{These properties can be anything that are contextually important about the represenation itself, e.g.: (a) for a word, one property is its length, (b) for a pair of representations,  one property is a measure of their relative cognitive complexity, and (c) for the curves in an Euler diagram, each of them has a (graphical) property that is the  colour it adopts.} that can be exploited by AI engines to, say, drive the selection of alternative representations. %We define each of these three parts of a representational system in section~\ref{sec:RSs}.

Indeed, our view is perhaps more closely aligned with that of Goldin, who posits that a representational system comprises \textit{primitive} ``characters or signs'' drawn from sets that may not be well-defined~\cite{GOLDIN1998137}. In turn, these primitives can be combined using potentially \textit{ambiguous rules} in order to produce \textit{permitted configurations}, i.e. new representations. Goldin explicitly acknowledges the existence of further ``ambiguous rules'' for ``moving from one [representation] ... to another'' in the context of inference.  Representational Systems Theory is directly aligned with Goldin's position and does not require  there to be formal rules that build new representations or define entailment relations: it is designed to be able to encode representational systems whose grammatical construction `rules' and entailment relations have not been precisely specified as well as those systems that are rigorously defined. In addition, Goldin observes that representational systems may have ``higher-level structure'' and points towards relations between representations, assigning semantics, and evaluating truth values as examples. This observation directly relates to our inclusion of meta-level properties that form part of a representational system. % Regarding semantics in particular, we view this is a `link' between representational systems, specifically a binary relation between their representations.

To reinforce some of the above remarks, there is no presupposition by Representational Systems Theory that any specific representational system is formal itself. The theory can be applied to, say: hand-drawn sketches, where there is no mathematically precise definition of a `well-formed sketch'; formal logics, where we are able to define what constitutes a well-formed formula; and natural languages which have understood, yet not completely defined, `rules' for producing words, phrases or sentences. We merely assume that any representation (strictly, token) is either a single entity in its own right (e.g. part of a sketch formed from a single pen stroke, a symbol, or a letter) or is \textit{constructed} from other representations. In this context, Representational Systems Theory provides a formal lens through which we can study both informal and formal representations. This is an important point, since it implies that  Representational Systems Theory can be applied to representations that arise in the real-world, which need not have any formal underpinning at all.
\end{paragraph}

%\subsection{Research Contributions}
\begin{paragraph}{Research Contribution} Having now framed our primary terms of reference, our major contributions can be summarised as follows:
\begin{enumerate}
\item[-] \textit{Contribution 1:} We introduce \textit{Representational Systems Theory}, which is \textit{the first general framework} designed to abstractly encode a wide variety of representations in a unified way from three fundamental perspectives: their grammatical construction, entailment relations between them, and their meta-level properties. A particular highlight is that Representational Systems Theory provides foundational tools for examining representational systems.
\item[-] \textit{Contribution 2:} Representational Systems Theory facilitates the development and implementation of \textit{representational-system-agnostic transformation algorithms}, removing the current need to devise system-specific algorithms, because of its unified nature. We present a novel approach to \textit{structural transformations} that is designed to support \textit{representation selection} based on desired structural relations and meta-level properties holding between representations. This removes a major theoretical barrier, since it is now no longer necessary to devise system-specific algorithms.
\item[-] \textit{Contribution 3:} Our novel approach to structure transfer \textit{permits the derivation of partial transformations} when no terminating algorithm can produce a complete transformation.  This is particularly important since the problem of finding a suitable transformation of one representation into another, subject to certain relationships holding between the two representations, is clearly non-terminating in general\footnote{For instance, one common example of such a transformation problem is to find a representation, $t'$ say, that is semantically equivalent to representation $t$. Establishing semantic equivalence is an example of a problem that is often undecidable and, thus, for which there is no terminating algorithm that guarantees to provide a `yes' or 'no' response. An algorithm that transformed $t$ into $t'$ would need to ensure that $t$ and $t'$ are semantically equivalent.}. Offering a partial transformation to an end-user may well provide sufficient information about a suitable target representation that they are able to complete the transformation process themselves.
\end{enumerate}
In essence, such is the universal nature of Representational Systems Theory, exploiting it to encode, analyse and transform representations is akin to exploiting category theory in the study of mathematical structures: both theories offer a high level of abstraction that can be broadly applied to the study of a wide range of similar entities.

There are practical reasons why our theoretical contributions have the potential to make a major impact, one of which we now discuss. Suppose an AI system is provided with a representation and is asked to identify a set of alternative representations from a variety of known representational systems; the importance of representation change and the potential role of AI in this context is discussed in~\cite{jamnik:emwtehatsrwah} and illustrated in Section~\ref{sec:example}. Based on the current state-of-the-art, one could develop system-specific transformation algorithms in order to explore a  range of possible alternative representations. Such an approach inherently leads to a requirement to devise multiple transformation algorithms, with their specific implementations depending on the systems that one is considering. This raises an immediate question: is it possible to devise a theoretical framework that enables a common approach to transforming one representation into another? Of course, affirmatively answering such a question would require a major theoretical advance in the study of representational systems. In particular, it requires a \textit{common language} with which to describe \textit{all} representations and representational systems. For, without such a language, transformation algorithms cannot be system-agnostic. Representational Systems Theory provides such a common language (developed in Section~\ref{sec:RSs}), based on a graph-theoretic conceptualisation, that permits all representations and representational systems that are encompassed by Goldin's thesis to be abstractly encoded and transformed.

Now, one could be forgiven for thinking that when we talk of transforming one representation into another, we are focusing on re-representation in the sense of \textit{semantic equivalence}. Indeed, the contextual examples that will be given in Section~\ref{sec:example} essentially correspond to transformation problems of that kind. Importantly, Representational Systems Theory goes far beyond supporting transformations from just this semantic perspective. It allows us to seek new representations that are in some \textit{specified relationship} with the original representation that we seek to transform (the relevant theoretical contributions are primarily given in Section~\ref{sec:structuralTransformations}, but they rely on all preceding sections). Such relations are varied, and can include -- to name a few possibilities -- being cognitively more effective (subject to a suitable model of cognitive cost being known\footnote{Providing such a model is beyond the scope of this paper. We are working towards a generic model,  not tied to any particular representational system,  that will estimate the relative cognitive costs of competing representations~\cite{cheng2021cognitive}.}), possessing different graphical properties, identifying representations that convey observable facts given the original representation~\cite{shimojima:spodatcp}, or more effectively supporting the derivation of a solution to a problem, via the entailment facilities that are available. This high degree of generality is a particular strength of Representational Systems Theory: allowing transformation algorithms to identify alternative representations, in any context where a change of representation is desired, leads to wide-ranging applicability.

The extent of our theory's applicability even goes beyond producing alternative representations. In fact, the process of finding a \textit{structural transformation} (Section~\ref{sec:structuralTransformations}) will identify the primitives from which a new representation is built and how those primitives are used to build the new representation. The most intuitive interpretation of these remarks lies at the grammatical level: a transformation will identify new primitive representations (called \textit{foundations} in our theoretical development) and establish how they are combined to produce the desired new representation. In an entailment sense, the new primitives correspond to given facts, and a transformation establishes the way in which they are used via entailment relations to infer other representations, that is, how we \textit{move from one representation to another} when solving problems. Thus, if we are provided with a \textit{solution to a problem} then our theory supports the derivation of alternative solutions, in other representational systems, via transformation algorithms.
\end{paragraph}

%%% rep choice was here %%%

\begin{paragraph}{Organisation}
The paper is structured as follows. Section~\ref{sec:example} begins by discussing the major contributions of the paper in more depth. Section~\ref{sec:prelims} covers fundamental, pre-existing terminology that is exploited throughout the paper and exemplifies notational conventions. The theory of representational systems is developed in Section~\ref{sec:RSs}, culminating in a mathematically precise conceptualisation of a representational system, delivering the first major contribution of the paper. A noteworthy aspect of our conceptualisation is the development of \textit{construction spaces}, which provide a unified approach to defining the grammatical, entailment and meta-level properties of representational systems\footnote{A representational system will thus comprise three construction spaces: grammatical and entailment spaces alongside an \textit{identification} space, for the meta-level properties.}. Our second and third major contributions are the derivation of general methods for transforming one representation into another and, respectively, providing a partial transformation when it is not possible to identify a complete transformation. In this context, we require an approach that \textit{decomposes} representations into \textit{other representations that can be constructed in simpler ways}. Thus, Section~\ref{sec:constructions} develops the theory of \textit{constructions} of representations and explains in some detail how to \textit{decompose} constructions into smaller constructions. Section~\ref{sec:patterns} develops the theory of \textit{patterns}, used to describe classes of constructions, which play a core role in structural transformations. Exploiting Sections~\ref{sec:constructions} and~\ref{sec:patterns}, the theory of structural transformations is given in Section~\ref{sec:structuralTransformations}, thus delivering the second and third major constributions. We conclude in section~\ref{sec:conclusion} where we also discuss future work on the development of AI tools that exploit Representational Systems Theory to invoke structural transformations.

Lastly, we point the reader towards the extensive appendices. To aid readability, the appendices include a summary of the novel notation (Appendix~\ref{sec:defnsAndNotation}) and terminology (Appendix~\ref{sec:defnsTerms}) that we introduce: a large number of new concepts are presented. Further references to the appendices are made at appropriate places throughout the paper.
\end{paragraph}

%% file: majorContributions.tex
We now set about providing further discussion, alongside a range of examples, in order to exemplify our major contributions. For technical precision, throughout this section we use the term \textit{token} rather than \textit{representation} to avoid any potential ambiguity that may arise. Thus, a representational system can be considered as comprising tokens (including grammatical information about how they are formed from other tokens), entailment relations between tokens, and information about properties of tokens; a formal definition of such a system will be developed in section~\ref{sec:RSs}, but for now this informal perspective is sufficient.

\begin{paragraph}{Contribution 1: Representational Systems Theory as a general framework for abstractly encoding representational systems.}
In order to devise a general framework that provides a common language, permitting all representational systems encompassed by Goldin's thesis to be abstractly encoded, we consider, generically, how:
\begin{enumerate}
\item tokens are formed in a grammatical sense,
\item entailment relations arise between tokens, and
\item meta-level properties of tokens can be captured.
\end{enumerate}
A question immediately arises: are we able to formalize each of these three `layers' in a way that can be applied to all representational systems? To answer this question, one must think abstractly about each layer in order to identify commonality across a wide range of representational systems, which can have diverse syntax and, as highlighted by Goldin, need not even be formal systems in their own right. In addition to this concern, the three \textit{different} layers of a representational system each focus on a particular fundamental aspect. We ask ourselves: can we find common structure in these three layers that would allow us to understand them under a single, unifying, paradigm? As we proceed, we illustrate that the answer to both of these questions is yes, and this commonality is captured via \textit{construction spaces} in Section~\ref{sec:RSs}. % where we link our observations to the Curry-Howard correspondence.

Firstly, then, let us consider the grammatical layer. Below are four tokens that are varied in nature, with respect to their syntax and level of formality. The right-most token is an area diagram and the second token is a probability statement. There is a metro map (without station names), followed by a hand-sketched, incomplete, class \begin{samepage}diagram.
\begin{center}
\adjustbox{valign=c, scale = 0.98}{\areasDiagram}
\hfill
\adjustbox{valign=c}{\small$P(D) = 0.04$}
\hfill
\adjustbox{valign=c, scale = 0.63}{\metro}\label{metromap}
\hfill
\adjustbox{valign=c}{\includegraphics[scale=0.68]{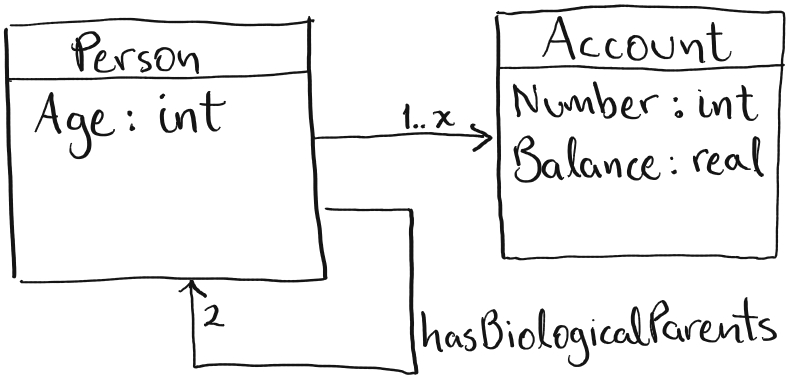}}
%\hspace{20pt}
%\AsubBdisjC
\end{center}\end{samepage}
In each case, the tokens are \textit{constructed} from other tokens, as illustrated in the four graphs, called \textit{constructions}\footnote{One might have expected some of these constructions to contain further vertices, showing how their leaves were constructed from other, simpler, tokens. However, this is not a requirement of our theory. In a construction, one can choose how much information to include about how a token is constructed.}, \newcommand{\metroscale}{0.52}\begin{samepage}below:
\begin{center}
	\adjustbox{valign=c}{%
	\begin{tikzpicture}[construction]
	\node[termrep, anchor=center] (t) {\scalebox{\metroscale}{\areasDiagram}};
	\node[constructorS = {$\texttt{superImpose}$}, below = 0.4cm of t] (u) {};
		\node[termrep, below left = 1.1cm and 1.8cm of u,anchor=center] (ta) {\scalebox{\metroscale}{\areasA}};
		\node[constructor = {$\texttt{superImpose}$}, below = 0.4cm of ta] (ua) {};
			\node[termrep, below left = 1.1cm and 0.8cm of ua,anchor=center] (taa) {\scalebox{\metroscale}{\areasAA}};
			\node[termrep, below right = 1.1cm and 0.8cm of ua,anchor=center] (tab) {\scalebox{\metroscale}{\areasAB}};
			\node[constructor = {$\texttt{superImpose}$}, below = 0.4cm of tab, xshift = 0.2cm] (uab) {};
				\node[termS, below left = 1.1cm and 0.8cm of uab,anchor=center] (taba) {\scalebox{\metroscale}{\areasABA}};
				\node[termS, below right = 1.1cm and 0.8cm of uab,anchor=center] (tabb) {\scalebox{\metroscale}{\areasABB}};
		\node[termrep, below right = 1.1cm and 1.8cm of u,anchor=center] (tb) {\scalebox{\metroscale}{\areasB}};
		\node[constructor = {$\texttt{superImpose}$}, below = 0.4cm of tb] (ub) {};
			\node[termrep, below left = 1.1cm and 0.8cm of ub,anchor=center] (tba) {\scalebox{0.5}{\areasBA}};
			\node[termrep, below right = 1.1cm and 0.8cm of ub,anchor=center] (tbb) {\scalebox{0.5}{\areasBB}};
	\path[->] (u) edge (t)
	(ta) edge[out = 70, in = 180] node[index label]{1} (u)
	(tb) edge[out = 110, in = -0] node[index label]{2} (u)
	(ua) edge (ta)
	(taa) edge[out = 80, in = 190] node[index label]{1} (ua)
	(tab) edge[out = 100, in = -10] node[index label]{2} (ua)
	(uab) edge ([xshift = 0.2cm]tab.south)
	(taba) edge[out = 80, in = 190] node[index label]{1} (uab)
	(tabb) edge[out = 100, in = -10] node[index label]{2} (uab)
	(ub) edge (tb)
	(tba) edge[out = 80, in = 190] node[index label]{1} (ub)
	(tbb) edge[out = 100, in = -10] node[index label]{2} (ub);
	\end{tikzpicture}}\hfill
\adjustbox{valign=c}{%
	\begin{tikzpicture}[construction]\footnotesize
	\node[termrep] (t) {$P(D) = 0.04$};
	\node[constructor = $\texttt{applyInfix}$,below = 0.6cm of t] (u) {};
	\node[termrep,below left = 0.6cm and 0.7cm of u] (t1) {$P(D)$};
	\node[constructor = $\texttt{applyUnary}$,below = 0.6cm of t1] (u1) {};
	\node[termrep,below left = 0.6cm and 0.7cm of u1] (t11) {$P$};
	\node[termrep,below left = 0.7cm and 0.05cm of u1] (t12) {$($};
	\node[termrep,below right = 0.75cm and -0.05cm of u1] (t13) {$D$};
	\node[termrep,below right = 0.55cm and 0.65cm of u1] (t14) {$)$};
	\node[termrep,below = 0.7cm of u] (t2) {\vphantom{j}$=$};
	\node[termrep,below right = 0.6cm and 0.7 of u] (t3) {$0.04$};
	\node[constructorNE = $\texttt{buildDecimal}$,below = 0.6cm of t3] (u3) {};
	\node[termrep,below left = 0.6cm and 0.4cm of u3] (t31) {$0$};
	\node[termrep,below = 0.7cm and 0.05cm of u3] (t32) {$.$};
	\node[termrep,below right = 0.6cm and 0.3cm of u3] (t33) {$04$};
	\node[constructorpos = {$\texttt{buildDigitSeq}$}{165}{0.75cm},below = 0.6cm of t33] (u33) {};
	\node[termrep,below left = 0.6cm and 0.05cm of u33] (t331) {$0$};
	\node[termrep,below right = 0.6cm and 0.1cm of u33] (t332) {$4$};
	\path[->] (u) edge (t)
	(t1) edge[bend left = 10] node[index label]{1} (u)
	(t2) edge node[index label]{2} (u)
	(t3) edge[bend right = 10] node[index label]{3} (u)
	(u1) edge (t1)
	(t11) edge[bend left = 10] node[index label]{1} (u1)
	(t12) edge[bend left = 10] node[index label]{2} (u1)
	(t13) edge[bend right = 10] node[index label]{3} (u1)
	(t14) edge[bend right = 10] node[index label]{4} (u1)
	(u3) edge (t3)
	(t31) edge[bend left = 10] node[index label]{1} (u3)
	(t32) edge node[index label]{2} (u3)
	(t33) edge[bend right = 10] node[index label]{3} (u3)
	(u33) edge (t33)
	(t331) edge[bend left = 10] node[index label]{1} (u33)
	(t332) edge[bend right = 10] node[index label]{2} (u33);
	\end{tikzpicture}}
\end{center}\end{samepage}
\begin{center}
	\adjustbox{valign=c}{%
	\begin{tikzpicture}[construction]
	\node[termrep] (t) {\scalebox{\metroscale}{\metro}};
	\node[constructor = {$\texttt{addLine}$}, below = 0.5cm of t] (u) {};
	\node[termrep, below left = 0.3cm and -0.1cm of u] (t1) {\scalebox{\metroscale}{\metroA}};
	\node[termrep, below right = 0.3cm and 0.0cm of u] (t2) {\scalebox{\metroscale}{\metroB}};
	\node[constructorNE = {$\texttt{addStop}$}, below = 0.5cm of t2,xshift = -0.7cm] (u2) {};
	\node[termrep, below left = 0.3cm and 0.6cm of u2] (t21) {\scalebox{\metroscale}{\metroBA}};
	\node[termrep, below right = 0.3cm and -0.7cm of u2] (t22) {\scalebox{\metroscale}{\metroBB}};
	\path[->] (u) edge (t)
			(t1) edge[bend left = 15] node[index label]{1} (u)
			(t2) edge[bend right = 15] node[index label]{2} (u)
			(u2) edge ([xshift = -0.7cm]t2.south)
			(t21) edge[out = 60, in = -170, looseness = 0.9] node[index label]{1} (u2)
			(t22) edge[out = 95, in = -10, looseness = 0.9] node[index label]{2} (u2);
	\end{tikzpicture}}\hfill
\adjustbox{valign=c}{%
	\begin{tikzpicture}[construction]
	\node[termrep] (t) {\includegraphics[scale=0.54]{UML-sketch.png}};
	\node[constructor = {$\texttt{addArrow}$}, below = 0.5cm of t] (u) {};
	\node[termrep, below left = 0.3cm and 0.2cm of u] (t1) {\includegraphics[scale=0.49]{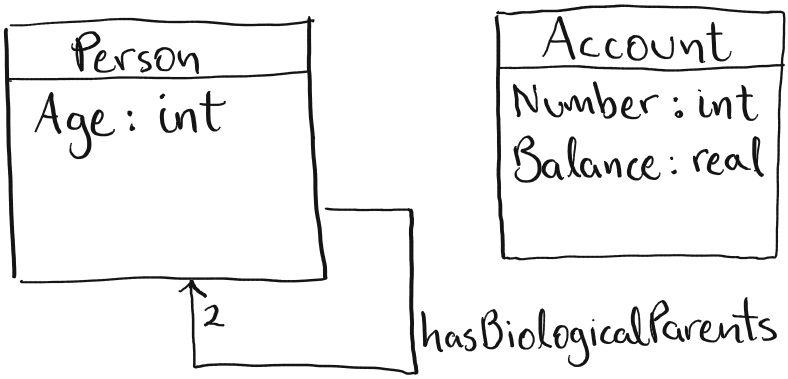}};
	\node[constructor = {$\texttt{addClassBox}$}, below = 0.5cm of t1,xshift=0.1cm] (u1) {};
	\node[termrep, below left = 0.6cm and -1.6cm of u1] (t11) {\includegraphics[scale=0.49]{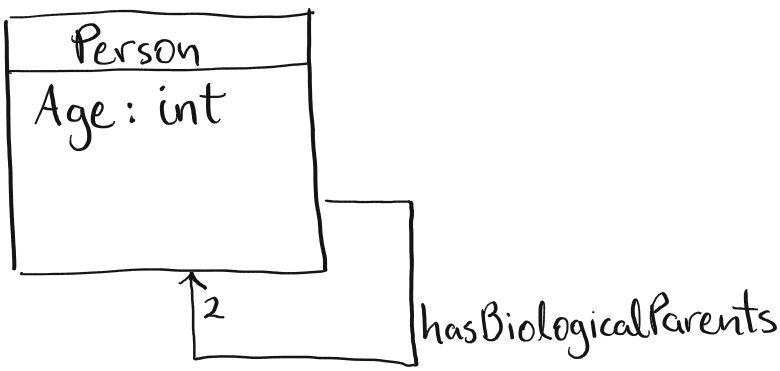}};
	\node[termrep, below right = -0.3cm and 1.5cm of u1] (t12) {\includegraphics[scale=0.49]{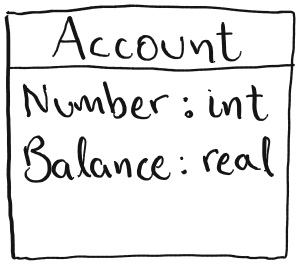}};
	\node[constructorNE = {$\texttt{addAttribute}$}, below = 0.5cm of t12] (u12) {};
	\node[termrep, below left = 0.6cm and -0.7cm of u12] (t121) {\includegraphics[scale=0.49]{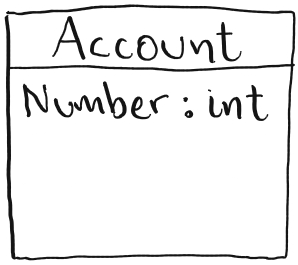}};
	\node[termrep, below right = 0.4cm and 0.5cm of u12] (t122) {\includegraphics[scale=0.49]{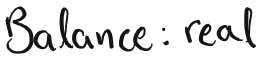}};
	\node[termrep, below right = -0.2cm and 1.1cm of u] (t2) {\includegraphics[scale=0.49]{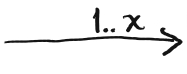}};
	\node[constructorNW = {$\texttt{numberArrow}$}, below = 0.5cm of t2] (u2) {};
	\node[termrep, below left = 0.5cm and 0.0cm of u2] (t21) {\includegraphics[scale=0.49]{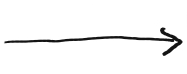}};
	\node[termrep, below right = 0.6cm and 0.1cm of u2] (t22) {\includegraphics[scale=0.49]{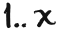}};
	\path[->] (u) edge (t)
	(t1) edge[out = 70, in = -175] node[index label]{1} (u)
	(t2) edge[out = 175, in = -5] node[index label]{2} (u)
	(u1) edge ([xshift=0.1cm]t1.south)
	(t11) edge[bend left = 15] node[index label]{1} (u1)
	(t12) edge[out = 158, in = -5] node[index label]{2} (u1)
	(u12) edge (t12)
	(t121) edge[bend left = 15] node[index label]{1} (u12)
	(t122) edge[out = 130, in = -10] node[index label]{2} (u12)
	(u2) edge (t2)
	(t21) edge[bend left = 15] node[index label]{1} (u2)
	(t22) edge[bend right = 15] node[index label]{2} (u2);
	\end{tikzpicture}}
\end{center}
%\begin{center}
%	\adjustbox{valign=c}{%
%	\begin{tikzpicture}[construction]\small
%	\node[termrep](t){\scalebox{0.5}{\AsubBdisjC}};
%	\node[constructor= {unify}, below = 0.5cm of t](u){};
%	\node[termrep, below left = 0.4cm and 0.6cm of u](t1){\scalebox{0.75}{\AsubB}};
%	\node[constructorNW = {subsetMerge}, below = 0.5cm of t1,xshift = 0.2cm](u1){};
%	\node[termrep, below left = 0.4 and 0.2cm of u1](t11){\scalebox{0.75}{\Acirc}};
%	\node[termrep, below right = 0.4 and 0.7cm of u1](t13){\scalebox{0.75}{\Bcirc}};
%	\node[constructor = {labelCurve}, below = 0.5 of t13](u13){};
%	\node[termrep, below left = 0.3 and 0.4cm of u13](t131){\scalebox{0.75}{\Ucirc}};
%	\node[termrep, below right = 0.6 and 0.2cm of u13](t132){$B$};
%	\node[termrep, below right = 0.4cm and 0.4cm of u](t3){\scalebox{0.75}{\BdisjC}};
%	\node[constructorNE = {disjointMerge}, below = 0.5cm of t3,xshift = -0.3cm](u2){};
%	\node[termrep, below right = 0.4 and 0.3cm of u2](t33){\scalebox{0.75}{\Ccirc}};
%	\path[->] (u) edge (t)
%	(t1) edge[bend left = 15] node[index label]{1} (u)
%	(u1) edge ([xshift = 0.2cm]t1.south)
%	(t11) edge[bend left = 10] node[index label]{1} (u1)
%	(t13) edge[bend right = 10] node[index label]{2} (u1)
%	(t3) edge[bend right = 15] node[index label, pos = 0.5]{2} (u)
%	(u2) edge ([xshift = -0.3cm]t3.south)
%	(t13) edge[bend left = 10] node[index label]{1} (u2)
%	(t33) edge[bend right = 10] node[index label]{2} (u2)
%	(u13) edge (t13.south)
%	(t131) edge[bend left = 10] node[index label]{1} (u13)
%	(t132) edge[bend right = 10] node[index label]{2} (u13);
%	\end{tikzpicture}}
%\end{center}
For instance, the area diagram can be built by superimposing two other area diagrams which, in turn, are built from other tokens. The arrows are numbered, indicating the order in which the \textit{input tokens} are `used' in order to `produce' the output token. To illustrate, $P(D) = 0.04$, which is the output from the \texttt{applyInfix} vertex, is formed from three other tokens, namely $P(D)$, $=$, and $0.04$; switching the indices on the associated arrows numbered $1$ and $3$ would instead yield $0.04=P(D)$.

Extrapolating from these illustrations, it is evident that \textit{all tokens} can be considered as either composite tokens (a token that is built from other tokens) or primitive tokens (a token that is not built from other tokens). Thus, we can consider the grammatical layer to be, essentially, a collection of relations between tokens: for each composite token, $t$, there is a collection of other tokens, $t_1,...,t_n$, from which $t$ is built; in turn, each $t_i$ could be primitive or composite. These observations are invariant under representational system choice. Regardless of the system, we can view the construction of a token as a relation, $c$ (for construction) between $t_1,...,t_n$ and $t$: $c(t_1,...,t_n,t)$\footnote{In our formalization, we will say that the $c$ has input-sequence $[t_1,...,t_n]$ and output $t$, which means, informally, that $c$ constructs $t$ from the ordered sequence of tokens $[t_1,...,t_n]$.}. In Representational Systems Theory, each of these relations is captured by the concept of a \textit{constructor} (definition~\ref{defn:constructionSpecification}) and they will all be encoded using a \textit{structure graph} (definition~\ref{defn:structureGraph}) within the context of a \textit{construction space} (definition~\ref{defn:constructionSpace}).

Similar observations can be readily made about entailments: they are widely viewed as relations where tokens $t_1,...,t_n$ entail $t$. It is easy, then, to see that one can formalize the grammatical and entailment layers of any representational system in a unified way. Our remaining consideration relates to meta-level properties of tokens, for which further examples are useful. Consider, first, the metro map token, on page~\pageref{metromap}, that includes six coloured routes. Two of its meta-level properties are encoded below, using constructions that include \textit{meta-tokens}; in the graph, vertices of the form \tikz[construction,baseline=-3pt]{%
\node[termIrep] (u) {$$}} are used for meta-tokens whereas \tikz[construction,baseline=-3pt]{
\node[termrep] (u) {$$}} is used for non-meta-tokens.{\pagebreak} The first construction identifies that the purple route has 8 stops and the other construction encodes the property that the purple and black routes share two \begin{samepage}stops:
\begin{center}
\adjustbox{valign = c}{\begin{tikzpicture}[construction]\footnotesize
	\node[termIrep] (t) {$8$};
	\node[constructorGE = {$\texttt{numberOfStops}$}, below = 0.4cm of t] (u) {};
	\node[termrep, below = 0.55cm of u] (t1) {\scalebox{0.5}{\metroB}};
	\path[->] (u) edge (t)
	(t1) edge[bend left = 0] node[index label]{1} (u);
	\end{tikzpicture}}
\hspace{1cm}
\adjustbox{valign = c}{\begin{tikzpicture}[construction]\footnotesize
\node[termIrep] (t) {$2$};
\node[constructor = {$\texttt{numberOfSharedStops}$}, below = 0.4cm of t] (u) {};
\node[termrep, below left = 0.3cm and 0cm of u] (t1) {\scalebox{0.5}{\metroB}};
\node[termrep, below right = 0.3cm and 0cm of u] (t2) {\scalebox{0.5}{\metroC}};
\path[->] (u) edge (t)
(t1) edge[out = 60, in = -170] node[index label]{1} (u)
(t2) edge[out = 120, in = -10] node[index label]{2} (u);
\end{tikzpicture}}
\end{center}
\end{samepage}
Selected properties of other tokens on page~\pageref{metromap} can be identified by the following \begin{samepage}constructions:
\begin{center}
\adjustbox{valign = c}{\begin{tikzpicture}[construction]\footnotesize
	\node[termIrep] (t) {$\bot$};
	\node[constructor = {$\texttt{ratioOfAreasEquals}$}, below = 0.4cm of t] (u) {};
	\node[termrep, below left = 0.3cm and 0.5cm of u] (t1) {\scalebox{0.5}{\areasBB}};
	\node[termrep, below right = 0.5cm and -0.6cm of u] (t2) {\scalebox{0.5}{\areasAA}};
	\node[termrep, below right = 0.5cm and 1.3cm of u] (t3) {$24.5$};
	\path[->] (u) edge (t)
	(t1) edge[out = 80, in = -170] node[index label]{1} (u)
	(t2) edge[out = 80, in = -70] node[index label]{2} (u)
	(t3) edge[out = 95, in = -10] node[index label]{3} (u);
	\end{tikzpicture}}
\hspace{0.5cm}
\adjustbox{valign = c}{\begin{tikzpicture}[construction]\footnotesize
\node[termIrep] (t) {$\top$};
\node[constructor = {$\texttt{areaEquals}$}, below = 0.4cm of t] (u) {};
\node[termrep, below left = 0.3cm and -0.1cm of u] (t1) {\scalebox{0.5}{\areasABA}};
\node[termrep, below right = 0.4cm and 0.3cm of u] (t2) {$0.04$};
\path[->] (u) edge (t)
(t1) edge[out = 80, in = -170] node[index label]{1} (u)
(t2) edge[out = 95, in = -10] node[index label]{2} (u);
\end{tikzpicture}}
\hspace{0.5cm}
\adjustbox{valign = c}{\begin{tikzpicture}[construction]\footnotesize
\node[termIrep] (t) {$\top$};
\node[constructor = {$\texttt{isSubtermOf}$}, below = 0.4cm of t] (u) {};
\node[termrep, below left = 0.4cm and 0.22cm of u] (t1) {$P(D)$};
\node[termrep, below right = 0.4cm and -0.1cm of u] (t2) {$P(D) = 0.04$};
\path[->] (u) edge (t)
(t1) edge[out = 80, in = -170] node[index label]{1} (u)
(t2) edge[out = 100, in = -10] node[index label]{2} (u);
\end{tikzpicture}}
\end{center}
\begin{center}
\adjustbox{valign = c}{\begin{tikzpicture}[construction]\footnotesize
	\node[termIrep] (t) {Account};
	\node[constructorGE = {$\texttt{className}$}, below = 0.4cm of t] (u) {};
	\node[termrep, below = 0.55cm of u] (t1) {\includegraphics[scale=0.49]{UML-sketch12.png}};
	\path[->] (u) edge (t)
	(t1) edge[bend left = 0] node[index label]{1} (u);
	\end{tikzpicture}}
\hspace{1.5cm}
\adjustbox{valign = c}{\begin{tikzpicture}[construction]\footnotesize
	\node[termIrep] (t) {$\top$};
	\node[constructor = {$\texttt{includesClass}$}, below = 0.4cm of t] (u) {};
	\node[termrep, below left = 0.3cm and -0.2cm of u] (t1) {\includegraphics[scale=0.49]{UML-sketch.png}};
	\node[termIrep, below right = 0.4cm and 0.3cm of u] (t2) {Person};
	\path[->] (u) edge (t)
	(t1) edge[out = 70, in = -170] node[index label]{1} (u)
	(t2) edge[out = 95, in = -10] node[index label]{2} (u);
	\end{tikzpicture}}
\end{center}\end{samepage}
Note that the bottom right construction includes a meta-token as an arrow's source: {Person} (which is type-written text) is not part of the representational system that contains hand-sketched class diagrams. This particular construction identifies the fact that the sketched class diagram includes Person as one of its class names.

Now, as well as identifying properties of tokens \textit{within} a given representational system, constructions can also be used to identify properties of tokens \textit{across} systems. Perhaps the most prominent application of this kind of inter-representational-system-encoding arises when semantics are assigned to tokens. As we can see below, as well as semantic relations, we can encode \textit{observability}\footnote{Example~\ref{ex:rs:encodingProperties} discusses the notion of observability in more detail.} relations using \begin{samepage}constructions:
\begin{center}
\begin{minipage}[t]{0.23\textwidth}\centering
	\textit{encoding semantics:}\\[1ex]
	\begin{tikzpicture}[construction]
	\node[termIrep] (v) at (3.2,4.2) {\footnotesize$\top$};
	\node[constructorINW={$\texttt{represents}$}] (u) at (3.2,3.4) {};
	\node[termrep] (v2) at (2.5,2.35) {\scalebox{0.8}{\abcDiagram}};
	\node[termrep] (v4) at (3.8,1.7) {\footnotesize$\{A\subseteq B, B\cap C=\emptyset\}$};
	\path[->]
	(u) edge[bend right = 0] (v)
	(v2) edge[bend right = -5] node[index label] {1} (u)
	(v4) edge[bend right = 10] node[index label] {2} (u);
	\end{tikzpicture}
\end{minipage}
\hspace{1cm}
\begin{minipage}[t]{0.4\textwidth}\centering
	\textit{encoding observability:}\\[1ex]
	\begin{tikzpicture}[construction]
	\node[termIrep] (v) at (3.2,4.2) {\footnotesize$\top$};
	\node[constructorIpos={\constructorObserve}{0}{0.58cm}] (u4) at (2.8,3.4) {};
	\node[constructorIpos={\constructorObserve}{143}{0.3cm}] (u3) at (1.7,3.6) {};
	\node[constructorIpos={\constructorObserve}{32}{0.38cm}] (u5) at (4.6,3) {};
	\node[termrep] (v4) at (3.3,2.4) {\footnotesize$B\cap C=\emptyset$};
	\node[termrep] (v3) at (1.3,2.5) {\footnotesize$A\subseteq B$};
	\node[termrep] (v5) at (4.9,1.7) {\footnotesize$A\cap C=\emptyset$};
	\node[termrep] (v1) at (2.3,1.7) {\scalebox{0.8}{\abcDiagram}};
	\path[->]
	(u4) edge[bend right = 0] (v)
	(u3) edge[bend right = -15] (v)
	(u5) edge[bend right = 20] (v)
	(v1) edge[in=235,out=85] node[index label] {1} (u4)
	(v1) edge[bend right = 5] node[index label] {1} (u3)
	(v1) edge[bend right = 35] node[index label] {1} (u5)
	(v4) edge[bend right = 10] node[index label] {2} (u4)
	(v3) edge[bend right = -10] node[index label] {2} (u3)
	(v5) edge[bend right = 5] node[index label] {2} (u5);
	\end{tikzpicture}
\end{minipage}
\end{center}
\end{samepage}

It is now evident that \textit{properties of tokens} can also be viewed as relations: a token,{\pagebreak} or collection of tokens, can be identified as having a certain property, $p$, using a constructor $c$, that is $c(t_1,...,t_n,p)$. \begin{samepage}For example, we have:
\begin{itemize}
\item[-] $\texttt{numberOfStops}\left(\scalebox{0.2}{\tikz[baseline=-12pt]{
		\clip(-1.5,-1.4) rectangle (3.4,1.7);
		\node () {\metroB};}},8\right)$,
\item[-] $\texttt{includesClass}\left(\adjustbox{raise=-0.3cm}{\includegraphics[scale=0.25]{UML-sketch.png}}\hspace{0.04cm},\textup{Person},\top\right)$, and
\item[-] $\texttt{represents}\left(\adjustbox{scale=0.9,raise=-0.2cm}{\abcdDiagram},\{A\subseteq B, D\subseteq C, B\cap C=\emptyset\}, \top\right)$.
\end{itemize}
\end{samepage}
The fact that all three layers can be viewed, at a certain level of abstraction, as having essentially the same structural properties (regardless of the representational system in question), represents a significant insight: our theory is unified in terms how each layer is formalised and, in general, it applies to all representational systems encompassed by Goldin's thesis. In summary, by exploiting our novel idea of a construction space to formalise each layer, Representational Systems Theory provides a general framework that can be used to abstractly encode representational systems.
\end{paragraph}

\begin{paragraph}{Contribution 2: Representational-system-agnostic transformation specifications for identifying complete transformations that ensure desired structural relations and meta-level properties hold.} Here we focus on how our theory exploits construction spaces to facilitate the transformation of one token into another; below, under \textit{Choices of Tokens: Motivation for Change}, we explore different choices of token and discuss the desire to be able to switch between them. Now, as remarked on above, a construction space will include a \textit{structure graph}, $\graphn$, that encodes relations between tokens; a structure graph is, essentially, a collection of constructions. As we have just witnessed, for any given token, $t$, a construction that is a sub-graph of $\graphn$ will capture from which tokens the token $t$ is \textit{constructed}. The constructions below of the semantically equivalent tokens  $\{A\subseteq B, B\cap C=\emptyset\}$ and \adjustbox{scale=0.61,raise=-0.08cm}{\abcDiagram} will be used to exemplify Representational Systems Theory's approach to structure \begin{samepage}transfer:
\begin{center}
\adjustbox{scale = 0.8,valign=c}{\STConstruction} \hspace{15pt}
\adjustbox{scale = 0.8,valign=c}{\EDConstruction}
\end{center}
\end{samepage}
In each construction, tokens are sources of indexed arrows that target vertices labelled by constructors; e.g. $A$, $\subseteq$ and $B$ are sources of arrows that target the vertex  labelled by the constructor \texttt{assertSubset}, and they are used to construct the output $A\subseteq B$ (in terms of relations, we can think of this as $\texttt{assertSubset}(A,\subseteq, B, A\subseteq B)$. Notice that the construction of $\{A\subseteq B, B\cap C=\emptyset\}$ uses four different constructors: \texttt{asSet}, \texttt{assertSubset}, \texttt{assertEmpty}, and \texttt{intersect}. By thinking about how the four constructors correspond to how one might construct parts of the Euler diagram, we are able to transform $\{A\subseteq B, B\cap C=\emptyset\}$ into \adjustbox{scale=0.61,raise=-0.08cm}{\abcDiagram}\!. In particular, we can \textit{split} the construction of $\{A\subseteq B, B\cap C=\emptyset\}$ into parts, each of which can subsequently be transformed into part of the required Euler diagram. One way of splitting each of the two constructions apart gives us the \textit{decompositions} below. The labels assigned to the parts of the decomposition on the left are then used on the right to show how the parts can be individually transformed into an Euler diagram.
\begin{center}
\adjustbox{scale=0.6,valign=c}{\STDecompositionTree}\hfill
\adjustbox{scale=0.6,valign=c}{\EDDecompositionTree}
\end{center}
By effecting this transformation, we produce an Euler diagram -- namely \adjustbox{scale=0.61,raise=-0.08cm}{\abcDiagram}\! --  that allows us to \textit{observe} facts that we would need to derive from the set-theoretic token, $\{A\subseteq B, B\cap C=\emptyset\}$, such as $A\cap C=\emptyset$. Thus, transformations such as this could, for some tasks, produce cognitively more effective tokens. We return to this example in Section~\ref{sec:cons:decompositions}, where we work through the steps involved in the transformation, including explaining why not all parts of the decomposition of $\{A\subseteq B, B\cap C=\emptyset\}$ need to be transformed into Euler diagrams.

Some general remarks are in order. Our structural transformation process exploits the unified nature of Representational Systems Theory: it is sufficient to define what constitutes a structural transformation from one construction space to another because each representational system is defined to comprise three construction spaces. The approach we take is to use \textit{patterns}, which are also constructions but where tokens are, roughly speaking, replaced by their assigned \textit{types}. Thus, any given pattern can \textit{describe} multiple constructions, where any described construction is isomorphic to the pattern and each token respects the type associated with its identified vertex in the pattern. To illustrate, the six patterns below each describe part of the decompositions above:\\[1ex]
\noindent \textit{-- Set-theoretic patterns}
\begin{center}\footnotesize
	\adjustbox{valign = c}{\STIntSubPattern}
	\hfill
	\adjustbox{valign = c}{\STIntDisPattern}
	\hfill
	\adjustbox{valign = c}{\STrootPattern}
\end{center}

\noindent \textit{-- Euler diagram patterns:}
\begin{center}\footnotesize
	\adjustbox{valign = c}{\EDintabPattern}
	\hspace{20pt}
	\adjustbox{valign = c}{\EDintbcPattern}
	\hspace{20pt}
	\adjustbox{valign = c}{\EDrootPattern}
\end{center}
Here, each of the set-theoretic patterns relates to the Euler diagram pattern drawn immediately underneath it. For instance, making a subset assertion using set-theoretic expressions relates to making a subset assertion using circles in an Euler diagram. Using knowledge about how patterns and tokens are related across construction spaces, we are able to transform one construction into another; the details are covered in Section~\ref{sec:structuralTransformations}.
\end{paragraph}

\begin{paragraph}{Contribution 3:  The ability to identify partial transformations when a complete transformation is not achievable.} As highlighted in the Introduction, the ability to produce partial transformations is of particular practical importance. A key reason for failure to produce a complete transformation, in an automated setting, is due to incomplete knowledge about the (target) construction space into which we are transforming. Incomplete knowledge can mean not all of the tokens are known or the structure graph is only partly known. Both of these are likely scenarios as representational systems often have infinitely many tokens and infinite structure graphs. In addition, there may be incomplete knowledge about relations that hold between the tokens in the original construction space and the target construction space. In such circumstances, it is beneficial to produce a partial transformation which is a pattern with, perhaps, some of its vertices instantiated as tokens.

A partial transformation of $\{A\subseteq B, B\cap C=\emptyset\}$  is given below, where the tokens that form circle labels are instantiated. Such a transformation could arise when the available knowledge base does not have enough information to identify the three circles at the leaves and, thus, the six missing Euler diagrams:
\begin{center}
\EDConstructionPartial
\end{center}

To summarise, Representational Systems Theory necessarily permits the derivation of partial transformations, even though there does not necessarily exist terminating algorithms that are guaranteed to produce complete transformations. This is particularly significant for practical applications, where a human end-user may be readily able to instantiate the missing tokens in order to produce a complete transformation.
\end{paragraph}

\begin{paragraph}{Choices of Tokens: Motivation for Change}
To exemplify one practical motivation for devising a unified theory of representational systems, we present four different tokens, each of which is an alternative representation of a problem. The scope of these tokens and, thus, the representational systems from which they are drawn is large, including natural language (NL), formal notation, a geometric figure, and a table. This variety further illustrates the breadth of representational systems to which our theory can be applied and illuminates just one context in which re-representation can be desirable. So, suppose that we have a token drawn from a natural language system, which could readily be presented in a lecture on probability:

\textsc{Medical Test Problem}\enspace Suppose that 4\% of the population has a disease~$D$. For those who have the disease, a test $T$ is accurate 95\% of the time. For those who do not have the disease, $T$ is accurate 90\% of the time. If you take the test and it comes out positive, what is the probability that you have the disease?

There are many other tokens that represent this problem, drawn from various representational systems. Each of these alternative systems has different entailment relations and, thus, different mechanisms can be used to solve the problem. The choice of token may depend on the ability of the students, their prior knowledge, or their own cognitive biases. An effective choice could better support their learning, or ability to solve the problem: it is well-known that token (representation) choice in mathematics is an important factor in students' progress~\cite{doi:10.1080/10986061003654240,janvier:porittalom}. For the purposes of illustration, we present three of these alternatives.

\newcommand{\Prob}{\textup{P}}
\textsc{Bayesian Probability}\enspace
It is given that $P(D)\!=\!0.04$, $P(T\! \mid \! D)\! = \!0.95$, $P(\overline{T}\! \mid \! \overline{D})\! =\! 0.9$. What is $P(D\!\mid\! T)$?

\noindent
\parbox{0.74\textwidth}{%
\ \ \ \ \textsc{Area Diagrams}\enspace
	Calculate the size ratio between the area of the patterned region, \,\tikz[scale=.25]{\draw[fill=black, fill opacity = 0.55, pattern=north west lines, pattern color=black] (0,0) rectangle (1,1)},\, and the area of the shaded region, \tikz[scale=.25]{\draw[fill=black, fill opacity = 0.15] (0,0) rectangle (1,1)}.%
    }%
    \hfill
        \parbox{0.2\textwidth}{%
       \scalebox{0.7}{
    		\begin{tikzpicture}[scale=0.9]
		\tikzstyle{every node}=[font=\scriptsize]
		\draw [fill=white] (0, 0) rectangle (2, 2);
		\draw [fill=white] (0, 0) rectangle (0.1, 2);
		\draw [fill=black, fill opacity = 0.15] (0, 0.1) rectangle (0.1, 2);
		\draw [pattern=north west lines, pattern color=black] (0, 0.1) rectangle (0.1, 2);
		\draw [fill=black, fill opacity = 0.15] (0.1, 1.8) rectangle (2, 2);
		%
		%		\node at (.05,-0.19) {$D$};
		%		\node at (1.1,-0.183) {$\overbar{D}$};
		%		\node at (-.2,1.1) {$T$};
		%		\node at (-.2,0.1) {$\overbar{T}$};
		%		\node at (2.2,1.9) {$T$};
		%		\node at (2.2,0.9) {$\overbar{T}$};
		%
		\draw [decorate,decoration={brace,amplitude=5pt,mirror},yshift=0pt]
		(2.05,0) -- (2.05,1.8) node [black,midway,xshift=0.35cm] {\scriptsize\rotatebox[]{270}{$0.9$}};
		\draw [decorate,decoration={brace,amplitude=5pt},yshift=0pt]
		(-0.05,0.1) -- (-0.05,2) node [black,midway,xshift=-0.35cm] {\scriptsize\rotatebox[]{90}{$0.95$}};
		\draw [decorate,decoration={brace,amplitude=0.8pt},yshift=0pt]
		(-0.01,2.05) -- (0.11,2.05) node[black,midway,yshift=0.2cm] {\scriptsize$0.04$};
		\draw [decorate,decoration={brace,amplitude=5pt,mirror},yshift=0pt]
		(-0.01,-0.05) -- (2.01,-0.05) node[black,midway,yshift=-0.35cm] {\scriptsize$1$};
		%		\node at (-.35,0.9) {$t \mid d$};
		%		\node at (-.35,1.9) {$\bar{t} \mid d$};
		%		\node at (2.35,0.1) {$t \mid \bar{d}$};
		%		\node at (2.35,1.1) {$\bar{t} \mid \bar{d}$};
		%
		\end{tikzpicture}
}
    }

\newcolumntype{P}[1]{>{\centering\arraybackslash}p{#1}}
\textsc{Contingency Tables}\label{medical-cont}\enspace Calculate the ratio of the value of cell $(T,D)$ against $total(T)$.
\begin{center}
\scalebox{0.85}{
		\begin{tabular}{|P{0.075\linewidth}|P{0.15\linewidth}|P{0.15\linewidth}|P{0.15\linewidth}|}
			\hline
			& \textbf{$D$} & \textbf{$\overline D$} & \textbf{total} \\
			\hline
			\textbf{$T$} & \cellcolor{gray!20}$0.95 \cdot 0.04$ &  & \cellcolor{gray!20} \\
			\hline
			\textbf{$\overline T$} &  & $0.9\cdot \text{total}(\overline D)$ &  \\
			\hline
			\textbf{total} & $0.04$ & & 1 \\
			\hline
		\end{tabular}
}
\end{center}

\vspace{5pt} It is self-evident that these four different representational systems have tokens with fundamentally different syntaxes and, by their very nature, afford different approaches to solving the medical problem. Representational Systems Theory provides a unified framework that allows each of these tokens to be encoding using a graph-theoretic approach via the existence of construction spaces. Given higher-level descriptions of the systems from which they are drawn, in the form of patterns, and suitable links between the pattens, we are able to transform between these tokens.

We previously published the medical problem representations in~\cite{raggi2020re}. In that prior work, we presented a much earlier version of the framework presented here. However, that prior work did not support the identification of new tokens in particular. Rather, it sought to ascertain \textit{whether} a token \textit{could be re-represented} in another representational system. That is, our prior work was geared towards saying whether \textit{an alternative representational system was likely to contain a token that re-represented the original representation}. Thus, the research in this paper represents a substantial advance over our previous contributions.
\end{paragraph} 

%% file: preliminaries.tex
%We give a brief overview of notations for sequences and graphs that we exploit in this paper.

We demonstrate the notational conventions adopted for a variety of (mostly) standard mathematical concepts. Appendices~\ref{sec:defnsAndNotation} and~\ref{sec:defnsTerms} provide summaries of the notation used for novel terminology introduced in the paper and, respectively the terminology itself.

\begin{paragraph}{Sequence notations}
A \textit{finite sequence}, $\genSeqn$, of \textit{length} $n$ over a set $A$ is a function,\linebreak $\genSeqn\colon \{1,\ldots, n\}\to A$, which we write, informally, as a list: $\genSeqn=[a_1,\ldots,a_n]$. The \textit{empty sequence}, $[]$, has length 0. An element, $a$, \textit{occurs} in $\genSeqn=[a_1,\ldots,a_n]$, denoted $a\in \genSeqn$,  provided $a=a_i$ for some $1\leq i \leq n$. The set of all sequences, of arbitrary lengths, over $A$ is denoted $\sequence(A)$. The \textit{concatenation} of sequences $\genSeqn_1=[a_1,\ldots,a_n]$ and $\genSeqn_2=[b_1,\ldots,b_m]$, denoted $\genSeqn_1\oplus \genSeqn_2$, is $[a_1,\ldots,a_n,b_1,\ldots,b_m]$. We write $\genSeqn_1\oplus \cdots\oplus \genSeqn_n$ to mean the concatenation of $n$ sequences; when $n=0$, $\genSeqn_1\oplus \cdots\oplus \genSeqn_n=[]$. Given a sequence of sequences, $[\genSeqn_1,\ldots,\genSeqn_n]$, and a sequence $\genSeqn$, the \textit{right product} of $[\genSeqn_1,\ldots,\genSeqn_n]$ with $\genSeqn$,  denoted $[\genSeqn_1,\ldots,\genSeqn_n]\triangleleft \genSeqn$, is defined to be $[\genSeqn_1\oplus \genSeqn,\ldots,\genSeqn_n\oplus \genSeqn]$.
\end{paragraph}

\begin{paragraph}{Graph notations}
A \textit{directed labelled bipartite graph}, which we will simply call a \textit{graph}, is a tuple, $\graphn=\graph$, where\footnote{The naming conventions $\pa$, $\pb$, $\arrowl$, $\tokenl$ and $\consl$ will become clear in Section~\ref{sec:RSs}.}: $\pa$ and $\pb$ are two disjoint sets of \textit{vertices}, $A$ is a set of \textit{arrows}, $\incVert\colon A \to (\pa\times \pb) \cup (\pb \times \pa)$ is a function that identifies a pair of \textit{incident vertices} for each arrow, and $\arrowl\colon \arrows \to \mathbb{N}$ is a function that assigns a label, called an \textit{index}, to each arrow. The functions $\tokenl$ and $\consl$ assign a label to each vertex in $\pa$ and, resp., $\pb$. Notably, graphs can have multiple edges (i.e. more than one edge between a pair of vertices) and, thus, need not be simple. Vertices in $\pa$ (resp. $\pb$) will typically be denoted by $t$, $t'$, $t_1$ and so forth (resp. $u$, $u'$, $u_1$). Vertices in either $\pa$ or $\pb$ will be denoted $v$, $v'$, $v_1$ and we set $V=\pa\cup \pb$.  Given any arrow, $a$, if $\incVert(a)=(v_1,v_2)$ then the \textit{source} (resp. \textit{target}) of $a$, denoted $\sor{a}$ (resp. $\tar{a}$), is $v_1$ (resp. $v_2$). Given any vertex, $v$: the set of \textit{incoming arrows} (resp. \textit{outgoing arrows}), denoted $\inA{v}$ (resp. $\outA{v}$), is $\{a\in A \colon  \tar{a} =v\}$ (resp. $\{a\in A \colon \sor{a} =v\}$), and the set of \textit{input vertices} (resp. \textit{ouput vertices}), denoted $\inV{v}$ (resp. $\outV{v}$),  is $\{\sor{a}\colon a\in \inA{v}\}$ (resp. $\{\tar{a}\colon a\in \outA{v}\}$). Given a graph, $\graphn$, the \textit{neighbourhood} of vertex $v$, denoted  $\neigh{v}$, is the largest subgraph of $\graphn$ whose vertex set is $\inV{v}\cup \outV{v}\cup \{v\}$.
\end{paragraph}

\begin{paragraph}{Isomorphisms}
If graphs $\graphn$ and $\graphn'$ are isomorphic and the corresponding bijections from the vertices and arrows in $\graphn$ to those in $\graphn'$ are such that the vertex and arrow labels are preserved then the isomorphism is \textit{label-preserving} and $\graphn$ and $\graphn'$ are \textit{label-isomorphic}. Sometimes we want an isomorphism to preserve some particular labels but not others. We will specify which elements do not necessarily have matching labels by saying that $\graphn=\graph$ and $\graphn'$, are \textit{label-isomorphic up to $X$}, where $X\subseteq \pa\cup \pb \cup \arrows$. That is, the labels of elements in $X$ do not necessarily match under the isomorphism, whereas the labels of elements in $(\pa\cup \pb \cup \arrows)\backslash X$ are preserved.
\end{paragraph}

\begin{paragraph}{Walks, trails and paths}
 A \textit{directed walk}, $w$, in $\graphn$ is a possibly empty sequence of arrows, $[a_1,\ldots,a_{n-1}]$, for which there is an associated sequence of vertices, $[v_1,\ldots,v_n]$, where, for each $a_i$ in $w$, $\incVert(a_i)=(v_i,v_{i+1})$. The \textit{length} of $w$ is $n-1$. A \textit{directed trail} is a directed walk, $[a_1,\ldots,a_{n-1}]$, where for all $a_i$ and $a_j$, if $a_i=a_j$ then $i=j$. A \textit{directed path} is a directed trail, $[a_1,\ldots,a_{n-1}]$, where for all $v_i$ and $v_j$ in the associated vertex sequence, $[v_1,\ldots,v_n]$, if  $v_i=v_j$ then $i=j$. As all our graphs are directed, we omit the prefix `directed' and simply say walk, trail and path. We often wish to talk about the \textit{source} or \textit{target} vertex ($v_1$ and $v_n$ in an associated vertex sequence) of a walk, $w$, which we denote by $\source{w}$ and $\target{w}$. However, the source and target are only directly derivable from $w=[a_1,...,a_n]$ when $n\geq 1$. This means that when $w=[]$, we \textit{must explicitly specify the associated single-vertex sequence}, $[v]$ say.  We assume that such a vertex sequence for $w=[]$ has been specified in any context where it is necessary, writing $[]_v$ to denote that the empty walk $w=[]$ is associated with vertex sequence $[v]$ and, therefore, we have $\source{[]_v}=\target{[]_v}=v$. A trail, $\trail$, in $\graphn$ is \textit{source-extendable} provided there is an arrow, $a$, in $\graphn$ such that either $\trail\neq []$ and $[a]\oplus \trail$ is a trail or $\trail=[]_v$ and $\tar{a}=v$. Taking $[]_v$ and $[]_{v'}$, where $v\neq v'$, together with an arrow, $a$, such that $\tar{a}=v$, it is the case that $[]_v$ is source-extendable by $a$ whereas $[]_{v'}$ is not. As we do not need the corresponding notion of target-extendable, we simply say \textit{extendable} to mean source-extendable.
\end{paragraph}

\begin{paragraph}{Visual notation for walks}
It can be cumbersome to express a walk, $w$, as a sequence of arrows, so we will often abuse notation. For example, given $w=[a_1,a_2]$, with associated vertex sequence $[v_1,v_2,v_3]$, we may visually express $w$ as $v_1\arrow v_2\arrow v_3$. However, for non-simple graphs this visual notation is not necessarily sufficient to identify the intended walk. In our example, this can occur when, say, $v_1$ is the source of another arrow, $a_1'$,  which also targets $v_2$. When $a_1$ and $a_1'$ have different indices, $\arrowl(a_1)=i$ and $\arrowl(a_1')=j$, in $\graphn$ we will visualize $w$ as $v_1\arrow[i] v_2\arrow v_3$. For our purposes, the addition of arrow indices will always be sufficient to identify a unique walk. When a walk, $[]_v$, has no arrows we will avoid using the walk-visualisation $v$ which could be confused with a single vertex.
\end{paragraph}

\begin{paragraph}{Sets, functions and disambiguation} The power set of a set, $X$, will be denoted by $\powerset(X)$. Given ordered tuples of sets, $A=(A_1,\ldots,A_n)$ and $B=(B_1,\ldots,B_n)$, we adopt the convention of writing $A\cup B$ to mean $(A_1\cup B_1,\ldots,A_n\cup B_n)$. This convention obviously extends to other set-theoretic operations.  We will exploit the formal definition of a function, $f\colon A\to B$, as a subset of $A\times B$ such that for each $a\in A$ there exists a unique $b\in B$ where $(a,b)\in f$ (i.e. $f(a)=b$). Thus, given functions $f$ and $f'$, we can form $f\cup f'$ and $f\cup \{(a,b)\}$, for example. Functions are often defined within some context, such as $\neigh{v}$ in the context of a graph $\graphn$. If $v$ is in graphs $\graphn$ and $\graphn'$ then we write $\neigh{v,\graphn}$ to mean the neighbourhood of $v$ in the graph $\graphn$. Thus, sometimes an additional parameter, $p$, is needed to disambiguate the application of a function, $f$: if $f(x_1,\ldots,x_n)$ is defined in the context of $p$ then we will write $f(x_1,\ldots,x_n,p)$ to mean the application of $f$ to $(x_1,\ldots,x_n)$ in the context of $p$; in our example, $f=\mathit{Nh}$ and $p=\graphn$.
\end{paragraph}

%% file: representationalSystems.tex
Intuitively, a \textit{representational system} comprises syntactic entities, called \textit{tokens}, that have the capacity to represent objects when assigned semantics~\cite{peirce:cp,Shimojima1996-SHIOTE-2,wetzel:tat}. The notion of a token is closely related to that of a \textit{term}, often seen in type theory~\cite{sep-principia-mathematica}: the different \textit{instances} of a term are precisely its tokens\footnote{This distinction is embodied by the type-token dichotomy which arises in the philosophical analysis of representations~\cite{peirce:cp,wetzel:tat}; see~\cite{bellucci:aaoeg} for further discussion.}. The tokens must adhere to the \textit{grammatical rules governing the syntax} of the system and may be subject to \textit{entailment relations}. Tokens, or collections of them, exhibit \textit{properties} that it may be desirable to \textit{identify} within the system. \textit{Representational Systems Theory} is designed to permit a wide variety of representational systems to be abstractly encoded, reflecting the key insights just given. The major contribution of this section is a formal definition of a representational system which comprises three \textit{spaces}: \textit{grammatical}, \textit{entailment} and \textit{identification}. Whilst each of these spaces serves a different purpose, the formal theory is unified in its approach to defining them: each space is a type of \textit{construction space} (see definition~\ref{defn:constructionSpace}) that encodes the respective aspects of representations, in the context of a \textit{type system}, using a directed labelled bipartite graph, called a \textit{structure graph}. Tokens are encoded as vertices in the structure graph and are assigned types by a vertex labelling function. This approach has similarities to research by Shimojima and Barker-Plummer, who view representations as tokens that are assigned types~\cite{10.1007/978-3-662-44043-8_25}.

The use of a single abstraction for encoding both the grammatical and entailment aspects of representational systems is related to the Curry-Howard correspondence, which associates types to formulae. More importantly, Curry-Howard associates the type-structure of any term to a proof~\cite{howard1980formulae}. Behind these abstractions that unify syntax and inference there is perhaps a deeper truth about how human constructs are structured. {\pagebreak}For instance, consider the similarity in the \begin{samepage}following:
	\begin{itemize}
		\item how a sentence relates to its constituent phrases,
		\item how a geometric figure relates to the shapes that form it,
		\item how the validity of a statement relates to the validity of the statements used to justify it, and
		\item how the output computed by a program relates to the outputs of its subroutines.
	\end{itemize}
\end{samepage}
Now, whenever we have an analogy between a number of structures we can try to find their shared abstraction. We contend that the notion of construction space formalises the abstraction shared by many structures, and that makes it ideal for capturing the different aspects of a single representational system, as well as a variety of representational systems.

\subsection{Motivating Examples}\label{sec:representationalSystems:ME}

The major contributions of Section~\ref{sec:RSs} will be contextualised by exemplifying \textit{one way} that Representational Systems Theory can encode a \textit{fragment}, $\rsystemn_{\FOA}$, of \textsc{First-Order Arithmetic} (FOA). In FOA, tokens include syntactic entities, such as numerals, operators, variables, and quantifiers, to name a few, together with expressions built using them. The following are all distinct tokens in FOA:
%\begin{multicols}{4}
\begin{itemize}
	\begin{minipage}[t][60pt]{0.09\linewidth}
		\item[] $x$
		\item[] $2$
		\item[] $91$
		\item[] $+$
	\end{minipage}
	\begin{minipage}[t][60pt]{0.16\linewidth}
		\item[] $\leq$
		\item[] $2/(x+2)$
		\item[] $1 \leq 4$
		\item[] $x = 4$
	\end{minipage}
	\begin{minipage}[t][60pt]{0.19\linewidth}
		\item[] $\sum_{i=0}^{n} 2^i$
		\item[] $2-3\cdot 4 > 1$
		\item[] $\frac{\sum_{i=0}^{n-1} x^i}{x^n-1}$
		\item[] $\frac{\phantom{xxxxxx}}{\phantom{xxxxxx}}$
	\end{minipage}
	\begin{minipage}[t][60pt]{0.21\linewidth}
		\item[] $\forall$
		\item[] $\forall x\; \forall y\;\; x \leq y$
		\item[] $x^2 + 2xy + y^2$
		\item[] $1+2+3=42$
	\end{minipage}
	\begin{minipage}[t][60pt]{0.3\linewidth}
		\item[] $\Rightarrow$
		\item[] $(x+y)^2 = x^2 + 2xy + y^2$
		\item[] $\neg(x=1 \wedge x=2 ) \vee 1=2$
		\item[] $\forall x\; x > x \Rightarrow \exists y\; y=x$.
	\end{minipage}
\end{itemize}
%\end{multicols}

The fragment, $\rsystemn_{\FOA}$, that we will encode includes, amongst its tokens: numerals, the two variables $x$ and $y$, parentheses $($ and $)$, and the quantifier $\forall$, along with the binary operators and relations, $+$, $-$, $=$, and $>$ but \textit{does not include} other tokens such as logical connectives like $\Rightarrow$; example~\ref{ex:tokensInFOA} precisely describes the tokens in $\rsystemn_{\FOA}$. This fragment is sufficiently rich to expose most of the core theoretical concepts that are necessary for defining representational systems in general.

\begin{example}[Tokens in $\rsystemn_{\FOA}$]\label{ex:tokensInFOA}
\textsc{First-Order Arithmetic} fragment of interest includes, amongst its tokens, representations formed of the following symbols:\vskip8pt
\begin{multicols}{2}
\begin{itemize}
	%\begin{minipage}[t]{0.48\linewidth}
	\item[-] all base 10 numerals, such as $0, 1, 2, 3, \ldots, 10$, and so forth
    \item[-] variables $x$ and $y$,
    \item[-] operator symbols $+$ and $-$,
	\item[-] relation symbols $=$ and $>$,
    \item[-] brackets, $($ and $)$, and
    \item[-] the quantifier symbol $\forall$.
\end{itemize}
\end{multicols}\vskip8pt
Tokens that are numerals or variables are \textit{numerical expressions}. Any representation in $\rsystemn_{\FOA}$ formed of a single symbol is called a \textit{primitive} token. For instance, $1$ is primitive but $1=2$ is not, since it is formed from $1$, $=$ and $2$. \textit{Composite} tokens are constructed using an inductive definition:
	\begin{itemize}
    \item[-] For all tokens, $t_1$ and $t_2$, that are numerical expressions it is the case that $t_1+t_2$, $t_1-t_2$, and $(t_1)$ are tokens\footnote{Note that $t_1+t_2$ and $t_1-t_2$ do not impose the presence of parentheses, reflecting the typical informal use of such expressions in practice. Hence, brackets can be placed around any numerical expression, $t_1$, where needed, giving $(t_1)$.} that are also numerical expressions.
    \item[-] For all tokens, $t_1$ and $t_2$, that are numerical expressions it is the case that $t_1=t_2$ and $t_1> t_2$ are tokens that are \textit{formulae}.
    \item[-] For each of the two variables, $x$ and $y$, and for all tokens, $t$, that are formulae it is the case that $\forall x \thinspace t$ and $\forall y \thinspace t$ are tokens that are also formulae.
    \end{itemize}
What we have just witnessed is the \textit{construction} of tokens following the grammatical rules of first-order arithmetic.  These tokens can naturally be classified using types such as \FOAnum, \FOAvar, \FOAnumexp, \FOAbop\ (binary operator), \FOAbrel\ (binary relation), and \FOAform.
\end{example}

Example~\ref{ex:tokensInFOA} concluded by giving some of the types that will be used in $\rsystemn_{\FOA}$, written in \texttt{typewriter} font, a convention we adopt throughout.  Example~\ref{ex:distinctTokensSameSymbols} will introduce more types, related to distinct tokens that are (a) built from the same symbols, and (b) follow equivalent inductive constructions. By design, Representational Systems Theory permits each pair of tokens that satisfy (a) and (b) to be distinguished, encoding them as distinct entities. Subsequently, a particular type can be assigned to all tokens that satisfy (a) and (b), illustrated in example~\ref{ex:distinctTokensSameSymbols}.

\begin{example}[Distinct Tokens Built from the Same Symbols]\label{ex:distinctTokensSameSymbols}
Consider the token $x=x$. This is built from two \textit{distinct} $x$  tokens and one $=$ token. Likewise, the composite tokens $x=2$ and $x=2$ are distinct. At one level, they appear to be the same (strictly, they are \textit{isomorphic}) since they both use the symbols $x$, $=$ and $2$ to form a representation where $x$ is to the left of $=$ and $2$ is to the right. However, they occupy different positions on the page and are, therefore, distinct tokens\footnote{Isomorphic tokens, i.e. those built from the same symbols following essentially the same inductive construction, can be considered indistinguishable. In such cases where, say, $x=2$ and $x=2$ are considered to be indistinguishable, they would sometimes be called \textit{the term} $x=2$. Distinguishing tokens from terms is akin, in the study of natural language, to differentiating the word \textit{the} from the occurrences of the word \textit{the} in the sentence `the cat sat on the mat': there are two \textit{the} tokens in the sentence, both of type \texttt{the}. See~\cite{10.1007/978-3-662-44043-8_25} for a similar discussion.}. We can group together isomorphic tokens using types: $x=2$ and $x=2$ both have the same type, which we call $\texttt{x}\_\FOAeq\_\texttt{2}$; essentially, this naming convention forms a type by concatenating the types assigned to the constituent parts. Similarly, $x$ and $x$ both have type $\texttt{x}$, and so forth. In essence, assigning isomorphic tokens the same type captures the fact that they are different tokens built using the same symbols, following equivalent inductive constructions. Lastly, the types $\texttt{x}\_\FOAeq\_\texttt{2}$ and $\texttt{x}$ are, respectively, subtypes of $\FOAform$ and $\FOAvar$. Further types for $\rsystemn_{\FOA}$ are given in example~\ref{ex:FOATypeSystem}.
\end{example}

In formal representational systems, such as \textsc{First-Order Arithmetic}, identifying the tokens tends to be a fairly straightforward process. Example~\ref{ex:termInstances} illustrates how we might identify tokens in a less formal system.

\begin{example}[Identifying Tokens]\label{ex:termInstances}
Consider a representational system comprising a collection of crossword puzzles, where each letter of the alphabet inscribed on some medium is a primitive token. Such a puzzle might contain the words `hello' and `world', sharing the letter `o':
\begin{center}
\begin{tikzpicture}[x=0.6cm, y = 0.6cm]
\fill[gray] (0,1) rectangle (4,4);
\fill[gray] (0,3) rectangle (4,4);
\fill[gray] (0,5) rectangle (4,6);
\fill[gray] (5,1) rectangle (6,5);
\draw[step=0.6cm,black,very thin] (0,1) grid (6,6);
\node[] () at (0.12,4.83) {\scriptsize 2};
\node[] () at (4.12,5.83) {\scriptsize 1};
\node[] () at (0.5,4.5) {h};
\node[] () at (1.5,4.5) {\vphantom{l}e};
\node[] () at (2.5,4.5) {l};
\node[] () at (3.5,4.5) {l};
\node[] () at (4.5,4.5) {\vphantom{l}o};
\node[] () at (4.5,1.5) {d};
\node[] () at (4.5,2.5) {l};
\node[] () at (4.5,3.5) {\vphantom{l}r};
\node[] () at (4.5,5.5) {\vphantom{l}w};
\end{tikzpicture}
\end{center}
In this representation, there are seven different letters: h, e, l, o, w, r, d. There are three \textit{instances} of the letter `l' (two in `hello' and one in `world'), one \textit{instance} of `o' (since the words share this letter in the representation), and one \textit{instance} of each of the remaining letters. An encoding of this representation will exploit \textit{one token for each letter instance}: h,e,l,l,o,w,r,l,d. Lastly, we note that this crossword puzzle representation also contains other tokens, such as the numbers 1 and 2, and the lines used to form the grid, amongst others.
\end{example}

In Representational Systems Theory, isomorphic tokens can be assigned the same type, as seen in example~\ref{ex:distinctTokensSameSymbols}. The assignment of types to tokens, in this way, is a direct embodiment of Peirce's theory of types and tokens~\cite{peirce:cp}. Now, we saw in example~\ref{ex:tokensInFOA} that tokens were built following an inductive \textit{construction}. Representational Systems Theory includes \textit{constructors} that embody the rules to which representations must conform. In essence, the constructors in a \textit{grammatical space} are used to encode grammatical rules\footnote{Goldin's conceptualisation of \textit{ambiguous rules}~\cite{GOLDIN1998137} is akin, in our formalisation, to  \textit{constructors} not being \textit{total} or \textit{functional}; these three concepts are developed in Section~\ref{sec:propertiesOfConstructionSpaces}, but for now it suffices to say that constructors behave as black boxes that are provided with `input representations' and `produce output representations' (i.e. from a \textit{grammatical} perspective, a new permitted configuration aka representation; from an \textit{entailment} perspective, an allowable entailed representation.)} using types and, subsequently, how tokens are built from other tokens. Example~\ref{ex:grammaticalConstructors}  will define the constructors for $\rsystemn_{\FOA}$'s grammatical space, $\gspacen_{\FOA}$. %The grammatical space, $\gspacen$, of $\rsystemn_{\FOA}$ assigns types to tokens and encodes the construction of tokens using its constructors.

Entailment and identification spaces also exploit constructors that encode further \textit{construction rules}. $\rsystemn_{\FOA}$'s entailment space, $\espacen_{\FOA}$, is assigned constructors that encode entailment relations, such as $x = 4$ entails $x > 1$. Finally, $\rsystemn_{\FOA}$'s identification space, $\ispacen_{\FOA}$, encodes properties of tokens that are not directly accessible from $\gspacen_{\FOA}$ or $\espacen_{\FOA}$. This might include properties such as `$3$ is a primitive token', `in $2>1$, $2$ is to the left of $>$' or even `the cognitive cost of understanding $3+1 >3 $ is less than that of understanding $\forall n\; n+1 >n$'. As such, $\ispacen_{\FOA}$ uses constructors that encode properties of tokens.

Whilst many examples in Section~\ref{sec:RSs} will be based on $\rsystemn_{\FOA}$, we also draw on other representational systems. Examples will be varied, including \textsc{Vector Spaces}, the \textsc{Graphical Representation of Vectors}, \textsc{Set Theory}, and \textsc{Euler Diagrams}; we have already seen an example using \textsc{Crossword Puzzles}. This rich source of examples conveys the wide-ranging applications of Representational Systems Theory for the general analysis of representations.

\subsection{Construction Spaces}\label{sec:constructionSpaces}

As already noted, Representational Systems Theory is  unified in its approach to defining the three spaces of which a representational system comprises: each of the three spaces is a \textit{construction space}\footnote{A construction space is a typical mathematical space: it comprises a set of mathematical objects that are associated with a structure (e.g. see~\cite{DALARSSON20151}) which, in our work, is embodied by a structure graph.} that includes a \textit{structure graph} associated with a \textit{type system}. These concepts are presented alongside that of a \textit{constructor}, which is a key part of our formal framework: at an intuitive level, each constructor takes an \textit{ordered} collection (sequence) of tokens, $[t_1,\ldots,t_n]$, as its input and produces a new token, $t$. 

The particular roles of constructors vary across the spaces. In the grammatical space, constructors reflect how $t$ is \textit{built} from $[t_1,\ldots,t_n]$, by encapsulating the \textit{grammatical rules} to which syntactic entities in the representational system must conform. As one might expect, these constructors are required to be \textit{functional}: every encoded grammatical rule ensures that for any input sequence, $[t_1,\ldots,t_n]$, a unique token, $t$, is built. In the entailment space, each constructor encodes \textit{deductions} or \textit{inferences} that can be made from $[t_1,\ldots,t_n]$, resulting in another token\footnote{There is no requirement for these deductions or inferences to be \textit{sound} in the formal sense. The framework is designed to allow any kind of reasoning to be encoded in the entailment space, provided the reasoning identifies tokens that \textit{follow} from other tokens, aligning with Goldin's view of inference~\cite{GOLDIN1998137}. For instance, the entailment space could encode a system in which the object of analysis is defeasible reasoning~\cite{koons:dr}, which models rationally-compelling human-like arguments that are not necessarily deductively valid.}. Constructors in this space need not be functional: every encoded entailment relation allows potentially many tokens to be inferred from $[t_1,\ldots,t_n]$. Lastly, in the identification space, each constructor allows properties of tokens to be encoded. Given input $[t_1,\ldots,t_n]$, the output of such a constructor is not a token that is built in the grammatical space but, rather, a meta-level token, $t$, with a specified meta-level type, $\omega$. Intuitively, $t$ could be of type $\FOAboolean$ (the property holds or it does not) or of type \texttt{real number} in $[0,1]$ (e.g. the probability that this property holds is $p$). Our theory allows meta-level types to be anything that is appropriate for the representational system. So, let us now define a \textit{type system} and \textit{constructors} in the context of \textit{constructor specifications}.

\begin{definition}\label{defn:typeSystem}
A \textit{type system}, $\tsystemn$,  is a pair, $\tsystemn=\tsystem$,
where
\begin{enumerate}
\item $\types$ is a set whose elements are called \textit{types}, and
\item $\leq$ is a partial order over $\types$.
\end{enumerate}
If $\tau_1\leq \tau_2$ then $\tau_1$ is a \textit{subtype} of $\tau_2$ and, respectively, $\tau_2$ is a \textit{supertype} of $\tau_1$.
\end{definition}

\begin{example}[A Type System]\label{ex:FOATypeSystem}
The fragment, $\rsystemn_{\FOA}$, of \textsc{First-Order Arithmetic} includes (a) one type for each symbol from which primitive tokens are formed, (b) types for composite tokens, and (c) types that enable a compact definition of the grammatical rules. The types for primitive tokens are:
%\begin{multicols}{2}
\begin{itemize}
	%\begin{minipage}[t]{0.48\linewidth}
	\item[-] for each base 10 numeral, $n$ (e.g. $21$), $\texttt{n}$ (e.g. \texttt{21}) is a type\footnote{Of course, this implies that there are infinitely many types because there are infinitely many numerals. In addition, \texttt{n} is not itself a type, but a meta-level schema for defining types. We use this meta-level schema elsewhere too.},
    \item[-] variables $x$ and $y$ are of types $\texttt{x}$ and $\texttt{y}$,
    \item[-] binary operators $+$ and $-$ are of types $\FOAplus$ and $\FOAminus$\footnote{Note the different naming convention used here is merely because $+$ and $-$ are almost indistinguishable in \texttt{typewriter} font from their appearance in the standard mathematics environment. The same holds for the remainder of the types assigned to primitive tokens.},
	\item[-] binary relations $=$ and $>$ are of types $\FOAeq$ and $\FOAgr$,
    \item[-] brackets $($ and $)$ are of types $\FOAopenb$ and $\FOAcloseb$, and
    \item[-] the  quantifier $\forall$ is of type $\FOAquant$.
\end{itemize}
%\end{multicols}
The types assigned to the primitive tokens will be augmented with further types to yield the type system over which $\rsystemn_{\FOA}$ is defined. The diagram below shows two representations, $3+x>y$ and $21-(x+2)=21$, and identifies the types of their primitive tokens:
\begin{center}
	\begin{tikzpicture}[node distance = 0.5cm,outer sep = -0.025cm,inner sep = 0.09cm]
	\node[] (n1) at (-0.7,-0.05) {\texttt{1}};
	\node[right = of n1,xshift=-0.2cm] (n2) {\texttt{2}};
	\node[right = of n2,xshift=-0.2cm] (n3) {\texttt{3}};
	\node[right = of n3,xshift=-0.3cm,yshift = -0.03cm] (ndots) {$\cdots$};
	\node[right = of ndots,xshift=-0.3cm,yshift = 0.03cm] (n21) {\texttt{21}};
	\node[right = of n21,xshift=-0.3cm,yshift = -0.03cm] (ndots') {$\cdots$};
	\node[right = of ndots',yshift = 0.03cm] (x) {\texttt{x}};
	\node[right = of x,yshift = -0.03cm] (y) {\texttt{y}};
	\node[right = of y] (+) {\texttt{pl}};
	\node[right = of +] (-) {\texttt{mi}};
	\node[right = of -,yshift = -0.03cm] (gr) {\texttt{gr}};
	\node[right = of gr] (eq) {\texttt{eq}};
	\node[right = of eq,yshift = 0.03cm] (openB) {\FOAopenb};
	\node[right = of openB,yshift = 0.03cm] (closeB) {\FOAcloseb};
	\draw[rounded corners, black!30, thick] (-1,-0.4) rectangle node[xshift = 4.4cm,yshift = 0.25cm,black!50] {some types for primitives} (12.3,0.6);
	\draw[rounded corners, black!30, thick] (-1,-1.75) rectangle node[xshift = 6cm,yshift = 0.2cm,black!50] {tokens} (12.3,-0.95);
	\coordinate (formula1) at (2.5,-1.45) ;
	\coordinate (formula2) at (7.5,-1.45) ;
	\node[at = (formula1),fill=darkgold, text=black, fill opacity = 0.25, text opacity=1, rectangle, rounded corners,inner sep =0.069cm,xshift = -0.63cm,yshift = 0.04cm] (n3t) {\large$3$};
	\node[at = (formula1),fill=darkgreen,text=black, fill opacity = 0.25, text opacity = 1, rectangle, rounded corners,inner sep =0.069cm, xshift = -0.34cm,yshift = 0.025cm] (+t) {\large$+$};
	\node[at = (formula1),fill=cyan!80,text=black, fill opacity = 0.25, text opacity = 1, rectangle, rounded corners,inner sep =0.089cm, xshift = -0.05cm,yshift = 0.01cm] (xt) {\large$x$};
	\node[at = (formula1),fill=darkblue,text=black, fill opacity = 0.25, text opacity = 1, rectangle, rounded corners,inner sep =0.079cm,xshift = 0.28cm,yshift = 0.03cm] (grt) {\large$>$};
	\node[at = (formula1),fill=darkpurple,text=black, fill opacity = 0.25, text opacity = 1, rectangle, rounded corners,inner sep =0.079cm,xshift = 0.625cm,yshift = -0.03cm] (yt) {\large$y$};
	\node[at = (formula2),fill=darkred,text=black, fill opacity = 0.25, text opacity = 1, rectangle, rounded corners,inner sep =0.079cm,xshift = -1.15cm,yshift = 0.02cm] (n21t) {\large$21$};
	\node[at = (formula2),fill=orange,text=black, fill opacity = 0.25, text opacity = 1, rectangle, rounded corners,inner sep =0.069cm,xshift = -0.755cm,yshift = 0.02cm] (-t) {\large\vphantom{\%}$-$};
	\node[at = (formula2),fill=darkgray!60,text=black, fill opacity = 0.25, text opacity = 1, rectangle, rounded corners,inner sep =0.069cm,xshift = -0.475cm,yshift = 0.0cm] (openBt) {\large$($};
	\node[at = (formula2),fill=cyan!80,text=black, fill opacity = 0.25, text opacity = 1, rectangle, rounded corners,inner sep =0.079cm,xshift = -0.29cm,yshift = -0.02cm] (xt') {\large$x$};
	\node[at = (formula2),fill=darkgreen,text=black, fill opacity = 0.25, text opacity = 1, rectangle, rounded corners,inner sep =0.069cm,xshift = 0.05=2cm,yshift = -0.0cm] (+t') {\large$+$};
	\node[at = (formula2),fill=darkblue,text=black, fill opacity = 0.25, text opacity = 1, rectangle, rounded corners,inner sep =0.079cm,xshift = 0.29cm,yshift = 0.02cm] (n2t) {\large$2$};
	\node[at = (formula2),fill=darkgray,text=black, fill opacity = 0.25, text opacity = 1, rectangle, rounded corners,inner sep =0.069cm,xshift = 0.46cm,yshift = 0.00cm] (closeBt) {\large$)$};
	\node[at = (formula2),fill=darkpurple,text=black, fill opacity = 0.25, text opacity = 1, rectangle, rounded corners,inner sep =0.069cm,xshift = 0.75cm,yshift = 0.02cm] (eqt) {\large\vphantom{l\%,}$=$};
	\node[at = (formula2),fill=darkred,text=black, fill opacity = 0.25, text opacity = 1, rectangle, rounded corners,inner sep =0.069cm,xshift = 1.149cm,yshift = 0.02cm] (n21t') {\large$21$};
	\draw[-,darkgold,in=-45,out=120] (n3t) edge (n3);
	\draw[-,darkgreen,in=-140,out=70] (+t) edge (+);
	\draw[-,cyan!80,in=-130,out=60] (xt) edge (x);
	\draw[-,darkblue,in=-130,out=50] (grt) edge (gr);
	\draw[-,darkpurple,in=-120,out=40] (yt) edge (y);
	\draw[-,darkred,in=-60,out=160] (n21t) edge (n21);
	\draw[-,orange,in=-80,out=100] (-t) edge (-);
	\draw[-,darkgray!60,in=-160,out=75] (openBt) edge (openB);
	\draw[-,cyan!80,in=-45,out=99] (xt') edge (x);
	\draw[-,darkgreen,in=-60,out=120] (+t') edge (+);
	\draw[-,darkblue,in=-30,out=135] (n2t) edge (n2);
	\draw[-,darkgray,in=-160,out=60] (closeBt) edge (closeB);
	\draw[-,darkpurple,in=-85,out=95] (eqt) edge (eq);
	\draw[-,darkred,in=-45,out=125] (n21t') edge (n21);
	\end{tikzpicture}
\end{center}
From the diagram, we can see that many tokens can have the same type, such as the two instances of type $\texttt{x}$. We also need types for composite tokens. In $\rsystemn_{\FOA}$, all composite tokens are \textit{directly built} from three other tokens. For example, $\forall x \thinspace 3=x$ is directly built from the three tokens $\forall$, $x$ and $3=x$, with $x=3$, in turn, also being directly built from three tokens. The name of the type we assign to $\forall x \thinspace 3=x$ is obtained by concatenating of the names of the types assigned to $\forall$, $x$ and $3=x$, namely $\FOAquant\_\texttt{x}\_\texttt{3}\_\FOAeq\_\texttt{x}$. In general, given a composite token, $t_1t_2t_3$, in $\rsystemn_{\FOA}$ its type is defined to be $\mathit{type}(t_1)\_\mathit{type}(t_2)\_\type(t_3)$. The diagram below shows how all \textit{composite} parts of $3+x>y$ and $21-(x+2)=21$ are assigned types:
	
\begin{center}
	\begin{tikzpicture}[node distance = 0.2cm and 0cm,outer sep = -0.01cm,inner sep = 0.09cm]
	\node[] (3+x) at (-0.8,1) {\texttt{3\_pl\_x}};
	\node[below right = of 3+x,xshift=-0.3cm] (3+x>y) {\texttt{3\_pl\_x\_gr\_y}};
	\node[above right = of 3+x>y,xshift=-0.cm] (x+2) {\texttt{x\_pl\_2}};
	\node[below right = of x+2,xshift=-0.cm] (x+2B) {\texttt{oB\_x\_pl\_2\_cB}};
	\node[above right = of x+2B,xshift=-0.3cm] (21-x+2) {\texttt{21\_mi\_oB\_x\_pl\_2\_cB}};
	\node[below right = of 21-x+2,xshift=-0.3cm] (21-x+2=21) {\texttt{21\_mi\_oB\_x\_pl\_2\_cB\_eq\_21}};
	\draw[rounded corners, black!30, thick] (-1.5,0) rectangle node[xshift = 5.1cm,yshift = 0.45cm,black!50] {some types for composites} (13.5,1.4);
	\draw[rounded corners, black!30, thick] (-1.5,-1.55) rectangle node[xshift = 6.7cm,yshift = 0.3cm,black!50] {tokens} (13.5,-0.45);
	\coordinate (formula1) at (1.5,-1) ;
	\coordinate (formula2) at (6,-1) ;
	\node[at = (formula1), text=black, fill opacity = 0.25, text opacity=1, rectangle, rounded corners,inner sep =0.069cm,xshift = -0.63cm,yshift = 0.04cm] (n3t) {\large$3$};
	\node[at = (formula1), text=black, draw=orange, fill opacity = 0.25, text opacity = 1, rectangle, rounded corners,inner sep =0.069cm, xshift = -0.34cm,yshift = 0.025cm] (3+xt) {$\phantom{3}+\phantom{x}$};
	\node[at = (formula1), text=black, fill opacity = 0.25, text opacity = 1, rectangle, rounded corners,inner sep =0.089cm, xshift = -0.05cm,yshift = 0.01cm] (xt) {\large$x$};
	\node[at = (formula1), text=black, draw=darkpurple, fill opacity = 0.25, text opacity = 1, rectangle, rounded corners,inner sep =0.12cm, xshift = -0.0cm,yshift = 0.02cm] (3+x>yt) {\phantom{$3+x>y$}};
	\node[at = (formula1), text=black, fill opacity = 0.25, text opacity = 1, rectangle, rounded corners,inner sep =0.079cm,xshift = 0.28cm,yshift = 0.03cm] (grt) {\large$>$};
	\node[at = (formula1), text=black, fill opacity = 0.25, text opacity = 1, rectangle, rounded corners,inner sep =0.079cm,xshift = 0.625cm,yshift = -0.03cm] (yt) {\large$y$};
	\node[at = (formula2), text=black, fill opacity = 0.25, text opacity = 1, rectangle, rounded corners,inner sep =0.079cm,xshift = -1.15cm,yshift = 0.02cm] (n21t) {\large$21$};
	\node[at = (formula2), fill opacity = 0.25, text opacity = 1, rectangle, rounded corners,inner sep =0.069cm,xshift = -0.755cm,yshift = 0.02cm] (-t) {\large\vphantom{\%}$-$};
	\node[at = (formula2), fill opacity = 0.25, text opacity = 1, rectangle, rounded corners,inner sep =0.069cm,xshift = -0.475cm,yshift = 0.0cm] (openBt) {\large$($};
	\node[at = (formula2), fill opacity = 0.25, text opacity = 1, rectangle, rounded corners,inner sep =0.079cm,xshift = -0.29cm,yshift = -0.02cm] (xt') {\large$x$};
	\node[at = (formula2), fill opacity = 0.25, draw=darkblue, text opacity = 1, rectangle, rounded corners,inner sep =0.16cm,xshift = -0.395cm,yshift = -0.0cm] (21-x+2t) {$\phantom{21-(x+2)\!\!\!}$};
	\node[at = (formula2), text=black, fill opacity = 0.25, text opacity = 1, rectangle, rounded corners,inner sep =0.069cm,xshift = 0.0cm,yshift = -0.0cm] (+t') {\large$+$};
	\node[at = (formula2), text=black,draw=darkgold, fill opacity = 0.25, text opacity = 1, rectangle, rounded corners,inner sep =0.06cm,xshift = -0.01cm,yshift = -0.0cm] (x+2t) {$\phantom{x+2\!}$};
	\node[at = (formula2), text=black,draw=darkgreen, fill opacity = 0.25, text opacity = 1, rectangle, rounded corners,inner sep =0.08cm,xshift = -0.01cm,yshift = -0.0cm] (x+2Bt) {$\phantom{(x+2)\!\!}$};
	\node[at = (formula2), text=black, fill opacity = 0.25, text opacity = 1, rectangle, rounded corners,inner sep =0.079cm,xshift = 0.29cm,yshift = 0.02cm] (n2t) {\large$2$};
	\node[at = (formula2), text=black,draw=darkred, fill opacity = 0.25, text opacity = 1, rectangle, rounded corners,inner sep =0.24cm,xshift = 0.0cm,yshift = 0.0cm] (21-x+2=21t) {$\phantom{21-(x+2)=21\!\!\!}$};
	\node[at = (formula2), text=black, fill opacity = 0.25, text opacity = 1, rectangle, rounded corners,inner sep =0.069cm,xshift = 0.46cm,yshift = 0.00cm] (closeBt) {\large$)$};
	\node[at = (formula2), text=black, fill opacity = 0.25, text opacity = 1, rectangle, rounded corners,inner sep =0.069cm,xshift = 0.75cm,yshift = 0.02cm] (eqt) {\large\vphantom{l\%,}$=$};
	\node[at = (formula2), text=black, fill opacity = 0.25, text opacity = 1, rectangle, rounded corners,inner sep =0.069cm,xshift = 1.149cm,yshift = 0.02cm] (n21t') {\large$21$};
	\draw[-,orange,out=140,in=-100] (3+xt) edge (3+x);
	\draw[-,darkpurple,out=105,in=-75] (3+x>yt) edge (3+x>y);
	\draw[-,darkgold,out=160,in=-85] (x+2t) edge (x+2);
	\draw[-,darkgreen,out=150,in=-60] (x+2Bt) edge (x+2B);
	\draw[-,darkblue,out=90,in=-90] (21-x+2t) edge (21-x+2);
	\draw[-,darkred,out=40,in=-130] (21-x+2=21t) edge (21-x+2=21);
	\end{tikzpicture}
\end{center}

The types for primitive and composite tokens are complemented by further useful types:\vskip8pt
\begin{multicols}{2}
\begin{itemize}
	%\begin{minipage}[t]{0.48\linewidth}
	\item[-] \FOAnum\ (numeral),
    \item[-] \FOAvar\ (variable),
    \item[-] \FOAnumexp\ (numerical expression),
    %
	%\end{minipage}
%\hfill
	%\begin{minipage}[t]{0.48\linewidth}
	\item[-]  \FOAbop\ (binary operator),
    \item[-] \FOAbrel\ (binary relation)\footnote{The types \FOAbop\ and \FOAbrel\ are \textit{composites}, and we note the more standard convention of writing $\FOAnumexp\times \FOAnumexp \to \FOAnumexp$ and $\FOAnumexp\times \FOAnumexp \to \FOAform$.},
    \item[-] $\FOApar$ (parentheses)
    \item[-] \FOAform\ (formula)\footnote{An alternative encoding of $\rsystemn_{\FOA}$ could include additional types such as (with illustrative tokens): $\FOAsum$, tokens: $1+2$, $3+26$; and $\FOAdifference$, tokens: $1-2$, $3-26$. This gives a finer-grained encoding, distinguishing tokens that exploit different operators such as $1+2$ and $1-2$. Types that distinguish tokens formed from different binary relations, $\texttt{bgr}$ and $\texttt{beq}$ to distinguish $1>2$ and $1=2$, may also be of use.}.
	%\end{minipage}
\end{itemize}
\end{multicols}\vskip8pt
The diagram below illustrates the partial order, $\leq$, over the set $\types$, including some of the types assigned to primitives and composites:
\begin{center}
	\begin{tikzpicture}[node distance = 0.5cm and 0.8cm]
	\node[] (numExp) {$\FOAnumexp$};
	\node[below left = of numExp, xshift = -0.2cm] (var) {$\FOAvar$};
	\node[below = of numExp, xshift = -0.55cm] (num) {$\FOAnum$};
	\node[below left = of var, xshift = 0.9cm] (x) {\texttt{x}};
	\node[below = of var, xshift = -0.0cm,yshift=0.15cm] (y) {\texttt{y}};
	\node[below = of var, xshift = 0.55cm,yshift=0.35cm] (vardots) {\rotatebox{20}{$\cdots$}};
	\node[below left = of num, xshift = 1cm,yshift=-0.1cm] (0) {\texttt0};
	\node[below = of num, xshift = -0.05cm,yshift=0.02cm] (1) {\texttt{1}};
	\node[below = of num, xshift = 0.35cm,yshift=0.15cm] (2) {\texttt{2}};
	\node[below right = of num, xshift = -0.8cm,yshift=0.35cm] (numdots) {\rotatebox{25}{$\cdots$}};
	\node[below right = 1.45cm and -0.15cm of numExp] (x-1) {\texttt{x\_mi\_1}};
	\node[below right = 0.9cm and 0.5cm of numExp] (2+y) {\texttt{2\_pl\_y}};
	\node[at = (2+y), yshift = 0.5cm, xshift = 0.45cm] (numExpdots) {\rotatebox{45}{$\cdots$}};
	\node[right = 2.6cm of numExp,yshift=-0.45cm] (bop) {\texttt{bop}};
	\node[below left = 0.5cm and -0.35cm of bop] (pl) {\texttt{pl}};
	\node[below right = 0.9cm and -0.35cm of bop] (mi) {\texttt{mi}};
	\node[right = 0.55cm of bop, yshift = 0.4cm] (brel) {\texttt{brel}};
	\node[below left = 0.6cm and -0.3cm of brel] (eq) {\texttt{eq}};
	\node[below right = 0.6cm and -0.2cm of brel] (gr) {\texttt{gr}};
	\node[right = 0.55cm of brel,yshift=-0.45cm] (par) {\texttt{par}};
	\node[below left = 0.9cm and -0.35cm of par] (oB) {\texttt{oB}};
	\node[below right = 0.5cm and -0.35cm of par] (cB) {\texttt{cB}};
	\node[right = 1.2cm of par, yshift = 0.3cm] (form) {\texttt{form}};
	\node[below left = 1.2cm and -0.6cm of form] (1eq2) {\texttt{1\_eq\_2}};
	\node[below right = 0.7cm and -0.7cm of form] (3xy) {\texttt{3\_pl\_x\_gr\_y}};
	\node[at  = (3xy),xshift = 0.8cm,yshift = 0.5cm] () {\rotatebox{30}{$\cdots$}};
	
	\draw[rounded corners, black!30, thick] (-3,-2.3) rectangle node[xshift = 6.4cm,yshift = 1cm,black!50] {$\tsystem$} (11.4,0.3);
	\coordinate (formula1) at (-0.5,-3.2) ;
	\draw[->] (var) edge (numExp);
	\draw[->] (num) edge (numExp);
	\draw[->] (x-1) edge (numExp);
	\draw[->] (2+y) edge (numExp);
	\draw[->] (x) edge (var);
	\draw[->] (y) edge (var);
	\draw[->] (0) edge (num);
	\draw[->] (1) edge (num);
	\draw[->] (2) edge (num);
	\draw[->] (pl) edge (bop);
	\draw[->] (mi) edge (bop);
	\draw[->] (eq) edge (brel);
	\draw[->] (gr) edge (brel);
	\draw[->] (oB) edge (par);
	\draw[->] (cB) edge (par);
	\draw[->] (1eq2) edge (form);
	\draw[->] (3xy) edge (form);
	\end{tikzpicture}
\end{center}
\end{example}

\begin{example}[Grammatical Constructors]\label{ex:grammaticalConstructors}
We now exemplify the \textit{constructors} that are used in $\rsystemn_{\FOA}$'s grammatical space. For each way of building tokens, embodied by their inductive construction, there is a constructor in the system: $\FOAcinfixop$, $\FOAcaddPar$, $\FOAcinfixrel$, and $\FOAcquantify$\footnote{Note, as with types, we adopt the convention of using \texttt{typewriter} font for constructor names: in examples, the vertices in graphs will consistently use \texttt{typewriter} font for their labels. Specifically in this example, the constructor $\FOAcinfixop$ does not impose the presence of parentheses, reflecting the typical informal use of such expressions in practice. Therefore, $\rsystemn_{\FOA}$ includes a bracketing constructor, $\FOAcaddPar$. If $\rsystemn_{\FOA}$ included additional types, such as $\FOAsum$, then more constructors could be included. One example is the constructor $\FOAcsum$ with input type-sequence $[\FOAnumexp,\FOAplus,\FOAnumexp]$ and output type $\FOAsum$.}. For instance, $\FOAcinfixop$ will build tokens such as $1+3$ and has \textit{input type-sequence} $[\FOAnumexp,\FOAbop,\FOAnumexp]$ and \textit{output type} $\FOAnumexp$.{\pagebreak} The constructors' visualisations, where the \textit{indexed} arrows indicate the order of their sources' labels in the input type-sequence, \begin{samepage}are:
\begin{center}
% \begin{minipage}[c][3cm]{0.2\textwidth}    % numerical expression
\begin{tikzpicture}[construction,yscale=0.9]
\node[typeE={$\FOAnumexp$}] (v) at (3.2,4.1) {};
\node[constructorNW={$\FOAcinfixop$}] (u) at (3.2,3.2) {};
\node[typeS={$\FOAnumexp$}] (v2) at (2.2,2.2) {};
\node[typeS={$\FOAbop$}] (v3) at (3.2,2.2) {};
\node[typeS={$\FOAnumexp$}] (v4) at (4.2,2.2) {};
%\node[termS={$\FOAopenb$}] (v1) at (1.6,2.2) {};
%\node[termS={$\FOAcloseb$}] (v5) at (4.8,2.2) {};
\path[->]
(u) edge[bend right = 0] (v)
%(v1) edge[bend right = -20] node[index label] {1} (u)
(v2) edge[bend right = -10] node[index label] {1} (u)
(v3) edge[bend right = 0] node[index label] {2} (u)
(v4) edge[bend right = 10] node[index label] {3} (u)
%(v5) edge[bend right = 20] node[index label] {5} (u)
;
\end{tikzpicture}
% \end{minipage}
%
\hspace{0.9cm}
%
% \begin{minipage}[c][3cm]{0.2\textwidth}    % brackets
\begin{tikzpicture}[construction,yscale=0.9]
\node[typeE={$\FOAnumexp$}] (v) at (3.2,4.1) {};
\node[constructorNW={$\FOAcaddPar$}] (u) at (3.2,3.2) {};
\node[typeS={$\FOAopenb$}] (v2) at (2.2,2.2) {};
\node[typeS={$\FOAnumexp$}] (v3) at (3.2,2.2) {};
\node[typeS={$\FOAcloseb$}] (v4) at (4.2,2.2) {};
%\node[termS={$\FOAopenb$}] (v1) at (1.6,2.2) {};
%\node[termS={$\FOAcloseb$}] (v5) at (4.8,2.2) {};
\path[->]
(u) edge[bend right = 0] (v)
%(v1) edge[bend right = -20] node[index label] {1} (u)
(v2) edge[bend right = -10] node[index label] {1} (u)
(v3) edge[bend right = 0] node[index label] {2} (u)
(v4) edge[bend right = 10] node[index label] {3} (u)
%(v5) edge[bend right = 20] node[index label] {5} (u)
;
\end{tikzpicture}
% \end{minipage}
%
\hspace{0.9cm}
%
% \begin{minipage}[c][3cm]{0.2\textwidth} % binary relation formulae
\begin{tikzpicture}[construction,yscale=0.9]
\node[typeE={$\FOAform$}] (v) at (3.2,4.1) {};
\node[constructorNW={$\FOAcinfixrel$}] (u) at (3.2,3.2) {};
\node[typeS={$\FOAnumexp$}] (v2) at (2.2,2.2) {};
\node[typeS={$\FOAbrel$}] (v3) at (3.2,2.2) {};
\node[typeS={$\FOAnumexp$}] (v4) at (4.2,2.2) {};
\path[->]
(u) edge[bend right = 0] (v)
(v2) edge[bend right = -10] node[index label] {1} (u)
(v3) edge[bend right = 0] node[index label] {2} (u)
(v4) edge[bend right = 10] node[index label] {3} (u)
;
\end{tikzpicture}
% \end{minipage}
%
\hspace{0.9cm}
%
% \begin{minipage}[c][3cm]{0.2\textwidth} %quantified formula
\begin{tikzpicture}[construction,yscale=0.9]
\node[typeE={$\FOAform$}] (v) at (3.2,4.1) {};
\node[constructorNW={$\FOAcquantify$}] (u) at (3.2,3.2) {};
\node[typeS={$\FOAquant$}] (v2) at (2.2,2.2) {};
\node[typeS={$\FOAvar$}] (v3) at (3.2,2.2) {};
\node[typeS={$\FOAform$}] (v4) at (4.2,2.2) {};
\path[->]
(u) edge[bend right = 0] (v)
(v2) edge[bend right = -10] node[index label] {1} (u)
(v3) edge[bend right = 0] node[index label] {2} (u)
(v4) edge[bend right = 10] node[index label] {3} (u);
\end{tikzpicture}
% \end{minipage}
\end{center}
\end{samepage}
% \end{itemize}
\end{example}

As is evident in example~\ref{ex:grammaticalConstructors}, each constructor can be \textit{specified} by its \textit{input type-sequence} and \textit{output type}, embodied in definition~\ref{defn:constructionSpecification}. Together, the input type-sequence and output type form the \textit{signature} of the constructor which is akin to the signature of a function or relation: it tells us the types of the tokens that the constructor requires as inputs (the `domain') and the type of the output (the `co-domain'). We emphasise that constructors are essentially acting as relation names and their signatures indicate the types of the tokens that are related.

\begin{definition}\label{defn:constructionSpecification}
A \textit{constructor specification}, $\cspecificationn$, over type system $\tsystemn=\tsystem$ is a pair, $\cspecificationn=(\constructors, \spec)$,  where
\begin{enumerate}
\item $\constructors$ is a set, disjoint from $\types$, whose elements are called \textit{constructors}, and
\item $\spec\colon \constructors \to \sequence(\types)\times \types$ is a function that returns, for each constructor, $c$, a \textit{signature}, $\spec(c)=([\tau_1,\ldots,\tau_n], \tau)$, where $[\tau_1,\ldots,\tau_n]$ is non-empty.
\end{enumerate}
The \textit{input type-sequence} and the \textit{output type} for $c$, denoted $\inputsty{c}$ and $\outputsty{c}$ respectively, are $\inputsty{c}=[\tau_1,\ldots,\tau_n]$ and  $\outputsty{c}=\tau$.
\end{definition}

Note that the requirement for $\constructors$ and $\types$ to be disjoint is simply to allow us to know, given an element, $e$, of $\constructors\cup \types$, whether $e$ is a constructor or a type. Given any constructor, $c$, with signature $([\tau_1,\ldots,\tau_n],\tau)$, $c$ can take many different input token-sequences and output many different tokens. The variety of input token-sequences will be encapsulated using constructor \textit{configurations}: bipartite graphs, whose vertices are tokens (labelled by types) or \textit{configurators} (labelled by constructors). In example~\ref{ex:grammaticalConstructors}, we visualised constructor signatures using graphs whose vertices were labelled by types or constructors. When we visualise \textit{configurations}, vertices that are tokens will be labelled by types and the remaining vertices will be labelled by constructors, as we shall see in example~\ref{ex:configOfCons}. In addition, just as we can think of constructors as relation names, a configuration is akin to an element of that relation.

\begin{example}[Configurations of Constructors]\label{ex:configOfCons}
The following constructor configurations\footnote{To avoid clutter in the graphs, the types that label the vertices that are tokens are omitted.} form part of the \textit{grammatical space} of $\rsystemn_{\FOA}$:
\begin{center}
% \begin{minipage}{0.22\textwidth}    % numerical expression
\begin{tikzpicture}[construction,yscale=0.9]
\node[termrep] (v) at (3.2,4.1) {\scriptsize$(7+25)-1$};
\node[constructorNW={$\FOAcinfixop$}] (u) at (3.2,3.2) {};
\node[termrep] (v2) at (2.2,2.2) {\scriptsize$(7+25)$};
\node[termrep] (v3) at (3.2,2.2) {\scriptsize\phantom{.}\vphantom{$1$}$-$\phantom{.}};
\node[termrep] (v4) at (4.2,2.2) {\scriptsize$1$};
%\node[termS={$($}] (v1) at (1.6,2.2) {$v_1$};
%\node[termS={$)$}] (v5) at (4.8,2.2) {$v_5$};
\path[->]
(u) edge[bend right = 0] (v)
%(v1) edge[bend right = -15] node[index label] {1} (u)
(v2) edge[bend right = -10] node[index label] {1} (u)
(v3) edge[bend right = 0] node[index label] {2} (u)
(v4) edge[bend right = 10] node[index label] {3} (u)
%(v5) edge[bend right = 15] node[index label] {5} (u)
;
\end{tikzpicture}
% \end{minipage}
%
\hspace{1cm}
%
% \begin{minipage}{0.22\textwidth}    % numerical expression
\begin{tikzpicture}[construction,yscale=0.9]
\node[termrep] (v) at (3.2,4.1) {\scriptsize$(7+25)$};
\node[constructorNW={$\FOAcaddPar$}] (u) at (3.2,3.2) {};
\node[termrep] (v2) at (2.3,2.2) {\scriptsize$($};
\node[termrep] (v3) at (3.2,2.2) {\scriptsize$7+25$};
\node[termrep] (v4) at (4.1,2.2) {\scriptsize$)$};
\path[->]
(u) edge[bend right = 0] (v)
(v2) edge[bend right = -10] node[index label] {1} (u)
(v3) edge[bend right = 0] node[index label] {2} (u)
(v4) edge[bend right = 10] node[index label] {3} (u)
;
\end{tikzpicture}
% \end{minipage}
%
\hspace{1cm}
%
% \begin{minipage}{0.20\textwidth} % binary relation formulae
\begin{tikzpicture}[construction,yscale=0.9]
\node[termrep] (v) at (3.2,4.1) {\scriptsize$x>3$};
\node[constructorNW={$\FOAcinfixrel$}] (u) at (3.2,3.2) {};
\node[termrep] (v2) at (2.3,2.2) {\scriptsize$x$};
\node[termrep] (v3) at (3.2,2.2) {\scriptsize$>$};
\node[termrep] (v4) at (4.1,2.2) {\scriptsize$3$};
\path[->]
(u) edge[bend right = 0] (v)
(v2) edge[bend right = -10] node[index label] {1} (u)
(v3) edge[bend right = 0] node[index label] {2} (u)
(v4) edge[bend right = 10] node[index label] {3} (u)
;
\end{tikzpicture}
% \end{minipage}
%
\hspace{1cm}
%
% \begin{minipage}{0.25\textwidth} %quantified formula
\begin{tikzpicture}[construction,yscale=0.9]
\node[termrep] (v) at (3.2,4.1) {\scriptsize$\forall x \thinspace x> 3$};
\node[constructorNW={$\FOAcquantify$}] (u) at (3.2,3.2) {};
\node[termrep] (v2) at (2.3,2.2) {\scriptsize$\forall$};
\node[termrep] (v3) at (3.2,2.2) {\scriptsize$x$};
\node[termrep] (v4) at (4.1,2.2) {\scriptsize$x> 3$};
\path[->]
(u) edge[bend right = 0] (v)
(v2) edge[bend right = -10] node[index label] {1} (u)
(v3) edge[bend right = 0] node[index label] {2} (u)
(v4) edge[bend right = 10] node[index label] {3} (u);
\end{tikzpicture}
% \end{minipage}\\
\end{center}
Each configuration associates its constructor with an \textit{input token-sequence}. In turn, the type sequence directly obtained from the input token-sequence \textit{specialises} the constructor's input type-sequence: the configuration of $\FOAcinfixop$ has input token-sequence $[(7+25),-,1]$, that gives rise to input type-sequence $[\FOAopenb\_\texttt{7}\_\FOAplus\_\texttt{25}\_\FOAcloseb,\FOAminus,\texttt{1}]$ which specialises $\inputsty{\FOAcinfixop}=[\FOAnumexp,\FOAbop, \FOAnumexp]$.
\end{example}

As is evident from example~\ref{ex:configOfCons}, the input token-sequence, $[t_1,\ldots,t_n]$, of constructor $c$ ensures that each $t_i$ is assigned a subtype of $\tau_i$, where $\spec(c)=([\tau_1,\ldots,\tau_n],\tau)$. We now define a \textit{specialisation} of a type sequence, where we can replace types with subtypes.

\begin{definition}\label{defn:specialization}
Let $\T=[\tau_1,\ldots,\tau_n]$ be a sequence over the set $\types$, given a type system $\tsystemn=\tsystem$. A \textit{specialisation} of $\T$ given $\tsystemn$ is a sequence, $\T'=[\tau_1',\ldots,\tau_n']$, such that for each $i$, it is the case that $\tau_i'\leq \tau_i$. Whenever $\T'$ is a specialisation of $\T$ we say that $\T$ is a \textit{generalisation} of $\T'$.
\end{definition}

\begin{definition}\label{defn:configuration}
Let $\cspecificationn=\cspecification$ be a constructor specification over $\tsystemn=\tsystem$.\linebreak Let $c\in \constructors$ where
$\spec(c) =([\tau_1,\ldots,\tau_n], \tau)$.  A \textit{configuration} of $c$ is a graph,\linebreak $\graphn=\graph$, where
\begin{enumerate}
\item $\pb$ contains a single vertex, $u$, called a $c$-\textit{configurator}: $\pb=\{u\}$,
\item the vertices in $\pa$, called \textit{tokens}, are all adjacent to $u$: $\pa=\inV{u}\cup \outV{u}$,
\item $u$ has exactly one outgoing arrow, $a_{0}$: $\outA{u}=\{a_{0}\}$,
\item $u$ has exactly $n$ incoming arrows: $\inA{u}=\{a_1,\ldots,a_n\}$,
\item $\arrowl\colon \arrows\to \{0,1,\ldots,n\}$ is a bijection that labels each arrow $a_{i}$ with $i$,
\item $\tokenl\colon \pa \to \types$ is a function that labels each token with a type, such that for all $t\in \pa$,
\begin{enumerate}
            \item if $\tar{a_{0}}=t$  then $\tokenl(t)\leq \tau$, and
            \item the vertex-label sequence $[\tokenl(\sor{a_1}),\ldots,\tokenl(\sor{a_n})]$ is a specialisation of $[\tau_1,\ldots,\tau_n]$,
            \end{enumerate}
and
\item $\consl\colon \pb \to \constructors$ is a function where $\consl(u)=c$.

\end{enumerate}
The \textit{output token}, $\tar{a_0}$, of $u$, is denoted $\outputsto{u}$ and the \textit{output type} of $u$ is $\outputsty{u}=\tokenl(\tar{a_0})$. We further define:
\begin{enumerate}
\item the \textit{input arrow-sequence} of $u$, denoted $\inputsA{u}$, to be the sequence of arrows in $\inA{u}$ ordered by their indices: $\inputsA{u} = [\arrowl^{-1}(1),\ldots,\arrowl^{-1}(n)] = [a_1,\ldots,a_n]$,
\item the \textit{input token-sequence} of $u$, denoted $\inputsto{u}$, replaces each arrow in $\inputsA{u}$ with its source: $\inputsto{u} = [\sor{a_1},\ldots,\sor{a_n}]$, and
\item the \textit{input type-sequence} of $u$, denoted $\inputsty{u}$, replaces each token in $\inputsto{u}$ with its assigned type: $\inputsty{u} = [\tokenl(\sor{a_1}),\ldots,\tokenl(\sor{a_n})]$.
\end{enumerate}
\end{definition}

In a configuration, the outgoing arrow $a_0$ is labelled by $0$ simply because the arrow labelling function is total, not partial. We adopt the convention, when visualizing configurations of constructors, of omitting the label $0$ that appears on $a_0$, since it can always be deduced. In addition, vertex labels will be omitted when they are not relevant for the context, as was done in example~\ref{ex:configOfCons}.

It is not yet evident, from the \textsc{First-Order Arithmetic} examples, why configurations can be non-simple graphs. In essence, this need arises due to representational systems that can use the \textit{same} token \textit{multiple times} as a constructor input. Example~\ref{ex:vectorMulti}, which illustrates the use of a multi-graph, makes use of the \textit{entailment space}, $\espacen$, of a system that graphically represents vectors in the plane. 

\begin{example}[The Need for Multi-Graphs]\label{ex:vectorMulti} Consider a geometric representational system that visualises vectors in the plane. The vector $(1,1)$ is graphically represented by token $t_1$, below left, which is of type $\GRVvec$. The signature of constructor, $\GRVcvecadd$, for drawn vector addition is visualised below middle. The vector $(1,1)+(1,1)=(2,2)$ is graphically represented  by $t_2$, below right.
\begin{center}
\begin{minipage}[t]{0.3\textwidth}\centering
\textit{representation, $t_1$, of $(1,1)$:}\\[1ex]
 \begin{tikzpicture}[x=0.5cm,y=0.5cm,step=0.5cm,>={angle 90}]  %projectile motion
 \draw[very thin, gray!20] (0,0) grid (3.5,3.5);
\draw[->] (-0.1,0) -- (3.5,0) node[right] {$x$};
\draw[->] (0,-0.1) -- (0,3.5) node[above] {$y$};
\draw[->, thick, color=darkblue] (0,0) -- (1,1){};% node[above] {$\underline{v}$};
\draw[-, ] (1,-0.1) -- (1,0.1) node[below] {\tiny $1$};
\draw[-, ] (2,-0.1) -- (2,0.1) node[below] {\tiny $2$};
\draw[-, ] (3,-0.1) -- (3,0.1) node[below] {\tiny $3$};
\draw[-, ] (-0.1,1) -- (0.1,1) node[left] {\tiny $1$};
\draw[-, ] (-0.1,2) -- (0.1,2) node[left] {\tiny $2$};
\draw[-, ] (-0.1,3) -- (0.1,3) node[left] {\tiny $3$};
\end{tikzpicture}
\end{minipage}
\hfill
\begin{minipage}[t]{0.32\textwidth}\centering
\textit{signature of $\GRVcvecadd$:}\\[2ex]
\begin{tikzpicture}[construction,yscale=0.85]
\node[typeE={$\GRVvec$}] (v) at (3.2,4.2) {};
\node[constructorENW={$\GRVcvecadd$}] (u) at (3.2,3.2) {};
\node[typeS={$\GRVvec$}] (v2) at (2.3,2.2) {};
\node[typeS={$\GRVvec$}] (v4) at (4.1,2.2) {};
\path[->]
(u) edge[bend right = 0] (v)
(v2) edge[bend right = -10] node[index label] {1} (u)
(v4) edge[bend right = 10] node[index label] {2} (u);
\end{tikzpicture}
\end{minipage}
\hfill
\begin{minipage}[t]{0.3\textwidth}\centering
\textit{representation, $t_2$, of $(2,2)$:}\\[1ex]
\begin{tikzpicture}[x=0.5cm,y=0.5cm,step=0.5cm,>={angle 90}]  %projectile motion
\draw[very thin, gray!20] (0,0) grid (3.5,3.5);
\draw[->, ] (-0.1,0) -- (3.5,0) node[right] {$x$};
\draw[->, ] (0,-0.1) -- (0,3.5) node[above] {$y$};
\draw[->, thick, color=darkblue] (0,0) -- (2,2) {};% node[above] {$\underline{v}+\underline{v}$};
\draw[-, ] (1,-0.1) -- (1,0.1) node[below] {\tiny $1$};
\draw[-, ] (2,-0.1) -- (2,0.1) node[below] {\tiny $2$};
\draw[-, ] (3,-0.1) -- (3,0.1) node[below] {\tiny $3$};
\draw[-, ] (-0.1,1) -- (0.1,1) node[left] {\tiny $1$};
\draw[-, ] (-0.1,2) -- (0.1,2) node[left] {\tiny $2$};
\draw[-, ] (-0.1,3) -- (0.1,3) node[left] {\tiny $3$};
\end{tikzpicture}
\end{minipage}
\end{center}
This non-simple configuration shows that adding $t_1$ to itself constructs $t_2$\\
\begin{center}
\begin{tikzpicture}[construction]
\begin{scope}[rotate=-90]
\node[termE={}] (vi) at (0,0) {\scalebox{.5}{\vectorVis{(1,1)}{0.4}}};
\node[constructorIpos={$\GRVcvecadd$}{60}{0.17cm}] (u) at (0,1.9) {};
\node[termpos={}{160}{0.17cm}] (vo) at (0,3.8) {\scalebox{.5}{\vectorVis{(2,2)}{0.4}}};
\end{scope}
\path[->]
(u) edge[bend right = 0] (vo)
(vi) edge[bend right = 20] node[index label] {2} (u)
(vi) edge[bend right = -20] node[index label] {1} (u);
\end{tikzpicture}
\end{center}
\end{example}

Each configuration of a constructor gives us \textit{local} information about tokens that construct another token. In turn, the constructing tokens need not be primitives and could well be built from other tokens. Thus, to encode an entire system, for each space (grammatical, entailment and identification) we exploit a \textit{structure graph} that is a union of configurations.

\begin{example}[Structure Graphs] The configurations in example~\ref{ex:configOfCons} are subgraphs of a structure graph, $\graphn$, that forms the grammatical space, $\gspacen_{\FOA}$, of $\rsystemn_{\FOA}$. Taking the two rightmost configurations in example~\ref{ex:configOfCons}, the graph below is also a subgraph of $\graphn$:
\begin{center}
\begin{tikzpicture}[construction,yscale=0.8]
\begin{scope}[rotate=-90]
\node[termrep] (v) at (3.2,5.0) {\scriptsize$\forall x \thinspace x> 3$};
\node[constructorpos={\scriptsize$\FOAcquantify$}{35}{0.34cm}] (u) at (3.2,3.4) {};
\node[termrep] (v2) at (2.2,2.4) {\scriptsize$\forall$};
\node[termrep] (v3) at (3.2,2.4) {\scriptsize$x$};
\node[termrep] (v4) at (4.2,2.4) {\scriptsize$x>3$};
\node[constructorpos={$\FOAcinfixrel$}{35}{0.34cm}] (1u) at (4.2,1) {};
\node[termrep] (1v2) at (3.2,0) {\scriptsize$x$};
\node[termrep] (1v3) at (4.2,0) {\scriptsize$>$};
\node[termrep] (1v4) at (5.2,0) {\scriptsize$3$};
\end{scope}
\path[->]
(u) edge[bend right = 0] (v)
(v2) edge[bend right = -10] node[index label] {1} (u)
(v3) edge[bend right = 0] node[index label] {2} (u)
(v4) edge[bend right = 10] node[index label] {3} (u);

\path[->]
(1u) edge[bend right = 0] (v4)
(1v2) edge[bend right = -10] node[index label] {1} (1u)
(1v3) edge[bend right = 0] node[index label] {2} (1u)
(1v4) edge[bend right = 10] node[index label] {3} (1u);
\end{tikzpicture}
\end{center}
Note the two distinct vertices containing distinct instances of $x$: by contrast to example~\ref{ex:vectorMulti}, which used one token, $t_1$, to construct $t_2$, here two $x$ tokens are used to build $\forall x \thinspace x > 3$.
\end{example}
\begin{definition}\label{defn:structureGraph}
A \textit{structure graph}, $\graphn=\graph$, for a constructor specification $\cspecificationn=(\constructors, \spec)$ is a graph where for all $u\in \pb$, $\consl(u)$ is in $\constructors$ and $\neigh{u}$ is a configuration of $\consl(u)$.
The elements of $\pb$ are called \textit{configurators}.
\end{definition}
We are now in a position to define a construction space.
\begin{definition}\label{defn:constructionSpace}
A \textit{construction space} is a triple, $\cspacen=\cspace$, where %
\begin{enumerate}
\item $\tsystemn=\tsystem$ is a type system,
\item $\cspecificationn=\cspecification$ is a constructor specification over $\tsystemn$, and
\item $\graphn=\graph$ is a structure graph for $\cspecificationn$.
\end{enumerate}
We say that $\cspacen$ is a construction space \textit{formed over} $\tsystemn$.
\end{definition}

Construction spaces enable a graph-theoretical encoding of representations, generalising the concept of a \textit{parse tree} by including non-directed and directed cycles. In particular, construction spaces support the encoding of:
\begin{itemize}
\item[-] the same token, $t$ say, being used multiple times when constructing another token, as seen in example~\ref{ex:vectorMulti}; when this happens, $t$ is the source of multiple arrows that all target the same configurator and, hence, the structure graph contains a non-directed cycle and is not a tree,
\item[-] a range of construction choices for any given token, $t$, as will be seen in example~\ref{ex:encodingSemantics}; when this happens, $t$ is the target of multiple arrows sourced on different configurators, and
\item[-] cyclical ways of constructing tokens, such as when $t$ can be constructed using $t'$ and vice-versa, as will be seen in Section~\ref{sec:constructions:AMEGROV};  when this happens, the associated structure graph contains a directed cycle and is not a DAG.
\end{itemize}
In Section~\ref{sec:formalizingRSs}, we will show how construction spaces can be used to encode the grammatical, entailment, and identification spaces of representational systems. In particular, each of these three spaces \textit{is a construction space} that satisfies certain properties. For instance, the grammatical space is taken to be \textit{functional}, defined in Section~\ref{sec:propertiesOfConstructionSpaces}. Hence, construction spaces play a fundamental role in our formalisation of representational systems.

\subsection{Properties of Construction Spaces}\label{sec:propertiesOfConstructionSpaces}

Construction spaces can have various properties, including being functional, that are sometimes useful to exploit. Here, we explore four particular aspects.

\begin{paragraph}{Compatibility} Construction spaces are compatible if their union is a construction space. We require that the three spaces over which a representational system is formed are compatible\footnote{This is because structural transformations, covered in Section~\ref{sec:structuralTransformations} sometimes need to exploit constructions formed from configurators drawn from any of the three spaces.}. Section~\ref{sec:compatibility} establishes that for any set of pairwise compatible construction spaces, their union is a construction space.
\end{paragraph}%

\begin{paragraph}{Functional} For a construction space to be functional we require that, if we know the input sequence, $\T$, of a configuration of a constructor then the output is uniquely determined. We define this property in Section~\ref{sec:determinism} at both token and type levels.
\end{paragraph}%

\begin{paragraph}{Totality} The property of totality means, informally, that for a constructor, $c$, and any sequence, $\T$, that conforms to $\spec(c)$, there exists a $c$-configurator with input sequence $\T$. Whilst not an essential property of any construction space, $\cspacen$, if $\cspacen$ is total (i.e., all of its constructors are total) then we can exploit that information when identifying structural transformations (Section~\ref{sec:structuralTransformations}). Section~\ref{sec:totality} defines totality at both the token and type levels. Notably, formal languages are likely to employ total constructors, whereas informal languages (such as natural language) are not\footnote{For instance, in NL, one can imagine a constructor that appends one sequence of words on to the end of another sequence to form a sentence; there will be cases were two input sequences of words do not form a sentence, so the constructor is not total.}.
\end{paragraph}%

\begin{paragraph}{Type extendability} In Section~\ref{sec:interRepSystems}, we develop theory that permits the definition of inter-repre\-sen\-ta\-tio\-nal-system relationships. In this context, it can be useful to exploit types that need not appear in the given representational systems. We show how to extend the set of types without inherently changing the given systems. Section~\ref{sec:extendability} focuses on extensions of the type system that provide sets of types with upper bounds.
\end{paragraph}\medskip

Sections~\ref{sec:compatibility} (compatibility) and~\ref{sec:determinism} (functionality) are pre-requisites for definition~\ref{defn:representationalSystem}, of a representational system, whereas Sections~\ref{sec:totality} (totality) and~\ref{sec:extendability} (extendability) are exploited when defining properties of representational systems and identifying relationships between them.

\subsubsection{Compatibility}\label{sec:compatibility}

The grammatical and entailment spaces, $\gspacen$ and $\espacen$, of a representational system, $\rsystemn$, will be formed over the same type system, $\tsystemn$. The identification space, $\ispacen$, will use additional types -- which we call meta-level types since they need not be in the type system of $\gspacen$ -- alongside those in $\tsystemn$\footnote{Recall that given a constructor, $c$, in the identification space with input $[t_1,\ldots,t_n]$, the output is not a token drawn from the grammatical space but, rather, is a meta-level token, $t$, with a specified meta-level type, $\omega$.}. It is necessary that these meta-level types form a type system that is \textit{compatible} with $\tsystemn$ and that the constructor specifications of $\gspacen$, $\espacen$ and $\ispacen$ are also \textit{compatible}.

\begin{definition}\label{defn:compatible}
Type systems $\tsystemn_1=\tsystemp{1}$ and $\tsystemn_2=\tsystemp{2}$ are \textit{compatible} provided $\tsystemn_1\cup \tsystemn_2$ is a type system.
\end{definition}

\begin{lemma}\label{lem:pairwiseCompatibleTTSs}
Let $\tsystemn_1=\tsystemp{1}$, $\tsystemn_2=\tsystemp{2}$ and $\tsystemn_3=\tsystemp{3}$ be pairwise compatible type systems. Then $\tsystemn_1\cup \tsystemn_2$ is compatible with $\tsystemn_3$.
\end{lemma}

The proof, which is straightforward, can be found in appendix~\ref{sec:app:representationalSystems} (see lemma~\ref{alem:pairwiseCompatibleTTSs}).

\begin{definition}\label{defn:compatibleCS}
Construction spaces $\cspacen_{1}=\cspacep{1}$ and $\cspacen_{2}=\cspacep{2}$ are \textit{compatible} provided $\cspacen_{1}\cup \cspacen_{2}$ is a construction space.
\end{definition}

Example~\ref{ex:compatibleConstructionSpaces} discusses the compatibility of $\rsystemn_{\FOA}$'s grammatical and identification spaces. Recall that meta-tokens in the identification space will be visualised by \tikz[construction,baseline=-3pt]{
\node[termIrep] (u) {$t$}}, with \tikz[construction,baseline=-3pt]{
\node[termrep] (u) {$t$}} reserved for non-meta-tokens (i.e. tokens in $\gspacen$ and $\espacen$).

\begin{example}[Compatible Construction Spaces]\label{ex:compatibleConstructionSpaces}
We require that the grammatical, entailment and identification spaces of $\rsystemn_{\FOA}$ are pairwise compatible. This means that $\gspacen_{\FOA}\cup \espacen_{\FOA}\cup \ispacen_{\FOA}$ is a construction space and, in particular, that the union of any of their structure graphs is also a structure graph. Now, $\rsystemn_{\FOA}$ includes a constructor, called $\FOAcprovable$, in its identification space, $\ispacen_{\FOA}$, with signature $([\FOAform],\FOAboolean)$. The structure graph for $\ispacen_{\FOA}$ uses $\FOAcprovable$ to encode whether numerical expressions are valid (universally true), using $\top$, or invalid, using $\bot$: the meta-tokens $\top$ and $\bot$ are both of type $\FOAboolean$.  The following subgraphs of $\ispacen_{\FOA}$'s structure graph encode the validity of $\forall x \thinspace x> x-1$ and the invalidity of $1+(7-4)=22$:
\begin{center}
\begin{tikzpicture}[construction]
\begin{scope}[rotate = 0]
\node[termrep] (v4) at (3.2,4.6) {\scriptsize$\forall x \thinspace x> x-1$};
\node[constructorI={\FOAcprovable}] (u) at (3.2,5.8) {};
\node[termIrep] (v) at (3.2,7) {\scriptsize$\top$};
\end{scope}
\path[->]
(u) edge[bend right = 0] (v)
(v4) edge[bend right = 0] node[index label] {1} (u);
\end{tikzpicture}
\hspace{1.5cm}
\begin{tikzpicture}[construction]
\begin{scope}[rotate = 0]
\node[termrep] (v4) at (3.2,4.6) {\scriptsize$1+(7-4)=22$};
\node[constructorI={\FOAcprovable}] (u) at (3.2,5.8) {};
\node[termIrep] (v) at (3.2,7) {\scriptsize$\bot$};
\end{scope}
\path[->]
(u) edge[bend right = 0] (v)
(v4) edge[bend right = 0] node[index label] {1} (u);
\end{tikzpicture}
\end{center}
Since $\gspacen_{\FOA}$ and $\ispacen_{\FOA}$ are compatible, their union contains a structure graph, $\graphn$. The graph below is a subgraph of $\graphn$ and shows one way of building the token $1+2=3$ using the constructor $\FOAcinfixrel$ found in the grammatical space. This graph encodes the fact that the expression $1+2=3$ is valid, using the $\FOAcprovable$ constructor found within $\ispacen_{\FOA}$, with the contours identifying the spaces in which subgraphs occur:
\begin{center}
	\begin{tikzpicture}[construction, node distance = 0.3cm and 0.1cm, yscale=0.8]
	\begin{scope}[rotate = -90]
	\node[termrep] (v4) at (3.2,4.2) {\scriptsize$1+2=3$};
	\node[termrep] (v1) at (2.2,1.4) {\scriptsize$1+2$};
	\node[constructorGpos={$\FOAcinfixrel$}{37}{0.28cm}] (u1) at (3.2,2.5) {};
	\node[termrep] (v3) at (4.2,1.4) {\scriptsize$3$};
	\node[constructorIN={\FOAcprovable}] (u) at (3.2,5.8) {};
	\node[termIrep] (v) at (3.2,6.8) {\scriptsize$\top$};
	\node[termrep] (v2) at (3.2,1.4) {\scriptsize$=$};
	\end{scope}
	\path[->]
	(u1) edge[bend right = 0] (v4)
	(v1) edge[bend right = -10] node[index label] {1} (u1)
	(v3) edge[bend right = 10] node[index label] {3} (u1)
	(u) edge[bend right = 0] (v)
	(v4) edge[bend right = 0] node[index label] {1} (u)
	(v2) edge[bend right = 0] node[index label] {2} (u1);
	\coordinate[above right = of v4, yshift = 0.3cm] (x1) ;
	\coordinate[below right = of v4, yshift = -0.3cm] (x2) ;
	\coordinate[below left = of v3,xshift = -0.1cm, yshift = 0.15cm] (x3) ;
	\coordinate[above left = of v1,xshift = -0.1cm, yshift = -0.1cm] (x4) ;
	\draw[rounded corners=10,thick, draw opacity = .30, fill opacity = .20, text opacity = 1] (x1) -- (x2) -- (x3) -- node[xshift = -0.35cm, yshift= 0.9cm] {$\gspacen_{\FOA}$} (x4) -- cycle;
	\coordinate[above right = of v,yshift = 0.1cm] (y1) ;
	\coordinate[below right = of v,yshift = -0.1cm] (y2) ;
	\coordinate[below left = of v4,yshift = -0.1cm, xshift=-0.1cm] (y3) ;
	\coordinate[above left = of v4,yshift = 0.1cm, xshift=-0.1cm] (y4) ;
	\draw[rounded corners=10, double,thick, draw opacity = .30, fill opacity = .20, text opacity = 1] (y1) --node[xshift = 0.25cm, yshift= 0.4cm)] {$\ispacen_{\FOA}$} (y2) -- (y3) -- (y4) -- cycle;
	\end{tikzpicture}
\end{center}
\end{example}

As one might expect, given a set of pairwise compatible construction spaces, $\{\cspacen_1, \ldots, \cspacen_i,\ldots\}$, the union of any pair of the spaces, $\cspacen_i\cup \cspacen_j$, is compatible with the remaining spaces, captured in lemma~\ref{lem:pairwiseCompatibleCSformCS}; a full proof can be found in appendix~\ref{sec:app:representationalSystems} (see lemma~\ref{alem:pairwiseCompatibleCSformCS}).

\begin{lemma}\label{lem:pairwiseCompatibleCSformCS}
Let $\cspacen_{1}$, $\cspacen_2$ and $\cspacen_3$ be pairwise compatible construction spaces. Then $\cspacen_1\cup \cspacen_2$ is com\-pa\-ti\-ble with $\cspacen_3$.
\end{lemma}

\begin{proof}[Proof Sketch]
Assume that $\cspacen_i=\cspacep{i}$ is a construction space where $\tsystemn_i=\tsystemp{i}$, $\cspecificationn_i=\cspecificationp{i}$, and $\graphn_i=\graphp{i}$, for $1\leq i \leq 3$. The proof has three main parts:
\begin{enumerate}
\item[(a)] show that $\tsystemn=(\tsystemn_1\cup \tsystemn_2)\cup \tsystemn_3$ is a type system,
\item[(b)] show that $\cspecificationn=(\cspecificationn_1\cup \cspecificationn_2)\cup \cspecificationn_3$ is a constructor specification over $\tsystemn$, and
\item[(c)] show that $\graphn=(\graphn_1\cup \graphn_2)\cup \graphn_3$ is a structure graph for $\cspecificationn$.
\end{enumerate}
The first two parts are particularly straightforward to prove. Part (a) is established by lemma~\ref{lem:pairwiseCompatibleTTSs}. For part (b), the proof uses the pairwise compatibility of the construction spaces to show that $\constructors$ is disjoint from $\types$. To establish that $\spec_1\cup \spec_2\cup \spec_3$ is a function, $\spec$, the proof uses the fact that the union of any pair of $\spec_1$, $\spec_2$, and $\spec_3$ is a function, by pairwise compatibility. These two steps are sufficient to show that (b) holds. For (c), it is straightforward to show that $\graphn$, the union of $\graphn_1$, $\graphn_2$ and $\graphn_3$, is a directed labelled bipartite graph. To show that $\graphn$ is a structure graph, the task is to show that for each configurator, $u$, in $\graphn$, the neighbourhood of $u$ is identical in each of the graphs, $\graphn_1$, $\graphn_2$ and $\graphn_3$, in which $u$ appears.
\end{proof}

\subsubsection{Functionality}\label{sec:determinism}

Construction spaces are functional when, given any constructor, $c$, if the inputs are known then the output is uniquely determined. In fact, the property of being functional arises in two distinct ways relating to tokens and types.

\begin{definition}\label{defn:deterministic}
Let $\cspacen=\cspace$ be a construction space containing a constructor, $c$. Then $c$ is:
\begin{enumerate}
\item \textit{token-functional}  provided that for any pair of $c$-configurators, $u$ and $u'$, in $\graphn$, if $\inputsto{u}=\inputsto{u'}$ then $\outputsto{u}=\outputsto{u'}$.
\item \textit{type-functional} provided that for any pair of $c$-configurators, $u$ and $u'$, in $\graphn$,  if $\inputsty{u}=\inputsty{u'}$ then $\outputsty{u}=\outputsty{u'}$.
\end{enumerate}
The construction space $\cspacen$ is \textit{token-functional} (resp. \textit{type-functional}) provided that each constructor in $\cspacen$ is \textit{token-functional} (resp. \textit{type-functional}). In addition, $\cspacen$ is \textit{functional} if it is both token- and type-functional.
\end{definition}

\begin{example}[Functional Constructors in a Grammatical Space]
The constructors specified in example~\ref{ex:grammaticalConstructors} for the grammatical space, $\gspacen_\FOA$, of $\rsystemn_{\FOA}$ are type-functional. Suppose that we have an $\FOAcinfixop$-configurator, $u$, in $\gspacen_{\FOA}$. Further suppose that $u$ has input token-sequence $[5,-,21]$. Then the input type-sequence of $u$ is $[\texttt{5},\FOAminus,\texttt{21}]$ and the uniquely constructed output token, $5-21$, has type $\texttt{5}\_\FOAminus\_\texttt{21}$. For type-functional, given any pair of $\FOAcinfixop$-configurators, $u'$ and $u''$, if $\inputsty{u'}=\inputsty{u''}=[\tau_1,\tau_2,\tau_3]$ then the output type of both $u'$ and $u''$ is $\tau_1\_\tau_2\_\tau_3$. Hence, $\FOAcinfixop$ is type-functional. As well as $\gspacen_{\FOA}$ being type-functional, we also require it to be token-functional.
\end{example}

\begin{example}[Non-Functional Constructors in an Entailment Space]
A constructor that forms part of $\rsystemn_{\FOA}$'s entailment layer, $\espacen_{\FOA}$, is $\FOAexplosion$. This constructor is used for inferring any formula from an unsatisfiable formula (principle of explosion). The following graphs visualise the signature of $\FOAexplosion$ and two configurations with the same input token-sequence:
\begin{center}
\begin{minipage}[t]{0.26\textwidth}\centering
\textit{signature of} $\FOAexplosion$:\\[0.5ex]
\begin{tikzpicture}[construction]
\node[typeS={$\FOAform$}] (v) at (4.6,3.2) {};
\node[constructorEN={$\FOAexplosion$}] (u) at (3.8,3.2) {};
\node[typeS={$\FOAform$}] (v2) at (2.2,3.2) {};
\path[->]
(u) edge[bend right = 0] (v)
(v2) edge[bend right = 0] node[index label] {1} (u);
\end{tikzpicture}
\end{minipage}
\hspace{0.5cm}
\begin{minipage}[t]{0.31\textwidth}\centering
\textit{configuration of} $\FOAexplosion$:\\[0.5ex]
\begin{tikzpicture}[construction]
\node[termrep] (v) at (4.8,3.2) {\scriptsize$x>x$};
\node[constructorEN={$\FOAexplosion$}] (u) at (3.8,3.2) {};
\node[termrep] (v2) at (1.8,3.2) {\scriptsize$\forall x \thinspace x-1> x$};
\path[->]
(u) edge[bend right = 0] (v)
(v2) edge[bend right = 0] node[index label] {1} (u);
\end{tikzpicture}
\end{minipage}
\hspace{0.5cm}
\begin{minipage}[t]{0.31\textwidth}\centering
\textit{configuration of} $\FOAexplosion$:\\[0.5ex]
\begin{tikzpicture}[construction]
\node[termrep] (v) at (4.8,3.2) {\scriptsize$7=8$};
\node[constructorEN={$\FOAexplosion$}] (u) at (3.8,3.2) {};
\node[termrep] (v2) at (1.8,3.2) {\scriptsize$\forall x \thinspace x-1> x$};
\path[->]
(u) edge[bend right = 0] (v)
(v2) edge[bend right = 0] node[index label] {1} (u);
\end{tikzpicture}
\end{minipage}
\end{center}
The input token-sequence $[\forall x \thinspace x-1> x]$ syntactically entails multiple tokens via the $\FOAexplosion$ constructor: $\espacen_{\FOA}$ is non-functional.
\end{example}

\begin{samepage}
The structure graphs below exemplify that token-functional does not imply type-functional (left) and that type-functional does not imply token-functional (right); the vertex labels illustrate the types assigned to the tokens and the constructors assigned to the configurators.
\begin{center}
\begin{minipage}[t]{0.35\textwidth}\centering
% for rep not token
\textit{token-functional only:}\\[1ex]
\begin{tikzpicture}[construction]
\node[termW={$\tau$}] (v) at (3.2,4.1) {\scriptsize$t$};
\node[termW={$\tau_1$}] (v1) at (2.2,2.2) {\scriptsize$t_1$};
\node[termE={$\tau_1$}] (v3) at (6.2,2.2) {\scriptsize$t_3$};
\node[constructorNW={$c$}] (u) at (3.2,3.2) {\scriptsize$u_1$};
\node[constructorNE={$c$}] (u') at (5.2,3.2) {\scriptsize$u_2$};
\node[termE={$\tau'$}] (v') at (5.2,4.1) {\scriptsize$t'$};
\node[termS={$\tau_2$}] (v2) at (4.2,2.2) {\scriptsize$t_2$};
\path[->]
(u) edge[bend right = 0] (v)
(u') edge[bend right = 0] (v')
(v1) edge[bend right = -10] node[index label] {1} (u)
(v2) edge[bend right = 10] node[index label] {2} (u)
(v2) edge[bend right = -10] node[index label] {2} (u')
(v3) edge[bend right = 10] node[index label] {1} (u');
\end{tikzpicture}
\end{minipage}
\hspace{2cm}
\begin{minipage}[t]{0.35\textwidth}\centering
% for token not rep
\textit{type-functional only:}\\[1ex]
\begin{tikzpicture}[construction]
\node[termW={$\tau$}] (v) at (3.2,4.1) {\scriptsize$t$};
\node[termW={$\tau_1$}] (v1) at (3.2,2.2) {\scriptsize$t_1$};
\node[termE={$\tau_2$}] (v2) at (5.2,2.2) {\scriptsize$t_2$};
\node[constructorNW={$c'$}] (u) at (3.2,3.2) {\scriptsize$u_3$};
\node[constructorNE={$c'$}] (u') at (5.2,3.2) {\scriptsize$u_4$};
\node[termE={$\tau$}] (v') at (5.2,4.1) {\scriptsize$t'$};
\node[] () at  (5.2,1.72) {};%for space
\path[->]
(u) edge[bend right = 0] (v)
(v1) edge[bend right = 0 ] node[index label] {1} (u)
(v2) edge[bend right = 15] node[index label] {2} (u)
(v1) edge[bend right = 15] node[index label] {1} (u')
(v2) edge[bend right = 0] node[index label] {2} (u')
(u') edge[bend right = 0] (v');
\end{tikzpicture}
\end{minipage}
\end{center}
\end{samepage}
On the left, $\inputsto{u_1} =[t_1,t_2]\neq [t_3,t_2] = \inputsto{u_2}$, so $c$ is token-functional; the antecedent in condition (1) of definition~\ref{defn:deterministic} is false. However,
$\inputsty{u_1} = [\tau_1,\tau_2]= \inputsty{u_2}$ and $\outputsty{u_1} = \tau \neq \tau' =\outputsty{u_1}$,
 so $c$ is not type-functional. Similarly, on the right, $\inputsto{u_3} = \inputsto{u_4}$ and $\outputsto{u_3} \neq \outputsto{u_4}$, so $c'$ is not token-functional. Yet, $c'$ is type-functional since $\outputsty{u_3} = \outputsty{u_4}$.

\subsubsection{Totality}\label{sec:totality}

As indicated above, the property of totality requires that for any valid input sequence, $\T$, to a constructor, $c$, there exists a $c$-configuration whose input sequence `conforms' to $\T$. As with functional constructors, $\T$ could be a sequence of terms or a sequence of types.

\begin{example}[Totality in a Grammatical Space]
The grammatical space, $\gspacen_\FOA$, of $\rsystemn_{\FOA}$ has a particularly strong property in that its structure graph is \textit{type-total}: for each constructor, $c$, given any type-sequence, $\T$, that specialises $\inputsty{c}$, there is a $c$-configurator whose input token-sequence \textit{instantiates} $\T$. For instance, consider $\FOAcinfixop$'s input type-sequence, $\inputsty{\FOAcinfixop}=[\FOAnumexp,\FOAbop,\FOAnumexp]$. For any type-sequence, $\T = [\tau_1,\tau_2,\tau_3]$, that specialises $\inputsty{\FOAcinfixop}$ \textit{there exists some $\FOAcinfixop$-configurator} with an input token-sequence that \textit{instantiates} $\T$. For example, the type-sequence
\begin{displaymath}
[\FOAopenb\_\texttt{1}\_\FOAminus\_ \texttt{91}\_\FOAcloseb,\FOAplus,\FOAopenb\_\texttt{6}\_\FOAminus\_\texttt{12}\_\FOAcloseb]
\end{displaymath}
specialises $\inputsty{\FOAcinfixop}$ and $\gspacen_{\FOA}$ contains an $\FOAcinfixop$-configurator  with input token-sequence
\begin{displaymath}
[(1-91),+,(6-12)]
\end{displaymath} that constructs the token $(1-91)+(6-12)$. We call this kind of totality \textit{type-total}.
\end{example}

\begin{definition}\label{defn:instantiated}
Let $\cspacen=\cspace$ be a construction space where $\graphn=\graph$. A type, $\tau$, in $\cspacen$ is \textit{instantiated} provided there exists a token, $t$, in $\pa$ such that $\tokenl(t)\leq \tau$. The token $t$ \textit{instantiates} $\tau$. If $\tokenl(t)=\tau$ then the instantiation is \textit{direct}.  A sequence, $\T=[\tau_1,\ldots,\tau_n]$, over $\types$ is \textit{instantiated} provided each $\tau_i$ is instantiated.  A sequence of tokens, $[t_1,\ldots,t_n]$, (\textit{directly}) \textit{instantiates} $[\tau_1,\ldots,\tau_n]$ provided each $t_i$ (directly) instantiates $\tau_i$.
\end{definition}

\begin{definition}\label{defn:total}
Let $\cspacen=\cspace$ be a construction space containing a constructor, $c$, with $\inputsty{c}=[\tau_1,\ldots,\tau_n]$. Then $c$ is:
\begin{enumerate}
\item \textit{token-total} provided, for any instantiation, $[t_1,\ldots,t_n]$, of $\inputsty{c}$ there exists a $c$-configurator, $u$, in $\graphn$  where $\inputsto{u}=[t_1,\ldots,t_n]$.
\item \textit{type-total} provided, for any specialisation, $[\tau_1',\ldots,\tau_n']$, of $\inputsty{c}$, there exists a $c$-configurator, $u$, in $\graphn$ where $\inputsto{u}$ instantiates $[\tau_1',\ldots,\tau_n']$.
\end{enumerate}
The construction space $\cspacen$ is \textit{token-total} (resp. \textit{type-total}) provided each constructor in $\cspacen$ is token-total (resp. type-total). In addition, $\cspacen$ is \textit{total} if it is both token- and type-total.
\end{definition}

The property of totality will be useful in Section~\ref{sec:structuralTransformations}: roughly speaking, we will seek to assign an input token-sequence to a constructor (strictly, a \textit{pattern}) that instantiates its input-type sequence.

\begin{example}[Non-Totality in an Entailment Space]
A constructor that forms part of $\rsystemn_{\FOA}$'s entailment layer, $\espacen_{\FOA}$, is $\FOAceqtran$, for inferring  equality relations using transitivity. The signature of $\FOAceqtran$ is shown below (left), alongside a configuration (middle) and a non-valid configuration (right):
\begin{center}
\begin{minipage}[t]{0.24\textwidth}\centering
\textit{signature of} $\FOAceqtran$:\\[0.5ex]
\begin{tikzpicture}[construction,yscale=0.9]
\node[typeE={$\FOAnumexp$}] (v) at (3.2,4.1) {};
\node[constructorENW={$\FOAceqtran$}] (u) at (3.2,3.2) {};
\node[typeS={$\FOAnumexp$}] (v2) at (2.6,2.4) {};
\node[typeS={$\FOAnumexp$}] (v4) at (3.8,2.4) {};
\path[->]
(u) edge[bend right = 0] (v)
(v2) edge[bend right = -10] node[index label] {1} (u)
(v4) edge[bend right = 10] node[index label] {2} (u);
\end{tikzpicture}
\end{minipage}
\hspace{0.5cm}
\begin{minipage}[t]{0.27\textwidth}\centering
\textit{configuration of} $\FOAceqtran$:\\[0.5ex]
\begin{tikzpicture}[construction,yscale=0.9]\small
\node[termrep] (v) at (3.2,4.1) {\scriptsize$1+1=5-3$};
\node[constructorENW={$\FOAceqtran$}] (u) at (3.2,3.2) {};
\node[termrep] (v2) at (2.4,2.2) {\scriptsize$1+1=2$};
\node[termrep] (v4) at (4,2.2) {\scriptsize$2=5-3$};
\path[->]
(u) edge[bend right = 0] (v)
(v2) edge[bend right = -10] node[index label] {1} (u)
(v4) edge[bend right = 10] node[index label] {2} (u);
\end{tikzpicture}
\end{minipage}
\hspace{0.5cm}
\begin{minipage}[t]{0.38\textwidth}\centering
\textit{non-valid configuration of} $\FOAceqtran$:\\[0.5ex]
\begin{tikzpicture}[construction,yscale=0.9]\small
\node[] (v) at (3.2,4.1) {\scriptsize{none}};
\node[constructorENW={$\FOAceqtran$}] (u) at (3.2,3.2) {};
\node[termrep] (v2) at (2.6,2.2) {\scriptsize$2>1$};
\node[termrep] (v4) at (3.8,2.2) {\scriptsize$1=x$};
\path[->]
(u) edge[bend right = 0] (v)
(v2) edge[bend right = -10] node[index label] {1} (u)
(v4) edge[bend right = 10] node[index label] {2} (u);
\end{tikzpicture}
\end{minipage}
\end{center}
In the non-valid configuration (no subgraph of $\espacen_{\FOA}$'s structure graph has that shape), the transitivity of equality cannot be used to construct an entailed token because $2>1$ does not use the $=$ relation. Thus $\espacen_{\FOA}$ is not type-total: $\T=[\texttt{2}\_\FOAgr\_\texttt{1}, \texttt{1}\_\FOAeq\_\texttt{x}]$ is a specialisation of $\inputsty{\FOAceqtran}=[\FOAnumexp,\FOAnumexp]$, which can be directly instantiated by $\T'=[2>1, 1=x]$, yet there is no $\FOAceqtran$-configuration with $\T'$ as its input token-sequence.
\end{example}

\subsubsection{Type Extendability}\label{sec:extendability}

Section~\ref{sec:interRepSystems} demonstrates how to encode semantics and other properties of representational systems. For this purpose, we may need access to new \textit{inter-system} constructors and new \textit{intra-system} types. Suppose that we enlarge the type system, $\tsystemn=\tsystem$, over which a construction space, $\cspacen$, is formed, by simply adding a fresh type (i.e. a type that is not in $\types$) and updating $\leq$ so that it remains a partial order, yielding $\tsystemn'$. We can form a new construction space from $\cspacen=\cspace$ by substituting $\tsystemn$ with $\tsystemn'$, yielding $\cspacen'=(\tsystemn', \cspecificationn, \graphn)$. Our focus is on enlarging the set of types to provide \textit{minimal upper bounds} over the subtype ordering. We do this in such a way that for a set of types, $\types'\subseteq \types$, there exists a minimal upper bound, $\tau$, where the set $\pa'=\{t\in \pa\colon \exists \tau'\in \types' \thinspace \tokenl(t)\leq \tau'\}$ can be re-expressed using $\tau$: $\pa'=\{t\in \pa \colon \tokenl(t)\leq \tau\}$. When defining the semantics of an entire representational system, we require access to such an upper bound of $\types$ itself\footnote{We note that this is primarily a matter of convenience and elegance: including a type, $\tau$, that is an upper bound of $\types$ enables us to exploit a \textit{unique} constructor, called \constructorRepresents\ in the examples in Section~\ref{sec:interRepSystems}, when defining semantics. Else, we would need to use one constructor, $\constructorRepresents_{\texttt{i}}$, for each subset, $T_i$, of $\types$ with an upper bound. This would lead to inelegant encodings and could be suggestive of different roles being played by the $\constructorRepresents_{\texttt{i}}$ constructors. In addition, this larger set of constructors would, in turn, increase the number of patterns needed to describe our representational systems (Section~\ref{sec:patterns}) and the number of relationships that need to be defined between patterns in the context of structural transformations (Section~\ref{sec:structuralTransformations}).}. The auxiliary definition~\ref{defn:subtypeClosure} leads to definition~\ref{defn:intertionofLUB}, which captures extensions with minimal upper bounds, and lemma~\ref{lem:upperBoundExtensionExists} establishes that such extensions exist.

\begin{definition}\label{defn:subtypeClosure}
Let $\tsystemn=\tsystem$ be a type system, and let $\types'$ be a set of types such that $\types' \subseteq \types$. The \textit{subtype closure of $\types'$ in $\types$}, denoted $\subtypeClose(\types', \tsystemn)$, is the set $\{\tau \in \types : \exists \tau' \in \types' \thinspace \tau \leq \tau'\}$. The set of all subtype closures formed over subsets of $\types$ given $\tsystemn$ is denoted $\subtypeSets(\tsystemn)$.
\end{definition}

\begin{definition}\label{defn:intertionofLUB}
An \textit{upper bound extension} of $\tsystemn=\tsystem$
is a type system $\tsystemn'=(\types\cup \types',\leq')$, where $\types'$ is a set of fresh types such that
\begin{enumerate}
\item there exists a bijection, $f\colon \subtypeSets(\tsystemn)\to \types'$ and
\item the partial order $\leq'$ is the reflexive closure of
\begin{displaymath}
\leq \cup \bigcup\limits_{\types''\in \mathit{SC}(\tsystemn)} \{(\tau,f(\types'')) : \tau\in \types''\}.
\end{displaymath}
\end{enumerate}
Given a set, $\types''\in \subtypeSets(\tsystemn)$, the type $f(\types'')$ is called a \textit{defining} upper bound for $\types''$.
\end{definition}

There may be many defining upper bounds for $\types''$, so $\tau''$ need not be a \textit{least} upper bound.

\begin{lemma}\label{lem:upperBoundExtensionExists}
Every type system, $\tsystemn=\tsystem$, has an upper bound extension.
\end{lemma}

The full proof of lemma~\ref{lem:upperBoundExtensionExists} can be found in appendix~\ref{sec:app:representationalSystems} (see lemma~\ref{alem:leastUpperBoundExtensionExists}). Lemma~\ref{lem:enlargeTypes}, whose proof is trivial and thus omitted, establishes that we can replace the set of types with an upper bound extension.

\begin{lemma}\label{lem:enlargeTypes}
Let $\cspacen=\cspace$ be a construction space. Let $\tsystemn'$ be an upper bound extension of $\tsystemn$. Then $(\tsystemn', \cspecificationn, \graphn)$ is a construction space.
\end{lemma}

As a consequence of lemma~\ref{lem:enlargeTypes}, we can assume (without loss of generality), that if we need a type that is a defining upper bound of a subset, $\types'$, of the set of types, $\types$, then it is available within the given type system.

\subsection{Formalizing Representational Systems}\label{sec:formalizingRSs}

In a representational system, the grammatical, entailment, and identification spaces are required to be compatible. With regards to the compatibility of their respective type systems, this is assured  for $\gspacen$ and $\espacen$: they will be formed over the same system, $\tsystemn=\tsystem$. However, $\ispacen$ augments $\tsystemn$ with meta-types, used to identify meta-level properties of tokens. We introduce the concepts of an \textit{identification space} and a \textit{meta-type system} in definition~\ref{defn:identificationSpace}.

\begin{definition}\label{defn:identificationSpace}
Given a type system, $\tsystemn=\tsystem$, an \textit{identification space} is a construction space, $\ispacen=\ispace$, such that $\tsystemn=\tsystem$ and $\mpsystemn=\mpsystem$ are compatible type systems. The type system $\mpsystemn$ is the \textit{meta-type system} of $\ispacen$. We say that $\ispacen$ is \textit{formed over} $\tsystemn$ and $\mpsystemn$.
\end{definition}

\begin{example}[Encoding Properties]\label{ex:rs:encodingProperties}
In $\rsystemn_{\FOA}$, a meta-type system, $\mpsystemn$, may include the type $\texttt{nat}$ and an identification space may include all natural numbers, of type $\FOAnat$, amongst its tokens. Given an identification space constructor $\FOAccount$, with signature $\spec(\FOAccount)=([\FOAform],\FOAnat)$, an identification space can encode how many primitive tokens appear in each formula. The graph below indicates that $\forall x\thinspace x> x-1$ is formed of seven primitives:
\begin{center}
\begin{tikzpicture}[construction]
\begin{scope}[rotate = -90]
\node[termrep] (v4) at (3.2,3.8) {\scriptsize$\forall x \thinspace x> x-1$};
\node[constructorIN={\FOAccount}] (u) at (3.2,5.8000000000000005) {};
\node[termIrep] (v) at (3.2,6.8) {\scriptsize$7$};
\end{scope}
\path[->]
(u) edge[bend right = 0] (v)
(v4) edge[bend right = 0] node[index label] {1} (u);
\end{tikzpicture}
\end{center}
\end{example}

Our formalisation of a representational system will require that tokens (not meta-tokens) used in the entailment space, $\espacen$, and identification space, $\ispacen$, are constructed in the grammatical space, $\gspacen$; there is a subtle distinction in how this is prescribed in the two spaces, as we shall see in definition~\ref{defn:representationalSystem}. Intuitively, $\espacen$ should only encode entailments using tokens that are constructed in $\gspacen$. Similarly, the identification space, $\ispacen$, should only encode properties of tokens that are constructed in $\gspacen$. Example~\ref{ex:relationshopBetweenCSpaces} presents an identification space for a representational system that naturally shares the symbol $\bot$ with the underlying grammatical space. In this case, we do not wish all instances of $\bot$ to occur in the grammatical space: when $\bot$ is used as meta-token, its occurrence in $\gspacen$ is not required.

\begin{example}[Relationships Between Construction Spaces]\label{ex:relationshopBetweenCSpaces}
In \textsc{Description Logic}~\cite{baader:tdlh}, there are tokens of the form $\bot$, used to denote an empty class, which is also a useful symbol to include in an identification space, $\ispacen_{\mathit{DL}}$. These graphs show one way of \textit{building} the token $C\sqcap D\sqsubseteq \bot$ in the grammatical space, $\gspacen_{\mathit{DL}}$, and \textit{identify} that the expression $C\sqcap D$ is not coherent, given $C\sqcap D\sqsubseteq \bot$, exploiting $\ispacen_{\mathit{DL}}$:
\begin{center}
	\begin{tikzpicture}[construction,node distance = 0.3cm and 0.1cm]
	\begin{scope}[rotate = -90]
	\node[termrep] (v4) at (3.2,4.2) {\scriptsize$C\sqcap D\sqsubseteq \bot$};
	\node[termrep] (v1) at (2.2,3.2) {\scriptsize$C\sqcap D$};
	\node[constructorGpos={$\texttt{subsumes}$}{-60}{0.15cm}] (u1) at (3.2,2.4) {};
	\node[termrep] (v3) at (3.5,1.4) {\scriptsize$\bot$};
	\node[constructorIpos={\texttt{isCoherent}}{45}{0.2cm}] (u) at (3.2,6.8) {};
	\node[termIrep] (v) at (3.2,8.4) {\scriptsize$\bot$};
	\node[termrep] (v2) at (2.7,1.4) {\scriptsize$\sqsubseteq$};
	\end{scope}
	\path[->]
	(u1) edge[bend right = 0] (v4)
	(v1) edge[bend right = 10] node[index label] {1} (u1)
	(v3) edge[bend right = 10] node[index label] {3} (u1)
	(u) edge[bend right = 0] (v)
	(v4) edge[bend right = 0] node[index label] {1} (u)
	(v1) edge[bend right = -10] node[index label] {2} (u)
	(v2) edge[bend right = -10] node[index label] {2} (u1);
	\coordinate[above right = of v4, yshift = -0.1cm,xshift=0.1cm] (x1) ;
	\coordinate[below right = of v4, yshift = 0.1cm,xshift=0.1cm] (x2) ;
	\coordinate[below left = of v3,yshift = 0.1cm,xshift=-0.1cm] (x3) ;
	\coordinate[above left = of v2,yshift = -0.2cm,xshift=-0.1cm] (x4) ;
	\coordinate[above left = of v1,yshift = -0.1cm,xshift=0.0cm] (x5) ;
	\coordinate[above right = of v1,yshift = -0.2cm,xshift=0.0cm] (x6) ;
	\draw[rounded corners=10,thick, draw opacity = .25, fill opacity = .15, text opacity = 1] (x1) -- (x2) -- (x3) -- (x4) -- node[xshift = -0.3cm, yshift= 0.2cm)] {$\gspacen_\mathit{DL}$} (x5) -- (x6) -- cycle;
	\coordinate[above right = of v,yshift = 0.1cm,xshift=0.2cm] (y1) ;
	\coordinate[below right = of v,yshift = -0.0cm,xshift=0.2cm] (y2) ;
	\coordinate[below left = of v4,yshift = -0.0cm,xshift=0.1cm] (y3) ;
	\coordinate[left = of v1,yshift = 0.0cm,xshift=-0.2cm] (y4) ;
	\coordinate[above  = of v1,yshift = 0.0cm,xshift=-0.2cm] (y5) ;
	\draw[rounded corners=10, double,thick, draw opacity = .25, fill opacity = .15, text opacity = 1] (y1) -- (y2) -- (y3) -- (y4) -- (y5) --node[xshift = 0.8cm, yshift= 0.15cm)] {$\ispacen_{\mathit{DL}}$} cycle;
	\end{tikzpicture}
\end{center}
It is necessary that the structure graphs for $\gspacen_{\mathit{DL}}$ and $\ispacen_{\mathit{DL}}$ share the tokens $C\sqcap D$ and $C\sqcap D\sqsubseteq \bot$. However, the symbol $\bot$ is represented as one token in $\gspacen_{\mathit{DL}}$ and a distinct token in $\ispacen_{\mathit{DL}}$. Here, the instance of $\bot$ in $\ispacen_{\mathit{DL}}$ is a meta-token since it does not appear in $\gspacen_{\mathit{DL}}$. It is \textit{precisely the meta-tokens} in $\ispacen_{\mathit{DL}}$ that do not occur in $\gspacen_{\mathit{DL}}$.
\end{example}

We are now in a position to define a representational system which can be used to abstractly encode the grammatical rules to which representations must adhere, entailment relations between representations, and meta-level properties of representations.

{\pagebreak}
\begin{definition}\label{defn:representationalSystem}
A \textit{representational system} is a triple, $\rsystemn= \rsystem$, \textit{formed over} type system, $\tsystemn$, and meta-type system, $\mpsystemn$, \begin{samepage}where:
	\begin{enumerate}
		\item $\gspacen=(\tsystemn,\cspecificationn_{\gspacen},\graphn_{\gspacen})$ is a functional construction space,
		\item $\espacen=(\tsystemn, \cspecificationn_{\espacen},\graphn_{\espacen})$ is a construction space such that every token in $\espacen$ is also a token in $\gspacen$,
		\item $\ispacen=(\tsystemn\cup \mpsystemn, \cspecificationn_{\ispacen},\graphn_{\ispacen})$ is an identification space formed over $\tsystemn$ and $\mpsystemn$ %given $\pb_{\mpsystemn}$
such that for every token, $t$, in $\ispacen$:
                \begin{enumerate}
                \item if $t$ is not in $\gspacen$ then the label of $t$ is a meta-type in $\mpsystemn$, and
                \item if $t$ is the output of a configurator in $\graphn_{\ispacen}$ then $t$ is not in $\gspacen$, and
                \end{enumerate}
        \item $\gspacen$, $\espacen$ and $\ispacen$ are pairwise compatible and their sets of constructors are pairwise disjoint.
	\end{enumerate}\end{samepage}
The spaces $\gspacen$, $\espacen$ and $\ispacen$ are called the \textit{grammatical}, \textit{entailment} and \textit{identification} spaces of $\rsystemn$. We denote the components of the structure graphs $\graphn_\gspacen$, $\graphn_\espacen$, and $\graphn_\ispacen$ by
\begin{enumerate}
\item $\graphn_{\gspacen}=\graphp{\gspacen}$,
\item $\graphn_{\espacen}=\graphp{\espacen}$, and
\item $\graphn_{\ispacen}=\graphp{\ispacen}$, respectively.
\end{enumerate}
The \textit{meta-tokens} in $\rsystemn$ are the elements of $\pa_{\ispacen}\backslash \pa_{\gspacen}$.
We say that $\tsystemn$ and $\mpsystemn$ are, respectively, the type system and meta-type system of $\rsystemn$.
\end{definition}

We now establish that the union of the grammatical, entailment and identification spaces is itself a construction space, with theorem~\ref{thm:superSpace} following immediately from lemma~\ref{lem:pairwiseCompatibleCSformCS}.

\begin{theorem}\label{thm:superSpace}
Let  $\rsystemn= \rsystem$, be a representational system. Then $\gspacen\cup \espacen\cup \ispacen$ is a construction space.
\end{theorem}

\begin{definition}\label{defn:universalSpace}
Let  $\rsystemn= \rsystem$, be a representational system. The \textit{universal space} of $\rsystemn$, denoted $\uspace{\rsystemn}$, is defined to be the construction space $\gspacen\cup \espacen\cup \ispacen$. The structure graph of $\uspace{\rsystemn}$, which we denote by $\ugraph{\rsystemn}$, is called the \textit{universal structure graph} for $\rsystemn$.
\end{definition}

Since theorem~\ref{thm:superSpace} tells us that we can form the union the spaces of which a representational system comprises to form another construction space, we specifically know that $\ugraph{\rsystemn}=\graphn_{\gspacen}\cup \graphn_{\espacen}\cup \graphn_{\ispacen}$ is a structure graph. Moreover, since the sets of constructors of $\gspacen$, $\espacen$ and $\ispacen$ are pairwise disjoint, given any configurator, $u$, in $\ugraph{\rsystemn}$, we know from which of the three individual spaces $u$, and its neighbourhood, $\neigh{u}$, arises. These insights allows us to exploit the universal structure graph, $\ugraph{\rsystemn}$, when we seek to describe representational systems (Section~\ref{sec:patterns}) and for the purposes of structural transformations (Section~\ref{sec:structuralTransformations}).

\subsection{Inter-Representational-System Relationships}\label{sec:interRepSystems}

We began Section~\ref{sec:RSs} by observing that a representational system, $\rsystemn$, comprises syntactic entities, called tokens, that have the capacity to represent objects when assigned semantics. To define semantics for $\rsystemn$, we need access to the objects that are being represented. The property of a token, $t$, representing an object, $o$, can be encoded (at a meta-level) by a (binary) relationship holding between $t$ and $o$. However, it is also possible to encode semantics using construction spaces by noting that the represented objects are themselves simply tokens of another representational system, $\rsystemn'$. To define the semantics of $\rsystemn$ in the context of $\rsystemn'$, we require that $\rsystemn$ and $\rsystemn'$ are \textit{compatible}\footnote{This is not a material restriction: one can always rename the types, or other components, that cause incompatibility by taking the disjoint union of $\rsystemn$ and $\rsystemn'$.}.

\begin{definition}\label{defn:compatibleRSs}
Let $\rsystemn=\rsystem$ and $\rsystemn'=\rsystemd$  be representational systems. Then $\rsystemn$ and $\rsystemn'$ are \textit{compatible} provided $\rsystemn\cup \rsystemn'$ is a representational system.
\end{definition}

Beyond simply encoding semantics, inter-representational-system encodings play an important role in structural transformations (Section~\ref{sec:structuralTransformations}), supporting \textit{representation selection.} There, we may have a graph, say $\cgraphn$, that shows one way of building a token, $t$, in $\rsystemn$. Such a pair, $\cpair$, is called a \textit{construction} and will be formally defined in Section~\ref{sec:constructions}. One goal of finding a structural transformation may be to identify another construction, $\cpaird$, in $\rsystemn'$ such that $t$ and $t'$ are in some specified relationship. In Section~\ref{sec:structuralTransformations}, such relationships are identified using an \textit{inter-representational-system encoding}. Example~\ref{ex:encodingSemantics} illustrates how semantics can be encoded.

\begin{example}[Encoding Semantics]\label{ex:encodingSemantics}
Example~\ref{ex:vectorMulti} demonstrated how a graphical representational system, $\rsystemn_{\mathit{GV}}$, could be used to visualise vectors in the plane; such a vector space can be encoded by a representational system, $\rsystemn_{\mathit{VS}}$. The two graphical tokens below (left, $t_1$, and right, $t_2$) \textit{represent} the vectors $(1,3)$ and $(0.5,2)$, respectively (the vertex labels $t_1$ and $t_2$ are naming tokens and are not intended to be types):
\begin{center}
\begin{tikzpicture}[construction]
\node[termIrep] (v) at (4.8,3.5) {\footnotesize$\top$};
\node[termN={$t_1$}] (v1) at (1.1,2.4) {\scalebox{0.8}{\vectorVis{(1,3)}{0.5}}};
\node[termrep] (v2) at (3.8,2.4) {\footnotesize$(1,3)$};
\node[constructorIpos={\constructorRepresents}{110}{0.17cm}] (u1) at (3.3,3.4) {\footnotesize$u_1$};
\node[constructorIpos={\constructorRepresents}{-110}{0.17cm}] (u3) at (3.3,1.5) {\footnotesize$u_3$};
\node[termrep] (v3) at (5.9,2.4) {\footnotesize$(0.5,2)$};
\node[termN={$t_2$}] (v4) at (8.5,2.4) {\scalebox{0.8}{\vectorVis{(0.5,2)}{0.5}}};
\node[constructorIpos={\constructorRepresents}{-70}{0.17cm}] (u4) at (6.3,1.5) {\footnotesize$u_4$};
\node[constructorIpos={\constructorRepresents}{70}{0.17cm}] (u2) at (6.3,3.4) {\footnotesize$u_2$};
\node[termIrep] (v') at (4.8,1.4) {\footnotesize$\bot$};
\path[->]
(u1) edge[bend right = 0] (v)
(v1) edge[bend right = -10] node[index label] {1} (u1)
(v2) edge[bend right = -10] node[index label] {2} (u1)
(v3) edge[bend left = -10] node[index label] {2} (u2)
(v4) edge[bend right = 10] node[index label] {1} (u2)
(u2) edge[bend right = 0] (v)
(v3) edge[out = 205, in = 55,looseness=0.9] node[index label] {2} (u3)
(v1) edge[bend right = 10] node[index label] {1} (u3)
(v2) edge[out = 315, in = 155,looseness=1.3] node[index label] {2} (u4)
(v4) edge[bend right = -10] node[index label] {1} (u4)
(u3) edge[bend right = 0] (v')
(u4) edge[bend right = 0] (v');
\end{tikzpicture}
%\end{minipage}
\end{center}
This structure graph includes an identification space constructor, \constructorRepresents, with signature $([\tau_1,\tau_2],\texttt{bool})$, where the (fresh) types $\tau_1$ and $\tau_2$ are defining upper bounds of the sets of types in $\rsystemn_{\mathit{GV}}$ and, respectively, $\rsystemn_{\mathit{VS}}$. Each configuration of \constructorRepresents\ encodes whether the first input (a token in the graphical system $\rsystemn_{\mathit{GV}}$) \textit{represents} the second input (a token in the vector space system $\rsystemn_{\mathit{VS}}$). Thus, the structure graph encodes the semantics that $t_1$ represents $(1,3)$, since the configurator $u_1$ outputs a meta-token, $\top$. Likewise, $t_1$ does not represent $(0.5,2)$ since the output of $u_3$ is $\bot$.
\end{example}

The presence of defining upper bounds supports the exploitation of \textit{inter-representational-system constructors} that can encode semantics (as seen in example~\ref{ex:encodingSemantics}) and, moreover, that can also encode other properties that exist between inter-system tokens. One such property relates to the notion of observability~\cite{stapleton:wmaeroiafaoo}: a token, $t$, can be \textit{observed} from another token, $t'$,  if there is a \textit{meaning-carrying relationship} between syntactic entities in $t'$ that corresponds to the informational content of $t$. The ability to observe information from tokens is suggestive of a cognitive advantage in some contexts~\cite{blake:efraoaisv}.

\begin{example}[Encoding Properties]\label{ex:encodingProperties}
Consider a system of \textsc{Euler Diagrams}, $\rsystemn_{\mathit{ED}}$, where any circle drawn in the plane, and any circle label, are primitive tokens. Each Euler diagram is a composite token comprising a finite set of circles, each with an associated label.{\pagebreak} The token \adjustbox{scale=0.61,raise=-0.08cm}{\abcDiagram}\! below, drawn from $\rsystemn_{\mathit{ED}}$, \textit{represents} token $\{A\subseteq B, B\cap C=\emptyset\}$, drawn from a set-theoretic representational system, $\rsystemn_{\mathit{ST}}$. This meaning of \adjustbox{scale=0.61,raise=-0.08cm}{\abcDiagram}\! is encoded below, in the \begin{samepage}middle.
\begin{center}
\begin{minipage}[t]{0.23\textwidth}\centering
	\textit{token:}\\[2ex]
	\begin{tikzpicture}[scale = 2]
	\draw (0.085,-0.035) circle [radius=0.125] node[above left= 0.02cm, yshift=-0.05cm, xshift=-0.13cm] {$A$};
	\draw (0,0) circle [radius=0.3] node[left = 0.2cm, yshift=0.52cm, xshift=-0.15cm] {$B$};
	\draw (0.7,0) circle [radius=0.3] node[right= 0.2cm, yshift=0.52cm, xshift=0.15cm] {$C$};
	\end{tikzpicture}
\end{minipage}
\hspace{1cm}
\begin{minipage}[t]{0.23\textwidth}\centering
	\textit{encoding semantics:}\\[1ex]
	\begin{tikzpicture}[construction]
	\node[termIrep] (v) at (3.2,4.2) {\footnotesize$\top$};
	\node[constructorINW={$\texttt{represents}$}] (u) at (3.2,3.4) {};
	\node[termrep] (v2) at (2.5,2.35) {\scalebox{0.8}{\abcDiagram}};
	\node[termrep] (v4) at (3.8,1.7) {\footnotesize$\{A\subseteq B, B\cap C=\emptyset\}$};
	\path[->]
	(u) edge[bend right = 0] (v)
	(v2) edge[bend right = -5] node[index label] {1} (u)
	(v4) edge[bend right = 10] node[index label] {2} (u);
	\end{tikzpicture}
\end{minipage}
\hspace{0.7cm}
\begin{minipage}[t]{0.39\textwidth}\centering
	\textit{encoding observability:}\\[1ex]
	\begin{tikzpicture}[construction]
	\node[termIrep] (v) at (3.2,4.2) {\footnotesize$\top$};
	\node[constructorIpos={\constructorObserve}{0}{0.58cm}] (u4) at (2.8,3.4) {};
	\node[constructorIpos={\constructorObserve}{143}{0.3cm}] (u3) at (1.7,3.6) {};
	\node[constructorIpos={\constructorObserve}{32}{0.38cm}] (u5) at (4.6,3) {};
	\node[termrep] (v4) at (3.3,2.4) {\footnotesize$B\cap C=\emptyset$};
	\node[termrep] (v3) at (1.3,2.5) {\footnotesize$A\subseteq B$};
	\node[termrep] (v5) at (4.9,1.7) {\footnotesize$A\cap C=\emptyset$};
	\node[termrep] (v1) at (2.3,1.7) {\scalebox{0.8}{\abcDiagram}};
	\path[->]
	(u4) edge[bend right = 0] (v)
	(u3) edge[bend right = -15] (v)
	(u5) edge[bend right = 20] (v)
	(v1) edge[in=235,out=85] node[index label] {1} (u4)
	(v1) edge[bend right = 5] node[index label] {1} (u3)
	(v1) edge[bend right = 35] node[index label] {1} (u5)
	(v4) edge[bend right = 10] node[index label] {2} (u4)
	(v3) edge[bend right = -10] node[index label] {2} (u3)
	(v5) edge[bend right = 5] node[index label] {2} (u5);
	\end{tikzpicture}
\end{minipage}
\end{center}
\end{samepage}

Now, \adjustbox{scale=0.61,raise=-0.08cm}{\abcDiagram}\! has various \textit{meaning-carrying relationships} such as: (a) the circle $A$ is inside $B$ (meaning: $A\subseteq B$), (b) the circles $B$ and $C$ do not overlap (meaning: $B\cap C=\emptyset$), and (c) the circles $A$ and $C$ do not overlap (meaning: $A\cap C=\emptyset$). It is possible to \textit{observe} the truth of $A\subseteq B$, $B\cap C=\emptyset$ and $A\cap C=\emptyset$ from \adjustbox{scale=0.61,raise=-0.08cm}{\abcDiagram}\!. In addition, while the former two statements can be directly obtained from $\{A\subseteq B, B\cap C=\emptyset\}$, the statement $A\cap C=\emptyset$ would need to be inferred using inference rules encoded in the entailment space of $\rsystemn_{\mathit{ST}}$. These \textit{observability} properties of \adjustbox{scale=0.61,raise=-0.08cm}{\abcDiagram}\! are encoded above right.
\end{example}

Let us now define an \textit{inter-representational-system encoding}, which embodies the idea that an identification space can encode relationships between representational systems. Definition~\ref{defn:interRSEncoding} exploits the fact that, given compatible representational systems, $\rsystemn=\rsystem$ and $\rsystemn'=\rsystemd$, the following are construction spaces: $\gspacen\cup \gspacen'$, $\espacen\cup \espacen'$, and  $\ispacen\cup \ispacen'$.

\begin{definition}\label{defn:interRSEncoding}
An \textit{inter-representational-system encoding} is a triple, $\irencodingn=\irencoding$, where:
\begin{enumerate}
\item $\rsystemn=\rsystem$ is a representational system formed over $\tsystemn$ and $\mpsystemn$,
\item $\rsystemn'=\rsystemd$ is a representational system formed over $\tsystemn'$ and $\mpsystemn'$, that is compatible with $\rsystemn$, and
\item $\ispacen''=(\tsystemn\cup \tsystemn'\cup \mpsystemn'', \cspecificationn,\graphn)$ is an identification space formed over $\tsystemn\cup \tsystemn'$ and $\mpsystemn''$ such that
        \begin{enumerate}
        \item all meta-types in $\rsystemn$ and $\rsystemn'$ are included in $\ispacen''$: $\mpsystemn\cup \mpsystemn'\subseteq \mpsystemn''$,
        \item for every token, $t''$ in $\ispacen''$:
                \begin{enumerate}
                \item if $t''$ is not in $\gspacen\cup \gspacen'$ then the label of $t''$ is a meta-type in $\mpsystemn''$, and
                \item if $t''$ is the output of a configurator in $\graphn$ then $t''$ is not in $\gspacen\cup \gspacen'$, and
                \end{enumerate}
         \item $\ispacen''$ is compatible with the construction spaces $\gspacen\cup \gspacen'$, $\espacen\cup \espacen'$, and $\ispacen\cup \ispacen'$ and their sets of constructors are pairwise disjoint.

        \end{enumerate}
\end{enumerate}
The space $\ipispacen=\ispacen\cup \ispacen'\cup \ispacen''$ is called an \textit{inter-property identification space} for $(\rsystemn,\rsystemn',\ispacen'')$.
\end{definition}

To conclude Section~\ref{sec:RSs}, theorem~\ref{thm:interRSisRS} establishes that an inter-representational-system encoding can be readily converted into a representational system.

\begin{theorem}\label{thm:interRSisRS}
Let $\irencodingn=\irencoding$ be an inter-representational-system encoding. Given $\rsystemn=\rsystem$ and $\rsystemn'=\rsystemd$ the triple
\begin{displaymath}
(\gspacen\cup \gspacen', \espacen\cup \espacen', \ipispacen)
\end{displaymath}
is a representational system.
\end{theorem}

\begin{proof}
Since $\rsystemn$ and $\rsystemn'$ are compatible, we know that $\rsystemn\cup \rsystemn'$ is a representational system. Hence conditions (1) and (2) of definition~\ref{defn:representationalSystem} hold for $(\gspacen\cup \gspacen', \espacen\cup \espacen', \ipispacen)$. Suppose that $\ispacen''=(\tsystemn\cup \tsystemn'\cup \mpsystemn'', \cspecificationn,\graphn)$ as in definition~\ref{defn:interRSEncoding}. We now show that $\ipispacen=\ispacen\cup \ispacen\cup \ispacen''$ is an identification space formed over $\tsystemn\cup \tsystemn'$ and $\mpsystemn''$, thus satisfying  definition~\ref{defn:identificationSpace}, such that conditions (3) and (4) of definition~\ref{defn:representationalSystem} hold. Firstly, focusing on definition~\ref{defn:identificationSpace}, it is trivial that $\tsystemn\cup \tsystemn'$ and $\mpsystemn''$ are compatible since $\ispacen''$ is a construction space. Hence $\ipispacen$ is an identification space formed over $\tsystemn\cup \tsystemn'$ and $\mpsystemn''$. It is readily seen that condition (3) in definition~\ref{defn:representationalSystem} is merely a restatement of condition 3(b) in definition~\ref{defn:interRSEncoding}. Condition (4) of definition~\ref{defn:representationalSystem} requires that $\gspacen\cup \gspacen'$, $\espacen\cup \espacen'$ and $\ipispacen=\ispacen\cup \ispacen'\cup \ispacen''$ are pairwise compatible:
\begin{enumerate}
\item[-] $\gspacen\cup \gspacen'$ and $\espacen\cup \espacen'$ are compatible since $\rsystemn\cup \rsystemn'$ is a representational system.
\item[-] Consider $\gspacen\cup \gspacen'$ and $\ispacen\cup \ispacen'\cup \ispacen''$. We establish the pairwise compatibility of $\gspacen\cup \gspacen'$, $\ispacen\cup \ispacen'$ and $\ispacen''$. We know that $\gspacen\cup \gspacen'$ and $\ispacen\cup \ispacen'$ are compatible since $\rsystemn\cup \rsystemn'$ is a representational system. Condition 3(c) of definition~\ref{defn:interRSEncoding} gives us that $\gspacen\cup \gspacen'$ and $\ispacen''$ are compatible. Similarly, $\ispacen\cup \ispacen'$ and $\ispacen''$ are compatible. By lemma~\ref{lem:pairwiseCompatibleCSformCS}, $\gspacen\cup \gspacen'$ and $\ispacen\cup \ispacen'\cup \ispacen''$ are compatible.
\item[-] Establishing the compatibility of $\espacen\cup \espacen'$ and $\ispacen\cup \ispacen'\cup \ispacen''$ is similar to the previous case.
\end{enumerate}
Therefore, we have pairwise compatibility, so the first part of condition (4) holds. The last part of condition (4) requires that $\gspacen\cup \gspacen'$, $\espacen\cup \espacen'$, and $\ipispacen$ have pairwise disjoint sets of constructors. This is given by condition 3(c) of definition~\ref{defn:interRSEncoding}.  Thus condition (4) of definition~\ref{defn:representationalSystem} holds. Hence
\begin{displaymath}
(\gspacen\cup \gspacen', \espacen\cup \espacen', \ipispacen)
\end{displaymath}
is a representational system.
\end{proof}

As a consequence of theorem~\ref{thm:interRSisRS}, any theoretical results concerning representational systems can be readily applied to inter-representational-system encodings, after such a conversion. Theorem~\ref{thm:interRSisRS} is also useful in the context of applications of structural transformations, as follows. If we are given two inter-representational-system encodings, $\irencodingn=(\rsystemn,\rsystemn',\ispacen'')$ and $\irencodingn'=(\rsystemn',\rsystemn'',\ispacen''')$, then we can readily derive a new inter-representational-system encoding (assuming the conditions of definition~\ref{defn:interRSEncoding} are met):
\begin{displaymath}
\irencodingn''=((\gspacen\cup \gspacen', \espacen\cup \espacen', \ipispacen),\rsystemn'',\ispacen''').
\end{displaymath}
This allows us to exploit information encoded using $\ispacen'''$, given in $\mathbb{I}(\irencodingn')$, in order to support transfers from $(\gspacen\cup \gspacen', \espacen\cup \espacen', \ipispacen)$ to $\rsystemn''$. Similarly, we could also devise transfers from the representational system $\rsystemn$ to $(\gspacen'\cup \gspacen'', \espacen'\cup \espacen'', \mathbb{I}(\irencodingn'))$.

%% file: constructions.tex
\definecolor{darkgreen}{rgb}{0.0, 0.7, 0.1}
\definecolor{darkblue}{rgb}{0.0, 0.3, 0.8}
\definecolor{darkred}{rgb}{0.7, 0.1, 0.1}
\definecolor{darkpurple}{rgb}{0.4, 0.0, 0.6}
\definecolor{darkgold}{rgb}{0.8, 0.6, 0}

\newcommand{\topC}{black}
\newcommand{\botC}{red!30!blue}
\newcommand{\conC}{darkred}
\newcommand{\found}{darkpurple}
\newcommand{\termN}{very thick}
\newcommand{\termC}{darkpurple}

In a representational system, any of its structure graphs may encode many ways to \textit{construct} each token. A \textit{construction}, which is essentially a subgraph of a structure graph, describes \textit{one way} to construct a token from other tokens; examples will be given in Section~\ref{sec:constructions:AMEGROV}. Constructions play a fundamental role in our approach to \textit{describing} representational systems using patterns (Section~\ref{sec:patterns}) and applying \textit{structural transformations} to representations (Section~\ref{sec:structuralTransformations}). Roughly speaking, a pattern will describe a class of constructions, thus providing a compact way of describing isomorphic constructions of tokens that are built using essentially the same configurators but different tokens. We formally define constructions and develop the theory of \textit{splits}, which supports the subdivision of a construction into multiple sub-constructions. The ability to split constructions allows fewer patterns to be used in system descriptions (explained in opening remarks of Section~\ref{sec:patterns}) and provides an essential simplification mechanism that forms a core part of our approach to defining structural transformations.

\subsection{A Motivating Example: \textsc{Graphical Representations of Vectors}}\label{sec:constructions:AMEGROV}
The major contributions of Section~\ref{sec:constructions} will be exemplified using a graphical representational system, $\rsystemn_{GV}$, that visualises vectors in the plane. We focus on two particular constructors in the entailment space: $\GRVcvecadd$, as seen in example~\ref{ex:vectorMulti}, and $\GRVcvecsub$, for subtraction entailments, with signature $\spec(\GRVcvecsub)=([\GRVvec,\GRVvec],\GRVvec)$. Even with only two constructors, the resulting entailment space has a non-trivial structure graph: it is not a DAG, since it contains cycles. To illustrate, the three tokens, $t_1$, $t_2$ and $t_3$ below, are visualisations of $(1,1)$, $(2,3)$ and $(1,1.5)$:
\begin{center}
\begin{minipage}[t]{0.32\textwidth}\centering
\textit{representation, $t_1$, of $(1,1)$:}\\
 \vectorVis{(1,1)}{0.5}
\end{minipage}
\hfill
\begin{minipage}[t]{0.32\textwidth}\centering
\textit{representation, $t_2$, of $(2,3)$:}\\
\vectorVis{(2,3)}{0.5}
\end{minipage}
\hfill
\begin{minipage}[t]{0.32\textwidth}\centering
\textit{representation, $t_3$, of $(1,1.5)$:}\\
\vectorVis{(1,1.5)}{0.5}
\end{minipage}
\end{center}

Taking $t_4$ and $t_5$ to be visualisations of $(3,4)$ and $(2,2.5)$, respectively, the following graph, $\cgraphn$, is a \textit{construction} of $t_3$ that contains a cycle (the labels $t_1,\ldots,t_5$ are the tokens' names, not types):
\begin{center}\label{example:cons:t3}
\vectorsCon
\end{center}
Working backwards along the arrows, from $t_3$, initially we see that the constructor $\GRVcvecsub$ encodes the fact that  $t_3$ is entailed by $[t_4,t_5]$, under a \textit{subtraction} rule; analogously, in the associated vector space, the expression $(3,4)-(2,2.5)$ is semantically equivalent to $(1,1.5)$. In turn, $t_4$ and $t_5$ are entailed by $[t_1,t_2]$ and $[t_1,t_3]$, respectively, encoded by $\GRVcvecadd$. Hence, the entire graph encodes the construction of $t_3$  from $t_1$ (used twice), $t_2$, and itself. These facts are derived from the graph as follows: $t_1$, $t_2$, and $t_3$ are sources of non-extendable trails\footnote{Recall that a trail, $\trail$, is (source-)extentable if there is an arrow, $a$, such that $[a]\oplus \trail$ is also a trail.} that target $t_3$; recall, the target of a trail is the final vertex in its associated vertex sequence. In fact, the \textit{foundation token-sequence}, $\foundationsto{\cgraphn,t_3}=[t_1,t_2,t_1,t_3]$, can be derived from the sources of \textit{all non-extendable trails} that target $t_3$, ordered in such a way that they respect the arrow indices\footnote{Essentially, the foundation token-sequence can be computed by performing a depth-first search of the graph, that does not visit any arrow more than once.}:
\begin{eqnarray*}
\trail_1 & = & t_1\arrow[1] u_2\arrow t_4\arrow[1] u\arrow t_3 \\
\trail_2 & = & t_2\arrow[2] u_2\arrow t_4\arrow[1] u\arrow t_3 \\
\trail_3 & = & t_1\arrow[1] u_1\arrow t_5\arrow[2] u\arrow t_3 \\
\trail_4 & = & t_3\arrow[2] u_1\arrow t_5\arrow[2] u\arrow t_3.
\end{eqnarray*}
The \textit{complete trail sequence} of this construction is $\ctsequence{\cgraphn,t_3}=[\trail_1,\trail_2,\trail_3,\trail_4]$, from which the foundation token-sequence is derived.

Complete trail sequences and foundation token-sequences play a role in our approach to splitting constructions into simpler constructions, which ultimately enables our development of \textit{structural transformations} (Section~\ref{sec:structuralTransformations}). The theory of splits, developed in Section~\ref{sec:constructions:GICAS}, will ensure that the complete trail sequence and, thus, the foundation token-sequence of the original construction can be readily obtained from any split.

\subsection{Constructions, Trail Sequences and Foundations}\label{sec:constructions:CTSAF}

A \textit{construction} will be defined to be a pair, $\cpair$, comprising a structure graph, $\cgraphn$, that satisfies certain conditions and a \textit{constructed token}, $t$; note the use of $\cgraphn$ in a construction, we reserve the notation $\graphn$ for arbitrary structure graphs. Given a construction of $t$, we can identify the tokens, called \textit{foundations}, from which $t$ can be constructed, as demonstrated in Section~\ref{sec:constructions:AMEGROV}. The order of the foundations respects the indexed arrows employed by configurations of constructors and they are identified using complete trail sequences. Trails feature heavily in the theory presented throughout Section~\ref{sec:constructions}. Section~\ref{sec:constructions:AMEGROV} gave an example of a construction whereas our next example shows a structure graph that fails to be a suitable $\cgraphn$.

\begin{example}[A Structure Graph that is not a Construction]\label{ex:Unistructured}
In the representational system $\rsystemn_{\mathit{GV}}$, consider the terms $t_1$ to $t_5$ as in Section~\ref{sec:constructions:AMEGROV} alongside terms $t_6$ and $t_7$ that are visualisations of the vectors $(3,2)$ and $(2,0.5)$ respectively. The following is a subgraph, $\graphn$, of the entailment space $\espacen_{\mathit{GV}}$'s structure graph:
\begin{center}
\begin{tikzpicture}[construction]
\begin{scope}[rotate=-90]
\node[termpos={$t_3$}{90}{0.15cm},inner sep = 1.5pt] (v) at (3.2,5.6) {\scalebox{.5}{\vectorVis{(1,1.5)}{0.4}}};
\node[termpos={$t_4$}{20}{0.17cm},inner sep = 1.5pt] (v1) at (2.4,2.6) {\scalebox{.5}{\vectorVis{(3,4)}{0.4}}};
\node[termpos={$t_5$}{-20}{0.17cm},inner sep = 1.5pt] (v2) at (4.0,2.6) {\scalebox{.5}{\vectorVis{(2,2.5)}{0.4}}};
\node[constructorEpos={$\GRVcvecsub$}{60}{0.17cm}] (u) at (3.2,4.1) {$u$};
\node[constructorEpos={$\GRVcvecadd$}{80}{0.18cm}] (u1) at (2.4,1.2) {$u_2$};
\node[constructorEpos={$\GRVcvecadd$}{80}{0.18cm}] (u2) at (4,1.2) {$u_1$};
\node[termpos={$t_2$}{160}{0.17cm},inner sep = 1.5pt] (v3) at (2.4,-0.5) {\scalebox{.5}{\vectorVis{(2,3)}{0.4}}};
\node[termpos={$t_1$}{200}{0.17cm},inner sep = 1.5pt] (v4) at (4,-0.5) {\scalebox{.5}{\vectorVis{(1,1)}{0.4}}};
\node[termpos={$t_6$}{20}{0.17cm},inner sep = 1.5pt] (v6) at (2.4,8.7) {\scalebox{.5}{\vectorVis{(3,2)}{0.4}}};
\node[termpos={$t_7$}{-20}{0.17cm},inner sep = 1.5pt] (v7) at (4.0,8.7) {\scalebox{.5}{\vectorVis{(2,0.5)}{0.4}}};
\node[constructorEpos={$\GRVcvecsub$}{120}{0.17cm}] (up) at (3.2,7.1) {$u'$};
\end{scope}
\path[->]
(u) edge[bend right = 0] (v)
(v1) edge[bend right = -10] node[index label] {1} (u)
(v2) edge[bend right = 10] node[index label] {2} (u)
(v3) edge[bend right = -10] node[index label] {2} (u1)
(v4) edge[bend right = 10] node[index label] {1} (u2)
(u1) edge[bend right = 0] (v1)
(u2) edge[bend right = 0] (v2)
(v4) edge[bend right = 10] node[index label] {1} (u1)
(v) edge[in = -85, out = 240, looseness=0.72] node[index label] {2} (u2)
(up) edge[bend right = 0] (v)
(v6) edge[bend right = 10] node[index label] {1} (up)
(v7) edge[bend right = -10] node[index label] {2} (up)
;
\end{tikzpicture}\vspace{-0.55cm}
\end{center}

The graph $\graphn$ shows two different ways of constructing $t_3$, since $t_3$ is the target of two arrows: $t_3$ can be constructed from $[t_1,t_2,t_1,t_3]$ and from $[t_6,t_7]$, using configurations of $\GRVcvecadd$ and $\GRVcvecsub$. Intuitively, if a token, $t$, is the target of two (or more) arrows, $a_0$ and $a_0'$, then the graph is encoding two different ways to construct $t$. Thus, since constructions are intended to show one way of constructing the identified token, $\graphn$ is \textit{not} a construction of $t_3$. However, any subgraph of $\graphn$ that is a structure graph, where $t_3$ is the target of a unique arrow, is a construction of $t_3$.
\end{example}

In example~\ref{ex:Unistructured}, the failure of the graph to be a construction of $t_3$ was due to the fact that more than one arrow targeted $t_3$. Similar issues arise when \textit{any} token is the target of more than one arrow: constructions prohibit graphs from exhibiting such a property.

\begin{definition}\label{defn:uniStructured}
A structure graph, $\graphn$, is \textit{uni-structured} provided that for every token, $t$, in $\graphn$ there is at most one arrow, $a$, in $\graphn$ such that $\tar{a}= t$.
\end{definition}

So, in a uni-structured graph, $\graphn$, every configurator has exactly one outgoing arrow (since $\graphn$ is a structure graph) and every token has at most one incoming arrow. A construction, $\cpair$, requires that $\cgraphn$ is uni-structured. The specification of the token $t$ in $\cpair$ is essential as it cannot be inferred from the graph: some constructions will contain directed cycles and there is no reason to require that there is a unique sink, say, that is the constructed token; in Section~\ref{sec:constructions:AMEGROV}, $t_3$ was not a sink. Constructions are formalised in definition~\ref{defn:construction} which implicitly assumes that a construction is formed within the context of a construction space, $\cspacen=\cspace$; that is, if $\cgraphn$ is a construction graph then $\cgraphn\subseteq \graphn$ and all the configurators used in $\cgraphn$ respect $\cspecificationn=\cspecification$. However, for readability, we do not typically make the presence of $\cspacen$ explicit unless it is needed for disambiguation.

\begin{definition}\label{defn:construction}
A \textit{construction} is a pair, $\cpair$, where
\begin{enumerate}
\item  $\cgraphn$ is a finite, uni-structured structure graph, and
\item $t$ is a token in $\cgraphn$ such that each vertex in $\cgraphn$ is the source of a trail in $\cgraphn$ that targets $t$.
\end{enumerate}
Given a construction, $\cpair$, we say that $\cgraphn$ \textit{constructs} $t$ and that $t$ is the \textit{construct} of $\cpair$. If $t$ is the only vertex in $\cgraphn$ then $\cpair$ is \textit{trivial}. If $\cpair$ contains exactly one configurator then $\cpair$ is \textit{basic}.
\end{definition}

\begin{definition}\label{defn:setOfConstructions}
Let $\rsystemn=\rsystem$ be a representational system. A construction, $\cpair$, is a \textit{construction in $\rsystemn$} provided for each configurator, $u$, in $\cgraphn$, $\neigh{u,\cgraphn}=\neigh{u,\ugraph{\rsystemn}}$. The set of constructions in $\rsystemn$ is denoted $\allconstructions{\rsystemn}$.
\end{definition}
The construction $(\cgraphn,t_3)$, given on page~\pageref{example:cons:t3}, is an element of $\allconstructions{\rsystemn_{\mathit{GV}}}$. The requirement that the neighbourhoods of configurators are `complete' in constructions (i.e. $\neigh{u,\cgraphn}=\neigh{u,\ugraph{\rsystemn}}$) means that if we were to remove any arrow from $\cgraphn$ then it would no longer be a construction in $\allconstructions{\rsystemn_{\mathit{GV}}}$. Notably, the condition that each vertex in $\cpair$ is the source of a trail that targets $t$ ensures that each token is used in the construction of $t$. For instance, the following graph is not a construction of $t$, since it includes superfluous vertices ($u'$ and $t'$):
\begin{center}
\begin{tikzpicture}[construction,yscale=0.8]
\node[termrep] (v) at (3.2,4.2) {$t$};
\node[termrep] (v1) at (2.4,2.2) {$t_1$};
\node[termrep] (v2) at (4,2.2) {$t_2$};
\node[constructorNW={$c$}] (u) at (3.2,3.2) {$u$};
\node[constructorNE={$c'$}] (u') at (4.8,3.2) {$u'$};
\node[termrep] (v') at (4.8,4.2) {$t'$};
\path[->]
(u) edge[bend right = 0] (v)
(v1) edge[bend right = -10] node[index label] {1} (u)
(v2) edge[bend right = 10] node[index label] {2} (u)
(v2) edge[bend right = 10] node[index label] {1} (u')
(u') edge[bend right = 0] (v');
\end{tikzpicture}
\end{center}

Of particular interest are trails that are sourced and targetted on tokens. We call such trails \textit{well-formed}. In essence, given a construction, $\cpair$, a well-formed trail with source $t'$ and target $t$ indicates some of the tokens used to construct $t$ and the configurators for which they are inputs.

\begin{definition}\label{defn:wellformedTrail}
Let $\cpair$ be a construction and let $\trail$ be a trail in $\cgraphn$. Then $\trail$ is \textit{well-formed} provided its source and target are both tokens.
\end{definition}

As we proceed, we wish to take an extendable well-formed trail, $\trail$, in a construction and identify trails that can be used to extend $\trail$, yielding longer trails. Suppose, then, that $\trail$ has source $t'$ and is extendable by $a$; this implies that $\sor{a}$ is a configurator. Typically, when we identify the trails that extend $\trail$, we wish to order them according to the indices assigned to $\sor{a}$'s incoming arrows. Therefore, we utilise the \textit{input arrow-sequence}, $\inputsA{\sor{a}}$, for the configurator $\sor{a}$, when defining trail extensions. For example, if we take the trail consisting solely of $a$, and $\sor{a}$ has input arrow-sequence $[a_1,a_2]$, then this sequence of trails, that extend $a$, respects the arrow indices:
\begin{displaymath}
[[a_1,a],[a_2,a]]=[[a_1],[a_2]]\triangleleft [a].
\end{displaymath}
This operation, of producing extensions of $\trail$ informed by \textit{the unique} configurator that is the source of \textit{the unique} arrow that targets  $\source{\trail}$ is only well-defined defined because $\cgraphn$ is uni-structured.

In addition to producing extensions of a well-formed trail, $\trail$, often we will want to produce \textit{all possible ways} that $\trail$ can be extended to yield well-formed \textit{non-extendable} trails. We capture this in definition~\ref{defn:CES}, of a \textit{trail extension sequence}, which we will denote by $\tesequence{\trail}$: any trail, $\trail'$, in $\tesequence{\trail}$, is guaranteed to be such that $\trail'\oplus \trail$ is a well-formed, non-extendable trail. In addition, the trails in $\tesequence{\trail}$ are ordered so that, \textit{for each configurator}, they respect the indices assigned to the incoming arrows. This latter property will be ensured by condition (2) in definition~\ref{defn:CES}.

\begin{example}[Trail Extension Sequences]
Referring again to the example in Section~\ref{sec:constructions:AMEGROV}, the construction $(\cgraphn,t_3)$ is repeated \begin{samepage}here:
\begin{center}
	\vectorsCon
\end{center}\end{samepage}
Consider the trail $\trail=t_5\arrow[2] u \arrow t_3$. Concatenating either of the following trails (each of which is well-formed) onto $\trail$ yields a well-formed, non-extendable trail:
\begin{eqnarray*}
\trail_1' & = & t_1\arrow[1] u_1\arrow t_5 \\
\trail_2' & = & t_3\arrow[2] u_1\arrow t_5.
\end{eqnarray*}
That is, $\trail_1'\oplus \trail$ and $\trail_2'\oplus \trail$ are well-formed, non-extendable trails. The sequence $\tesequence{\cgraphn,t_3}=[\trail_1',\trail_2']$ is a \textit{trail extension sequence} for $\trail$. However, $[\trail_2',\trail_1']$ and $[\trail_1']$ are not, since the former does not respect the order of the arrows targeting $u_1$ and the latter omits $\trail_2'$.  The complete trail sequence, $\ctsequence{\cgraphn,t_3}=[\trail_1,\trail_2,\trail_3,\trail_4]$, given in Section~\ref{sec:constructions:AMEGROV}, is the trail extension sequence of $[]_{t_3}$.
\end{example}

Definition~\ref{defn:tes} makes the notion of a trail extension sequence precise, exploiting the facts that any extendable well-formed trail, $\trail$, in a construction is extendable by a unique arrow, $a$, and that $\sor{a}$ is a configurator. In turn, $\sor{a}$  has at least one incoming arrow, $a_1$: definition~\ref{defn:constructionSpecification}, of a constructor specification, implies that configurators have at least one input.

\begin{definition}\label{defn:DTS}\label{defn:CES}\label{defn:tes}
Let $\cpair$ be a construction containing a well-formed trail, $\trail$, with source $t'$. The \textit{trail extension sequence} of $\trail$, denoted $\tesequence{\trail}$, is defined as follows:
\begin{enumerate}
\item if $\trail$ is non-extendable then $\tesequence{\trail}=[[]_{t'}]$, and
\item if $\trail$ is extendable by arrow $a$ then,
            given $\inputsA{\sor{a}}=[a_1,\ldots,a_n]$,
            \begin{displaymath}
              \tesequence{\trail}=\big(\tesequence{[a_1,a]\oplus \trail}\triangleleft [a_1,a] \big) \oplus \cdots \oplus \big(\tesequence{[a_n,a]\oplus \trail}\triangleleft [a_n,a]\big).
            \end{displaymath}
\end{enumerate}
\end{definition}

We now prove two lemmas that establish properties of the trail extension sequence of any trail, $\trail$, in $\cpair$ whose target is $t$.  Lemma~\ref{lem:non-extendableWellFormedTrails} proves that any trail, $\trail'$, in $\tesequence{\trail}$, is such that $\trail'\oplus \trail$ is both well-formed and non-extendable, and it also establishes that the trail $\trail'$ is itself well-formed. The proof uses the fact that trail extension sequences only contain trails of even length, since they are sourced and targeted on vertices in $\pa$ (i.e. the tokens) in a bipartite graph. Lemma~\ref{lem:CESisComplete} establishes that  $\tesequence{\trail}$ contains \textit{all} trails that ensure  $\trail'\oplus \trail$ is both well-formed and non-extendable. In particular, these two lemmas combine to tell us that $\tesequence{[]_t}$ contains \textit{all and only} the well-formed, non-extendable trails in $\cpair$ that target $t$. As indicated earlier, in order to identify the foundation token-sequence from which $t$ is constructed we exploit all non-extendable trails that target $t$ in $\cgraphn$.

\begin{lemma}\label{lem:non-extendableWellFormedTrails}
Let $\cpair$ be a construction containing a well-formed trail, $\trail$, with target $t$. Then, given any trail, $\trail'$, in $\tesequence{\trail}$,%
\begin{enumerate}
\item $\trail'\oplus \trail$ is a non-extendable trail that targets $t$, and
\item $\trail'$ is well-formed.
\end{enumerate}
\end{lemma}

\begin{proof}
Suppose that $t'$ is the source of $\trail$. The proof proceeds by induction over the length of $\trail'$.
For the base case, $\trail'$ has length 0, meaning that the only trail in $\tesequence{\trail}$ is the empty trail, $[]_{t'}$, that is $\trail'=[]_{t'}$. Therefore, we have
\begin{displaymath}
  \trail'\oplus \trail =[]_{t'} \oplus \trail= \trail.
\end{displaymath}
It is given that the trail $\trail$ has target $t$. Moreover, the only way that $\tesequence{\trail}=[[]_{t'}]$ is when $\trail$ is non-extendable which gives us that $[]_{t'} \oplus \trail$ is non-extendable. Trivially, $\trail'$ is well-formed. Hence the base case holds.

Assume, for all $\trail'$ with length at most $k$, that  $\trail'\oplus \trail$ is a non-extendable trail that targets $t$ and that $\trail'$ is well-formed. For the inductive step, let $\trail'$ be such that $\trail'=[a_1,\cdots,a_k,a_{k+1},a_{k+2}]$. We must show that $\trail'\oplus \trail$ is a non-extendable trail that targets $t$, and that $\trail'$ is well-formed. Given $\trail'$ is in $\tesequence{\trail}$, we consider part (2) of definition~\ref{defn:DTS}, since $\trail$ is extendable by $a_{k+2}$. It can only be that $\trail'$ is in
\begin{displaymath}
  \tesequence{[a_{k+1},a_{k+2}]\oplus \trail}\triangleleft [a_{k+1},a_{k+2}].
\end{displaymath}
Thus, we have it that $[a_1,\ldots,a_{k}]$ is in  $\tesequence{[a_{k+1},a_{k+2}]\oplus \trail}$. By the inductive assumption
\begin{displaymath}
  [a_1,\ldots,a_{k}] \oplus ([a_{k+1},a_{k+2}]\oplus \trail)
\end{displaymath}
is a non-extendable trail that targets $t$. Since $\trail'=[a_1,\ldots,a_{k}] \oplus [a_{k+1},a_{k+2}]$, it follows that
$\trail'\oplus \trail$ is a non-extendable trail that targets $t$. Now, it is trivial that $\target{\trail'}$ is a token, since the source of $\trail$ is a token. By the inductive assumption, $[a_1,\ldots,a_{k}]$ is well-formed, therefore $[a_1,\ldots,a_{k-1},a_k,a_{k+1},a_{k+2}]=\trail'$ is also well-formed. Hence, $\trail'\oplus \trail$ is a non-extendable trail that targets $t$ and $\trail'$ is well-formed.
\end{proof}

\begin{corollary}\label{cor:TESOnlyNonExtendable}
Let $\cpair$ be a construction. Then $\tesequence{[]_t}$ contains only non-extendable trails that target $t$.
\end{corollary}

We now establish the completeness of $\tesequence{\trail}$, in the sense that any trail, $\trail'$, where $\trail'\oplus \trail$ is a non-extendable trail, appears in $\tesequence{\trail}$.

\begin{lemma}\label{lem:CESisComplete}
Let $\cpair$ be a construction containing a well-formed trail, $\trail$, with source $t'$ and target $t$. Then any trail, $\trail'$, where $\trail'\oplus \trail$ is a  well-formed, non-extendable trail, occurs in $\tesequence{\trail}$.
\end{lemma}

\begin{proof}
The proof is by induction, over the length of $\trail'$. For the base case, $\trail'$ has length 0, meaning that $\trail'=[]_{t'}$ and $\trail$ is non-extendable. By definition~\ref{defn:CES}, $\tesequence{\trail}=[[]_{t'}]=[\trail']$. Hence the  base case holds. Assume, for all $\trail'$ with length at most $k$, that $\trail'$ occurs in $\tesequence{\trail}$. For the inductive step, let $\trail'$ be such that $\trail'=[a_1,\ldots,a_k,a_{k+1},a_{k+2}]$. We must show that $\trail'$ occurs in $\tesequence{\trail}$. By assumption, $[a_1,\ldots,a_{k}]$ occurs in $\tesequence{[a_{k+1},a_{k+2}]\oplus \trail}$. By definition~\ref{defn:CES}, part (2), the trail
\begin{displaymath}
  [a_1,\ldots,a_{k}]\oplus [a_{k+1},a_{k+2}]=\trail'
\end{displaymath}
occurs in $\tesequence{\trail}$, as required. Hence, the trail $\trail'$ occurs in $\tesequence{\trail}$.
\end{proof}

\begin{corollary}\label{cor:TESAllandOnlyNonExtendable}
Let $\cpair$ be a construction. Then $\tesequence{[]_t}$ contains all and only the non-extendable trails in $\cpair$ that target $t$.
\end{corollary}

As we proceed, we often exploit $\tesequence{[]_t}$, with definition~\ref{defn:CTS} introducing convenient terminology and notation.

\begin{definition}\label{defn:CTS}
Let $\cpair$ be a construction. The \textit{complete trail sequence} of $\cpair$, denoted $\ctsequence{\cgraphn,t}$, is defined to be
$\ctsequence{\cgraphn,t}=\tesequence{[]_t}.$
\end{definition}

As our earlier examples indicate, the \textit{foundation token-sequence}, $\foundationsto{\cgraphn,t}$, contains precisely the sources of the trails in $\ctsequence{\cgraphn,t}$. Thus, to conclude Section~\ref{sec:constructions:CTSAF}, we firstly define a function that produces the sequence of trail-sources from a sequence of trails. Thereafter, we define $\foundationsto{\cgraphn,t}$, which captures the sequence of tokens, $[t_1,\ldots,t_n]$ from which a token, $t$, is constructed. From that, we can extract the \textit{foundation type-sequence}, replacing tokens with their assigned types.

\begin{definition}\label{defn:sourceSequence}
Let $\trailSeq=[\trail_1,\ldots,\trail_n]$ be a sequence of trails in some graph. The \textit{source sequence} of $\trailSeq$, denoted $\ssequence{\trailSeq}$, is defined to be
$\ssequence{\trailSeq} = [\source{\trail_1},\ldots,\source{\trail_n}].$
\end{definition}

\begin{definition}\label{defn:foundations}
Let $\cpair$ be a construction. The \textit{foundation token-sequence} of $\cpair$, denoted $\foundationsto{\cgraphn,t}$ is defined to be
$\foundationsto{\cgraphn,t}=\ssequence{\ctsequence{\cgraphn,t}}$.
The \textit{foundation type-sequence} of $\cpair$, denoted $\foundationsty{\cgraphn,t}$,
is obtained from $\foundationsto{\cgraphn,t}$ by replacing each token with its assigned type. The \textit{arity} of $\cpair$ is the length of $\foundationsto{\cgraphn,t}$.
\end{definition}

We conclude Section~\ref{sec:constructions:CTSAF} with a remark about the importance of foundation sequences. Section~\ref{sec:structuralTransformations} elucidates how Representation Systems Theory supports the structural transformation of a construction, $\cpair$, into another construction, $\cpaird$, which may not be known in advance, but which may be built to satisfy some desirable properties.  By exploiting the theory of patterns, developed in Section~\ref{sec:patterns}, it is possible to \textit{describe} $\cpair$ using a pattern, $\ppair$, which is itself a construction similar to $\cpair$, but one which `abstracts' the types and `forgets' the specific tokens used in $\cpair$; so $\pgraphn$ is isomorphic to $\cgraphn$, where vertices in $\pgraphn$ are labelled with super-types of their corresponding vertices in $\cgraphn$. In turn, $\ppair$ can be turned into another pattern, $\ppaird$, that describes the sought-after $\cpaird$. The last step in this process  -- finding
 a specific $\cpaird$ -- involves determining how to replace each vertex, $v'$, in $\pgraphn'$ with a suitable token, $t_{v'}$ to yield $\cgraphn'$. Here, foundation sequences are particularly useful: the structural transformation process ensures that $\foundationsty{\pgraphn',v'}$ is a generalisation of $\foundationsty{\cgraphn',t'}$.   Thus, we can \textit{infer} possible token-sequences that could form $\foundationsto{\cgraphn',t'}$. Subsequently, given a \textit{functional} construction space (e.g. any grammatical space), if we are provided with $\foundationsto{\cgraphn',t'}$ then one can \textit{uniquely instantiate the rest of the types as tokens} in $\pgraphn'$ to yield $\cpaird$, provided such a $\cpaird$ exists. Even when a construction space is not functional, knowing $\foundationsto{\cgraphn',t'}$ can sometimes allow the rest of the types to be instantiated as tokens, but not necessarily uniquely. This powerful insight demonstrates the role that foundations will play in Section~\ref{sec:structuralTransformations}.

\subsection{Generators, Induced Constructions and Splits}\label{sec:constructions:GICAS}

The major goal of this section is to demonstrate how we can decompose a construction, $\cpair$, into a collection of sub-constructions such that (a) the complete trail sequence, $\ctsequence{\cgraphn,t}$, can be reconstituted from the complete trail sequences of the sub-constructions \textit{and} (b) the foundation token-sequence for $\cpair$ is readily derivable from the sub-constructions. The general approach is to, firstly, identify a sub-construction, $\genpair$, called a \textit{generator}, of $\cpair$, that also constructs $t$. Any given generator induces a sequence of sub-constructions\footnote{If $\cgraphn$ is a tree then, taking $F_V=\{t' : t'\in \foundationsto{\cgraphn',t}\}$ as a vertex cut set, these induced sub-constructions would be readily \textit{derivable} from the components of $\cgraphn\backslash F_V$. However, constructions are not necessarily trees, or even DAGs, so the process of decomposing a construction is not straightforward.} of $\cpair$. By design, $\ctsequence{\cgraphn,t}$ and $\foundationsto{\cgraphn,t}$ can be obtained from the generator and the induced sub-constructions, as we shall prove in theorems~\ref{thm:splitPreservesCompleteTrailSequence} and~\ref{thm:splitPreservesfoundations}. Notably, the generator and the induced sub-constructions completely `cover' $\cpair$: all vertices and arrows in $\cgraphn$ occur in the generator or in an induced sub-construction, captured in theorem~\ref{thm:SplitsCoverConstruction}.  A generator together with its induced constructions form a \textit{split}. This crafted approach to decomposing a construction into simpler sub-constructions is vital for our approach to describing representational systems (Section~\ref{sec:patterns}) and for structural transformations (Section~\ref{sec:structuralTransformations}).

The first step in this endeavour is to define a \textit{generator}. We assume, in all that follows, that the generator is obtained from the same structure graph, $\graphn$, as $\cpair$, given a construction space, $\cspacen=\cspace$.  This means that $\cpair$ and any of its generators are defined over the same type system and constructor specification. As previously noted, prior to definition~\ref{defn:construction}, for readability we omit explicit mention of $\cspacen$ from the following theoretical development and leave it implied.

\begin{definition}\label{defn:generator}
Let $\cpair$ be a construction. A \textit{generator} of $\cpair$ is a construction, $\genpair$, such that $\cgraphn'\subseteq \cgraphn$.
\end{definition}

Example~\ref{ex:generatorsInducedConsAndSplits} will use a construction, $\cpair$,  to demonstrate the concepts of generators, induced constructions and splits. The construction $\cpair$ is not contextualised within some specific representational system in order to allow a variety of features of splits to be more succinctly exposed.

\begin{example}[Generators, Induced Constructions, and Splits]\label{ex:generatorsInducedConsAndSplits}
Consider the following construction, $\cpair$, on the left, alongside one of its generators, $\genpair$:
\begin{center}
\begin{tikzpicture}[construction,yscale=0.7,xscale=0.85]\small
\node[termrep] (v) at (3,-0.1) {$t$};
\node[constructorrep] (u) at (3,-1.1) {$u$};
\node[termrep] (v8) at (3,-2.3) {$t_2$};
\node[constructorrep] (u8) at (3,-3.4) {$u_2$};
\node[termrep] (v1) at (1,-1.8) {$t_1$};
\node[termrep,\termN, \found] (v2) at (5,-1.8) {$t_3$};
\node[termrep] (v4) at (1.8,-4.2) {$t_5$};
\node[termrep] (v5) at (4.2,-4.2) {$t_6$};
\node[termrep] (v3) at (0.2,-4.2) {$t_4$};
\node[constructorNW={}] (u1) at (1,-2.9) {$u_1$};
\node[constructorNE={}] (u2) at (5,-2.9) {$u_3$};
\node[constructorW={}] (u3) at (0.2,-5.3) {$u_4$};
\node[termrep,\termN,\found] (v9) at (0.2,-6.6) {$t_8$}; %%%
\node[termrep] (v6) at (5.8,-4.2) {$t_7$};
\node[constructorW={}] (u4) at (1.8,-5.3) {$u_5$};
\node[constructorSE={}] (u5) at (5.8,-5.3) {$u_7$};
\node[termrep,\termN, \found] (v7) at (3.8,-6.6) {$t_9$};
\node[constructorGE={}] (u9) at (4.2,-5.3) {$u_6$};
\path[->]
(u) edge[bend right = 0] (v)
(v1) edge[bend right = -10] node[index label] {1} (u)
(v2) edge[bend right = 10] node[index label] {3} (u)
(u1) edge[bend right = 0] (v1)
(u2) edge[bend right = 0] (v2)
(u3) edge[bend right = 0] (v3)
(v9) edge[bend right = 0] node[index label] {1} (u3) %%%
(v3) edge[bend right = 0] node[index label] {1} (u1)
(v4) edge[bend right = 10] node[index label] {2} (u1)
(v5) edge[bend right = -10] node[index label] {1} (u2)
(v6) edge[bend right = 10] node[index label] {2} (u2)
(u4) edge[bend right = 0] (v4)
(v7) edge[bend right = -10] node[index label] {1} (u4)
(v7) edge[bend right = 10] node[index label] {2} (u5)
(v2) edge[bend left = 60] node[index label] {1} (u5)
(u5) edge[bend right = 0] (v6)
(v3) edge[bend left = 30] node[index label] {3} (u1)
(v8) edge[bend right = 0] node[index label] {2} (u)
(u8) edge[bend right = 0] (v8)
(v4) edge[bend right = -10] node[index label] {1} (u8)
(v4) edge[bend right = -10] node[index label] {1} (u9)
(u9) edge[bend right = 0] (v5);
\end{tikzpicture}%$
\hspace{0.9cm}
\begin{tikzpicture}[construction,yscale=0.7,xscale=0.85]\small
\node[termrep] (v) at (3,-0.1) {$t$};
\node[constructorrep] (u) at (3,-1.1) {$u$};
\node[termrep] (v8) at (3,-2.3) {$t_2$};
\node[constructorrep] (u8) at (3,-3.4) {$u_2$};
\node[termrep] (v1) at (1,-1.8) {$t_1$};
\node[termrep,\termN, \found] (v2) at (5,-1.8) {$t_3$};
\node[termrep] (v4) at (1.8,-4.2) {$t_5$};
\node[termrep,\termN, \found] (v3) at (0.2,-4.2) {$t_4$};
\node[constructorNW={}] (u1) at (1,-2.9) {$u_1$};
\node[constructorW={}] (u4) at (1.8,-5.3) {$u_5$};
\node[termrep,\termN, \found] (v7) at (3.8,-6.6) {$t_9$};
\path[->]
(u) edge[bend right = 0] (v)
(v1) edge[bend right = -10] node[index label] {1} (u)
(v2) edge[bend right = 10] node[index label] {3} (u)
(u1) edge[bend right = 0] (v1)
(v3) edge[bend right = 0] node[index label] {1} (u1)
(v4) edge[bend right = 10] node[index label] {2} (u1)
(u4) edge[bend right = 0] (v4)
(v7) edge[bend right = -10] node[index label] {1} (u4)
(v3) edge[bend left = 30] node[index label] {3} (u1)
(v8) edge[bend right = 0] node[index label] {2} (u)
(u8) edge[bend right = 0] (v8)
(v4) edge[bend right = -10] node[index label] {1} (u8)
;
\end{tikzpicture}
\end{center}
The first part of this example explores the foundation token-sequence of $\cpair$, which relies on constructing $\ctsequence{\cgraphn,t}$. The first trail in $\ctsequence{\cgraphn,t}$ is $t_8\arrow[1] u_4 \arrow t_4 \arrow[1] u_1 \arrow t_1 \arrow[1] u \arrow t$, which is obtained by traversing backwards along the arrows from $t$, choosing arrows with index 1 (since this is the first trail) each time there is a choice to be made and stopping when the trail is non-extendable (i.e. no more arrows can be added at the beginning of the trail because either there is no arrow or the arrow is already in the trail). The second and third trails are
\begin{eqnarray*}
 & & t_9\arrow[1] u_5 \arrow t_5\arrow[2] u_1 \arrow t_1 \arrow[1] u \arrow t, \textup{and}  \\
& & t_8\arrow[3] u_4 \arrow t_4 \arrow[1] u_1 \arrow t_1 \arrow[1] u \arrow t;
\end{eqnarray*}
we omit the remaining four trails. The foundation token-sequence of $\cpair$ is obtained from the sources of the trails in $\ctsequence{\cgraphn,t}$: $\foundationsto{\cgraphn,v}=[t_8,t_9,t_8,t_9,t_9,t_3,t_9]$ and the foundations are highlighted in the graph.  The approach taken to produce induced constructions of $\cpair$ given $(\cgraphn',t)$, relies on using the complete trail sequence, $\ctsequence{\cgraphn',t}  =  [\trail_1, \trail_2, \trail_3, \trail_4, \trail_5]$, of $\genpair$ where:
\begin{eqnarray*}
\trail_1 & = & t_4\arrow[1] u_1\arrow t_1 \arrow[1] u \arrow t\\
\trail_2 & = & t_9\arrow[1] u_5\arrow t_5\arrow[2] u_1\arrow t_1 \arrow[1] u \arrow t\\
\trail_3 & = & t_4\arrow[3] u_1\arrow t_1 \arrow[1] u \arrow t\\
\trail_4 & = & t_9\arrow[1] u_5\arrow t_5\arrow[1] u_2\arrow t_2 \arrow[2] u \arrow t\\
\trail_5 & = & t_3 \arrow[3] u \arrow t.
\end{eqnarray*}
Each $\trail_i$ in $\ctsequence{\cgraphn',t}$ gives rise to an \textit{induced construction}, $\icp{\trail_i}=(\icgp{\trail_i},\source{\trail_i})$. The \textit{induced construction graph}, $\icgp{\trail_i}$, of $\trail_i$ is formed from the trail extension sequence, $\tesequence{\trail_i}$, in $\cpair$; we are exploiting more than one construction -- $\cpair$ and $\genpair$ -- so we write $\tesequence{\trail_i,\cpair}$ to indicate that we are creating the trail extension sequence of $\trail_i$ in $\cpair$. The graph we require, $\icgp{\trail_i}$, is the subgraph of $\cgraphn$ that contains precisely the arrows that occur in the trails in $\tesequence{\trail_i,\cpair}$. The trail $\trail_1$ has $\tesequence{\trail_1,(\cgraphn,t)}=[t_8\arrow[1] u_4\arrow t_4]$, and the trail $t_8\arrow[1] u_4\arrow t_4$ is, essentially, the induced construction graph that we seek. All five induced constructions, alongside the generator, can be seen here:
\begin{center}
\begin{tikzpicture}[construction,yscale=0.7,xscale=0.85]\small
\node[termrep] (v) at (3,-0.1) {$t$};
\node[constructorrep] (u) at (3,-1.1) {$u$};
\node[termrep] (v8) at (3,-2.3) {$t_2$};
\node[constructorrep] (u8) at (3,-3.3) {$u_2$};

\node[termrep] (v1) at (1,-1.8) {$t_1$};
\node[termrep,\termN, \found] (v2) at (5,-1.8) {$t_3$};
\node[termrep] (v4) at (1.8,-4.1) {$t_5$};
\node[termrep,\termN, \found] (v3) at (-0.1,-4.1) {$t_4$};
\node[constructorNW={}] (u1) at (1,-2.8) {$u_1$};
\node[constructorW={}] (u4) at (1.8,-5.2) {$u_5$};
\node[termrep,\termN, \found] (v7) at (3.4,-6.2) {$t_9$};
% b1
\node[termrep] (v31) at (0.6,-5.7) {$t_4$};
\node[constructorW={}] (u31) at (0.6,-6.7) {$u_4$};
\node[termrep,\termN,\found] (v91) at (0.6,-8) {$t_8$};
% b2
\node[termrep, \termN, \found] (v73) at (2.3,-8) {$t_9$};
%
% b3
\node[termrep] (v33) at (-1.6,-5.6) {$t_4$};
\node[constructorW={}] (u33) at (-1.6,-6.6) {$u_4$};
\node[termrep,\termN,\found] (v93) at (-1.6,-7.9) {$t_8$};
%
% b4
\node[termrep, \termN, \found] (v21) at (8.3,-2.8) {$t_3$};
\node[termrep] (v41) at (5.4,-5.1) {$t_5$};
\node[termrep] (v5) at (7.5,-5.1) {$t_6$};
\node[constructorNE={}] (u2) at (8.3,-3.9) {$u_3$};
\node[termrep] (v6) at (9.1,-5.1) {$t_7$};
\node[constructorGW={}] (u41) at (5.4,-6.2) {$u_5$};
\node[constructorSE={}] (u5) at (9.1,-6.2) {$u_7$};
\node[constructorGE={}] (u9) at (7.5,-6.2) {$u_6$};
\node[termrep,\termN, \found] (v72) at (7.2,-7.4) {$t_9$};
%
% b5
\node[termrep, \termN, \found] (v71) at (4.5,-8) {$t_9$};
\path[->]
(u) edge[bend right = 0] (v)
(v1) edge[bend right = -10] node[index label] {1} (u)
(v2) edge[bend right = 10] node[index label] {3} (u)
(u1) edge[bend right = 0] (v1)
(v3) edge[bend right = 0] node[index label] {1} (u1)
(v4) edge[bend right = 10] node[index label] {2} (u1)
(u4) edge[bend right = 0] (v4)
(v7) edge[bend right = -10] node[index label] {1} (u4)
(v3) edge[bend left = 30] node[index label] {3} (u1)
(v8) edge[bend right = 0] node[index label] {2} (u)
(u8) edge[bend right = 0] (v8)
(v4) edge[bend right = -10] node[index label] {1} (u8)
;
\path[->]
(u5) edge[bend right = 0] (v6)
(v72) edge[bend right = 10] node[index label] {2} (u5)
(v21) edge[bend left = 60] node[index label] {1} (u5)
(u2) edge[bend right = 0] (v21)
(v5) edge[bend right = -10] node[index label] {1} (u2)
(v6) edge[bend right = 10] node[index label] {2} (u2)
(u33) edge[bend right = 0] (v33)
(u31) edge[bend right = 0] (v31)
(v91) edge[bend right = 0] node[index label] {1} (u31)
(v93) edge[bend right = 0] node[index label] {1} (u33)
(v41) edge[bend right = -10] node[index label] {1} (u9)
(u9) edge[bend right = 0] (v5)
(u41) edge[bend right = 0] (v41)
(v72) edge[bend right = -10] node[index label] {1} (u41)
;
\path[-,color=\conC, very thick, dashed]
(v33) edge[bend right = 0] node[fill=white, font=\scriptsize,inner sep = 0.5pt] {$\icp{\trail_3}$} (v3)
(v31) edge[bend right = 0] node[fill=white, font=\scriptsize,inner sep = 0.5pt] {$\icp{\trail_1}$} (v3)
(v73) edge[bend right = 0] node[fill=white, font=\scriptsize,inner sep = 0.5pt] {$\icp{\trail_2}$} (v7)
(v71) edge[bend right = 0] node[fill=white, font=\scriptsize,inner sep = 0.5pt] {$\icp{\trail_4}$} (v7)
(v21) edge[bend right = 0] node[fill=white, font=\scriptsize,inner sep = 0.5pt] {$\icp{\trail_5}$} (v2)
;
\end{tikzpicture}
\end{center}
Each induced construction is connected to a foundation of the generator by way of a dashed line. For instance, the trail $\trail_1$ induces the construction $\icp{\trail_1}=(\neigh{u_4},t_4)$; $\trail_3$ induces the same construction: $\icp{\trail_3}=\icp{\trail_1}$. The ordered sequence of induced constructions, derived from $\ctsequence{\cgraphn',t}$, is $[\icp{\trail_1},\icp{\trail_2},\icp{\trail_3},\icp{\trail_4},\icp{\trail_5}]$ and we see that
\begin{eqnarray*}
\ctsequence{\cgraphn,v} & = & (\ctsequence{\icp{\trail_1}}\triangleleft \trail_1)\oplus (\ctsequence{\icp{\trail_2}}\triangleleft \trail_2) \oplus (\ctsequence{\icp{\trail_3}}\triangleleft \trail_3) \oplus \\
 & & (\ctsequence{\icp{\trail_4}}\triangleleft \trail_4) \oplus (\ctsequence{\icp{\trail_5}}\triangleleft \trail_5).
\end{eqnarray*}
and
\begin{eqnarray*}
\foundationsto{\cgraphn,v}  & = & [t_8,t_9,t_8,t_9,t_9,t_3,t_9] \\
                        & = & [t_8] \oplus [t_9] \oplus [t_8] \oplus [t_9]\oplus  [t_9,t_3,t_9] \\
                        & = & \foundationsto{\icp{\trail_1}}\oplus \foundationsto{\icp{\trail_2}} \oplus \foundationsto{\icp{\trail_3}} \oplus \foundationsto{\icp{\trail_4}} \oplus \foundationsto{\icp{\trail_5}}.
\end{eqnarray*}
Each induced construction, $\icp{\trail_i}$, necessarily shares one vertex, $\source{\trail_i}$, with the generator.  Notice also that induced construction $\icp{\trail_5}$ shares other vertices and arrows with the generator, such as the vertex $t_5$ and the arrow from $u_5$ to $t_5$.
\end{example}

Our next task is to define the graph obtained from a sequence of trails, $\trailSeq$. For our purposes, it is sufficient to define this graph in the context of constructions and trails.

\begin{definition}\label{defn:tr-graph}
Let $\cpair$ be a construction and let $\trailSeq=[\trail_1,\ldots,\trail_n]$ be a (possibly empty) sequence of trails in $\cgraphn$. The \textit{$\trailSeq$-graph}, denoted $\trailToGraph{\trailSeq}$, is the sub-graph of $\cgraphn$ containing precisely the arrows that occur in some trail in $\trailSeq$, their incident vertices, and, given any empty trail, $[]_t$, in $\trailSeq$, its source vertex, $t$.
\end{definition}

Before we define induced constructions, we explore one further aspect. In example~\ref{ex:generatorsInducedConsAndSplits}, a notable feature is that the induced constructions that happened to construct the same token are identical: $\icp{\trail_1}=\icp{\trail_3}$ and $\icp{\trail_2}=\icp{\trail_4}$. In part, this is because $\trail_1$ and $\trail_3$ have the same associated vertex sequence, $[t_4,u_1,t_1,u,t]$, as do $\trail_2$ and $\trail_4$, although that property alone is not sufficient, in general, for the equality of induced constructions.

\begin{example}[Distinct Induced Constructions with the Same Construct]\label{ex:distinctICwithSC}
Consider the construction $\cpair$, on the left, alongside a generator, $\genpair$, on the right:
\begin{center}
	\adjustbox{valign=t}{%
		\begin{tikzpicture}[construction,yscale=0.75]\small
		\node[termrep] (v) at (3.2,3.7) {$t$};
		\node[termrep,\termN, \found] (v1) at (1.8,2.2) {$t_1$};
		\node[termrep,\termN, \found] (v3) at (4.6,2.2) {$t_2$};
		\node[constructorNW={}] (u) at (3.2,2.8) {$u$};
		\node[constructorW={}] (u1) at (2.2,1.2) {$u_1$};
		\node[constructorGE={}] (u3) at (4.2,1.2) {$u_2$};
		\node[termrep,\termN, \found] (v4) at (3.2,0.4) {$t_3$};
		\node[constructorNE={}] (u4) at (3.2,-0.6) {$u_3$};
		\node[termrep] (v5) at (2.2,-1.5) {$t_4$};
        \node[termrep] (v6) at (4.2,-1.5) {$t_5$};
		\node[constructorW={}] (u5) at (1.6,0) {$u_4$};
		\node[constructorSE={}] (u6) at (4.8,0) {$u_5$};
		\path[->]
		(u) edge[bend right = 0] (v)
		(v1) edge[bend right = -10] node[index label] {1} (u)
		(v3) edge[bend right = 10] node[index label] {2} (u)
		(u1) edge[bend right = 0] (v1)
		(u3) edge[bend right = 0] (v3)
		(v4) edge[bend right = -10] node[index label] {1} (u1)
		(v4) edge[bend right = 10] node[index label] {1} (u3)
		(u4) edge[bend right = 0] (v4)
		(v1) edge[bend right = 25] node[index label] {1} (u5)
		(v3) edge[bend right = -25] node[index label] {1} (u6)
		(u6) edge[bend right = 00] (v6)
		(u5) edge[bend right = 00] (v5)
		(v5) edge[bend right = 10] node[index label] {1} (u4)
        (v6) edge[bend right = -10] node[index label] {2} (u4)
		;
		\end{tikzpicture}}
\hspace{2.7cm}
\adjustbox{valign=t}{%
\begin{tikzpicture}[construction, yscale=0.75]\small
\node[termrep] (v) at (3.2,3.7) {$t$};
\node[termrep] (v1) at (1.8,2.2) {$t_1$};
\node[termrep] (v3) at (4.6,2.2) {$t_2$};
\node[constructorNW={}] (u) at (3.2,2.8) {$u$};
\node[constructorW={}] (u1) at (2.2,1.2) {$u_1$};
\node[constructorGE={}] (u3) at (4.2,1.2) {$u_2$};
\node[termrep,\termN, \found] (v4) at (3.2,0.4) {$t_3$};
\path[->]
(u) edge[bend right = 0] (v)
(v1) edge[bend right = -10] node[index label] {1} (u)
(v3) edge[bend right = 10] node[index label] {2} (u)
(u1) edge[bend right = 0] (v1)
(u3) edge[bend right = 0] (v3)
(v4) edge[bend right = -10] node[index label] {1} (u1)
(v4) edge[bend right = 10] node[index label] {1} (u3)
;
\end{tikzpicture}}
\end{center}
Here, $\ctsequence{\cgraphn',t}=[\trail_1,\trail_2]$ where
\begin{eqnarray*}
\trail_1 & = & t_3\arrow[1] u_1 \arrow t_1 \arrow[1] u \arrow t\\
\trail_2 & = & t_3\arrow[1] u_2\arrow t_2\arrow[2] u \arrow t.
\end{eqnarray*}
The trails $\trail_1$ and $\trail_2$ are both sourced on $t_3$. However, they give rise to different induced constructions, $(\icgp{\trail_1},t_3)$, on the left, and $(\icgp{\trail_2},t_3)$, on the right:
\begin{center}
\begin{tikzpicture}[construction,yscale=0.75]\small
\node[termrep,\termN, \found] (v1) at (1.8,2.2) {$t_1$};
\node[termrep] (v3) at (4.6,2.2) {$t_2$};
\node[constructorGE={}] (u3) at (4.2,1.2) {$u_2$};
\node[termrep,\termN, \found] (v4) at (3.2,0.4) {$t_3$};
\node[constructorNE={}] (u4) at (3.2,-0.6) {$u_3$};
\node[termrep] (v5) at (2.2,-1.5) {$t_4$};
\node[constructorW={}] (u5) at (1.6,0) {$u_4$};
\node[constructorSE={}] (u6) at (4.8,0) {$u_5$};
 \node[termrep] (v6) at (4.2,-1.5) {$t_5$};
\path[->]
(u3) edge[bend right = 0] (v3)
(v4) edge[bend right = 10] node[index label] {1} (u3)
(u4) edge[bend right = 0] (v4)
(v1) edge[bend right = 25] node[index label] {1} (u5)
(v3) edge[bend right = -25] node[index label] {1} (u6)
(u6) edge[bend right = 0] (v6)
(u5) edge[bend right = 0] (v5)
(v5) edge[bend right = 10] node[index label] {1} (u4)
(v6) edge[bend right = -10] node[index label] {2} (u4)
;
\end{tikzpicture}
\hspace{2.3cm}
\begin{tikzpicture}[construction,yscale=0.75]
\node[termrep] (v1) at (1.8,2.2) {$t_1$};
\node[termrep,\termN, \found] (v3) at (4.6,2.2) {$t_2$};
\node[constructorW={}] (u1) at (2.2,1.2) {$u_1$};
\node[termrep,\termN, \found] (v4) at (3.2,0.4) {$t_3$};
\node[constructorNE={}] (u4) at (3.2,-0.6) {$u_3$};
\node[termrep] (v5) at (2.2,-1.5) {$t_4$};
\node[constructorW={}] (u5) at (1.6,0) {$u_4$};
\node[constructorSE={}] (u6) at (4.8,0) {$u_5$};
 \node[termrep] (v6) at (4.2,-1.5) {$t_5$};
\path[->]
(u1) edge[bend right = 0] (v1)
(v4) edge[bend right = -10] node[index label] {1} (u1)
(u4) edge[bend right = 0] (v4)
(v1) edge[bend right = 25] node[index label] {1} (u5)
(v3) edge[bend right = -25] node[index label] {1} (u6)
(u6) edge[bend right = 0] (v6)
(u5) edge[bend right = 0] (v5)
(v5) edge[bend right = 10] node[index label] {1} (u4)
(v6) edge[bend right = -10] node[index label] {2} (u4)
;
\end{tikzpicture}
\end{center}

Essentially, these two induced constructions are different due to the definition of a trail extension sequence, from which the foundations are derived. In particular, given any trail, $\trail$, no trail in $\tesequence{\trail}$ shares an arrow with $\trail$. The trail extension sequence for $\trail_1$ includes the trail $\trail_1'=t_3\arrow[1] u_2 \arrow t_2 \arrow[1] u_5 \arrow t_5 \arrow[2] u_3 \arrow t_3$, which \textit{shares arrows} with $\trail_2$. This implies that $\trail_1'$ is not in $\tesequence{\trail_2}$. In fact, no trail in $\tesequence{\trail_2}$ includes the arrow $t_3\arrow[1] u_2$, for example, since this arrow \textit{is part of} $\trail_2$. Thus, the two induced constructions, whilst both constructing $t_3$, are necessarily distinct.
\end{example}

In summary,  knowing only the source of a trail, $\trail$, in a generator, $\genpair$, is not sufficient to identify its induced construction: the whole of $\trail$ must be considered as well as the entirety of $\cgraphn$. We now define the induced construction of a well-formed trail, $\trail$, whose target is $t$, the construct of $\cpair$.

\begin{definition}\label{defn:inducedConstruction}
Let $\cpair$ be a construction containing a well-formed trail, $\trail$, with target $t$. The \textit{induced construction graph} and the \textit{induced construction} of $\trail$, denoted $\icgp{\trail}$ and $\icp{\trail}$ respectively, are defined as follows:
\begin{enumerate}
\item $\icgp{\trail}$ is the $\tesequence{\trail}$-graph, and
\item $\icp{\trail}=(\icgp{\trail},\source{\trail})$.
\end{enumerate}

\end{definition}
\noindent We now prove that induced constructions are indeed constructions.

\begin{lemma}
Let $\cpair$ be a construction containing a well-formed trail, $\trail$, with target $t$. The pair $(\icgp{\trail},\source{\trail})$ is a construction.
\end{lemma}

\begin{proof}
The vertex $\source{\trail}$ is a token since $\trail$ is well-formed. Clearly, $\icgp{\trail}$ is finite and uni-structured since it is a subgraph of $\cgraphn$. Let $v$ be a vertex in  $\icgp{\trail}$. By the definition of an induced construction, this implies that there is a trail, $\trail'$, in $\tesequence{\trail,\cpair}$ whose associated vertex sequence includes $v$, either because it is the source of an arrow in $\trail'$ or it is the source of $\trail'$. By the definition of a trail extension sequence, the target of $\trail'$ is $\source{\trail}$. Hence, there is a trail from $v$ to $\source{\trail}$, as required. Hence $(\icgp{\trail},\source{\trail})$ is a construction.
\end{proof}

\begin{example}[Choices of Generator and Induced Construction Sequences]\label{ex:choicesOfGenAndICS}
Returning to the graphical representation of vectors, in Section~\ref{sec:constructions:AMEGROV} we saw the following construction of $t_3$ in the entailment layer:
\begin{center}
\vectorsCon
\end{center}
There are five different choices of generator of this construction: the trivial construction of $t_3$, $(\neigh{u},t_3)$, $(\neigh{u_1}\cup \neigh{u},t_3)$, $(\neigh{u_2}\cup \neigh{u},t_3)$, and $(\neigh{u_1}\cup \neigh{u_2}\cup \neigh{u},t_3)$.  The generator $(\neigh{u_1}\cup \neigh{u},t_3)$ is shown below, together with its three induced constructions, where the dashed red lines identify, for each induced construction,  its construct and the associated foundation of the generator:
\begin{center}
\begin{tikzpicture}[construction]
\begin{scope}[rotate=-90]
\node[termN={$t_3$},inner sep = 1.5pt] (vc1) at (3.2,7.6) {\scalebox{.5}{\vectorVis{(1,1.5)}{0.4}}};
\node[termN={$t_3$},inner sep = 1.5pt] (v) at (3.2,5.6) {\scalebox{.5}{\vectorVis{(1,1.5)}{0.4}}};
\node[termpos={$t_4$}{20}{0.17cm},inner sep = 1.5pt] (v1) at (2.4,2.6) {\scalebox{.5}{\vectorVis{(3,4)}{0.4}}};
\node[termpos={$t_4$}{20}{0.17cm},inner sep = 1.5pt] (v1c) at (2.4,-1.4) {\scalebox{.5}{\vectorVis{(3,4)}{0.4}}};
\node[termpos={$t_5$}{-20}{0.17cm},inner sep = 1.5pt] (v2) at (4.0,2.6) {\scalebox{.5}{\vectorVis{(2,2.5)}{0.4}}};
\node[constructorEpos={$\GRVcvecsub$}{60}{0.17cm}] (u) at (3.2,4.1) {$u$};
\node[constructorEpos={$\GRVcvecadd$}{80}{0.18cm}] (u1) at (2.4,-2.8) {$u_2$};
\node[constructorEpos={$\GRVcvecadd$}{80}{0.18cm}] (u2) at (4,1.2) {$u_1$};
\node[termpos={$t_2$}{160}{0.17cm},inner sep = 1.5pt] (v3) at (2.4,-4.5) {\scalebox{.5}{\vectorVis{(2,3)}{0.4}}};
\node[termpos={$t_1$}{200}{0.17cm},inner sep = 1.5pt] (v4) at (4,-4.5) {\scalebox{.5}{\vectorVis{(1,1)}{0.4}}};
\node[termpos={$t_1$}{200}{0.17cm},inner sep = 1.5pt] (v4c) at (4,-0.4) {\scalebox{.5}{\vectorVis{(1,1)}{0.4}}};
\node[termpos={$t_1$}{200}{0.17cm},inner sep = 1.5pt] (v4c2) at (4,-2.4) {\scalebox{.5}{\vectorVis{(1,1)}{0.4}}};
\end{scope}
\path[->]
(u) edge[bend right = 0] (v)
(v1) edge[bend right = -10] node[index label] {1} (u)
(v2) edge[bend right = 10] node[index label] {2} (u)
(v3) edge[bend right = -10] node[index label] {2} (u1)
(v4c) edge[bend right = 10] node[index label] {1} (u2)
(u1) edge[bend right = 0] (v1c)
(u2) edge[bend right = 0] (v2)
(v4) edge[bend right = 10] node[index label] {1} (u1)
(v) edge[in = -80, out = 250, looseness=0.67] node[index label] {2} (u2);
\path[-, dashed, very thick, color=\conC]
(vc1) edge[bend right = 0] (v)
(v1c) edge[bend right = 0] (v1)
(v4c2) edge[bend right = 0] (v4c);
\end{tikzpicture}\vspace{-0.55cm}
\end{center}
The induced constructions can be ordered, respecting the ordering of the trails in the complete trail sequence of the generator -- $\ctsequence{\neigh{u_1}\cup \neigh{u},t_3}=[\trail_1,\trail_2,\trail_3]$, where $\trail_1=t_4\arrow[1] u \arrow t_3$, $\trail_2= t_1\arrow[1] u_1\arrow t_5\arrow[2] u \arrow t_3$, and $\trail_3=t_3\arrow[2] u_1\arrow t_5\arrow[2] u \arrow t_3$ -- yielding an \textit{induced construction sequence}:
\begin{displaymath}
  \ics_1 = [(\neigh{u_2}, t_4), (\cgraphn_1,t_1), (\cgraphn_3, t_3)]
\end{displaymath}
where $\cgraphn_1$ and $\cgraphn_3$ are the trivial constructions of $t_1$ and $t_3$ respectively.{\pagebreak}  An alternative generator, $(\neigh{u},t_3)$, with its two induced constructions -- again with their constructs joined by dashed lines to generator foundations -- is given \begin{samepage}below:
\begin{center}
\begin{tikzpicture}[construction]
\begin{scope}[rotate=-90]
\node[termN={$t_3$},inner sep = 1.5pt] (vc1) at (3.2,7.6) {\scalebox{.5}{\vectorVis{(1,1.5)}{0.4}}};
\node[termN={$t_3$},inner sep = 1.5pt] (v) at (3.2,5.6) {\scalebox{.5}{\vectorVis{(1,1.5)}{0.4}}};
\node[termpos={$t_4$}{20}{0.17cm},inner sep = 1.5pt] (v1) at (2.4,2.6) {\scalebox{.5}{\vectorVis{(3,4)}{0.4}}};
\node[termpos={$t_4$}{20}{0.17cm},inner sep = 1.5pt] (v1c) at (2.4,-1.4) {\scalebox{.5}{\vectorVis{(3,4)}{0.4}}};
\node[termpos={$t_5$}{-20}{0.17cm},inner sep = 1.5pt] (v2) at (4.0,0.6) {\scalebox{.5}{\vectorVis{(2,2.5)}{0.4}}};
\node[termpos={$t_5$}{-20}{0.17cm},inner sep = 1.5pt] (v2c) at (4.0,2.6) {\scalebox{.5}{\vectorVis{(2,2.5)}{0.4}}};
\node[constructorEpos={$\GRVcvecsub$}{60}{0.17cm}] (u) at (3.2,4.1) {$u$};
\node[constructorEpos={$\GRVcvecadd$}{80}{0.18cm}] (u1) at (2.4,-2.8) {$u_2$};
\node[constructorEpos={$\GRVcvecadd$}{80}{0.18cm}] (u2) at (4,-0.8) {$u_1$};
\node[termpos={$t_2$}{160}{0.17cm},inner sep = 1.5pt] (v3) at (2.4,-4.5) {\scalebox{.5}{\vectorVis{(2,3)}{0.4}}};
\node[termpos={$t_1$}{200}{0.17cm},inner sep = 1.5pt] (v4) at (4,-4.5) {\scalebox{.5}{\vectorVis{(1,1)}{0.4}}};
%\node[termpos={$t_1$}{200}{0.17cm}] (v4c) at (4,-0.4) {\scalebox{.5}{\vectorVis{(1,1)}{0.4}}};
\node[termpos={$t_1$}{200}{0.17cm},inner sep = 1.5pt] (v4c2) at (4,-2.4) {\scalebox{.5}{\vectorVis{(1,1)}{0.4}}};
\end{scope}
\path[->]
(u) edge[bend right = 0] (v)
(v1) edge[bend right = -10] node[index label] {1} (u)
(v2c) edge[bend right = 10] node[index label] {2} (u)
(v3) edge[bend right = -10] node[index label] {2} (u1)
(u1) edge[bend right = 0] (v1c)
(u2) edge[bend right = 0] (v2)
(v4) edge[bend right = 10] node[index label] {1} (u1)
(v4c2) edge[bend right = 10] node[index label] {1} (u2)
(vc1) edge[in = -85, out = 235, looseness=0.39] node[index label] {2} (u2);
\path[-, dashed, very thick, color=\conC]
(v2) edge[bend right = 0] (v2c)
(v1c) edge[bend right = 0] (v1);
\end{tikzpicture}\vspace{-0.6cm}
\end{center}
\end{samepage}
Here, the induced construction sequence is
\begin{displaymath}
  \ics_2 = [(\neigh{u_2}, t_4), (\neigh{u_1}, t_5)].
\end{displaymath}
\end{example}

We can now define the \textit{induced construction sequence}\footnote{It might be anticipated that definition~\ref{defn:inducedConstruction} would be presented with $\trailSeq$ being restricted to $\ctsequence{\cgraphn',t}$, where $\genpair$ is some generator of $\cpair$. Appendix~\ref{sec:app:constructions} makes use of definition~\ref{defn:inducedConstruction} in cases where $\trailSeq\neq \ctsequence{\cgraphn,t}$, on the way to proving theorems~\ref{thm:splitPreservesCompleteTrailSequence} and~\ref{thm:splitPreservesfoundations}. Therefore, we cannot restrict definition~\ref{defn:inducedConstruction} to $\ctsequence{\cgraphn',t}$.} that arises from a sequence, $\trailSeq$, of trails: simply, the induced constructions are ordered in alignment with the trails in $\trailSeq$.

\begin{definition}\label{defn:inducedConstruction}
Let $\cpair$ be a construction and let $\trailSeq=[\trail_1,\ldots,{\trail_n}]$ be a sequence of well-formed trails, each of which targets $t$. The \textit{induced construction sequence} of $\trailSeq$, denoted $\icsequence{\trailSeq}$, is defined to be:
$$\icsequence{\trailSeq}=[\icp{\trail_1},\ldots,\icp{\trail_n}].$$
\end{definition}

\begin{definition}\label{defn:split}
Let $\cpair$ be a construction. A \textit{split} of $\cpair$ is a pair, $\csplit$, where
\begin{enumerate}
\item $\genpair$ is a generator of $\cpair$, and
\item $\ics$ is the induced construction sequence obtained by extending trails in $\ctsequence{\cgraphn',t}$ in $\cpair$:  $\ics=\icsequence{\ctsequence{\cgraphn',t},\cpair}$.
\end{enumerate}
%We write $\cpair\splitarrow \csplit$ to mean that $\csplit$ is a split of $\cpair$.
\end{definition}

\noindent As noted earlier, a key use of splits arises in our theoretical approach to structural transformations. In Section~\ref{sec:structuralTransformations}, it will be taken that when we split a construction, the complete trail sequence and the foundation sequence for $\cpair$ can be readily derived from the generator and induced construction sequence (theorems~\ref{thm:splitPreservesCompleteTrailSequence} and~\ref{thm:splitPreservesfoundations}).

\begin{example}[Deriving a Construction's Complete Trail Sequence and Foundation Token-Sequence from a Split]
In example~\ref{ex:choicesOfGenAndICS}, we saw two generators and their induced constructions, which give rise to the following \textit{splits}:  %
\begin{eqnarray*}
 & & \big((\neigh{u_1}\cup \neigh{u},t_3), [(\neigh{u_2}, t_4), (\cgraphn_1,t_1), (\cgraphn_3, t_3)]\big), \qquad \textup{and} \\
 & & \big((\neigh{u},t_3),[(\neigh{u_2}, t_4), (\neigh{u_1}, t_5)]\big).
\end{eqnarray*}
Here we focus on the second of these and show how the complete trail sequence and foundation token-sequence for the original construction,  $(\cgraphn,t_3)$, can be derived from the split.  First, we remark that the complete trail sequence $\ctsequence{\cgraphn,t_3}$ \textit{exactly captures} how $t_3$ is constructed in $\cgraphn$. How does this relate to the generator, the induced constructions and their complete trail sequences? We see that:
\begin{eqnarray*}
  \ctsequence{\cgraphn,t_3} &=& \big(\ctsequence{(\neigh{u_2}, t_4)}\triangleleft (t_4\arrow[1] u \arrow t_3)\big) \oplus \\
                            & & \big(\ctsequence{(\neigh{u_1}, t_5)}\triangleleft (t_5\arrow[2] u \arrow t_3)\big).
\end{eqnarray*}
Here, the two trails, $\,t_4\arrow[1] u \arrow t_3\,$ and $\,t_5\arrow[2] u \arrow t_3$, form the complete trail sequence,\break$\ctsequence{(\neigh{u},t_3)}$, of the generator. In both the generator and original construction, the construct, $t_3$, is constructed from only $t_4$ and $t_5$ via the configurator $u$. Now, in the construction $(\cgraphn,t_3)$, the token $t_4$ is built from $t_1$ and $t_2$, via $u_2$; this is precisely what we see in the associated induced construction, $(\neigh{u_2}, t_4)$. The two trails in $\ctsequence{(\neigh{u_2}, t_4)}$, namely $t_1\arrow[1] u_2 \arrow t_4$ and $t_2\arrow[2] u_2 \arrow t_4$ directly embody this observation about how $t_4$ is constructed within $(\cgraphn,t_3)$. In addition, $t_5$ is built from $t_1$ and $t_3$ as can be seen in the other induced construction, $(\neigh{u_1}, t_5)$, and again we have it that the trails in $\ctsequence{(\neigh{u_1}, t_5)}$ reflect this. In essence, the relationship between the construction sequences across the original construction, generator and induced constructions exemplifies that the splitting operation preserves all knowledge about how $t_3$ is constructed and, in the induced constructions, preserves how their specific constructs are constructed \textit{in the context of the original construction.}

Regarding the foundations, we see that
\begin{eqnarray*}
  \foundationsto{\cgraphn,t_3} & = & [t_1,t_2,t_1,t_3] \\
   &=& \foundationsto{\neigh{u_2}, t_4} \oplus \foundationsto{\neigh{u_1}, t_5}.
\end{eqnarray*}
In both cases, the foundation token-sequence, $\foundationsto{\cgraphn,t_3}$, of the original construction, $(\cgraphn,t_3)$, is directly obtained by concatenating the foundation token-sequences of the induced constructions in the split.
\end{example}

We now present the first major theorem of Section~\ref{sec:constructions}: splitting a construction, $\cpair$, using a generator enables the recreation of $\cpair$'s complete trail sequence, $\ctsequence{\cgraphn,t}$ from the generator and the associated induced construction sequence. As a result, we can also recreate $\cpair$'s foundation token-sequence using the induced construction sequence. These results are embodied in theorems~\ref{thm:splitPreservesCompleteTrailSequence} and~\ref{thm:splitPreservesfoundations}; a full proof of theorem~\ref{thm:splitPreservesCompleteTrailSequence} can be found in Appendix~\ref{sec:app:constructions}.

\begin{theorem}\label{thm:splitPreservesCompleteTrailSequence}
Let $\cpair$ be a construction with  split $\csplit$. Given $\ctsequence{\cgraphn',t} = {\linebreak} [\trail_1,\ldots,\trail_n]$ and $\ics=[\icp{\trail_1},\ldots,\icp{\trail_n}]$, it is the case that
\begin{displaymath}
\ctsequence{\cgraphn,t} = \big(\ctsequence{\icp{\trail_1}}\triangleleft \trail_1\big) \oplus\cdots \oplus \big(\ctsequence{\icp{\trail_n}}\triangleleft\trail_n\big).
\end{displaymath}
\end{theorem}

\begin{proof}[Proof Sketch]
Recall that, for any construction, $(\cgraphn'',t'')$, its complete trail sequence is the trail extension sequence of $[]_t$. The key parts of the proof strategy establish the following facts, beginning with trail extension sequences:
\begin{enumerate}
\item The trail extension sequence for $[]_t$ in $\cpair$ can be readily obtained from the trail extension sequences arising from the trails in $\ctsequence{\cgraphn',t}=[\trail_1,\ldots,\trail_n]$:
\begin{displaymath}
\tesequence{[]_t,\cpair} = \big(\tesequence{\trail_1,\cpair}\triangleleft \trail_1\big) \oplus\cdots \oplus \big(\tesequence{\trail_n,\cpair}\triangleleft\trail_n\big).
\end{displaymath}
%See theorem~\ref{thm:trailExtensionSequencesForSplitsMatch} in Appendix~\ref{sec:app:constructions}.
%
\item For each trail, $\trail_i$, in $\ctsequence{\cgraphn',t}=[\trail_1,\ldots,\trail_n]$ the trail extension sequence{\linebreak} $\tesequence{\trail_i,\cpair}$ is the same as $\ctsequence{\icp{\trail_i}}$:
\begin{displaymath}
  \tesequence{\trail_i,\cpair}=\ctsequence{\icp{\trail_i}}.
\end{displaymath}
\item Now, using (2) and the fact that $\tesequence{[]_t,\cpair}=\ctsequence{\cgraphn,t}$, substituting into (1), we obtain:
\begin{displaymath}
\ctsequence{\cgraphn,t} = \big(\ctsequence{\icp{\trail_1}}\triangleleft \trail_1\big) \oplus\cdots \oplus \big(\ctsequence{\icp{\trail_n}}\triangleleft\trail_n\big).
\end{displaymath}
as required.
\end{enumerate}
\end{proof}

\begin{theorem}\label{thm:splitPreservesfoundations}
Let $\cpair$ be a construction with  split $\csplit$. Given $\ctsequence{\cgraphn',t}=[\trail_1,\ldots,\trail_n]$ and $\ics=[\icp{\trail_1},\ldots,\icgp{\trail_n}]$, it is the case that
\begin{displaymath}
  \foundationsto{\cgraphn,t}=\foundationsto{\icp{\trail_1}}\oplus \cdots \oplus \foundationsto{\icp{\trail_n}}.
\end{displaymath}
\end{theorem}

\begin{proof}
Theorem~\ref{thm:splitPreservesCompleteTrailSequence} gives us
\begin{displaymath}
\ctsequence{\cgraphn,t} = \big(\ctsequence{\icp{\trail_1}}\triangleleft \trail_1\big) \oplus\cdots \oplus \big(\ctsequence{\icp{\trail_n}}\triangleleft\trail_n\big).
\end{displaymath}
By definition~\ref{defn:foundations}, the foundation token-sequence is the source sequence of the complete trail sequence, that is $\foundationsto{\cgraphn,t}=\ssequence{\ctsequence{\cgraphn,t}}$, so it follows that
\begin{displaymath}
\foundationsto{\cgraphn,t}  =  \ssequence{\ctsequence{\icp{\trail_1}}\triangleleft \trail_1} \oplus\cdots \oplus \ssequence{\ctsequence{\icp{\trail_n}}\triangleleft\trail_n}
\end{displaymath}
For each trail, $\trail_i$, in $\ctsequence{\cgraphn',t}=[\trail_1,\ldots,\trail_n]$, we can remove the right product operation used above, since concatenating a trail onto the target of another trail does not alter the source, yielding
\begin{displaymath}
\foundationsto{\cgraphn,t}  =  \ssequence{\ctsequence{\icp{\trail_1}}} \oplus\cdots \oplus \ssequence{\ctsequence{\icp{\trail_n}}}
\end{displaymath}
\item By the definition of a foundation token-sequence, for each trail, $\trail_i$, in $\ctsequence{\cgraphn',t}=[t_1,\ldots,t_n]$,
    \begin{displaymath}
      \foundationsto{\icp{\trail_i}}=\ssequence{\ctsequence{\icp{\trail_i}}}.
    \end{displaymath}
Thus, we deduce that
\begin{displaymath}
  \foundationsto{\cgraphn,t}=\foundationsto{\icp{\trail_1}}\oplus \cdots \oplus \foundationsto{\icp{\trail_n}}.
\end{displaymath}
as required.
\end{proof}

Our last task in this section on splitting constructions is to show that a split can be used to form the original construction. In particular, we will show that the original construction's graph, $\cgraphn$, is the union of the generator's graph, $\cgraphn'$, and all of the induced construction graphs, captured by theorem~\ref{thm:SplitsCoverConstruction}. The proof uses three facts concerning trails and the graphs to which they give rise: given any sequences of trails, $\trailSeq$ and $\trailSeq'$, and trail, $\trail$, it is the case that
\begin{enumerate}
\item[-] $\trailToGraph{\trailSeq\oplus \trailSeq'}=\trailToGraph{\trailSeq}\cup \trailToGraph{\trailSeq'}$,
\item[-] $\trailToGraph{\trailSeq\triangleleft \trail}=\trailToGraph{\trailSeq}\cup \trailToGraph{[\trail]}$, and
\item[-] $\trailToGraph{[\trail_1,...,\trail_i]}\cup \trailToGraph{[\trail_{i+1}]}= \trailToGraph{[\trail_1,...\trail_i,\trail_{i+1}]}$.
\end{enumerate}
It also uses the fact that for any construction, $\cpair$, it is the case that $\cgraphn=\trailToGraph{\ctsequence{\cpair}}$, established in Appendix~\ref{sec:app:constructions}, corollary~\ref{cor:DTSforTopGeneratesTop}.

\begin{theorem}\label{thm:SplitsCoverConstruction}
Let $\cpair$ be a construction with split $\csplit$, where $\ctsequence{\cgraphn',t}=[\trail_1,...,\trail_n]$. Then
\begin{displaymath}
  \cgraphn= \cgraphn' \cup \icgp{\trail_1}\cup \cdots \cup \icgp{\trail_n}.
\end{displaymath}
\end{theorem}

\begin{proof}
By theorem~\ref{thm:splitPreservesCompleteTrailSequence}, we know that
\begin{displaymath}
  \ctsequence{\cgraphn,t}= (\ctsequence{\icp{\trail_1}}\triangleleft \trail_1) \oplus ... \oplus (\ctsequence{\icp{\trail_n}}\triangleleft \trail_n).
\end{displaymath}
Converting each complete trail sequence into a graph, we see that
\begin{eqnarray*}
\trailToGraph{\ctsequence{\cgraphn,t}}  & = & \trailToGraph{(\ctsequence{\icp{\trail_1}}\triangleleft \trail_1) \oplus ... \oplus (\ctsequence{\icp{\trail_n}}\triangleleft \trail_m)}\\
%& = & \trailToGraph{(\ctsequence{\icp{\trail_1}}\triangleleft \trail_1)} \cup ... \cup \trailToGraph{(\ctsequence{\icp{\trail_n}}\triangleleft \trail_m)}
 & = & \trailToGraph{\ctsequence{\icp{\trail_1}}}\cup \trailToGraph{[\trail_1]} \cup ...  \cup \\ & & \trailToGraph{\ctsequence{\icp{\trail_n}}} \cup \trailToGraph{[\trail_n]}\\
& = & \trailToGraph{[\trail_1]} \cup ... \cup \trailToGraph{[\trail_n]}\cup \\ & &  \trailToGraph{\ctsequence{\icp{\trail_1}}}\cup  ... \cup \trailToGraph{\ctsequence{\icp{\trail_n}}}\\
 & = & \trailToGraph{[\trail_1,...,\trail_n]}\cup \trailToGraph{\ctsequence{\icp{\trail_1}}}\cup  ... \cup \trailToGraph{\ctsequence{\icp{\trail_n}}} \\
& = & \trailToGraph{\ctsequence{\cgraphn',t}}\cup \trailToGraph{\ctsequence{\icp{\trail_1}}}\cup  ... \cup \trailToGraph{\ctsequence{\icp{\trail_n}}}
%& = & \cgraphn'\cup \trailToGraph{\ctsequence{\icp{\trail_1}}}\cup  ... \cup \trailToGraph{\ctsequence{\icp{\trail_n}}}\\
\end{eqnarray*}
Using corollary~\ref{cor:DTSforTopGeneratesTop}, and noting that $\icp{\trail_i}=(\icgp{\trail_i},\source{\trail_i})$, we have
\begin{displaymath}
\cgraphn =  \cgraphn' \cup \icgp{\trail_1}\cup \cdots \cup \icgp{\trail_n}
\end{displaymath}
as required.
\end{proof}

In essence, theorems~\ref{thm:splitPreservesCompleteTrailSequence} to~\ref{thm:SplitsCoverConstruction} combine to show that when we split a construction, we do not lose any of its structure.

\subsection{Decompositions}\label{sec:cons:decompositions}

Splits support the decomposition of a construction into a generator and a sequence of induced constructions. Our approach to structural transformations requires that the splitting operation can also be applied to the induced constructions. This motivates the definition of a \textit{decomposition}, which embodies the iterative splitting of constructions. The concept of a decomposition is formalised using a directed rooted tree, where the root is labelled by the generator of a split or, simply, the original construction if no splitting has taken place. The root's adjacent vertices are, in turn, the roots of decompositions of the induced constructions. Firstly, we include an example.

\begin{example}[Splitting Splits]\label{ex:SplitingSplits}
Consider the following construction, $(\cgraphn,t_3)$, drawn from the graphical representation of vectors representational system; the enclosing contours show one way of decomposing $(\cgraphn,t_3)$ using splits. The initial generator, $(\neigh{u},t_3)$, is enclosed by the thick (green) contour. The remaining contours enclose the other constructions in the decomposition.
\begin{center}
	\begin{tikzpicture}[construction]
	\begin{scope}[rotate=-90]
	\node[termE={$t_3$}] (v) at (3.7,5.6) {\scalebox{.5}{\vectorVis{(1,1.5)}{0.4}}};
	\node[termpos={$t_4$}{20}{0.17cm}] (v1) at (2.4,2.6) {\scalebox{.5}{\vectorVis{(3,4)}{0.4}}};
	\node[termpos={$t_5$}{-20}{0.17cm}] (v2) at (4.5,2.6) {\scalebox{.5}{\vectorVis{(2,2.5)}{0.4}}};
	\node[constructorEpos={$\GRVcvecsub$}{60}{0.17cm}] (u) at (3.2,4.1) {$u$};
	\node[constructorEpos={$\GRVcvecadd$}{80}{0.18cm}] (u1) at (2.4,1.2) {$u_2$};
	\node[constructorEpos={$\GRVcvecadd$}{80}{0.18cm}] (u2) at (4.5,1.2) {$u_1$};
	\node[termpos={$t_2$}{160}{0.17cm}] (v3) at (2.4,-0.5) {\scalebox{.5}{\vectorVis{(2,3)}{0.4}}};
	\node[termpos={$t_1$}{200}{0.17cm}] (v4) at (4.5,-0.5) {\scalebox{.5}{\vectorVis{(1,1)}{0.4}}};
	\node[constructorEpos={$\GRVcvecsub$}{80}{0.18cm}] (u0) at (2.4,-2.2) {$u_3$};
	\node[termpos={$t_6$}{160}{0.17cm}] (v6) at (2.4,-3.9) {\scalebox{.5}{\vectorVis{(2.5,3.5)}{0.4}}};
	\node[termpos={$t_7$}{200}{0.17cm}] (v7) at (4.5,-3.9) {\scalebox{.5}{\vectorVis{(0.5,0.5)}{0.4}}};
	\end{scope}
	%green:
	\coordinate[above right = 0.2cm and 0.4cm of v] (var);
	\coordinate[below right = 0.2cm and 0.4cm of v] (vbr);
	\coordinate[below right = 0.15cm and 0.3cm of v2] (v2br);
	\coordinate[below left = 0.15cm and 0.2cm of v2] (v2bl);
	\coordinate[above left = 0.25cm and 0.2cm of v1] (v1al);
	\coordinate[above right = 0.25cm and 0.3cm of v1] (v1ar);
	\draw[rounded corners,ultra thick,darkgreen,draw opacity=.7,text=black] (var)--(vbr)--(v2br)--(v2bl)--(v1al)--(v1ar)--node[yshift =0.18cm,xshift=0.18cm] {\rotatebox{-24}{$(\neigh{u},t_3)$}} cycle;
	%blue
	\coordinate[above right = 0.15cm and 0.35cm of v1] (v1ar);
	\coordinate[below right = 0.1cm and 0.35cm of v1] (v1br);
	\coordinate[above right = 0.1cm and 0.3cm of v4] (v4ar);
	\coordinate[below right = 0.2cm and 0.2cm of v4] (v4br);
	\coordinate[below left = 0.2cm and 0.4cm of v4] (v4bl);
	\coordinate[above left = 0.2cm and 0.4cm of v3] (v3al);
	\draw[rounded corners, thick,darkblue,draw opacity=.7,text=black] (v1ar)--(v1br)--(v4ar)--(v4br)--(v4bl)--(v3al)--node[yshift =0.24cm,xshift=-0.1cm] {\rotatebox{0}{$(\neigh{u_2},t_4)$}}cycle;
	%red1
	\coordinate[above left = 0.1cm and 0.1cm of v] (r0);
	\coordinate[above right = 0.1cm and 0.3cm of v] (r1);
	\coordinate[below right = 0.1cm and 0.3cm of v] (r2);
	\coordinate[below right = 0.32cm and 0.4cm of v2] (r3);
	\coordinate[below left = 0.32cm and 0.5cm of v4] (r4);
	\coordinate[above left = 0.55cm and 0.5cm of v4] (r5);
	\coordinate[above right = 0.55cm and 1cm of v4] (r6);
	\coordinate[above left = -0.5cm and 0.7 cm of v] (r7);
	\draw[rounded corners,darkred,draw opacity=.7,text=black] (r0)--(r1)--(r2)--(r3)--node[yshift =-0.24cm,xshift=0.2cm] {\rotatebox{0}{$(\neigh{u_1},t_5)$}}(r4)--(r5)--(r6)--(r7)--cycle;
	%red2
	\coordinate[above right = 0.1cm and 0.1cm of v4] (rr1);
	\coordinate[below right = 0.1cm and 0.1cm of v4] (rr2);
	\coordinate[below left = 0.1cm and 0.3cm of v4] (rr3);
	\coordinate[above left = 0.1cm and 0.3cm of v4] (rr4);
	\draw[rounded corners,darkred,draw opacity=.7,text=black] (rr1)--(rr2)--(rr3)--(rr4)--node[yshift =0.18cm,xshift=-0.0cm] {\rotatebox{0}{$(g_1,t_1)$}}cycle;
	%red3
	\coordinate[above right = 0.1cm and 0.1cm of v3] (v3ar);
	\coordinate[below right = 0.1cm and 0.1cm of v3] (v3br);
	\coordinate[below left = 0.1cm and 0.6cm of v3] (v3bl);
	\coordinate[below right = 0.2cm and 0.1cm of v7] (v7br);
	\coordinate[below left = 0.2cm and 0.35cm of v7] (v7bl);
	\coordinate[above left = 0.15cm and 0.35cm of v6] (v6al);
	\draw[rounded corners,darkred,draw opacity=.7,text=black] (v3ar)--(v3br)--(v3bl)--(v7br)--(v7bl)--(v6al)--node[yshift =0.24cm,xshift=-0.2cm] {\rotatebox{0}{$(\neigh{u_3},t_2)$}} cycle;
	\path[->]
(u) edge[bend right = 0] (v)
(v1) edge[bend right = -10] node[index label] {1} (u)
(v2) edge[bend right = 10] node[index label] {2} (u)
(v3) edge[bend right = -10] node[index label] {2} (u1)
(v4) edge[bend right = 10] node[index label] {1} (u2)
(u1) edge[bend right = 0] (v1)
(u2) edge[bend right = 0] (v2)
(v4) edge[bend right = 10] node[index label] {1} (u1)
(v) edge[in = -55, out = 225] node[index label] {2} (u2)
    (v6) edge[bend right = -10] node[index label] {1} (u0)
	(v7) edge[bend right = 10] node[index label] {2} (u0)
	(u0) edge[bend right = 0] (v3);
	\end{tikzpicture}
\end{center}
The arising decomposition, $\decompositionn$, is visualised by this \begin{samepage}tree:
\begin{center}
	\begin{tikzpicture}[]
	\node[draw=darkred,rounded corners,label=180:{}] (t1) at (0.60,0.2) {$(g_1,t_1)$};
	\node[draw=darkred,rounded corners,label=180:{}] (t2) at (-0.35,1.5) {$(\neigh{u_3},t_2)$};
	\node[draw=darkgreen,rounded corners,label=0:{}] (t3) at (5,0.85) {$(\neigh{u},t_3)$};
	\node[draw=darkblue,rounded corners,label={[label distance =0.65cm,anchor=center]-25:{}}] (t4) at (2.5,1.5) {$(\neigh{u_2},t_4)$};
	\node[draw=darkred,rounded corners,label=180:{}] (t5) at (2.5,0.2) {$(\neigh{u_1},t_5)$};
	\path[->]
    (t4) edge[bend right = -10] node[index label] {1} (t3)
	(t5) edge[bend right = 10] node[index label] {2} (t3)
	(t1) edge[bend right = 5] node[index label] {1} (t4)
	(t2) edge[bend right = -5] node[index label] {2} (t4);
	\end{tikzpicture}\\\vspace{5pt}
\end{center}\end{samepage}
The production of $\decompositionn$ starts with a split that has $(\neigh{u},t_3)$ as a generator:
\begin{displaymath}
((\neigh{u},t_3),[(\neigh{u_2}\cup \neigh{u_3},t_4),(\neigh{u_1},t_5)])
\end{displaymath}
In turn, $(\neigh{u_2}\cup \neigh{u_3},t_4)$ can be split, using $(\neigh{u_2},t_4)$ as a generator, enclosed by the medium-thickness (blue) contour. This yields the split
\begin{displaymath}
((\neigh{u_2},t_4),[(\cgraphn_1,t_1),(\neigh{u_3},t_2)]).
\end{displaymath}
The leaves of the decomposition are, thus, $(\neigh{u_1},t_5)$ from the first split, along with $(\cgraphn_1,t_1)$ and $(\neigh{u_3},t_2)$ from the second split.
\end{example}

Definition~\ref{defn:decompositionTree} formalises the trees that form part of a decomposition, with definition~\ref{defn:decomposition} augmenting these trees with constructions as vertex labels. Firstly, recall that in a graph, $\incVert$ identifies the incident vertices (i.e. those connected by an arrow) and $\arrowl$ is a function that labels arrows with indices.

\begin{definition}\label{defn:decompositionTree}
A \textit{decomposition tree} is a directed arrow-labelled rooted in-tree\footnote{That is, the arrows are directed \textit{towards} the root: the tree is an anti-arborescence.},
\begin{displaymath}
\decompositionn=(V,\arrows, \incVert\colon \arrows\to V\times V, \arrowl\colon \arrows \to \mathbb{N})
\end{displaymath}
such that for each vertex, $v$, in $V$, the function $\arrowl$ with its domain restricted to $\inA{v}$ is a bijection with codomain $\{1,\ldots,|\inA{v}|\}$.
\end{definition}

We now introduce some notation and various concepts related to decomposition trees. Some of these concepts will be used in this section, whereas others are needed in Section~\ref{sec:structuralTransformations}.

\begin{definition}\label{defn:decTreeSequence}
Let $\decompositionn=(V,\arrows, \incVert, \arrowl)$ be a decomposition tree containing a vertex, $v$.
\begin{enumerate}
\item The root of $\decompositionn$ is denoted $\treeroot{\decompositionn}$.
%
%\item The sub-tree of $\decompositionn$ containing precisely $v$ is denoted $\decompositionn_v$.
%
\item The \textit{$v$-induced decomposition tree}, denoted $\videctree{v}$, is the maximal sub-tree of $\decompositionn$ with root $v$.
\item The \textit{decomposition-tree sequence} of $v$, denoted $\dsequence{v}$, is the sequence of sub-trees of $\decompositionn$, $\dsequence{v}=[\videctree{v_1},\ldots,\videctree{v_n}]$, where each $v_i$ is the vertex sourced on the arrow targeting $v$ that is labelled by $i$ and $n=|\inA{v}|$.
\item The set of \textit{$v$-ancestors}, denoted $\ancestors{v}$,  is defined as follows:

\begin{displaymath}
  \ancestors{v} = \bigcup_{\decompositionn_i \in \dsequence{v}} \ancestors{\treeroot{\decompositionn_i}} \cup \{\treeroot{\decompositionn_i}\}
\end{displaymath}
We further define $\ancestorsandvertex{v}=\ancestors{v}\cup \{v\}$.
\end{enumerate}
\end{definition}

We now extend the definition of a decomposition tree to that of a decomposition, which augments decomposition trees with a vertex labelling function. Notably, all concepts in definition~\ref{defn:decTreeSequence} extend in the obvious way to such augmented trees.

\begin{definition}\label{defn:decomposition}
Let $\cpair$ be a construction. A \textit{decomposition} of a $\cpair$ is a directed labelled rooted tree, $\decompositionn=\decomposition$, such that $(V,\arrows, \incVert, \arrowl)$ is a decomposition tree and $\labc$ is a vertex-labelling function where either
\begin{enumerate}
\item $\decompositionn$ contains only $\treeroot{\decompositionn}$ and $\labc(\treeroot{\decompositionn})=\cpair$, or
\item there exists a split $(\genpair,[(\cgraphn_1,t_1),\ldots,(\cgraphn_n,t_n)])$ of $\cpair$, where the following hold:
    \begin{enumerate}
    \item the construction that labels $\treeroot{\decompositionn}$ is $\genpair$: $\labc(\treeroot{\decompositionn})=\genpair$,
    \item given $\dsequence{\treeroot{\decompositionn}}=[\videctree{v_1},\ldots,\videctree{v_n}]$, each $\videctree{v_i}$ is a decomposition of $(\cgraphn_i,t_i)$.
    \end{enumerate}
\end{enumerate}
%Such a split is said to \textit{comply} with $\decompositionn$.
\end{definition}

\begin{example}[Using Leaves to Derive Foundation Sequences]
In example~\ref{ex:SplitingSplits}, the \textit{leaves} of the decomposition tree can be used to derive the foundation sequence of the entire construction, $(\cgraphn,t_3)$. To effect this derivation, the labels of the leaves need to be \textit{ordered} so that they respect the indices assigned to the arrows, as shown here:
\begin{displaymath}
[(g_1,t_1),(\neigh{u_3},t_2),(\neigh{u_1},t_5)].
\end{displaymath}
Taking the ordered leaves, we are able to concatenate their foundation sequences in order to obtain the foundation sequence for the original construction. Here, we see that
\begin{eqnarray*}
\foundationsto{g_1,t_1} \oplus \foundationsto{\neigh{u_3},t_2)} \oplus \foundationsto{\neigh{u_1},t_3)}
   &=& [t_1,t_6,t_7,t_1,t_3] \\
                                &=& \foundationsto{\cgraphn, t_3}.
\end{eqnarray*}
\end{example}

\begin{definition}\label{defn:leafConstructionSequence}
Let $\cpair$ be a construction with decomposition $\decompositionn$. The \textit{leaf construction-sequence} of $\decompositionn$, denoted $\lcsequence{\decompositionn}$, is defined recursively:
\begin{enumerate}
\item if $\decompositionn$ contains only $\treeroot{\decompositionn}$ then $\lcsequence{\decompositionn}=[\cpair]$,
\item otherwise, given $\dsequence{\treeroot{\decompositionn}}=[\videctree{v_1},\ldots,\videctree{v_n}]$ %and $\dsequence{\treeroot{\decompositionn}}=[\videctree{v_1},\ldots,\videctree{v_n}]$,
% $\dsequence{\treeroot{\decompositionn}}=[\decompositionn_1,\ldots,\decompositionn_n]$,
%
\begin{displaymath}
  \lcsequence{\decompositionn}=\lcsequence{\videctree{v_1}}\oplus \cdots \oplus \lcsequence{\videctree{v_n}}.
\end{displaymath}
\end{enumerate}
\end{definition}

Theorem~\ref{thm:splitPreservesfoundations} immediately tells us that the foundations of $\cpair$ are derivable from the foundations of the leaf construction-sequence in any decomposition, captured by theorem~\ref{thm:decompositionPreservesFoundations}.

\begin{theorem}\label{thm:decompositionPreservesFoundations}
Let $\cpair$ be a construction with decomposition $\decompositionn$. Given the leaf construction-sequence  $\lcsequence{\decompositionn}=[\cpairp{1},\ldots,\cpairp{n}]$,
\begin{displaymath}
\foundationsto{\cgraphn,t} = \foundationsto{\cgraphn_1,t_1}\oplus \cdots \oplus \foundationsto{\cgraphn_n,t_n}.
\end{displaymath}
\end{theorem}

Our last result in this section, theorem~\ref{them:DecompositonTreeUniqueConstruction}, establishes that a decomposition tree is associated with a unique construction.

\begin{theorem}\label{them:DecompositonTreeUniqueConstruction}
Let $\cpair$ and $\cpaird$ be a construction with decompositions $\decompositionn$ and $\decompositionn'$, respectively. If $\decompositionn=\decompositionn'$ then $\cpair=\cpaird$.
\end{theorem}

\begin{proof}
Suppose that $\decompositionn=\decompositionn'$. Then the labels, $(\cgraphn_v,t_v)$ and $(\cgraphn_{v'},t_{v'})$ of their roots, $v$ and $v'$, are equal, implying that $t_v=t_{v'}$. By the definition of a decomposition, $(\cgraphn_v,t_v)$ and $(\cgraphn_{v'},t_{v'})$ are generators for $\cpair$ and $\cpaird$, so $t_v=t$ and $t_{v'}=t'$.  Hence $t=t'$. By theorem~\ref{thm:SplitsCoverConstruction}, it can readily be shown that $\cgraphn=\cgraphn'$. It follows that $\cpair=\cpaird$.
\end{proof}

Consequently, if we had a decomposition, $\decompositionn$, which was the product of some algorithm to, say, invoke a structural transformation of some construction, $\cpair$, we could convert $\decompositionn$ into a unique construction, $\cpaird$ -- the sought-after transformation of $\cpair$.

\begin{definition}\label{defn:constructionOfD}
Let $\decompositionn=\decomposition$ be a decomposition and let $t$ be the construct of $\labc(\treeroot{\decompositionn})$. The \textit{construction of} $\decompositionn$, denoted $\constructionofD{\decompositionn}$, is the construction
\begin{displaymath}
\constructionofD{\decompositionn}=\bigg(\bigcup \{\cgraphn_v: \exists v\in V \thinspace \exists t_v \thickspace \labc(v)=(\cgraphn_v,t_v)\},t\bigg).
\end{displaymath}
\end{definition}

We now illustrate how decompositions can be used in the context of structural transformations. In particular, we include an example that illustrates transforming a set-theoretic construction into an Euler diagram construction. Decompositions allow us to use `local' knowledge, that one sub-construction can be transformed into another, when producing a structural transformation.

\begin{example}[Visualizing Information Using a Structural Transformation]\label{ex:visInfoUsingStructureTransfer}
A common problem in information visualisation is identifying a suitable diagrammatic representation of sets. Example~\ref{ex:encodingProperties}  visualised the set-theoretic axioms $\{A\subseteq B, B\cap C=\emptyset\}$ using an Euler diagram. This current example demonstrates, whilst omitting various significant details, how decompositions can be used to find an Euler diagram drawn from $\rsystemn_{ED}$ that represents $\{A\subseteq B, B\cap C=\emptyset\}$; we first saw this example in Section~\ref{sec:example}. The process of finding a structural transformation starts with a construction, \begin{samepage}$(\cgraphn, \{A\subseteq B, B\cap C=\emptyset\})$:
\begin{center}
\scalebox{1}{%
\STConstruction}
\end{center}
\end{samepage}
The structural transformation task is to decompose the above construction into simpler constructions for which individual (local) structural transformations can be produced. 

Consider the following decomposition, \begin{samepage}$\decompositionn$:
\begin{center}
\scalebox{1}{%
\STDecompositionTree}
\end{center}\end{samepage}
Further, suppose that we know some relationships between the foundations of the given set-theoretic construction of $\{A\subseteq B, B\cap C=\emptyset\}$ and tokens in $\rsystemn_{ED}$. In particular, suppose it is given that an atomic set expression (i.e. one that does not include any operators), such as $A$, corresponds to a closed curve that has been labelled by $A$. This means that the induced construction comprising just the foundation $A$ (see vertex $v_5$ above), of $(\cgraphn, \{A\subseteq B, B\cap C=\emptyset\})$ can be associated with this construction in $\rsystemn_{ED}$:
\begin{center}
\scalebox{1}{%
\begin{tikzpicture}
\node[draw=darkred, thick, rounded corners, label=-90:{$v_5$}] (EDlabelaCurve) at (2.3,0) {\STLeafA};
\draw[|-{angle 90},thick] (2.7,0) -- (3.25,0);
\node[draw=darkred, thick, rounded corners] (EDlabelaCurve) at (5,0) {\EDlabelaCurve};
\end{tikzpicture}}\vspace{-2pt}%
\end{center}
Similar relationships hold for $B$ and $C$. As such, the induced (trivial) constructions of basic sets $A$, $B$, and $C$ (vertices $v_5$, $v_7$, $v_{8}$ and $v_{10}$) can readily have its structure transferred to some construction in $\rsystemn_{ED}$. In fact, each sub-construction of $(\cgraphn, \{A\subseteq B, B\cap C=\emptyset\})$ in $\decompositionn$ can be structurally transferred to a construction in the decomposition, $\decompositionn'$, shown below, where the vertex names (from $\decompositionn$) and $\mapsto$ arrows indicate a structural \begin{samepage}transformation:
\begin{center}
\scalebox{0.9}{\EDDecompositionTree}\vspace{-2pt}%
\end{center}\end{samepage}
A key insight in this example is that the structure transfer process is contingent on establishing some basic transformations between constructions at the leaves of $\decompositionn$, to produce leaves of $\decompositionn'$. Considering the leaves of $\decompositionn$, suppose
\begin{enumerate}
\item it is given that any trivial construction, $(\cgraphn,X)$, of a basic set, $X$, \textit{corresponds} to some basic construction, $(\cgraphn', Y)$, that has $\foundationsty{\cgraphn',Y}=[\texttt{curve}, \texttt{label}]$ and
that uses the \texttt{assignLabel} constructor; this information provides a link between constructions, across the two systems, that have a \textit{prescribed structure} (this prescribed structure is exactly captured by patterns, developed in Section~\ref{sec:patterns}), and
\item we have such a pair of linked constructions, $(\cgraphn,X)$ and $(\cgraphn',Y)$,  where $\foundationsto{\cgraphn,X}=[X]$ and $\foundationsto{\cgraphn',Y}=[c,Z]$; this information allows us to define relationship between foundation sequences: $[X]$ is related to $[c,Z]$ whenever $X=Z$ (i.e., when the name, $X$, of the set is the label, $Z$, of the curve).
\end{enumerate}
Then, given $(\cgraphn,X)$ and $(\cgraphn',Y)$, where $X=Z$ we can \textit{deduce} that the construct $Y$ \textit{represents} $X$. This general approach allows the constructions associated with leaves $v_5$, $v_7$, $v_{8}$ and $v_{10}$ in $\decompositionn$ to have their structure transferred in such a way that it is \textit{guaranteed} that the constructed Euler diagrams represent the sets $A$, $B$ and $C$. Moreover, these leaf-level transformations \textit{provide the foundations} to the constructions in $\decompositionn'$ associated with \texttt{mergeSubset} and \texttt{mergeDisjoint}. As such, by transferring the structures of the individual constructions in $\decompositionn$, working from the leaves up to the root, we produce $\decompositionn'$. 

Notably, not all of the leaves of $\decompositionn$ have their structure transferred. For instance, the subset operator, $\subseteq$, has no direct counterpart token in the Euler diagram system: the construction assigned to $v_6$ does not require its structure to be transferred to another token. Lastly, the construction in the Euler diagram system obtained from the produced decomposition, $\decompositionn'$, \begin{samepage}is:
\begin{center}\vspace{-0.3cm}
\scalebox{0.9}{\EDConstruction}
\end{center}\vspace{-0.2cm}\end{samepage}
Thus the set of axioms, $\{A\subseteq B, B\cap C=\emptyset\}$, can be visualised by the Euler diagram constructed above\footnote{In an information visualisation setting, we would expect the result of such a structural transformation to be the \textit{abstract syntax} of the required diagram. The abstract syntax would then be provided to a system-specific layout algorithm. We do not claim that the process of structure transfer solves the general, system-specific, problem of how to find layouts of diagrams.}.
\end{example}

As example~\ref{ex:visInfoUsingStructureTransfer} demonstrates, decompositions play a vital role in structural transformations. They are also fundamental to how we describe a representational system using patterns, introduced in Section~\ref{sec:patterns}. 

%% file: patterns.tex
Patterns provide a way of describing the constructions in a representational system, $\rsystemn$, and are fundamental to the definition of structural transformations, developed in Section~\ref{sec:structuralTransformations}. A pattern is a construction, $\ppair$, drawn from a \textit{pattern space}; roughly speaking, a construction, $\cpair$, in $\rsystemn$ \textit{matches} a pattern, $\ppair$, provided $\cgraphn$ and $\pgraphn$ are essentially the same, but $\pgraphn$ `forgets' the tokens of $\cgraphn$ and uses types instead.  Consequently, a pattern can \textit{describe} many (matching) constructions in $\rsystemn$. A description of $\rsystemn$ will comprise a set, $\descriptionn$, of \textit{decompositions of patterns}. At first thought, it may appear more straightforward to simply say that a set of patterns, $\setofpatterns$, is a description of $\rsystemn$, without any need for using pattern decompositions. However, such an approach has limitations, as we now explain.

A structural transformation, given an inter-representational-system encoding, $\irencoding$,  will require an association of a pair of patterns, $\ppair$ and $\ppaird$, drawn from systems $\rsystemn$ and $\rsystemn'$, with a set of patterns, $\setofpatternstc$, drawn from $\ispacen''$, . Given such a triple, $(\ppair,\ppaird,\setofpatternstc)$ -- which we call a \textit{transformation constraint} -- the process of \textit{structure transfer} will take a construction, $\cpair$, that matches $\ppair$ and identify a construction, $\cpaird$, that is described by (i.e. matches) $\ppaird$:
\begin{displaymath}
  \cpair \longrightarrowY{\mathit{matches}} \ppair \longrightarrowX{}  \ppaird \longrightarrowY{\mathit{describes}} \cpaird.
\end{displaymath}
The role of $\setofpatternstc$, will be explained in Section~\ref{sec:structuralTransformations}; for now, it is sufficient to know that  $\setofpatternstc$ places an additional constraint on the relationship between $\cpair$ and $\cpaird$  -- e.g. the relationship that $t$ and $t'$ are semantically equivalent -- to support \textit{representation selection}. Using $\setofpatterns$ as a description of $\rsystemn$ would mean that -- in many cases -- descriptions of representational systems would be infinite, even if $\setofpatterns$ contains no pair of label-isomorphic patterns: there are likely to be infinitely many constructions of arbitrarily large complexity (measured by number of vertices) in $\uspace{\rsystemn}$. Hence, $\setofpatterns$ would contain patterns of arbitrary large complexity, too, which is a practical limitation: the set of transformation constraints that describe \textit{allowable structure transfers} would also need to be infinite.

Of course, even when using the set $\descriptionn$ to describe $\rsystemn$, it is likely that $\descriptionn$ would be infinite. But, the decompositions in $\descriptionn$ may only involve a \textit{finite} set of patterns (up to label-preserving isomorphism); $\descriptionn$ is called \textit{compact} provided the set of equivalence classes of the set of patterns that label the vertices of decompositions in $\descriptionn$ is finite. In essence, given a compact $\descriptionn$, a \textit{finite set of transformation constraints}, arising from (equivalence classes of) patterns can be used to support the structure transfer of \textit{constructions of arbitrary complexity}. This is the high-level process: construction $\cpair$ is decomposed, giving $\decompositionn$; each construction involved in $\decompositionn$ is, essentially, replaced by a matching pattern, to give a decomposition, $\pdecompositionn$, of some pattern, $\ppair$, that describes $\cpair$; next, transformation constraints are used to create $\pdecompositionn'$, a decomposition of some pattern associated with $\rsystemn'$; each pattern in $\pdecompositionn'$ is replaced by a construction that it describes to yield $\decompositionn'$, in such a way that the transformation constraints are satisfied; lastly, $\decompositionn'$ can be replaced by its construction, $\constructionofD{\decompositionn'}=\cpaird$, following definition~\ref{defn:constructionOfD}. This process is visualised here, omitting explicit depiction of the role of transformation constraints:
\begin{displaymath}
  \cpair \longrightarrowY{\mathit{decompose}} \decompositionn \longrightarrowY{\mathit{\vphantom{p}match}} \pdecompositionn \longrightarrowY{\mathit{\vphantom{p}transfer}}  \pdecompositionn' \longrightarrowY{\mathit{\vphantom{p}describe}} \decompositionn' \longrightarrowY{\mathit{recompose}}\cpaird.
\end{displaymath}
Hence, we can exploit structure transfer algorithms to convert $\cpair$ into $\cpaird$ via decompositions. Shortly the process of finding a structural transformation will be illustrated by example, to motivate what follows in this section, but it is formalised in section~\ref{sec:structuralTransformations}.

\subsection{A Motivating Example: Representing Data Using Tables and Tree Maps}\label{sec:patterns:ME}

To introduce patterns and their role in describing representational systems, we appeal to some constructions in the context of data that can be represented using a table as well as a variety of visual modes. These constructions, whilst simple, are sufficient to exemplify the role of patterns in describing representational systems and in defining structural transformations, building on the insights provided at the end of section~\ref{sec:constructions}. Firstly, we show how a \textsc{Tables} representational system, $\rsystemn_{T}$, can be described using patterns, specifically decompositions of patterns. For the purposes of illustration, only the grammatical space of $\rsystemn_{T}$ is described here; we assume that the entailment and identification spaces are empty\footnote{In general, a description of a representational system, $\rsystemn$, will exploit the universal space, $\uspace{\rsystemn}$.}. Secondly, we show how a \textsc{Tree Maps} representational system, $\rsystemn_{\mathit{TM}}$, can be described and we identify associations between patterns that permit structural transformations to be found. These transfers will take some data represented using a table and exploit associated patterns to identify a tree map that represents the same data.

The \textsc{Tables} system, $\rsystemn_{T}$, has three constructors in its grammatical space\footnote{Many other ways of encoding tables using Representational Systems  Theory are possible. The constructors provided are just one such way and they are sufficient to build any table whose cells include text or numerical values.} with the following signatures:
\begin{center}
% \begin{minipage}[c][3cm]{0.2\textwidth}    % numerical expression
\begin{tikzpicture}[construction,yscale=0.75]
\node[typeE={$\Ttable$}] (v) at (3.2,4.2) {};
\node[constructorNW={$\TcaddRow$}] (u) at (3.2,3.2) {};
\node[typeS={$\Ttable$}] (v2) at (2.2,2.5) {};
\node[typeS={$\Trow$}] (v4) at (4.2,2.5) {};
\path[->]
(u) edge[bend right = 0] (v)
(v2) edge[bend right = -10] node[index label] {1} (u)
(v4) edge[bend right = 10] node[index label] {2} (u)
;
\end{tikzpicture}
% \end{minipage}
%
\hspace{1.5cm}
%
% \begin{minipage}[c][3cm]{0.2\textwidth}    % brackets
\begin{tikzpicture}[construction,yscale=0.75]
\node[typeE={$\Trow$}] (v) at (3.2,4.2) {};
\node[constructorNW={$\TcaddCell$}] (u) at (3.2,3.2) {};
\node[typeS={$\Trow$}] (v2) at (2.2,2.5) {};
\node[typeS={$\Tcell$}] (v4) at (4.2,2.5) {};
\path[->]
(u) edge[bend right = 0] (v)
(v2) edge[bend right = -10] node[index label] {1} (u)
(v4) edge[bend right = 10] node[index label] {2} (u)
;
\end{tikzpicture}
% \end{minipage}
%
\hspace{1.5cm}
%
% \begin{minipage}[c][3cm]{0.2\textwidth} % binary relation formulae
\begin{tikzpicture}[construction,yscale=0.75]
\node[typeE={$\Tcell$}] (v) at (3.2,4.2) {};
\node[constructorNW={$\TcinsertText$}] (u) at (3.2,3.2) {};
\node[typeS={$\TemptyCell$}] (v2) at (2.2,2.5) {};
\node[typeS={$\TcellEntry$}] (v4) at (4.2,2.5) {};
\path[->]
(u) edge[bend right = 0] (v)
(v2) edge[bend right = -10] node[index label] {1} (u)
(v4) edge[bend right = 10] node[index label] {2} (u)
;
\end{tikzpicture}
% \end{minipage}
\end{center}
Whilst we do not introduce all of the types in $\rsystemn_{T}$, those used in the signatures above, along with a further two types, $\Ttext$ and $\Tnum$, respect the following ordering:
\begin{displaymath}
\TemptyCell \leq \Tcell \leq \Trow \leq \Ttable \qquad \Ttext \leq \TcellEntry \qquad \Tnum \leq \TcellEntry.
\end{displaymath}

\begin{wrapfigure}[5]{r}{3cm}\centering
	\vspace*{-0.4cm}
\scalebox{0.85}{%
\begin{tikzpicture}
\tablefive
\end{tikzpicture}
}
\end{wrapfigure}
Consider the table on the right, drawn from $\rsystemn_{T}$, representing how many seats are held by political parties in the UK parliament in 2021. This particular table has a construction that uses all three constructors, with each row being individually built before being integrated into the table. This construction, which we call $\cpair$, can be seen below 
and is indicative of how any table can be built using the three provided constructors:
\begin{center}
	\scalebox{0.9}{\tableConstruction}
\end{center}
Each of the constructor signatures  -- such as $\sig(\TcaddRow)=([\Ttable,\Trow],\Ttable)$  -- can, when viewed as a structure graph, be considered a pattern for $\rsystemn_{T}$. Some more complex patterns can also be observed by considering the UK parliament data table's construction, $\cpair$, which highlights some common ways of combining configurations to produce parts of tables. 

This commonality can be captured through the use of more complex patterns that involve multiple configurators. One common combination of configurators builds a two-cell row of a table and appends it to the bottom of an existing table, as seen in this generator of \begin{samepage}$\cpair$:
\begin{center}
\scalebox{0.95}{\tableConstructionGenerator}
\end{center}
\end{samepage}
The above generator matches the pattern, which we call $\ppairp{1}$\label{pageref:patternsppairp1}, shown below:
\begin{center}
\scalebox{0.95}{\tableConstructionPattern}
\end{center}
Notice that $\ppairp{1}$ exploits subtypes of those assigned to the constructor $\TcinsertText$, where $\TcellEntry$ is replaced by $\Ttext$ in one instance and by $\Tnum$ in another. Another pattern, which we call $\ppairp{2}$, is shown below which embodies the creation of a table comprising two rows and two columns, where the first row contains text (e.g. a `heading' row) and the second row contains text in the first cell and a numeral in the other:
\begin{center}
\scalebox{0.95}{\tableConstructionPatternFoundations}
\end{center}
The patterns $\ppairp{1}$ and $\ppairp{2}$, or label-isomorphic copies of them,  along with three further (trivial) patterns, are sufficient to \textit{describe} $\cpair$. These trivial patterns are $\ppairp{3}$, $\ppairp{4}$ and $\ppairp{5}$:
\begin{center}
\begin{tikzpicture}[construction]
\node[typeS={\TemptyCell}] (c5lb) {$v_3$};
\end{tikzpicture}
\quad
\begin{tikzpicture}[construction]
\node[typeS={\Ttext}] (c5lt) {$v_4$};
\end{tikzpicture}
\quad%
\begin{tikzpicture}[construction]
\node[typeS={\Tnum}] (c5rt) {$v_5$};
\end{tikzpicture}
\end{center}
A description of $\cpair$ can be obtained from a decomposition, $\decompositionn$, by replacing the sub-constructions with matching patterns.{\pagebreak} Such a decomposition, $\decompositionn$, is given below, alongside a description, $\pdecompositionn$, that exploits label-isomorphic copies of the five \begin{samepage}patterns:\label{ex:decompostionForStructureTransfer}
\begin{center}
	\begin{tikzpicture}
		\node[anchor = center] at (0,0) {\scalebox{0.65}{\TableDecompositionTree}};
		\node[anchor = center] at (6.5,0) {\scalebox{0.65}{\TableDescription}};
	\end{tikzpicture}
\end{center}\end{samepage}
Here, the construct in each non-root pattern in $\pdecompositionn$ is a foundation vertex of the pattern above. So, for example, the construct $n_{12}$ of the bottom pattern is a foundation of the pattern above. These common vertices allow us to join up the patterns in $\Delta$ to form a single, more complex, pattern that is guaranteed to describe $\cpair$.

So far, we have witnessed the conversion of a construction, $\cpair$, into a decomposition, $\decompositionn$, and, subsequently, a description, $\pdecompositionn$, which reflects the first stage of the structure transfer process:
\begin{center}
	\begin{tikzpicture}
	\node (f) at (0,0) {$\cpair \longrightarrowY{\mathit{decompose}} \decompositionn \longrightarrowY{\mathit{\vphantom{p}match}} \pdecompositionn \longrightarrowY{\mathit{\vphantom{p}transfer}}  \pdecompositionn' \longrightarrowY{\mathit{\vphantom{p}describe}} \decompositionn' \longrightarrowY{\mathit{recompose}}\cpaird.$};
	%\draw[rounded corners,very thick, opacity=0.75,darkgold] (-1.8,-0.5) rectangle (1.28,0.4);
	\draw[rounded corners,very thick, opacity=0.5,darkpurple] (-7,-0.4) rectangle (-1.2,0.3);
	%\draw[rounded corners,very thick, opacity=0.5,darkpurple] (0.57,-0.4) rectangle (7.0,0.3);
	\end{tikzpicture}
\end{center}
Our focus now turns to the step of going from $\pdecompositionn$ to $\pdecompositionn'$, a description of some construction in an alternative representational system. Afterwards, we will consider the final stage: taking $\pdecompositionn'$ and identifying $\decompositionn'$ and $\cpaird$.

\begin{wrapfigure}{r}{2.7cm}\centering
\treemap
\end{wrapfigure}
Our alternative representational system is that of tree maps, where we will seek a construction of a tree map that represents the UK parliament data, given the table with construction $\cpair$. In fact, we will derive a construction, $\cpaird$, of the tree map given on the right.  This diagram does not include a legend\footnote{It would be readily possible to include legends within the\textsc{Tree Maps} representational system, we merely omit them for simplicity of discussion.}, but simply represents the numbers of seats held by various parties. Now, in order to identify $\pdecompositionn'$, before finding $\decompositionn'$ and $\cpaird$, we first specify the types and constructors of $\rsystemn_{\mathit{TM}}$. In our encoding of tree maps, there are the following types with the specified relations holding:
\begin{displaymath}\small
\TMemptyTreeMap\leq \TMincompleteTreeMap \leq \TMtreeMap,  \TMfilledTreeMap\leq \TMtreeMap, \textup{ and } \TMnumeral.
\end{displaymath}

The type $\TMfilledTreeMap$ is assigned to each rectangle containing only a number, and to tree maps whose outer-most rectangle is \textit{exhaustively filled} by sub-tree maps, each of which are also complete; other tree maps are of type $\TMincompleteTreeMap$. Of the three tree maps below, that on the left is not exhaustively filled, since the outermost rectangle has space for further sub-tree maps to be inserted, so it is of type $\TMincompleteTreeMap$:
\begin{center}
\begin{minipage}{0.35\textwidth}\centering
{\treemapA}
\end{minipage}
\begin{minipage}{0.25\textwidth}\centering
{\treemapB}
\end{minipage}
\begin{minipage}{0.25\textwidth}\centering
{\treemapC}
\end{minipage}
\end{center}
By contrast, the tree map in the middle is of type $\TMfilledTreeMap$ since (a) its outer rectangle does not have space for further sub-tree maps to be inserted, and (b) the two sub-tree maps it contains are exhaustively filled (in their cases, their unique rectangle simply contains a number). The rightmost rectangle is a tree map of type $\TMemptyTreeMap$.

We define just two constructors\footnote{Our constructors prohibit the construction of tree maps whose outer-most rectangle is exhaustively filled but where one, or more, of the tree maps it contains is not of type $\TMfilledTreeMap$. This prohibition is purely for simplicity.} for the $\rsystemn_{\mathit{TM}}$ system, $\TMcinsertTreeMap$ and $\TMcannotate$, whose signatures are visualised here:
\begin{center}
% \begin{minipage}[c][3cm]{0.2\textwidth}    % numerical expression
\begin{tikzpicture}[construction,yscale = 0.8]\small
\node[typeE={$\TMtreeMap$}] (v) at (3.2,3.8) {};
\node[constructorNW={$\TMcinsertTreeMap$}] (u) at (3.2,2.9) {};
\node[typeS={$\TMincompleteTreeMap$}] (v2) at (2,2.2) {};
\node[typeS={$\TMfilledTreeMap$}] (v4) at (4.4,2.2) {};
\path[->]
(u) edge[bend right = 0] (v)
(v2) edge[bend right = -10] node[index label] {1} (u)
(v4) edge[bend right = 10] node[index label] {2} (u)
;
\end{tikzpicture}
% \end{minipage}
%
\hspace{2cm}
%
% \begin{minipage}[c][3cm]{0.2\textwidth}    % brackets
\begin{tikzpicture}[construction,yscale = 0.8]\small
\node[typeE={$\TMfilledTreeMap$}] (v) at (3.2,3.8) {};
\node[constructorNW={$\TMcannotate$}] (u) at (3.2,2.9) {};
\node[typeS={$\TMemptyTreeMap$}] (v2) at (2.2,2.2) {};
\node[typeS={$\TMnumeral$}] (v4) at (4.2,2.2) {};
\path[->]
(u) edge[bend right = 0] (v)
(v2) edge[bend right = -10] node[index label] {1} (u)
(v4) edge[bend right = 10] node[index label] {2} (u)
;
\end{tikzpicture}
\end{center}

The construction below shows how an incomplete tree map can have another tree map inserted into it:
\begin{center}
\scalebox{0.7}{%
\begin{tikzpicture}[construction,yscale = 0.8]\small
\node[termrep] (ab) {\treemapAB};
\node[constructor={\TMcinsertTreeMap}, below = 0.5cm of ab] (c) {};
\node[termrep, below left = 0.4cm and 0.1cm of c] (a) {\treemapA};
\node[termrep, below right = 0.5cm and 1cm of c] (b) {\treemapB};
\path[->] (c) edge (ab)
		  (a) edge[bend right =-10] node[index label] {1} (c)
		  (b) edge[bend right =10] node[index label] {2} (c);
\end{tikzpicture}}
\end{center}
In this case, the constructed tree map is also of type $\TMincompleteTreeMap$ but, in general, tree maps of type $\TMfilledTreeMap$ can be constructed by the $\TMcinsertTreeMap$ constructor.

Having introduced the type system and constructors for $\rsystemn_{\mathit{T}}$ and $\rsystemn_{\mathit{TM}}$, we are in a position to consider associations between their patterns. Now, in $\rsystemn_{T}$ we used, essentially, five patterns, namely $\ppairp{1},\hdots,\ppairp{5}$ in the description, $\pdecompositionn$,  of $\cpair$. The pattern $\ppairp{1}$ describes constructions that add a two-cell row to an existing table, $\mathit{ta}$, where the first cell contains text and the second cell contains a numeral. Such a row corresponds to a single-rectangle tree map containing the same numeral since, by design, $\rsystemn_{\mathit{TM}}$ does not include legends. In this context, the existing table, $\mathit{ta}$, can be taken to correspond to an incomplete tree map, $\mathit{tm}$. Thus, adding a two-cell row to $\mathit{ta}$ corresponds to inserting a tree map containing a single value into $\mathit{tm}$. Based on this observation, the pattern $\ppairp{1}$ can be associated with two patterns, $(\pgraphn_1',v_1')$ and $(\pgraphn_2',v_2')$ , for tree maps (reflecting whether the constructed tree map is complete)\footnote{A third link could be defined where the type of the constructed tree map is $\TMtreeMap$, but we do not need this pattern.}:
\begin{center}
\begin{minipage}{0.45\textwidth}\centering
	\scalebox{0.75}{\tableConstructionPattern}
\end{minipage}
\begin{minipage}{0.05\textwidth}\centering
	\longrightarrowX{}
\end{minipage}
\begin{minipage}{0.35\textwidth}\centering
	\scalebox{0.75}{\treemapPatternRoot{v_1'}{}{}{}}
\end{minipage}
\\[1ex]
\begin{minipage}{0.45\textwidth}\centering
	\scalebox{0.75}{\tableConstructionPattern}
\end{minipage}
\begin{minipage}{0.05\textwidth}\centering
	\longrightarrowX{}
\end{minipage}
\begin{minipage}{0.35\textwidth}\centering
	\scalebox{0.75}{\treemapPattern{v_2'}{}{}{}}
\end{minipage}
\end{center}

The pattern $\ppairp{2}$ describes constructions of tables comprising two rows and two columns, where the first row contains textual entries only (the column headings in the UK parliament table) and the second row contains text and a numeral. Such a table is akin to a tree map comprising a single rectangle which, in turn, contains a single numeral. We define the following association between $\ppairp{2}$ and a third $\rsystemn_{\mathit{TM}}$ pattern, $(\pgraphn_3',v_3')$ \footnote{Again, other patterns can be defined where the construct is of different types, namely $\TMcompleteTreeMap$ and $\TMtreeMap$.}:
\begin{center}
\begin{minipage}{0.5\textwidth}\centering
	\scalebox{0.75}{\tableConstructionPatternFoundations}
\end{minipage}\hspace{-0cm}
\begin{minipage}{0.07\textwidth}\centering\vspace{-1.5cm}
	\longrightarrowX{}
\end{minipage}\hspace{-0cm}
\begin{minipage}{0.3\textwidth}\centering
	\scalebox{0.75}{\treemapPatternLeaf{v_3'}{}}
\end{minipage}
\end{center}
Lastly, we define an association between a pair of trivial patterns, $\ppairp{5}$ and $(\pgraphn_4',v_4')$ :
\begin{center}
\begin{minipage}{0.15\textwidth}\centering
\begin{tikzpicture}[construction]
\node[typeSE={\Tnum}] (c5rt) {$v_5$};
\end{tikzpicture}
\end{minipage}
\begin{minipage}{0.07\textwidth}\centering\vspace{-0.5cm}
	\longrightarrowX{}
\end{minipage}
\begin{minipage}{0.15\textwidth}\centering
	\begin{tikzpicture}[construction]
\node[typeSW={\TMnumeral}] (c5rt) {$v_{4}'$};
\end{tikzpicture}
\end{minipage}
\end{center}

The remaining two patterns, $\ppairp{3}$ and $\ppairp{4}$ are not associated with patterns in $\rsystemn_{\mathit{TM}}$: tree maps, as we have defined them, do not contain tokens that have a direct analogy with empty cells or text. However, one further pattern, $(\pgraphn_5',v_5')$ , is needed for $\rsystemn_{\mathit{TM}}$, that comprises a single vertex labelled by $\TMemptyTreeMap$. Using{\pagebreak} the identified pattern associations it is now possible to transfer the structure of $\pdecompositionn$ to another decomposition, $\pdecompositionn'$, that describes a construction of the UK parliament data tree \begin{samepage}map:
\begin{center}
	\begin{tikzpicture}
	\node[anchor = north] (td) at (0,0) {\scalebox{0.62}{\TableDescriptionShorterArrows}};
	\node[anchor = north] (td) at (5.5,0) {\scalebox{0.62}{\TreeMapDescriptionShorterArrows}};
	\end{tikzpicture}
\end{center}\end{samepage}

We have now schematically demonstrated how to effect these steps in the structure transfer process:
\begin{center}
	\begin{tikzpicture}
	\node (f) at (0,0) {$\cpair \longrightarrowY{\mathit{decompose}} \decompositionn \longrightarrowY{\mathit{\vphantom{p}match}} \pdecompositionn \longrightarrowY{\mathit{\vphantom{p}transfer}}  \pdecompositionn' \longrightarrowY{\mathit{\vphantom{p}describe}} \decompositionn' \longrightarrowY{\mathit{recompose}}\cpaird.$};
	\draw[rounded corners,very thick, opacity=0.75,darkred] (-1.7,-0.5) rectangle (1.28,0.4);
	\draw[rounded corners,very thick, opacity=0.5,darkpurple] (-7.0,-0.4) rectangle (-1.2,0.3);
	\end{tikzpicture}
\end{center}
Given the above, it is now possible to produce, from $\pdecompositionn'$, a decomposition, $\decompositionn'$,{\pagebreak} of a construction, $\cpaird$, of a tree map that represents the UK parliament seats \begin{samepage}data:
%
%\treeMapConstruction
%
\begin{center}
	\begin{tikzpicture}
	\node[anchor = north] (td) at (0,0) {\scalebox{0.6}{\TreeMapDescription}};
	\node[anchor = north] (td) at (5.5,0) {\scalebox{0.6}{\TreeMapDecomposition}};
	\end{tikzpicture}
\end{center}\end{samepage}

In this example, we know which tree map we seek to produce via the structure transfer process. However, when we do not know exactly which tree map we seek, the step where patterns are replaced by constructions, $\pdecompositionn'\longrightarrowY{\mathit{describe}}\decompositionn'$, is non-deterministic: there are many choices of tree map that represent the UK parliament data and hence many choices of decompositions, $\decompositionn'$, described by $\pdecompositionn'$. By appealing to foundation sequences, we can eliminate some of the otherwise non-deterministic steps. Firstly, we recall theorem~\ref{thm:decompositionPreservesFoundations}, which states how the foundation token-sequence (e.g. $\foundationsto{\constructionofD{\pdecompositionn'}}$) of a construction and, thus, how the foundation type-sequence (e.g. $\foundationsty{\constructionofD{\pdecompositionn'}}$) can be deduced from any of its decompositions. By extension, we know that the foundation type-sequence, $\foundationsty{\cgraphn',t'}$, of the construction, $\cpaird$, that we seek necessarily specialises
\begin{eqnarray*}
 \foundationsty{\constructionofD{\pdecompositionn'}}  &  =  & [\TMemptyTreeMap, \TMemptyTreeMap, \TMnumeral, \TMemptyTreeMap, \TMnumeral, \\
            & & \TMemptyTreeMap, \TMnumeral, \TMemptyTreeMap, \TMnumeral].
\end{eqnarray*}
This type-sequence, along with $\foundationsto{\cgraphn,t}$, allows us to identify some of the tokens that $\cpaird$ will exploit. The table's construction has
\begin{eqnarray*}
  \foundationsto{\cgraphn,t} & = & [\mathit{ec}_1, \textup{Party}, \mathit{ec}_2, \textup{Seats}, \mathit{ec}_3, \textup{Conservative}, \mathit{ec}_4, 364, \mathit{ec}_5, \\
   & & \textup{Labour}, \mathit{ec}_6, 198, \mathit{ec}_7, \textup{SNP}, \mathit{ec}_8, 45, \mathit{ec}_9, \textup{Other}, \mathit{ec}_{10}, 43]
\end{eqnarray*}
where each $\mathit{ec}_i$ is the associated empty cell. Since we are seeking a construction of a tree map that represents the same data, we know that a suitable tree map's construction, given $\pdecompositionn'$, must have
\begin{displaymath}
  \foundationsto{\cgraphn',t'}=[\mathit{tm}_1, \mathit{tm}_2, 364, \mathit{tm}_3, 198, \mathit{tm}_4, 45, \mathit{tm}_5, 43]
\end{displaymath}
where each $\mathit{tm}_i$ is an empty tree map. What remains is to identify suitable empty rectangles, which is a problem specific to the \textsc{Tree Maps} system but for which algorithms can be readily devised, to complete the foundation token-sequence. The construction we seek to produce from $\pdecompositionn'$ is in the grammatical space, which is functional by definition~\ref{defn:representationalSystem}. Thus, if we know the foundation token-sequence then the remainder of the tokens can be instantiated. Thus, we can convert the table representing the UK parliament seats data into a tree map following this process:
\begin{center}
	\begin{tikzpicture}
	\node (f) at (0,0) {$\cpair \longrightarrowY{\mathit{decompose}} \decompositionn \longrightarrowY{\mathit{match}} \pdecompositionn \longrightarrowY{\textup{transfer}}  \pdecompositionn' \longrightarrowY{\mathit{describe}} \decompositionn' \longrightarrowY{\mathit{recompose}}\cpaird.$};
	\draw[rounded corners,very thick, opacity=0.75,darkred] (-1.77,-0.5) rectangle (1.28,0.4);
	\draw[rounded corners,very thick, opacity=0.5,darkpurple] (-6.9,-0.4) rectangle (-1.1,0.3);
	\draw[rounded corners,very thick, opacity=0.5,darkpurple] (0.52,-0.4) rectangle (6.9,0.3);
	\end{tikzpicture}
\end{center}

We can extract a general way of converting a two-column table, with a heading and where the second column contains numerical values, into an equivalent tree map. Suppose that we have a table in $\rsystemn_{T}$, with construction, $\cpair$, and a decomposition, $\decompositionn$, that can be described by some $\pdecompositionn$ where:
\begin{enumerate}
\item Each non-leaf vertex, $v$, is labelled by a pattern in the equivalence class, under label-preserving isomorphism, $[\ppairp{1}]$,  and has exactly five in-vertices, $v_1,\hdots, v_5$ that are sources of $v$'s in-arrows $a_1,\hdots, a_5$, indexed in the obvious way. The vertices $v_1,\hdots, v_5$ are labelled by patterns drawn from $[\ppairp{1}]$, $[\ppairp{2}]$, $[\ppairp{3}]$, $[\ppairp{4}]$, $[\ppairp{3}]$, and $[\ppairp{5}]$.

\item A unique leaf is labelled by a pattern in $[\ppairp{2}]$.
\end{enumerate}
Then we can produce a decomposition, $\pdecompositionn'$, exploiting the pattern associations given above, where:
\begin{enumerate}
\item The root of $\pdecompositionn'$ is labelled by a pattern in the equivalence class $[(\pgraphn_1',v_1')]$, noting that we have $\ppairp{1}\longrightarrowX{}(\pgraphn_1',v_1')$.
\item Each non-leaf vertex, $v$, of $\pdecompositionn'$, other than the root, is labelled by a pattern drawn from $[(\pgraphn_2',v_2')]$, since $\ppairp{2}\longrightarrowX{}(\pgraphn_2',v_2')$, and it has exactly three in-vertices, $v_1,\hdots, v_3$ that are the source of arrows $a_1,\hdots, a_3$. The vertices $v_1,\hdots, v_3$ are  labelled by patterns drawn from $[(\pgraphn_2',v_2')]\cup [(\pgraphn_3',v_3')]$ $[(\pgraphn_5',v_5')]$, $[(\pgraphn_4',v_4')]$, noting that $\ppairp{1}\longrightarrowX{}(\pgraphn_2',v_2')$, $\ppairp{2}\longrightarrowX{}(\pgraphn_3',v_3')$, whilst $(\pgraphn_5',v_5')$ is not associated with any pattern in $\rsystemn_{\mathit{T}}$.
\item There is a unique leaf labelled by $[(\pgraphn_3',v_3')]$, noting that $\ppairp{2}\longrightarrowX{}(\pgraphn_3',v_3')$.
\end{enumerate}
The foundation tokens of the desired construction, $\cpaird$, in $\rsystemn_{\mathit{TM}}$ that are of type $\TMnumeral$ can be instantiated by appealing to $\foundations{\cgraphn,t}$. The remainder of the foundation tokens are of type $\TMemptyTreeMap$ and they can be computed (given a suitable algorithm): it is known that the construction, $\cpaird$, must yield a complete tree map, subdivided by rectangles in the desired proportions. Importantly, \textit{any} resulting construction, $\cpaird$, that is described by such a $\pdecompositionn'$ ensures that the constructed tree map, $t'$, represents the same data as the table $t$.

This example demonstrated the key steps in the structure transfer process. Based on the insights provided here, the remainder of this section formalises pattern spaces and patterns, alongside descriptions of constructions, construction spaces, and representational systems: we exploit decompositions of patterns to describe decompositions of constructions. In turn, these decompositions of patterns describe a construction space. Section~\ref{sec:structuralTransformations} builds on the resulting theory of patterns, by formally defining structural transformations using transformation constraints.

\subsection{Pattern Spaces and Patterns}\label{sec:patterns:PSAP}

The primary contribution of section~\ref{sec:patterns:PSAP} is the introduction of \textit{patterns} as constructions in a \textit{pattern space}. Such a space is associated with a construction space, $\cspacen$, exploiting the same type system and constructor specification, but with a potentially different structure graph. As such,  our convention is to denote the structure graphs of $\cspacen$ and $\pspacen$ by $\graphn$ and, respectively, $\patterngraphn$. Similarly, constructions in $\pspacen$ will be denoted using $\ppair$, rather than $\cpair$, as in section~\ref{sec:patterns:ME}. Now, not every space that exploits the same type system and the same constructor specification as $\cspacen$ is a suitable pattern space: $\pspacen$ will ensure that each construction in $\cspacen$ matches some pattern, formalised in definition~\ref{defn:patternSpace}. This builds on definition~\ref{defn:embeddingPattern} which specifies when a structure graph in $\cspacen$ \textit{specialises} a structure graph in another space. This notion of specialisation will form the basis of definition~\ref{defn:matchAndDescribe}, which formalises the notion of a \textit{pattern} and the \textit{matching} relation. More generally, we will require that a pattern space describes all constructions in a \textit{representational system}, $\rsystemn$. Now, theorem~\ref{thm:superSpace} established that the union, $\uspace{\rsystemn}$, of $\rsystemn$'s grammatical, entailment and identification spaces is also a construction space. A pattern space for $\uspace{\rsystemn}$ will contain patterns that describe constructions in $\uspace{\rsystemn}$ and, hence, in $\rsystemn$.

\begin{definition}\label{defn:embeddingPattern}
Let  $\cspacen=\cspace$ and $\pspacen=\pspace$ be construction spaces where $\tsystemn=\tsystem$. Let $\cgraphn$ and $\pgraphn$ be  subgraphs of $\graphn$ and $\patterngraphn$. We say that $\cgraphn$ is a \textit{specialisation} of $\pgraphn$, denoted $\cgraphn\matches \pgraphn$, provided there exists an isomorphism, $f\colon \cgraphn \to \pgraphn$, such that
\begin{enumerate}
%
%\item $t$ maps to $v$: $f(t)=v$,
%
\item the subtype relation is respected by $f$: for all $t'\in \pa$, the vertex-labelling functions, $\tokenl$ and $\tokenl'$, of $\cgraphn$ and $\pgraphn$ ensure that $\tokenl(t')\leq \tokenl'(f(t'))$,
\item the constructors assigned to configurators match: for all $u\in \pb$, the configurator-labelling functions, $\consl$ and $\consl'$, of $\cgraphn$ and $\pgraphn$ ensure that $\consl(u)=\consl'(f(u))$,  and
\item the arrow indices are identical: for all $a\in \arrows$, the arrow-labelling functions, $\arrowl$ and $\arrowl'$, of $\cgraphn$ and $\pgraphn$ ensure that $\arrowl(a)=\arrowl'(f(a))$.
\end{enumerate}
Such an isomorphism is called an \textit{embedding}.
\end{definition}

\begin{definition}\label{defn:patternSpace}
A \textit{pattern space} for a construction space, $\cspacen=\cspace$, is a construction space, $\pspacen=\pspace$, such that for each construction, $\cpair$, in $\cspacen$ there exists a construction, $\ppair$, in $\pspacen$ where there exists an embedding, $f\colon \cgraphn\to \pgraphn$ such that $f(t)=v$.
\end{definition}

Going forwards, for constructions $\cpair$ and $\ppair$, if we wish an embedding $f\colon \cgraphn\to \pgraphn$ to ensure that $f(t)=v$, we write $f\colon \cpair \to \ppair$. It is trivial that every construction space, $\cspacen$, has a pattern space: $\cspacen$ is its own pattern space since every construction specialises itself. However, $\cspacen$ need not have a structure graph, $\graphn$, that directly exploits all of the types in $\cspacen$'s type system: there may exist a type, $\tau$, that is not assigned to any token via $\graphn$'s vertex labelling function. However, using such types in a pattern space for $\cspacen$ can be beneficial: they can allow a single pattern to describe many constructions.

\begin{example}[Use of Non-Assigned Types in a Pattern Space]
Consider the \textsc{First-Order Arithmetic} representational system, $\rsystemn_{\FOA}$, introduced in section~\ref{sec:representationalSystems:ME}, which has many types, such as $\FOAnumexp$, that are not assigned to any token but are supertypes of the infinitely many assigned types. Suppose that we have a pattern space that includes the patterns $\ppair$ -- which contains a single vertex -- and $\ppaird$ shown below:
\begin{center}
\begin{minipage}[c]{0.2\textwidth}
\begin{tikzpicture}
\node[typeE={$\FOAnum$}] (v) at (3.2,4.2) {$v$};
\end{tikzpicture}
\end{minipage}
\qquad
\begin{minipage}[c][3cm]{0.2\textwidth}    % numerical expression
\begin{tikzpicture}[construction,yscale=0.8]
\node[typeE={$\FOAnumexp$}] (v) at (3.2,4.2) {$v'$};
\node[constructorNW={$\FOAcinfixop$}] (u) at (3.2,3.2) {};
\node[typeS={$\FOAnumexp$}] (v2) at (2.2,2.2) {};
\node[typeS={$\FOAbop$}] (v3) at (3.2,2.2) {};
\node[typeS={$\FOAnumexp$}] (v4) at (4.2,2.2) {};
%\node[termS={$\FOAopenb$}] (v1) at (1.6,2.2) {};
%\node[termS={$\FOAcloseb$}] (v5) at (4.8,2.2) {};
\path[->]
(u) edge[bend right = 0] (v)
%(v1) edge[bend right = -20] node[index label] {1} (u)
(v2) edge[bend right = -10] node[index label] {1} (u)
(v3) edge[bend right = 0] node[index label] {2} (u)
(v4) edge[bend right = 10] node[index label] {3} (u)
%(v5) edge[bend right = 20] node[index label] {5} (u)
;
\end{tikzpicture}
\end{minipage}
 \end{center}
No token in $\rsystemn_{\FOA}$ is assigned a type used in the above patterns, but many tokens are assigned subtypes. When using patterns such as $\ppair$ and $\ppaird$, only a finite number of patterns is actually required to \textit{describe} each construction in the grammatical space of $\rsystemn_{\FOA}$ via \textit{decompositions}, as illustrated for tables and tree maps in section~\ref{sec:patterns:ME}. One such finite set of patterns, $\setofpatterns$, comprises a trivial (i.e. single-vertex) pattern for each of these types: $\FOAnum$, $\FOAvar$, $\FOAplus$, $\FOAminus$, $\FOAeq$, $\FOAgr$,  $\FOAopenb$, $\FOAcloseb$, and $\FOAquant$. In addition, $\setofpatterns$ includes one basic pattern for each constructor: $\FOAcinfixop$, as shown above, along with similar patterns for the remaining three constructors, $\FOAcaddPar$, $\FOAcinfixrel$, and $\FOAcquantify$. This \textit{finite} set of patterns, $\setofpatterns$, has the property that every construction, $\cpair$, in the grammatical space has a decomposition, $\decompositionn$, where each construction in $\decompositionn$ matches a pattern in $\setofpatterns$.
\end{example}

\begin{definition}\label{defn:matchAndDescribe}
Let $\cspacen$ be a construction space. A \textit{pattern} for $\cspacen$  is a
construction, $(\pgraphn,v)$, in some pattern space for $\cspacen$. A construction, $\cpair$, in $\cspacen$ \textit{matches} $\ppair$ provided $\cpair$ specialises $\ppair$ and the associated embedding, $f\colon \cpair\to \ppair$ ensures that $f(t)=v$. Whenever $\cpair$ matches $\ppair$ we say that $\ppair$ \textit{describes} $\cpair$.
\end{definition}

We now establish that, when $\cpair$ matches $\ppair$, the associated embedding (which is trivially unique, since $t$ maps to $v$ and the arrow indices are identical) ensures that the expected relationships hold between their complete trail sequences and foundation sequences:

\begin{lemma}\label{lem:embeddingRels}
Let $\cspacen$ be a construction space. Let $\cpair$ be a construction in $\cspacen$ that matches some pattern, $\ppair$. Given an embedding, $f\colon \cpair \to \ppair$,  the following hold:
\begin{enumerate}
\item if
\begin{displaymath}
  \ctsequence{\cgraphn,t}= [[a_{1,1},\ldots,a_{j,1}],\ldots,[a_{1,n},\ldots,a_{k,n}]]]
\end{displaymath}
then
\begin{displaymath}
  \ctsequence{\pgraphn,v}= [[f(a_{1,1}),\ldots,f(a_{j,1})],\ldots,[f(a_{1,n}),\ldots,f(a_{k,n})]]].
\end{displaymath}
\item if\hspace{2pt}  $\foundationsto{\cgraphn,t}=[t_1,\ldots,t_n]$ then $\foundationsto{\pgraphn,v}=[f(t_1),\ldots,f(t_n)]$, and
\item $\foundationsty{\cgraphn,t}$ is a specialisation of\hspace{2pt} $\foundationsty{\pgraphn,v}$.
\end{enumerate}
\end{lemma}

The proof of lemma~\ref{lem:embeddingRels} can be found in Appendix~\ref{sec:app:patterns} (see lemma~\ref{lema:embeddingRels}). Lemma~\ref{lem:unionOfPatternSpaces}, the proof of which is straightforward and thus omitted, tells us that there exists a single pattern space containing (isomorphic copies of) all patterns drawn from any of $\cspacen$'s pattern spaces.

\begin{lemma}\label{lem:unionOfPatternSpaces}
Let $\cspacen=\cspace$ be a construction space with compatible pattern spaces $\pspacen_1=(\tsystemn, \cspecificationn, \patterngraphn_1)$ and $\pspacen_2=(\tsystemn, \cspecificationn, \patterngraphn_2)$. Then $\pspacen_1\cup \pspacen_2$ is a pattern space for $\cspacen$.
\end{lemma}

Thus, from this point forward, it is assumed that whenever we have multiple patterns, there is no need to specify the space from which they are drawn. Merely, we assume they are all drawn from the same pattern space, $\pspacen$. %In addition, we likewise abuse notation and write $\setofpatterns(\cspacen)$ to mean a set of patterns for $\cspacen$ that is derived from such a $\pspacen$.

\subsection{Describing Constructions, Spaces, and Systems using Decompositions of Patterns}

Using pattern decompositions, it is possible to \textit{describe} constructions, spaces, and representational systems. A key step
is to extend the concept of a construction, $\cpair$, matching a pattern, $\ppair$, to a decomposition, $\decompositionn$, matching a pattern decomposition, $\pdecompositionn$. Regarding notation, the vertex-labelling function in pattern decompositions will typically be called $\labp$, to serve as a reminder that these decomposition trees have vertices labelled by patterns.

\begin{definition}\label{defn:decompositionMatchesPattern}
Let $\cspacen$ be a construction space. Let $\decompositionn$ and $\pdecompositionn$ be decompositions of some construction in $\cspacen$ and some pattern for $\cspacen$, respectively. Then $\decompositionn$ \textit{matches} $\pdecompositionn$ provided there exists an isomorphism, $f\colon \decompositionn \to \pdecompositionn$, where
\begin{enumerate}
%\item the root of $\decompositionn$ maps to the root of $\pdecompositionn$: $f(\treeroot{\decompositionn})=\treeroot{\pdecompositionn}$, %and
    %
\item the arrow indices are identical: for all $a\in \arrows$, the arrow-labelling functions, $\arrowl$ and $\arrowl'$, of $\decompositionn$ and $\pdecompositionn$ are such that $\arrowl(a)=\arrowl'(f(a))$, and

\item there exists an embedding, $h \colon \constructionofD{\decompositionn} \to \constructionofD{\pdecompositionn}$, such that for all vertices, $v_i$, in $\decompositionn$, given the vertex-labelling functions, $\labc$ and $\labp$, of $\decompositionn$ and $\pdecompositionn$, the function $h_i\colon \labc(v_i)\to \labp(f(v_i))$ obtained by restricting the domain of $h$ to $\labc(v_i)$ is an embedding.
\end{enumerate}
Such an isomorphism, $f\colon \decompositionn \to \pdecompositionn$, is called an \textit{embedding}. Whenever $\decompositionn$ matches $\pdecompositionn$ we say that $\pdecompositionn$ \textit{describes} $\decompositionn$.
\end{definition}

As such, when we have a construction, $\cpair$, with a decomposition, $\decompositionn$, that matches $\pdecompositionn$ we will simply say that $\cpair$ matches $\pdecompositionn$. The implications of definition~\ref{defn:decompositionMatchesPattern} in the context of the structure transfer process are visualised here:
\begin{center}
	\begin{tikzpicture}[yscale=0.8]\small
	\node (f) at (0,0) {$\cpair \longrightarrowY{\mathit{decompose}} \decompositionn \longrightarrowY{\mathit{\vphantom{p}match}} \pdecompositionn \longrightarrowY{\mathit{\vphantom{p}transfer}}  \pdecompositionn' \longrightarrowY{\mathit{\vphantom{p}describe}} \decompositionn' \longrightarrowY{\mathit{recompose}}\cpaird$};
    \node (g) at (0,-1.5) {};
	\node (con) at (-1.55,-1.3) {$\constructionofD{\pdecompositionn}$};
	\node (con') at (1.15,-1.3) {$\constructionofD{\pdecompositionn'}$};
	\draw[rounded corners,very thick, opacity=0.75,darkred] (-1.7,-0.55) rectangle (1.3,0.45);
	\draw[rounded corners,very thick, opacity=0.5,darkpurple] (-6.8,-0.45) rectangle (-1.0,0.35);
	\draw[rounded corners,very thick, opacity=0.5,darkpurple] (0.57,-0.45) rectangle (6.8,0.35);
%
%\draw[ ->] (-1.37,-0.3) --(-1.37,-1.3);
\draw[ ->] (-1.37,-0.3) --([xshift=0.18cm]con.north);
%\draw[ ->] (0.93,-0.3) --(0.93,-1.3);
\draw[ ->] (0.93,-0.3) --([xshift=-0.22cm]con'.north);
\draw[ ->] (-5.8,-0.3) --node[above,yshift=-0.29cm,xshift=1.1cm]{\rotatebox{-10}{\footnotesize$\mathit{match}$}} (con);
\draw[->] (con') -- node[above,yshift=-0.29cm,xshift=-1.1cm]{\rotatebox{9}{\footnotesize$\mathit{describe}$}} (5.62,-0.35);
\draw[rounded corners, very thick, opacity=1,darkblue] (-6.9,-1.6) rectangle (-0.9,0.55);
\draw[rounded corners, very thick, opacity=1,darkblue] (0.45,-1.6) rectangle (6.9,0.55);
%
%\node (h) at (-6,-1.6) {\textcolor[rgb]{0.00,0.07,1.00}{theorem~\ref{thm:decMatchPdecImpliesCpairMatchesPpair}}};
%\node (i) at (5.9,-1.6) {\textcolor[rgb]{0.00,0.07,1.00}{corollary~\ref{cor:decMatchPdecImpliesCpairMatchesPpair}}};
	\end{tikzpicture}
\end{center}

So far, we have defined how patterns and pattern decompositions are  used to describe constructions and their decompositions. These concepts are now extended to capture how construction spaces are described.

\begin{definition}\label{defn:descriptionOfCSpace}
Let $\cspacen$ be a construction space. A \textit{description} of $\cspacen$ is a set, $\descriptionn$, of decompositions arising from patterns for $\cspacen$.
\end{definition}

In section~\ref{sec:structuralTransformations}, it is useful for us to define the set of vertices associated with a description, $\descriptionn$: these are vertices that appear in decompositions in $\descriptionn$.

\begin{definition}\label{defn:verticesOfDecomposition}
Let $\rsystemn=\rsystem$ be a representational system  with description $\descriptionn$. We define the set of \textit{vertices of $\descriptionn$}, denoted $V(\descriptionn)$, by
\begin{displaymath}
  V(\descriptionn)=\bigcup\limits_{\decompositionn\in \descriptionn} V(\decompositionn)
\end{displaymath}
where $V(\decompositionn)$ is the set of vertices in the decomposition $\decompositionn$.
\end{definition}

Definitions~\ref{defn:rsdescription} and~\ref{defn:irsedescription} extend the notion of a description to a representational system, $\rsystemn$, and, respectively, an inter-representational-system encoding, $\irencodingn=\irencoding$. Definition~\ref{defn:rsdescription} requires us to form the universal construction space, $\uspace{\rsystemn}$, of $\rsystemn$, which is the union of the grammatical, entailment, and identification spaces. Definition~\ref{defn:irsedescription} will exploit $\ipispacen=\ispacen\cup \ispacen'\cup \ispacen''$, where $\ispacen$ and $\ispacen'$ are the identification spaces of $\rsystemn$ and $\rsystemn'$ respectively, for reasons that will become clear later (e.g. example~\ref{ex:STMErerepresentingTokenSE}, see the $\ipispacen$ patterns on page~\pageref{ex:pageref:ispaceexample}). Notably, our approach to structural transformations does not require decompositions of patterns drawn from a pattern space for $\ipispacen$. Rather, as mentioned above,a set of patterns drawn from such a pattern space forms a transformation constraint.

\begin{definition}\label{defn:rsdescription}
Let $\rsystemn=\rsystem$ be a representational system. Any description of the universal construction space, $\uspace{\rsystemn}$, is a \textit{description} of $\rsystemn$.
\end{definition}

\begin{definition}\label{defn:irsedescription}
Let $\irencoding$ be an inter-representational-system encoding.  A \textit{description} of $\irencoding$ is a tuple, $\descriptionirse$ where $\descriptionn$ and $\descriptionn'$ are descriptions of $\rsystemn$ and $\rsystemn'$, resp., and $\setofpatterns''$ is a set of patterns for $\ipispacen$.
\end{definition}

Section~\ref{sec:structuralTransformations} formalises the notion of structural transformation, which takes an inter-repre\-sentational-system encoding, $\irencoding$, along with a description, $\descriptionirse$, and converts constructions in $\rsystemn$ into constructions in $\rsystemn'$. In this context, $\ipispacen$ is used to capture relations that are required to hold between tokens in the respective constructions, $\cpair$ and $\cpaird$: in a transformation constraint, $(\ppair, \ppaird, \setofpatternstc)$, the set of patterns $\setofpatternstc$ describes the relationships that are desired to hold between tokens in $\cpair$ and $\cpaird$. This aspect of structural transformations, which enables \textit{targeted representation selections} to be made, will be exemplified in section~\ref{sec:structureTransfer:ME}.

%in a transformation constraint, $(\ppair, \ppaird, (\pgraphn'',v''))$, the pattern $(\pgraphn'',v'')$ describes the relations that are desired to hold between tokens in $\cpair$ and $\cpaird$. This aspect of structural transformations, which enables \textit{intelligent representation selection}, will be exemplified in section~\ref{sec:structureTransfer:ME}.

\subsection{Properties of Descriptions: Completeness and Compactness}

A description, $\descriptionn$, of a construction space, $\cspacen$, can have desirable properties: \textit{completeness} and \textit{compactness}. The former relates to whether each construction, $\cpair$, in $\cspacen$ has a decomposition with a description, $\pdecompositionn$, in $\descriptionn$. If such a $\pdecompositionn$ exists, it means that $\cpair$ can be built from only sub-constructions that match patterns in the set of patterns, $\setofpatterns(\descriptionn)$, that are used by decompositions in $\descriptionn$. In other words, it is possible to split $\cpair$ into a generator, $(\cgraphn',t)$, and a sequence of induced constructions, $[\cpairp{1},\ldots,\cpairp{n}]$, where $(\cgraphn',t)$ matches a pattern in $\setofpatterns(\descriptionn)$ and each $\cpairp{i}$ has a description in $\descriptionn$. Compactness arises when the set of equivalence classes of $\setofpatterns(\descriptionn)$, under label-preserving isomorphism, is finite.
\begin{definition}\label{defn:complete}
	Let $\cspacen$ be a construction space with description $\descriptionn$. Then $\descriptionn$ is \textit{complete} provided every construction, $\cpair$, in $\cspacen$ has a decomposition, $\decompositionn$, that is described by some decomposition, $\pdecompositionn$, in $\descriptionn$.
\end{definition}
\begin{definition}\label{defn:descriptionPatterns}
	Let $\cspacen$ be a construction space with description $\descriptionn$. The set of \textit{$\descriptionn$-patterns}, denoted $\setofpatterns(\descriptionn)$, is defined to be
	\begin{displaymath}
	\setofpatterns(\descriptionn)= \bigcup\limits_{\substack{\pdecompositionn\in \descriptionn \\ \pdecompositionn=\pdecomposition}} \{\ppair \colon \exists \thinspace v'\in V \; \labp(v')=\ppair \}.
	\end{displaymath}
	The set of equivalence classes of $\setofpatterns(\descriptionn)$, under label-preserving isomorphism, is denoted $\patternequivclasses{\descriptionn}$.
\end{definition}
\begin{definition}\label{defn:compact}
	Let $\descriptionn$ be some description of a construction space. Then $\descriptionn$ is \textit{compact} whenever $\patternequivclasses{\descriptionn}$ is finite.
\end{definition}

In the context of structure transfer, given $\irencoding$, it is desirable to exploit complete and compact descriptions, $\descriptionn$ and $\descriptionn'$, of $\rsystemn$ and $\rsystemn'$. Completeness is important, otherwise there would be constructions which cannot be decomposed into anything matching a description in $\descriptionn$\footnote{This is akin to a complete set of inference rules in a logic: a set of rules is complete if every true formula has a proof. Different sets of inference rules can yield different proofs. Analogously, different descriptions may describe different decompositions of a construction.}. Given compact $\descriptionn$ and $\descriptionn'$ along with a finite $\setofpatterns''$ for $\ipispacen$\footnote{It would be readily possible to extend the concepts of completeness and compactness to $\setofpatterns''$, using label-isomorphism equivalence classes and examining properties of the set of types to derive sufficient conditions for a complete, compact set of patterns for $\ipispacen$.}, any chosen set of (equivalence classes\footnote{The particular equivalence relation is given in definition~\ref{defn:equivalentTokenRelSpecs}, based on embeddings which are restricted forms of isomorphism.} of) transformation constraints, used for structure transfer, is necessarily finite. Thus, from a practical perspective, compactness ensures that there are only finitely many choices of $\pdecompositionn'$ that can arise from $\pdecompositionn$ in the process of structure transfer. In turn, this means that there is a terminating algorithm that begins with $\cpair$ and produces $\pdecompositionn'$, should it exist, which can be considered a partial transformation of $\cpair$ into some $\cpaird$: the transfer is partial since $\pdecompositionn'$ indicates one way of building $t'$, using specific constructors in particular combinations, without specifying the actual tokens from which $t'$ is constructed. Theorem~\ref{thm:completeCompactConditions} provides sufficient, but not necessary, conditions on a construction space, $\cspacen$, for the existence of a complete, compact description. The proof strategy employed builds such a description, starting with patterns that match the constructions in $\cspacen$. As a pre-requisite, lemma~\ref{lem:decTrivialOrBasic} proves that every construction has a decomposition where each sub-construction is either trivial (contains a single vertex) or basic (contains a single configurator).

\begin{definition}\label{defn:decoupling}
Let $\cspacen=\cspace$ be a construction space containing construction $\cpair$. A \textit{decoupling} of $\cpair$ is a decomposition, $\decompositionn$, of $\cpair$ where, for every $v\in V$, its label, $\labc(v)$, is trivial or basic.
\end{definition}

\begin{example}[Decoupling Decompositions]\label{ex:decoupling} Consider the pattern space for $\rsystemn_{\mathit{T}}$, the \textsc{Tables} representational system. The example given in section~\ref{sec:patterns:ME} used the pattern $\ppairp{1}$, on page~\pageref{pageref:patternsppairp1}, when invoking structure transfer: some of the vertices in $\pdecompositionn$, on page~\pageref{ex:decompostionForStructureTransfer}, were labelled by patterns that are label-isomorphic to $\ppairp{1}$. The decomposition $\pdecompositionn$ was used to convert the UK parliament data table into a tree map. However, since $\ppairp{1}$ is not trivial or basic -- it has four configurators --  $\pdecompositionn$ is not a decoupling of $\constructionofD{\pdecompositionn}$: obtaining a decoupling would require all of the patterns labelling the vertices of $\pdecompositionn$ to be further decomposed. Here we provide a decoupling of $\ppairp{1}$, omitting the remaining steps required to decouple $\constructionofD{\pdecompositionn}$:
\begin{center}
	\scalebox{0.8}{\decoupling}
\end{center}
\end{example}

\begin{lemma}\label{lem:decTrivialOrBasic}
Let $\cspacen=\cspace$ be a construction space containing construction $\cpair$. Then there exists a decomposition, $\decompositionn$, that is a decoupling of $\cpair$.
\end{lemma}

See Appendix~\ref{sec:app:patterns}, lemma~\ref{lema:decTrivialOrBasic}, for a proof of lemma~\ref{lem:decTrivialOrBasic}. Our next step towards the proof of theorem~\ref{thm:completeCompactConditions} establishes that if  the construction $\cpair$ matches the pattern $\ppair$, which has decomposition $\pdecompositionn$, then $\pdecompositionn$ essentially determines a unique decomposition, $\decompositionn$, of $\cpair$ which matches $\pdecompositionn$. This fact allows us to define a \textit{$\pdecompositionn$-canonical} decomposition of $\cpair$. This is useful for the proof strategy employed by theorem~\ref{thm:completeCompactConditions}: build a description, $\descriptionn$, of $\cspacen$ and show it is complete and compact. Lemma~\ref{lem:CpairMatchesPpairImpliesExistsdecMatchPdecImplies} allows us to choose a decoupling decomposition of $\ppair$ to include in $\descriptionn$ that describes some decomposition of $\cpair$, namely the $\pdecompositionn$-canonical decomposition, establishing completeness.

\begin{lemma}\label{lem:CpairMatchesPpairImpliesExistsdecMatchPdecImplies}
Let $\cspacen$ be a construction space containing construction $\cpair$ that matches pattern $\ppair$. Given a decomposition, $\pdecompositionn$, of $\ppair$, there exists a decomposition, $\decompositionn$, of $\cpair$ such that $\decompositionn$ matches $\pdecompositionn$.
\end{lemma}

\begin{proof}[Proof Sketch]
Suppose that $\cpair$ matches $\ppair$ and let $\pdecompositionn$ be some decomposition of $\ppair$. Since $\cpair$ matches $\ppair$, there exists an embedding, say $f\colon \cpair \to \ppair$.  Consider the pattern, $(\pgraphn',v)$, that labels the root of $\pdecompositionn$. Since  the isomorphism $f\colon \cpair\to \ppair$ maps $t$ to $v$, and $(\pgraphn',v)$ is a generator for $\ppair$, we use $(\pgraphn',v)$ to choose a particular generator, $(\cgraphn',t)$, of $\cpair$: take $\cgraphn'$ to be the largest subgraph of $\cgraphn$ containing vertices that map, under $f$, to vertices of $\pgraphn'$. To construct a suitable $\decompositionn$, take $(\cgraphn',t)$ to label a vertex, $v_r$, that will form the root of $\decompositionn$. Since splits are uniquely determined by their generators, an induction approach gives that the induced construction sequence, $[\ic_1,\ldots,\ic_n]$, obtained from $\cpair$ and with generator $(\cgraphn',v)$ each have decompositions, $\decompositionn_1,\ldots, \decompositionn_n$, that match the respective $\dsequence{v_r'}=[\pdecompositionn_1,\ldots, \pdecompositionn_n]$, where $v_r'$ is the root of $\pdecompositionn$. Take the disjoint union of $\decompositionn_1,\ldots, \decompositionn_n$ and join them with arrows, indexed in the obvious way, to $v_r$. The result is a decomposition of $\cpair$ that matches $\pdecompositionn$.
\end{proof}

Consequently, if we have a description, $\pdecompositionn$, of $\cpair$ then there is a unique generator, $(\cgraphn',t)$, of $\cpair$ that matches the pattern that labels the root of $\pdecompositionn$. By extension, $\pdecompositionn$ uniquely determines the splits applied to the induced constructions of $(\cgraphn',t)$ when yielding a decomposition, $\decompositionn$, of $\cpair$ where $\decompositionn$ matches $\pdecompositionn$. So, any given description, $\pdecompositionn$, of $\cpair$ uniquely determines a decomposition, $\decompositionn$, of $\cpair$ up to label-preserving isomorphism. This allows us to define the \textit{canonical} decomposition of $\cpair$ given $\pdecompositionn$, where $\decompositionn$ is obtained by replacing the patterns that label vertices of $\pdecompositionn$ with the induced constructions that are produced by the determined splits. This concept is embodied in the next definition.

\begin{definition}\label{defn:deltaCanonical}
Let $\cspacen$ be a construction space containing construction $\cpair$ that matches $\ppair$ which has decomposition $\pdecompositionn=\pdecomposition$. The decomposition $\decompositionn=\linebreak \decomposition$ of $\cpair$, where $\decompositionn$ matches $\pdecompositionn$, is the \textit{$\pdecompositionn$-canonical} decomposition of $\cpair$, denoted $\deltacanonical{\cpair,\pdecompositionn}$.
\end{definition}

\begin{theorem}\label{thm:completeCompactConditions}
Let $\cspacen=\cspace$ be a construction space formed over $\tsystemn=\tsystem$ where $\cspecificationn=\cspecification$. If
\begin{enumerate}
\item $\constructors$ is finite,
\item the set of maximal elements of $\leq$ is finite, and
\item for every type, $\tau \in \types$, there exists a maximal type, $\tau'$, of \ $\tsystemn$ such that $\tau \leq \tau'$
\end{enumerate}
then $\cspacen$ has a complete, compact description.
\end{theorem}

\begin{proof}
We construct a complete, compact description of $\cspacen$. The proof strategy starts by showing that each construction, $\cpair$, in $\cspacen$ matches some pattern whose vertices are labelled only by maximal elements of $\leq$. Let $\cpair$ be a construction in $\cspacen$. Trivially, $\cpair$ matches some pattern, $\ppair$, where $\pgraphn=\graph$. We construct a pattern, $(\pgraphn',v)$, that can be obtained from $\ppair$ by replacing types as follows. For every type, $\tau$, in $\types$, there exists a maximal type, $\tau'$, such that $\tau\leq \tau'$. Therefore, we can create a new pattern, $(\pgraphn',v)$, from $\ppair$ such that $\pgraphn'=(\pa,\pb,\arrows,\incVert,\arrowl,\tokenl')$ where for each $v\in \pa$, $\tokenl'(v)=\tau$ and $\tau$ is a maximal element that bounds $\tokenl(v)$, that is $\tokenl(v)\leq \tau$. Trivially, since $\cpair$ matches $\ppair$, it follows that $\cpair$ matches $(\pgraphn',v)$.  Hence $\cpair$ matches some pattern whose vertices are labelled only by maximal elements of $\leq$. For each construction, $\cpair$, in $\cspacen$ choose one such matching pattern and define $\setofpatterns_m$ to be the set of such patterns.

By lemma~\ref{lem:decTrivialOrBasic}, each $(\pgraphn',v)$ in $\setofpatterns_m$ has a decoupling decomposition, $\pdecompositionn$. Produce a decoupling decomposition of each pattern in $\setofpatterns_m$, to yield a description, $\descriptionn$ of $\cspacen$. By lemma~\ref{lem:CpairMatchesPpairImpliesExistsdecMatchPdecImplies}, there exists a decomposition of $\cpair$ -- namely the $\pdecompositionn$-canonical decomposition, $\deltacanonical{\cpair,\pdecompositionn}$ -- that is described by the selected decoupling decomposition, $\pdecompositionn$, of $(\pgraphn',v)$. Now, $\descriptionn$ was produced by taking, for each $\cpair$ in $\cspacen$, a decoupling decomposition of a matching pattern, thus $\descriptionn$ is complete.

What remains is to show that $\descriptionn$ is compact. Since each $(\pgraphn',v)$ in $\setofpatterns_m$ includes only tokens whose types are maximal, the same is true of the patterns that label the vertices in the selected decoupling decomposition of $(\pgraphn',v)$ in $\descriptionn$: splitting constructions does not change the assigned types. Thus,  $\setofpatterns(\descriptionn)$ contains only patterns whose vertices are labelled by maximal types. Moreover, $\setofpatterns(\descriptionn)$ contains only patterns that are trivial (contain exactly one vertex) or basic (contain exactly one configurator), since it is derived from decoupling decompositions. There is only a finite number of trivial constructions in $\setofpatterns(\descriptionn)$, up to label-preserving isomorphism, whose vertices are labelled by maximal types, since there are finitely many maximal types. In addition, since there are
\begin{enumerate}
\item[-] finitely many constructors, each of which has a finite number of input types and exactly one output type, and
\item[-]  finitely many maximal elements of $\leq$,
\end{enumerate}
there is only a finite number of basic constructions, up to label-preserving isomorphism, whose vertices are labelled by maximal types. Combining these two facts allows us to deduce that $\patternequivclasses{\descriptionn}$ is finite. Therefore, $\descriptionn$ is a complete, compact description of $\cspacen$.
\end{proof}

From theorem~\ref{thm:completeCompactConditions}, it is easy to see that if both $\constructors$ and $\types$ are finite then a complete, compact description exists. It should also be evident, from example~\ref{ex:decoupling} and section~\ref{sec:patterns:ME}, that \textit{decoupling decompositions} -- whilst useful for proving theorem~\ref{thm:completeCompactConditions} -- are not always sufficient for effective structure transfer. It is anticipated, in a practical setting where one wishes to exploit algorithms for producing structural transformations, that a complete, compact description, $\descriptionn$, would be based on decouplings augmented with additional decompositions involving a finite number of multi-configurator patterns, such as $\ppairp{1}$ in example~\ref{ex:decoupling}. Such descriptions would still be complete and remain compact.

%% file: dotAndFOAcommands.tex
\newcommand{\crcD}{%
	\begin{tikzpicture}[inner sep = 0pt,baseline]
	\draw (0,0) circle (.074cm);
	\fill[black, opacity = 0.15] (0,0) circle (.074cm);
	\end{tikzpicture}%
}
\newcommand{\oo}{%
	\begin{tikzpicture}[inner sep = 0pt,baseline,scale=0.935]
	\node[anchor = east] (o1) at (0,0) {\crcD};
	\node[anchor = east] (o2) at (6pt,0) {\crcD};
	\end{tikzpicture}%
}
\newcommand{\ooo}{%
\begin{tikzpicture}[inner sep = 0pt,baseline,scale=0.935]
\node[anchor = east] (o1) at (0,0) {\crcD};
\node[anchor = east] (o2) at (6pt,0) {\crcD};
\node[anchor = east] (o3) at (12pt,0) {\crcD};
\end{tikzpicture}%
}

\newcommand{\oooo}{%
\begin{tikzpicture}[inner sep = 0pt,baseline,scale=0.935]
\node[anchor = east] (o1) at (0,0) {\crcD};
\node[anchor = east] (o2) at (6pt,0) {\crcD};
\node[anchor = east] (o3) at (12pt,0) {\crcD};
\node[anchor = east] (o4) at (18pt,0) {\crcD};
\end{tikzpicture}
}%

\newcommand{\oT}{%
\begin{tikzpicture}[inner sep = 0pt,baseline,scale=0.935]
\node[anchor = east] (o1) at (0,0) {\crcD};
\node[anchor = east] (o2) at (6pt,0) {\crcD};
\node[anchor = east] (o1) at (12pt,0) {\crcD};
\node[anchor = east] (o1) at (0,6pt) {\crcD};
\node[anchor = east] (o2) at (6pt,6pt) {\crcD};
\node[anchor = east] (o1) at (0,12pt) {\crcD};
\end{tikzpicture}%
}
\newcommand{\oTrot}{%
	%\rotatebox{180}{%
		\begin{tikzpicture}[inner sep = 0pt,baseline,scale=0.935]
		\node[anchor = east] (o1) at (12pt,12pt) {\crcD};
		\node[anchor = east] (o2) at (6pt,12pt) {\crcD};
		\node[anchor = east] (o1) at (0pt,12pt) {\crcD};
		\node[anchor = east] (o1) at (12pt,6pt) {\crcD};
		\node[anchor = east] (o2) at (6pt,6pt) {\crcD};
		\node[anchor = east] (o1) at (12pt,0pt) {\crcD};
		\end{tikzpicture}%
	%}%
}
\newcommand{\oR}{%
	\begin{tikzpicture}[inner sep = 0pt,baseline,scale=0.935]
	\node[anchor = east] (o1) at (0,0) {\crcD};
	\node[anchor = east] (o2) at (6pt,0) {\crcD};
	\node[anchor = east] (o1) at (12pt,0) {\crcD};
	\node[anchor = east] (o1) at (18pt,0) {\crcD};
	\node[anchor = east] (o1) at (0,6pt) {\crcD};
	\node[anchor = east] (o2) at (6pt,6pt) {\crcD};
	\node[anchor = east] (o1) at (12pt,6pt) {\crcD};
	\node[anchor = east] (o1) at (18pt,6pt) {\crcD};
	\node[anchor = east] (o1) at (0,12pt) {\crcD};
	\node[anchor = east] (o2) at (6pt,12pt) {\crcD};
	\node[anchor = east] (o1) at (12pt,12pt) {\crcD};
	\node[anchor = east] (o1) at (18pt,12pt) {\crcD};
	\end{tikzpicture}%
}
\newcommand{\oLo}{%
\begin{tikzpicture}[inner sep = 0pt,baseline,scale=0.935]
\node[anchor = east] (o1) at (0,0) {\crcD};
\node[anchor = east] (o2) at (6pt,0) {\crcD};
\node[anchor = east] (o3) at (0,6pt) {\crcD};
\node[anchor = east] (o4) at (6pt,6pt) {\crcD};
\end{tikzpicture}%
}
\newcommand{\oL}{%
\begin{tikzpicture}[inner sep = 0pt,baseline,scale=0.935]
\node[anchor = east] (o1) at (0,0) {\crcD};
\node[anchor = east] (o2) at (6pt,0) {\crcD};
\node[anchor = east] (o3) at (0,6pt) {\crcD};
\end{tikzpicture}%
}
\newcommand{\oV}{%
\begin{tikzpicture}[inner sep = 0pt,baseline,scale=0.935]
\node[anchor = east] (o1) at (0,0) {\crcD};
\node[anchor = east] (o1) at (0,6pt) {\crcD};
\node[anchor = east] (o1) at (0,12pt) {\crcD};
\end{tikzpicture}%
}
\newcommand{\oH}{%
\begin{tikzpicture}[inner sep = 0pt,baseline,scale=0.935]
\node[anchor = east] (o1) at (0,0) {\crcD};
\node[anchor = east] (o2) at (6pt,0) {\crcD};
\node[anchor = east] (o1) at (12pt,0) {\crcD};
\node[anchor = east] (o1) at (18pt,0) {\crcD};
\end{tikzpicture}%
}

%% FOA configs

\newcommand{\foau}{
\node[constructorNE = {\FOAcisValid}] at (0,-1.1) (u) {};
\path[->]
(u) edge (v)
(v1) edge[bend right = -5] node[index label] {1} (u);}

\newcommand{\foauone}{
\node[constructorNE = {\FOAcinfixrel}] at (0,-3.2) (u1) {};
\path[->]
(u1) edge (v1)
(v2) edge[bend right = -10] node[index label] {1} (u1)
(v3) edge[bend right = -5] node[index label] {2} (u1)
(v4) edge[bend right = 10] node[index label] {3} (u1);}

\newcommand{\foautwo}{
\node[constructorNW = {\FOAcinfixop}] at (-2.2,-4.9) (u2) {};
\path[->]
(u2) edge (v2)
(v5) edge[bend right = -10] node[index label] {1} (u2)
(v6) edge[bend right = -5] node[index label] {2} (u2)
(v7) edge[bend right = 10] node[index label] {3} (u2);}

\newcommand{\foaufive}{
\node[constructorNW = {\FOAcinfixop}] at (-3.7,-6.7) (u5) {};
\path[->]
(u5) edge (v5)
(v11) edge[bend right = -10] node[index label] {1} (u5)
(v12) edge[bend right = 5] node[index label] {2} (u5)
(v13) edge[bend right = 10] node[index label] {3} (u5);}

\newcommand{\foaufour}{
\node[constructorNE = {\FOAcfrac}] at (2.3,-4.7) (u4) {};
\path[->]
(u4) edge (v4)
(v8) edge[bend right = -5] node[index label] {1} (u4)
(v9) edge[bend right = 5] node[index label] {2} (u4)
(v10) edge[bend right = 10] node[index label] {3} (u4);}

\newcommand{\foaueight}{
\node[constructorpos = {\FOAcimplicitMult}{166}{0.82cm}] at (0.7,-5.95) (u8) {};
\path[->]
(u8) edge (v8)
(v14) edge[bend right = -10] node[index label] {1} (u8)
(v15) edge[bend right = 10] node[index label] {2} (u8);}

\newcommand{\foaufifteen}{
\node[constructorNE = {\FOAcaddPar}] at (1.8,-7.4) (u15) {};
\path[->]
(u15) edge (v15)
(v16) edge[bend right = -10] node[index label] {1} (u15)
(v17) edge[bend right = -5] node[index label] {2} (u15)
(v18) edge[bend right = 10] node[index label] {3} (u15);}

\newcommand{\foauseventeen}{
\node[constructorNW = {\FOAcinfixop}] at (1.6,-9.3) (u17) {};
\path[->]
(u17) edge (v17)
(v19) edge[bend right = -10] node[index label] {1} (u17)
(v20) edge[bend right = -5] node[index label] {2} (u17)
(v21) edge[bend right = 10] node[index label] {3} (u17);}

%%% dots constructors with their edges
\newcommand{\dotsu}{
\node[constructorINE = {\DDisTranslationOf}] at (0,-1.3) (u) {};
    \path[->]
	(v1) edge[bend right = -10] node[ index label] {1} (u)
	(v2) edge[bend right = 10] node[ index label] {2} (u)
	(u) edge (v);}

\newcommand{\dotsuone}{
\node[constructorGNE ={\DDcstackLeft}] at (-1.5,-2.9) (u1) {};
	\path[->]
	(v3) edge[bend right = -10] node[ index label] {1} (u1)
	(v4) edge[bend right = 10] node[ index label] {2} (u1)
	(u1) edge (v1);}

\newcommand{\dotsutwo}{\node[constructorGW = {\DDcreflect}] at (1.5,-3) (u2) {};
    \path[->]
    (v5) edge[bend right = -10] node[index label] {1} (u2)
    %(v2') edge[bend right = 10] node[index label] {2} (u2)
    (u2) edge (v2);}

\newcommand{\dotsuthree}{\node[constructor = {\DDcstackLeft}] at (-2.4,-4.8) (u3) {};
	\path[->]
	(v6) edge[bend right = -10] node[ index label] {1} (u3)
	(v7) edge[bend right = 10] node[ index label] {2} (u3)
	(u3) edge (v3);}

\newcommand{\dotsufive}{\node[constructorGNE = {\DDcremove}] at (1.5,-5.3)(u5) {};	
	\path[->]
	(v8) edge[bend right = -10] node[index label] {1} (u5)	
	(v2') edge[bend right = 10] node[index label] {2} (u5)	
	(u5) edge (v5);	}

\newcommand{\dotsufivep}{\node[constructorGNE = {\DDcremove}] at (1.5,-5.3)(u5) {};	
\path[->]
(v8) edge[bend right = -10] node[index label] {1} (u5)	
(v2') edge[bend right = 10] node[index label] {2} (u5)	
(u5) edge (v5);	}

\newcommand{\dotsueight}{\node[constructorpos = {\DDcrectangulate}{168}{0.88cm}] at (0.3,-6.7) (u8) {};
	\path[->]
	(v9) edge[bend right = -10] node[index label] {1} (u8)
	(v10) edge[bend right = 10] node[index label] {2} (u8)
	(u8) edge (v8);}

\newcommand{\dotsueightp}{\node[constructorpos = {\DDcrectangulate}{168}{0.88cm}] at (0.3,-6.7) (u8) {};
\path[->]
(v9) edge[bend right = -10] node[index label] {1} (u8)
(v10) edge[bend right = 10] node[index label] {2} (u8)
(u8) edge (v8);}

\newcommand{\dotsuten}{\node[constructorGNE = {\DDcappend}] at (1.7,-7.7) (u10) {};
	\path[->]
	(u10) edge (v10)
	(v11) edge[bend right = -10] node[index label] {1} (u10)
	(v12) edge[bend right = 10] node[index label] {2} (u10);}

\newcommand{\dotsutenp}{\node[constructorGNE = {\DDcappend}] at (1.7,-7.7) (u10) {};
\path[->]
(u10) edge (v10)
(v11) edge[bend right = -10] node[index label] {1} (u10)
(v12) edge[bend right = 10] node[index label] {2} (u10);}

%%% dots diagram individual term nodes

\newcommand{\dotsv}{\node[termIE = {$s_1$}] at (0,-0.4) (v) {$\top$};}
\newcommand{\dotsvp}{\node[typeIE = {\typetrue}] at (0,-0.4) (v) {$\bcolour{b_1}$};}

\newcommand{\dotsvone}{\node[termW = {$s_2$},inner sep = 1pt] at (-1.5,-1.9) (v1) {\scalebox{0.85}{\oT}};}
\newcommand{\dotsvonep}{\node[typeW = {\DDstack}] at (-1.5,-1.9) (v1) {$\bcolour{b_2}$};}

\newcommand{\dotsvtwo}{\node[termE = {$s_7$},inner sep = 1pt] at  (1.5,-2) (v2) {\phantom{\crcD}\scalebox{0.85}{\oT}};}
\newcommand{\dotsvtwop}{\node[typeE = {\DDstack}] at  (1.5,-2) (v2) {$\bcolour{b_7}$};}

%%% new
\newcommand{\dotsvtwonew}{\node[termE = {$s_{10}$},inner sep = 1pt] at (2.9,-5.8) (v2') {\scalebox{0.85}{\oT}\phantom{\crcD}};}
\newcommand{\dotsvtwonewp}{\node[typeE = {\DDstack}] at (2.9,-5.8) (v2') {$\bcolour{b_{10}}$};}

\newcommand{\dotsvthree}{\node[termW = {$s_3$},inner sep = 1pt] at (-2.4,-3.8) (v3) {\scalebox{0.85}{\oL}};}
\newcommand{\dotsvthreep}{\node[typeW = {\DDstack}] at (-2.4,-3.8) (v3) {$\bcolour{b_3}$};}

\newcommand{\dotsvfour}{\node[termS = {$s_4$},inner sep = 1pt] at (-0.6,-3.8) (v4) {\scalebox{0.85}{\ooo}};}
\newcommand{\dotsvfourp}{\node[typeS = {\DDstack}] at (-0.6,-3.8) (v4) {$\bcolour{b_4}$};}

\newcommand{\dotsvfive}{\node[termSW = {$s_8$},inner sep = 1pt] at (1.3,-4.3) (v5) {\phantom{\crcD}\scalebox{0.85}{\oTrot}};}
\newcommand{\dotsvfivep}{\node[typeSW = {\DDtriangulation}] at (1.3,-4.3) (v5) {$\bcolour{b_8}$};}

\newcommand{\dotsvsix}{\node[termS = {$s_5$},inner sep = 1pt] at (-3.2,-5.6) (v6) {\scalebox{0.85}{\crcD}};}
\newcommand{\dotsvsixp}{\node[typeS = {\DDstack}] at (-3.2,-5.6) (v6) {$\bcolour{b_5}$};}

\newcommand{\dotsvseven}{\node[termS = {$s_6$},inner sep = 1pt] at (-1.6,-5.6) (v7) {\scalebox{0.85}{\oo}};}
\newcommand{\dotsvsevenp}{\node[typeS = {\DDstack}] at (-1.6,-5.6) (v7) {$\bcolour{b_6}$};}

\newcommand{\dotsveight}{\node[termW = {$s_9$},inner sep = 1pt] at (0.2,-5.7) (v8) {\scalebox{0.85}{\oR}};}
\newcommand{\dotsveightp}{\node[typeW = {\DDrectangulation}] at (0.2,-5.7) (v8) {$\bcolour{b_9}$};}

\newcommand{\dotsvnine}{\node[termS = {$s_{11}$},inner sep = 1pt] at (-0.7,-7.7) (v9) {\scalebox{0.85}{\oV}\phantom{\crcD\crcD\crcD\,}};}
\newcommand{\dotsvninep}{\node[typeS = {\DDverDots}] at (-0.7,-7.7) (v9) {$\bcolour{b_{11}}$};}

\newcommand{\dotsvten}{\node[termNE = {$s_{12}$},inner sep = 1pt] at (1.7,-6.8) (v10) {\scalebox{0.85}{\oH}};}
\newcommand{\dotsvtenp}{\node[typeNE = {\DDhorDots}] at (1.7,-6.8) (v10) {$\bcolour{b_{12}}$};}

\newcommand{\dotsveleven}{\node[termS = {$s_{13}$},inner sep = 1pt] at (0.9,-8.5) (v11) {\scalebox{0.85}{\ooo}};}
\newcommand{\dotsvelevenp}{\node[typeS = {\DDhorDots}] at (0.9,-8.5) (v11) {$\bcolour{b_{13}}$};}

\newcommand{\dotsvtwelve}{\node[termS = {$s_{14}$},inner sep = 1pt] at (2.5,-8.5) (v12) {\scalebox{0.85}{\crcD}};}
\newcommand{\dotsvtwelvep}{\node[typeS = {\DDhorDots}] at (2.5,-8.5) (v12) {$\bcolour{b_{14}}$};}

%
%%%%%% FOA CONSTRUCTION %%%

% FOA term reps

\newcommand{\foav}{\node[termIW = {$t$}] (v) at (0,-0.3) {$\top$};}
\newcommand{\foavp}{\node[typeIW = {\typebool}] (v) at (0,-0.3) {$\acolour{x}$};}

\newcommand{\foavone}{\node[termW = {$t_1$}] at (-0.3,-2.2) (v1) {$1+2+3=\frac{3(3+1)}{2}$};}
\newcommand{\foavonep}{\node[typeW= {\FOAform}] at (-0.3,-2.2) (v1) {$\acolour{a_1}$};}

\newcommand{\foavtwo}{\node[termW = {$t_2$}] at (-2.2,-4) (v2) {$1+2+3$};}
\newcommand{\foavtwop}{\node[typeW = {\FOAnumexp}] at (-2.2,-4) (v2) {$\acolour{a_2}$};}

\newcommand{\foavthree}{\node[termS = {$t_9$}] at (-0.5,-4.2) (v3)  {\vphantom{\scalebox{0.83}{i}}$=$};}
\newcommand{\foavthreep}{\node[typeS = {\FOAeq}] at (-0.5,-4.2) (v3)  {$\acolour{a_{9}}$};}

\newcommand{\foavfour}{\node[termE = {$t_{10}$}] at (2.1,-3.8) (v4) {$\frac{3(3+1)}{2}$};}
\newcommand{\foavfourp}{\node[typeE = {\FOAnumexp}] at (2.1,-3.8) (v4) {$\acolour{a_{10}}$};}

\newcommand{\foavfive}{\node[termW = {$t_3$}] at (-3.7,-5.8) (v5) {$1+2$};}
\newcommand{\foavfivep}{\node[typeW = {\FOAnumexp}] at (-3.7,-5.8) (v5) {$\acolour{a_3}$};}

\newcommand{\foavsix}{\node[termS = {$t_4$}] at (-2.5,-6) (v6) {$+$};}
\newcommand{\foavsixp}{\node[typeS = {\FOAplus}] at (-2.5,-6) (v6) {$\acolour{a_4}$};}

\newcommand{\foavseven}{\node[termS = {$t_5$}] at (-1.4,-5.8)  (v7) {$3$};}
\newcommand{\foavsevenp}{\node[typeS = {\FOAnumexp}] at (-1.4,-5.8)  (v7) {$\acolour{a_5}$};}

\newcommand{\foaveight}{\node[termNW = {$t_{11}$}] at (0.7,-5.1) (v8) {$3(3+1)$};}
\newcommand{\foaveightp}{\node[typeNW = {\FOAnumexp}] at (0.7,-5.1) (v8) {$\acolour{a_{11}}$};}

\newcommand{\foavnine}{\node[termS = {$t_{12}$}] at (3,-5.8) (v9) {$\frac{\phantom{3(3+1)}}{\phantom{2}}$};}
\newcommand{\foavninep}{\node[typeS = {\FOAdiv}] at (3,-5.8) (v9) {$\acolour{a_{12}}$};}

\newcommand{\foavten}{\node[termS = {$t_{13}$}] at (4.0,-5.3) (v10) {$2$};}
\newcommand{\foavtenp}{\node[typeS = {\textsf{2}}] at (4.0,-5.3) (v10) {$\acolour{a_{13}}$};}

\newcommand{\foaveleven}{\node[termS = {$t_6$}] at (-4.6,-7.6) (v11) {$1$};}
\newcommand{\foavelevenp}{\node[typeS = {\FOAnumexp}] at (-4.6,-7.6) (v11) {$\acolour{a_{6}}$};}

\newcommand{\foavtwelve}{\node[termS = {$t_7$}] at (-3.6,-7.7) (v12) {$+$};}
\newcommand{\foavtwelvep}{\node[typeS = {\FOAplus}] at (-3.6,-7.7) (v12) {$\acolour{a_{7}}$};}

\newcommand{\foavthirteen}{\node[termS = {$t_8$}] at (-2.6,-7.6) (v13) {$2$};}
\newcommand{\foavthirteenp}{\node[typeS = {\FOAnumexp}] at (-2.6,-7.6) (v13) {$\acolour{a_{8}}$};}

\newcommand{\foavfourteen}{\node[termS = {$t_{14}$}] at (-0.2,-6.7) (v14) {$3$};}
\newcommand{\foavfourteenp}{\node[typeS = {\FOAnumexp}] at (-0.2,-6.7) (v14) {$\acolour{a_{14}}$};}

\newcommand{\foavfifteen}{\node[termpos = {$t_{15}$}{45}{0.19cm}] at (1.7,-6.6) (v15) {$(3+1)$};}
\newcommand{\foavfifteenp}{\node[typepos = {\FOAnumexp}{45}{0.19cm}] at (1.7,-6.6) (v15) {$\acolour{a_{15}}$};}

\newcommand{\foavsixteen}{\node[termSW = {$t_{16}$}] at (0.2,-8.2) (v16) {$($};}
\newcommand{\foavsixteenp}{\node[typeSW = {\FOAopenb}] at (0.2,-8.2) (v16) {$\acolour{a_{16}}$};}

\newcommand{\foavseventeen}{\node[termSE = {$t_{17}$}] at (1.7,-8.5) (v17) {$3+1$};}
\newcommand{\foavseventeenp}{\node[typeSE = {\FOAnumexp}] at (1.7,-8.5) (v17) {$\acolour{a_{17}}$};}

\newcommand{\foaveighteen}{\node[termSE = {$t_{18}$}] at (3.2,-8.2) (v18) {$)$};}
\newcommand{\foaveighteenp}{\node[typeSE = {\FOAcloseb}] at (3.2,-8.2) (v18) {$\acolour{a_{18}}$};}

\newcommand{\foavnineteen}{\node[termS = {$t_{19}$}] at (0.2,-10.1) (v19) {$3$};}
\newcommand{\foavnineteenp}{\node[typeS = {\FOAnumexp}] at (0.2,-10.1) (v19) {$\acolour{a_{19}}$};}

\newcommand{\foavtwenty}{\node[termS = {$t_{20}$}] at (1.6,-10.3) (v20) {$+$};}
\newcommand{\foavtwentyp}{\node[typeS = {\FOAplus}] at (1.6,-10.3) (v20) {$\acolour{a_{20}}$};}

\newcommand{\foavtwentyone}{\node[termS = {$t_{21}$}] at (3.2,-10.1) (v21) {$1$};}
\newcommand{\foavtwentyonep}{\node[typeS = {\FOAnumexp}] at (3.2,-10.1) (v21) {$\acolour{a_{21}}$};}

%%% for more complex example

% FOA term reps

\newcommand{\foavn}{\node[termE = {$t$}] (v) {$\top$};}

\newcommand{\foavonen}{\node[termW = {$t_1$}] at (-0.3,-2) (v1) {$1+2+3=\frac{3(3+1)}{2}$};}

\newcommand{\foavtwon}{\node[termW = {$t_2$}] at (-2.2,-4) (v2) {$2+2+4$};}

\newcommand{\foavthreen}{\node[termS = {$t_9$}] at (-0.3,-4) (v3)  {\vphantom{\scalebox{0.83}{i}}$=$};}

\newcommand{\foavfourn}{\node[termE = {$t_{10}$}] at (2.5,-3.8) (v4) {$\frac{3(3+1)}{2}$};}

\newcommand{\foavfiven}{\node[termW = {$t_3$}] at (-3.7,-5.8) (v5) {$2+2$};}

\newcommand{\foavsixn}{\node[termSE = {$t_4$}] at (-2.5,-6) (v6) {$+$};}

\newcommand{\foavsevenn}{\node[termS = {$t_5$}] at (-1.4,-5.8) (v7) {$4$};}

\newcommand{\foaveightn}{\node[termNW = {$t_{11}$}] at (0.7,-5.1) (v8) {$3(3+1)$};}

\newcommand{\foavninen}{\node[termS = {$t_{12}$}] at (3,-5.8) (v9) {$\frac{\phantom{3(3+1)}}{\phantom{2}}$};}

\newcommand{\foavtenn}{\node[termS = {$t_{13}$}] at (4.0,-5.3) (v10) {$2$};}

\newcommand{\foavelevenn}{\node[termS = {$t_6$}] at (-4.6,-7.6) (v11) {$2$};}

\newcommand{\foavtwelven}{\node[termS = {$t_7$}] at (-3.6,-7.7) (v12) {$+$};}

\newcommand{\foavthirteenn}{\node[termS = {$t_8$}] at (-2.6,-7.6) (v13) {$2$};}

\newcommand{\foavfourteenn}{\node[termS = {$t_{14}$}] at (-0.5,-7.8) (v14) {$3$};}

\newcommand{\foavfifteenn}{\node[termW = {$t_{15}$}] at (2.5,-7.8) (v15) {$(3+1)$};}

\newcommand{\foavsixteenn}{\node[termSW = {$t_{16}$}] at (1,-9.8) (v16) {$($};}

\newcommand{\foavseventeenn}{\node[termSE = {$t_{17}$}] at (2.6,-9.8) (v17) {$3+1$};}

\newcommand{\foaveighteenn}{\node[termSE = {$t_{18}$}] at (4,-9.8) (v18) {$)$};}

\newcommand{\foavnineteenn}{\node[termS = {$t_{19}$}] at (0.2,-11.7) (v19) {$3$};}

\newcommand{\foavtwentyn}{\node[termS = {$t_{20}$}] at (2.2,-11.8) (v20) {$+$};}

\newcommand{\foavtwentyonen}{\node[termS = {$t_{21}$}] at (3.7,-11.7) (v21) {$1$};}

%%% dots more comple

%%% new more complex example

\newcommand{\dotsvthreen}{\node[termE = {\DDstack},inner sep = 1pt] at (-2.4,-3.8) (v3) {{\scalebox{0.85}{\oLo}}};} %%

\newcommand{\dotsvfourn}{\node[termS = {\DDstack},inner sep = 1pt] at (-0.6,-3.8) (v4) {{\scalebox{0.85}{\oLo}}};} %%

\newcommand{\dotsvsixn}{\node[termS = {\DDstack},inner sep = 1pt] at (-3.3,-5.7) (v6) {{\scalebox{0.85}{\oo}}};}

\newcommand{\dotsvsevenn}{\node[termS = {$s_6$},inner sep = 1pt] at (-1.5,-5.7) (v7) {{\scalebox{0.85}{\oo}}};}

%% file: structuralTransformations.tex
\newcommand{\acolour}[1]{\textcolor[rgb]{0.90,0.10,0.10}{#1}}
\newcommand{\bcolour}[1]{\textcolor[rgb]{0.10,0.25,0.90}{#1}}

This paper has, so far, introduced a range of theoretical concepts, starting with how Representational Systems Theory can be used to encode representations, their entailment relations, and properties of them (Section~\ref{sec:RSs}). Using patterns, this section defines \textit{structural transformations}, exploiting an inter-representational-system encoding, $\irencodingn=\irencoding$. The practical reasons for seeking a structural transformation of $\cpair$ in $\rsystemn$ broadly fit into two classes:
%\
\begin{enumerate}
\item[-] \textit{Class 1: identify a token, $t'$,  in $\rsystemn'$ where $t$ and $t'$ are in some specified relationship}. For instance, this relationship may be that $t$ is semantically equivalent to $t'$. Here, knowledge about how to construct $t$, provided by $\cgraphn$, is used to construct a suitable $t'$.
\item[-] \label{class2}\textit{Class 2: identify a construction, $\cpaird$, in $\rsystemn'$ where some specified relationships hold between sub-constructions of, and tokens in, $\cpair$ and $\cpaird$.} For instance, when $t$ and $t'$ are semantically equivalent and $\cpair$ is an entailment-space construction of $t$, we may seek $\cpaird$ in the entailment space of $\rsystemn'$, that shows how to derive $t'$. Knowledge about how $t$ is derived, provided by $\cpair$, is used to identify $\cpaird$.
\end{enumerate}
The theory we present applies equally to both classes.

We now digress briefly into remarks on notation. Since many patterns and multiple construction spaces are involved in structural transformations we adopt the following conventions:
\begin{enumerate}
\item Given an inter-representational-system encoding, $\irencodingn=\irencoding$, constructions in $\rsystemn$, $\rsystemn'$ and $\ipispacen$ will be denoted by $\cpair$, $\cpairhs$ and $\cpairiw$, sometimes with indices. Sets of constructions drawn from $\ipispacen$ will be denoted by $\setofconstructionsirse$.
\item Given a description, $\descriptionirse$, of an inter-representational-system encoding, patterns in $\setofpatterns(\descriptionn)$, $\setofpatterns(\descriptionn')$ and $\setofpatterns''$ (i.e. the set of patterns used to describe $\ipispacen$) will be denoted by $\ppaira$, $\ppairb$, and $\ppair$. Sets of patterns drawn from $\setofpatterns''$ will be denoted by $\setofpatternstc$.
\item Given a decomposition, $\decompositionn$ (resp. $\pdecompositionn$) of a construction (resp. pattern), the vertices in $\decompositionn$ (resp. $\pdecompositionn$)  will be denoted by, for example, $d$, $d_1$, etc. (resp. $\delta$, $\delta_1$, etc.).
\end{enumerate}

Now, in essence, a structural transformation is a \textit{relation} between the sets of constructions $\allconstructions{\rsystemn}$ and $\allconstructions{\rsystemn'}$. At the most basic level, the relationship between $\cpair$ and $\cpairhs$ is defined using a triple:  $\tRelSpecn=(\ppaira, \ppairb, \setofpatternstc)$. Related constructions, $\cpair$ and $\cpairhs$, are required to match the patterns $\ppaira$ and $\ppairb$, respectively. The role of $\setofpatternstc$, which is a set of patterns drawn from $\setofpatterns''$, is to specify that certain relations hold between the tokens of $\cpair$ and $\cpairhs$. Of particular importance is that exploiting relations encoded within $\ipispacen$, by using $\setofpatternstc$ in the triple $\tRelSpecn$, supports \textit{representation selection} based on properties that tokens or constructions possess, such as their relative cognitive effectiveness; we have already seen examples of how structure transfer can be used to derive an Euler diagram and a Tree Map, from set-theoretic expressions and, respectively, a table in Sections~\ref{sec:constructions} and~\ref{sec:patterns}.

In general, the theory of structural transformations \textit{exploits decompositions}, $\decompositionn$ and $\decompositionn'$, of $\cpair$ and $\cpairhs$ when defining structural transformations: selected pairs of constructions, $\cpaird$ and $\cpairhsd$, that label vertices in $\decompositionn$ and $\decompositionn'$, are each required to satisfy constraints as imposed by triples of the form $\tRelSpecn=(\ppaira, \ppairb, \setofpatternstc)$. This approach breaks down the problem of producing a structural transformation of $\cpair$ into a set of structural transformation problems involving sub-constructions of $\cpair$.

The concept of a structural transformation is \textit{distinct} from an algorithmic process that seeks to \textit{produce} a structural transformation, $\cpairhs$, of $\cpair$. Any such procedure is called a \textit{structure transfer} algorithm. The theory presented in this section is predominantly focused on structural transformations but is developed in such a way that it provides a useful framework for the subsequent development of structure transfer algorithms.  The approach we take is inspired by the ideas of Reynolds~\cite{reynolds1983types} and the implementation of the Transfer package in Isabelle/HOL by Huffman and Kun{\v{c}}ar~\cite{huffman2013lifting}.

\subsection{Motivating Examples: Representation Selection}\label{sec:structureTransfer:ME}

This section demonstrates a range of application areas where structure transfer can be exploited to turn one construction, $\cpair$, into another, $\cpairhs$. The examples do not present all of the details needed to invoke such structural transformations -- as Section~\ref{sec:patterns:ME} makes evident, such details are space-hungry. Rather, in the case of Class 1 problems, we show how to encode relations between $t$ and $s$ using $\ipispacen$. As will become clear, encoding relations between $t$ and $s$ is \textit{theoretically} easy. However, in a \textit{practical} setting it will typically be the case that many desired relations are not fully known: we may only have partial knowledge of the construction space $\ipispacen$. Later, we define the concept of a \textit{soundness} which allows relations between tokens to be inferred without needing to exploit the entirety of $\ipispacen$. All of the examples in this section illustrate how Representational Systems Theory can be used to support representation selection based on desired properties being exhibited. Example~\ref{ex:structureTransferObservation} begins with a simple case showing how the theory can be exploited to assert that a certain relation holds between constructs $t$ and $s$.

\begin{example}[Observability]\label{ex:structureTransferObservation}
A \textsc{Position Diagrams} representational system places tokens in the plane to express information about relative positions~\cite{shimojima:spodatcp}. For instance, the diagram below expresses that \textit{Amrita is left of Beckie} and \textit{Beckie is left of Cabir}:
\begin{center}
Amrita \quad Beckie \quad Cabir
\end{center}
From this diagram it can be \textit{observed} that \textit{Amrita is left of Cabir}, \textit{Beckie is right of Amrita}, and \textit{Cabir is right of Amrita}. Such an observation relation between natural language statements and position diagrams can be encoded using configurations that can be embedded into this pattern, $\ppair$:
\begin{center}
  \begin{tikzpicture}[construction]
	\node[typeIW={\texttt{true}}] (t5) {$v$};
	\node[constructorINE = {\constructorObserve}, below = 0.45cm of t5] (u) {};
	\node[typeS={\texttt{positionDiagram}},below left = 0.15cm and 1.1cm of u] (t4) {$a$};
	\node[typeS={\texttt{natLangStatement}},below right = 0.08cm and 1.1cm of u] (r5) {$b$};

	\path[->]
	(u) edge (t5)
	(t4) edge[bend right = -10] node[index label]{1} (u)
    (r5) edge[bend right = 10] node[index label]{2} (u);
	\end{tikzpicture}
\end{center}
Taking $\ppaira$ and $\ppairb$ to be trivial constructions, where $a$ and $b$ are assigned types\linebreak \texttt{positionDiagram} and \texttt{natLangStatement}, respectively, we would set $\tRelSpecn= (\ppaira,\ppairb,\linebreak \{\ppair\})$ to encode Class 1 structural transformations of this nature.
\end{example}

\begin{example}[Re-representing a Token: Colouring Problems]\label{ex:structureTransColour}
Suppose a rectangle is divided into regions and one wishes to know how many ways there are of colouring its regions, with $k$ colours, such that no two regions sharing a border are assigned the same colour. Such a representation can be converted to a graph, $\graphn$, where an equivalent problem is to count how many ways there are of colouring $\graphn$'s vertices such that no two adjacent vertices are assigned the same colour. In this context, the divided rectangle, $r$, below left, is equivalent to $\graphn_1$ but not $\graphn_2$:
\begin{center}
	\begin{tikzpicture}
		\draw[thick] (0,0) rectangle (1,1);
		\draw[thick] (0.3,0.3) rectangle (0.7,0.7);
		\draw[thick] (0.5,0.7) -- (0.5,1);
		\draw[thick] (0.7,0.3) -- (1,0.3);
		\draw[thick] (0.3,0.3) -- (0.3,0);
        \node (1) at (0.5,-0.25) {$r$};
	\end{tikzpicture}
 \qquad \qquad
\begin{tikzpicture}
	\node[shape=circle, fill=black, inner sep = 0pt, minimum size=4pt] (1) at (0,0.5) {};
    \node[shape=circle, fill=black, inner sep = 0pt, minimum size=4pt] (2) at (-0.5,-0.5) {};
    \node[shape=circle, fill=black, inner sep = 0pt, minimum size=4pt] (3) at (0.5,-0.5) {};
    \node[shape=circle, fill=black, inner sep = 0pt, minimum size=4pt] (4) at (0,-0.1) {};
    \node (5) at (0,-0.8) {$\graphn_1$};
	\path[-]
	(1) edge (2)
	(1) edge (3)
    (1) edge (4)
    (2) edge (3)
    (2) edge (4)
    (3) edge (4);
	\end{tikzpicture}
\qquad \qquad
\begin{tikzpicture}
	\node[shape=circle, fill=black, inner sep = 0pt, minimum size=4pt] (1) at (0,0.5) {};
    \node[shape=circle, fill=black, inner sep = 0pt, minimum size=4pt] (2) at (-0.5,-0.5) {};
    \node[shape=circle, fill=black, inner sep = 0pt, minimum size=4pt] (3) at (0.5,-0.5) {};
    \node[shape=circle, fill=black, inner sep = 0pt, minimum size=4pt] (4) at (0,-0.1) {};
        \node (5) at (0,-0.8) {$\graphn_2$};
	\path[-]
	(1) edge (2)
	(1) edge (3)
    %(1) edge (4)
    (2) edge (3)
    (2) edge (4)
    (3) edge (4);
	\end{tikzpicture}
\end{center}
For this divided rectangle, $r$, the re-representation problem can be stated as: given a construction, $(\cgraphn,r)$, drawn from the representational system $\rsystemn_{R}$ (\textsc{Divided Rectangles}), find a construction, $(\cgraphhn,\graphn)$, in $\rsystemn_{G}$ (\textsc{Graphs}) where $r$ and $\graphn$ have the same number of $k$-colourings. That is, the constructs $r$ and $\graphn$ should be related by the $\texttt{equal-k-Colourings}$ relation. In the case of $\graphn_1$ and $\graphn_2$, this relation can be encoded in $\ipispacen$, where $\irencodingn = (\rsystemn_{R}, \rsystemn_{G}, \ispacen'')$, using these constructions, $(\cgraphin_1,\top)$ and $(\cgraphin_2,\bot)$:
\begin{center}
\begin{tikzpicture}[construction]
	\node[termIW={\typetrue}] (t5) {$\top$};
	\node[constructorINE = {\texttt{equal-k-Colourings}}, below = 0.45cm of t5] (u) {};
	\node[termS={\texttt{dividedRectangle}},below left = 0.1cm and 1.2cm of u] (t4) {%
	\begin{tikzpicture}[rounded corners = 0,scale=0.75]
		\draw[thick] (0,0) rectangle (1,1);
		\draw[thick] (0.3,0.3) rectangle (0.7,0.7);
		\draw[thick] (0.5,0.7) -- (0.5,1);
		\draw[thick] (0.7,0.3) -- (1,0.3);
		\draw[thick] (0.3,0.3) -- (0.3,0);
        %\node (1) at (0.5,-0.2) {$r$};
	\end{tikzpicture}
};
	\node[termS={\texttt{graph}},below right = 0.02cm and 1.2cm of u] (r5)%
{
\begin{tikzpicture}[scale=0.75]
	\node[shape=circle, fill=black, inner sep = 0pt, minimum size=3pt] (1) at (0,0.5) {};
    \node[shape=circle, fill=black, inner sep = 0pt, minimum size=3pt] (2) at (-0.5,-0.5) {};
    \node[shape=circle, fill=black, inner sep = 0pt, minimum size=3pt] (3) at (0.5,-0.5) {};
    \node[shape=circle, fill=black, inner sep = 0pt, minimum size=3pt] (4) at (0,-0.1) {};
	\path[-]
	(1) edge (2)
	(1) edge (3)
    (1) edge (4)
    (2) edge (3)
    (2) edge (4)
    (3) edge (4);
	\end{tikzpicture}
};

	\path[->]
	(u) edge (t5)
	(t4) edge[bend right = -10] node[index label]{1} (u)
    (r5) edge[bend right = 10] node[index label]{2} (u);
	\end{tikzpicture}
\qquad
\begin{tikzpicture}[construction]
	\node[termIW={\typefalse}] (t5) {$\bot$};
	\node[constructorINE = {\texttt{equal-k-Colourings}}, below = 0.45cm of t5] (u) {};
	\node[termS={\texttt{dividedRectangle}},below left = 0.1cm and 1.2cm of u] (t4) {%
	\begin{tikzpicture}[rounded corners = 0,scale=0.75]
		\draw[thick] (0,0) rectangle (1,1);
		\draw[thick] (0.3,0.3) rectangle (0.7,0.7);
		\draw[thick] (0.5,0.7) -- (0.5,1);
		\draw[thick] (0.7,0.3) -- (1,0.3);
		\draw[thick] (0.3,0.3) -- (0.3,0);
        %\node (1) at (0.5,-0.2) {$r$};
	\end{tikzpicture}
};
	\node[termS={\texttt{graph}},below right = 0.02cm and 1.2cm of u] (r5)%
{%
\begin{tikzpicture}[scale=0.75]
	\node[shape=circle, fill=black, inner sep = 0pt, minimum size=3pt] (1) at (0,0.5) {};
    \node[shape=circle, fill=black, inner sep = 0pt, minimum size=3pt] (2) at (-0.5,-0.5) {};
    \node[shape=circle, fill=black, inner sep = 0pt, minimum size=3pt] (3) at (0.5,-0.5) {};
    \node[shape=circle, fill=black, inner sep = 0pt, minimum size=3pt] (4) at (0,-0.1) {};
	\path[-]
	(1) edge (2)
	(1) edge (3)
    %(1) edge (4)
    (2) edge (3)
    (2) edge (4)
    (3) edge (4);
	\end{tikzpicture}
};
	\path[->]
	(u) edge (t5)
	(t4) edge[bend right = -10] node[index label]{1} (u)
    (r5) edge[bend right = 10] node[index label]{2} (u);
	\end{tikzpicture}
\end{center}
The following pattern, $\ppair$ captures the required construct relation:
\begin{center}
\begin{tikzpicture}[construction]
	\node[typeIW={\typetrue}] (t5) {$v$};
	\node[constructorINE = {\texttt{equal-k-Colourings}}, below = 0.45cm of t5] (u) {};
	\node[typeS={\texttt{dividedRectangle}},below left = 0.15cm and 1.1cm of u] (t4) {$a$};
	\node[typeS={\texttt{graph}},below right = 0.07cm and 1.1cm of u] (r5) {$b$};
	\path[->]
	(u) edge (t5)
	(t4) edge[bend right = -10] node[index label]{1} (u)
    (r5) edge[bend right = 10] node[index label]{2} (u);
	\end{tikzpicture}\vspace{-0.2cm}
\end{center}

\end{example}
%\dnote{Just noting here that if we need space the example above can go. While I like the example in theory, we are not expanding it into a sophisticated example of structure transfer, so it ends up being simply an example of encoding a relation in $\ipispacen$.}

As is evident in examples~\ref{ex:structureTransferObservation} and~\ref{ex:structureTransColour}, formally encoding the relation that we wish to hold between a pair of constructs, $t$ and $s$, is trivial: given a relation, $\{(t_{1},s_1),\ldots,(t_i,s_i),\ldots\}$, ensure that $\ispacen''$ includes a constructor, $\texttt{R}$, where for each $(t_i,s_i)$, the structure graph of $\ispacen''$ contains an $\texttt{R}$-configurator with inputs $t_i$ and $s_i$ and output $\top$. Extending this observation, the output need not be $\top$ but could be, for instance, a value in the interval $\texttt{[0,1]}$ representing the probability that $t_i$ is related to $s_i$ under the given relation. Thus, these more general configurations of $\texttt{R}$ would match this \begin{samepage}pattern:
\begin{center}
  \begin{tikzpicture}[construction]
	\node[typeIW={\texttt{[0,1]}}] (t5) {$v$};
	\node[constructorINE = {$\texttt{R}$}, below = 0.45cm of t5] (u) {};
	\node[typeS={$\tau_1$},below left = 0.15cm and 1.1cm of u] (t4) {$a$};
	\node[typeS={$\tau_2$},below right = 0.08cm and 1.1cm of u] (r5) {$b$};
	\path[->]
	(u) edge (t5)
	(t4) edge[bend right = -10] node[index label]{1} (u)
    (r5) edge[bend right = 10] node[index label]{2} (u);
	\end{tikzpicture}\vspace{-0.2cm}
\end{center}\end{samepage}
Hence, $\ispacen''$ (and $\ipispacen$) can be used to encode uncertainty. Moreover, setting the type assigned to $v$ to be, say, $\texttt{[0.95,1]}$ instead of $\texttt{[0,1]}$ would represent a constraint that the probability of being related is at least $0.95$.

Whilst example~\ref{ex:structureTransferObservation} used an intuitive understanding of position diagrams and natural language to observe that Amrita is left of Cabir, for instance, it does not provide any insight into how we might go about \textit{constructing} the observed sentence. Similarly, example~\ref{ex:structureTransColour} gave no insight into how to construct $\graphn_1$. If we have only partial knowledge of our representational systems, how do we go about constructing $s$, given a construction, $\cpair$, of $t$? Example~\ref{ex:STMErerepresentingTokenSE} begins to illuminate how constructions and patterns are exploited in this context.

\newcommand{\inlinecrc}{\tikz[baseline=-0.1cm,inner sep = 0pt]{\draw[fill opacity = 0.2, fill = black,thin] (0,0) circle (.07cm);}}
\begin{example}[Re-representing a Token: Semantic Equivalence]\label{ex:STMErerepresentingTokenSE}
The system, $\rsystemn_{\FOA}$, of \textsc{First-Order Arithmetic} given in Section~\ref{sec:RSs} embodied various ways of constructing numerical expressions. Numerals, which are tokens in $\rsystemn_{\FOA}$, can also be represented by dots, using a system, $\rsystemn_{D}$, of \textsc{Dot Diagrams}: $1$ is represented by $\inlinecrc$, $2$ by $\inlinecrc \dotspace\inlinecrc$, and so forth. Summations can be represented by \textit{joining} dots: $1+2$ is represented by $\inlinecrc$ joined with $\inlinecrc \dotspace\inlinecrc$, giving $\inlinecrc \dotspace\inlinecrc \dotspace\inlinecrc$. Each dot occupies an integer-coordinate position, $(x,y)\in \mathbb{Z}\times \mathbb{Z}$ and, in our formalisation, is assigned type $\{(\texttt{x},\texttt{y})\}$ which is a subtype of \texttt{dotDiagram}. Given $n$ individual dots, $\inlinecrc, \hdots, \inlinecrc$, with individual types $\{(\texttt{x}_\texttt{1},\texttt{y}_\texttt{1})\},\ldots,\{(\texttt{x}_\texttt{n},\texttt{y}_\texttt{n})\}$, the single dot diagram comprising all $n$ dots has type $\{(\texttt{x}_\texttt{1},\texttt{y}_\texttt{1}),\ldots,(\texttt{x}_\texttt{n},\texttt{y}_\texttt{n})\}$. So, when joining dots to represent $1+2$, it is necessary that the joined sets of dots occupy different positions: the first dot, $\inlinecrc$, does not occupy either of the positions taken by the two dots in $\inlinecrc \dotspace \inlinecrc$. In other words, their types are disjoint sets. To encode a numerical expression being represented by a dot diagram, we exploit a constructor called \texttt{representedBy}, found in $\ispacen''$, given $\irencodingn=(\rsystemn_{\FOA},\rsystemn_{D},\ispacen'')$. We can now encode $1+2$ being represented by $\inlinecrc \dotspace\inlinecrc \dotspace\inlinecrc$ using the construction $(\cgraphin,\top)$, below left, whereas $(\cgraphin',\bot)$ encodes the fact that $1+2$ is not represented by $\inlinecrc \dotspace \inlinecrc$:
\begin{center}
\begin{tikzpicture}[construction]
	\node[termIW={\typetrue}] (t5) {$\top$};
	\node[constructorINE = {\constructorRepresentedBy}, below = 0.45cm of t5] (u) {};
	\node[termS={\texttt{\texttt{1\_pl\_2}}},below left = of u, yshift=0.7cm, xshift=-0.3cm] (t4) {$1+2$};
	\node[termS={$\{(\texttt{1},\texttt{1}),(\texttt{2},\texttt{1}),(\texttt{3},\texttt{1})\}$},below right = of u,yshift=0.7cm,xshift=0.3cm] (r5){$\inlinecrc \dotspace \inlinecrc \dotspace \inlinecrc$};
	\path[->]
	(u) edge (t5)
	(t4) edge[bend right = -10] node[index label]{1} (u)
    (r5) edge[bend right = 10] node[index label]{2} (u);
	\end{tikzpicture}
\qquad
\begin{tikzpicture}[construction]
	\node[termIW={\typefalse}] (t5) {$\bot$};
	\node[constructorINE = {\constructorRepresentedBy}, below = 0.45cm of t5] (u) {};
	\node[termS={\texttt{1\_pl\_2}},below left = of u, yshift=0.7cm, xshift=-0.2cm] (t4) {$1+2$};
	\node[termS={$\{(\texttt{1},\texttt{1}),(\texttt{2},\texttt{1})\}$},below right = of u,yshift=0.7cm,xshift=0.2cm] (r5) {$\inlinecrc \dotspace\inlinecrc$};

	\path[->]
	(u) edge (t5)
	(t4) edge[bend right = -10] node[index label]{1} (u)
    (r5) edge[bend right = 10] node[index label]{2} (u);
	\end{tikzpicture}
\end{center}
From this specific example, we can identify a pattern, $\ppair$, that captures the general case of a numeral being represented by a dot diagram:
\begin{center}
\begin{tikzpicture}[construction]
	\node[typeIW={\typetrue}] (t5) {$v$};
	\node[constructorINE = {\constructorRepresentedBy}, below = 0.45cm of t5] (u) {};
	\node[typeS={\FOAnumexp},below left = of u, yshift=0.7cm, xshift=-0.2cm] (t4) {$\acolour{a}$};
	\node[typeS={\DDdotDiagram},below right = of u,yshift=0.75cm,xshift=0.2cm] (r5) {$\bcolour{b}$};

	\path[->]
	(u) edge (t5)
	(t4) edge[bend right = -10] node[index label]{1} (u)
    (r5) edge[bend right = 10] node[index label]{2} (u);
	\end{tikzpicture}
\end{center}
{\pagebreak}Given constructions, $\cpair$ (e.g. where $\foundationsto{\cgraphn,t}=[1,+,2]$ and $t=1+2$) and $\cpairhs$ (e.g. where $\foundationsto{\cgraphhn,s}=[\inlinecrc, \inlinecrc \dotspace \inlinecrc]$ and $s=\inlinecrc \dotspace \inlinecrc \dotspace \inlinecrc$), that match $\ppaira$ and $\ppairb$, shown \begin{samepage}here,
\begin{center}
\begin{tikzpicture}[construction]
	\node[typeW={\FOAnumexp}] (c) {$\acolour{a}$};
	\node[constructorINE = {\FOAcinfixop}, below = 0.45cm of c] (u) {};
	\node[typeS={\FOAnumexp},below left = of u, yshift=0.3cm, xshift=-0.2cm] (l) {$\acolour{a_1}$};
    \node[typeS={\FOAplus},below  = of u, yshift=0.2cm, xshift=-0.0cm] (m) {$\acolour{a_2}$};
	\node[typeS={\FOAnumexp},below right = of u,yshift=0.3cm,xshift=0.2cm] (r) {$\acolour{a_3}$};
	\path[->]
	(u) edge (c)
	(l) edge[bend right = -10] node[index label]{1} (u)
    (m) edge[bend right = 0] node[index label]{2} (u)
    (r) edge[bend right = 10] node[index label]{3} (u);
	\end{tikzpicture}
\qquad
\begin{tikzpicture}[construction]
	\node[typeW={\DDdotDiagram}] (c) {$\bcolour{b}$};
	\node[constructorINE = {\DDcjoin}, below = 0.45cm of c] (u) {};
	\node[typeS={\DDdotDiagram},below left = of u, yshift=0.4cm, xshift=-0.0cm] (l) {$\bcolour{b_1}$};
    %\node[typeS={\FOAplus},below  = of u, yshift=0.5cm, xshift=-0.0cm] (m) {$v_1''$};
	\node[typeS={\DDdotDiagram},below right = of u,yshift=0.4cm,xshift=0.0cm] (r) {$\bcolour{b_2}$};
	\path[->]
	(u) edge (c)
	(l) edge[bend right = -10] node[index label]{1} (u)
    %(m) edge[bend right = 0] node[index label]{2} (u)
    (r) edge[bend right = 10] node[index label]{2} (u);
	\end{tikzpicture}
\end{center}
\end{samepage}
%\qquad
%
if there exists a set of constructions, $\setofconstructionsirse=\{\cpairiwp{1},\cpairiwp{2},\cpairiwp{3}\}$, drawn from $\ipispacen$ that, respectively, embed into the patterns in $\setofpatternstc=\{\ppairp{1},\ppairp{2},\ppairp{3}\}$, given here
\begin{center}\label{ex:pageref:ispaceexample}
	\begin{tikzpicture}[construction]
	%\node[typeIW={\typetrue}] (c) {$v'$};
	%\node[constructorINW = {\constructorAnd}, below = 0.45cm of c] (uc) {};
	%\node[typeIW={\typebool},below left = 0.1cm and 1.2cm of uc] (luc) {};
	%
	%\node[constructorINE = {\constructorAnd}, below = 0.45cm of luc] (uluc) {};
	\node[typeIW={\typetrue},anchor=center] (lluc) {$v_1$};
	\node[typeIE={\typetrue}, right =  2.5cm of lluc,anchor=center] (rluc) {$v_2$};
	\node[typeIE={\typetrue}, right =  2.5cm of rluc,anchor=center] (ruc) {$v_3$};
	\node[constructorINW = {\constructorRepresents}, below = 0.45cm of lluc] (ululuc) {};
	\node[typeW={\FOAnumexp},below left = 0.4cm and 1cm of ululuc] (llluc) {$\acolour{a_1}$};
	\node[typeS={\DDdotDiagram},below right = 0.4cm and 1cm of ululuc] (rlluc) {$\bcolour{b_1}$};
	\node[constructorINE = {\DDcdisjoint}, below = 0.45cm of rluc] (uruluc) {};
	\node[typeS={\DDdotDiagram},below right = 0.4cm and 1cm of uruluc] (rrluc) {$\bcolour{b_2}$};
	\node[constructorINE = {\constructorRepresents}, right = 2.3cm of uruluc] (uv) {};
	\node[typeE={\FOAnumexp}, below right = 0.4cm and 1cm of uv] (v1''') {$\acolour{a_3}$};
	\path[->]
	%(uc) edge (c)
	%(luc) edge[bend right = -10] node[index label]{1} (uc)
	%(ruc) edge[bend right = 10] node[index label]{2} (uc)
	%
	%(uluc) edge (luc)
	%(lluc) edge[bend right = -10] node[index label]{1} (uluc)
	%(rluc) edge[bend right = 10] node[index label]{2} (uluc)
	%
	(ululuc) edge (lluc)
	(llluc) edge[bend right = -10] node[index label]{1} (ululuc)
	(rlluc) edge[bend right = 10] node[index label]{2} (ululuc)
	(uruluc) edge (rluc)
	(rlluc) edge[bend right = -10] node[index label]{1} (uruluc)
	(rrluc) edge[bend right = 10] node[index label]{2} (uruluc)
	(uv) edge[] (ruc)
	(rrluc) edge[bend right = -10] node[index label]{2} (uv)
	(v1''') edge[bend right = 10] node[index label]{1} (uv);
	\coordinate[below left = of llluc,yshift =0.5cm,xshift=0.0cm] (a1);
	\coordinate[above left = of llluc,yshift =-0.5cm] (a2);
	\coordinate[above left = of lluc,yshift =-0.55cm,xshift=0.4cm] (a3);
	\coordinate[above right = of lluc,yshift =-0.55cm,xshift=-0.8cm] (a4);
	\coordinate[above right = of rlluc,xshift=-1.0cm] (a5);
	\coordinate[below right = of rlluc,yshift =0.35cm,xshift=-0.35cm] (a6);
	\draw[rounded corners=8,thick, draw opacity = 0.30, text opacity = 0.6] (a1)--(a2)--node[yshift=0.5cm,xshift=-0.2cm]{$(\gamma_1,v_1)$}(a3)--(a4)--(a5)--(a6)--cycle;
	\coordinate[below left = of rlluc,yshift =0.3cm,xshift=0.35cm] (b1);
	\coordinate[above left = of rlluc,yshift =-0.5cm,xshift=0.8cm] (b2);
	\coordinate[above left = of rluc,yshift =-0.55cm,xshift=1cm] (b3);
	\coordinate[above right = of rluc,yshift =-0.55cm,xshift=-0.5cm] (b4);
	\coordinate[above right = of rrluc,yshift = 0.4cm,xshift=-0.8cm] (b5);
	\coordinate[below right = of rrluc,yshift =0.3cm,xshift=-0.35cm] (b6);
	\draw[rounded corners=8,thick, draw opacity = 0.30, text opacity = 0.6] (b1)--(b2)--node[yshift=0.6cm,xshift=-0.1cm]{$(\gamma_2,v_2)$}(b3)--(b4)--(b5)--(b6)--cycle;
	\coordinate[below left = of rrluc,yshift =0.35cm,xshift=0.35cm] (g1);
	\coordinate[above left = of rrluc,yshift =-0.6cm,xshift=0.9cm] (g2);
	\coordinate[above left = of ruc,yshift =-0.55cm,xshift=1cm] (g3);
	\coordinate[above right = of ruc,yshift =-0.55cm,xshift=-0.4cm] (g4);
	\coordinate[above right = of v1''',xshift=-0.0cm,yshift=-0.5cm] (g5);
	\coordinate[below right = of v1''',yshift =0.5cm,xshift=0.0cm] (g6);
	\draw[rounded corners=8,thick, draw opacity = 0.30, text opacity = 0.6] (g1)--(g2)--(g3)--(g4)--node[yshift=0.55cm,xshift=0.2cm]{$(\gamma_3,v_3)$}(g5)--(g6)--cycle;
	\end{tikzpicture}
\end{center}
such that $\cgraphn\cup \cgraphhn\cup \cgraphin_1\cup \cgraphin_2\cup \cgraphin_3$ is isomorphic to $\pgraphan\cup \pgraphbn\cup \pgraphmn_1\cup \pgraphmn_2\cup \pgraphmn_3$ -- noting the graphs' shared vertices $a_1$, $a_3$, $b_1$ and $b_2$, with all other vertices taken to be distinct --  then we are \textit{assured} that there is a construction, $\cpairiw$, in $\ispacen''$ with foundation token-sequence $[t,s]$ that embeds into $\ppair$. That is, $t$ (e.g. $1+2$) is represented by $s$ (e.g. $\inlinecrc \dotspace \inlinecrc \dotspace \inlinecrc$). In other words, if
\begin{enumerate}
\item[] constructions $\cpair$ and $\cpairhs$ match $\ppaira$ and $\ppairb$ and there exists a set of constructions, $\setofconstructionsirse=\{\cpairiwp{1},\cpairiwp{2},\cpairiwp{3}\}$, drawn from $\ispacen''$ that, respectively, embed into $\setofpatternstc=\{\ppairp{1},\ppairp{2},\ppairp{3}\}$ such that $\cgraphn\cup \cgraphhn\cup \cgraphin_1\cup \cgraphin_2\cup\cgraphin_3$ is isomorphic to $\pgraphan\cup \pgraphbn\cup \pgraphmn_1 \cup \pgraphmn_2 \cup\pgraphmn_3$
\end{enumerate}
then
\begin{enumerate}
\item[] given the trivial patterns, $(\pgraphan,a)$ and $(\pgraphbn,b)$, and the matching trivial constructions, $(\cgraphn',t)$ and $(\cgraphhn',s)$, it is necessary that there exists a construction, $\cpairiw$, in $\ispacen''$ that embeds into $\ppair$ and has foundation token-sequence $[t,s]$.
\end{enumerate}
In addition, if we know a priori that the aforementioned relations satisfy the conditional statement above, $1$ is represented by $\inlinecrc$, and $2$ is represented by $\inlinecrc \dotspace \inlinecrc$, then, given functional constructors and system-specific knowledge about how they take their inputs to produce their outputs, it is possible to \textit{build} $(\cgraphhn,\inlinecrc \dotspace \inlinecrc \dotspace \inlinecrc)$: we know the foundation token-sequence is $[\inlinecrc, \inlinecrc \dotspace \inlinecrc]$ and we know that the construction matches $\ppairb$, so we can infer the remaining tokens in $(\cgraphhn,\inlinecrc \dotspace \inlinecrc \dotspace \inlinecrc)$.
\end{example}

\begin{example}[Improving Cognitive Effectiveness]\label{ex:stme:linearDiagramsConstructRelations}
Consider the system, $\rsystemn_{L}$, of \textsc{Linear Diagrams}\footnote{Linear diagrams were introduced by Leibniz in 1686~\cite{couturat:oefiedl} and are similar to parallel bargrams~\cite{wittenburg:pbfcbieac}, double decker plots~\cite{hofmann:varwimp} and UpSet~\cite{lex:uvois}. They are known to have cognitive advantages over Euler diagrams when representing sets~\cite{chapman:vsaecodt,gottfried:acsolarbd}.} that represent relationships between sets, which includes the following tokens, $l$ and $l'$:
\begin{center}
\begin{minipage}[c]{0.3\columnwidth}
 \begin{center}
  \linearone\\
    $l$
 \end{center}
\end{minipage}
 \qquad
\begin{minipage}[c]{0.3\columnwidth}
 \begin{center}
  \lineartwo\\
    $l'$
 \end{center}
\end{minipage}
\end{center}

Both diagrams express that all lions are cats (in each case, the `Lions' line shares all of its $x$-coordinates with the `Cats' line) and no dog is also a cat (the `Dogs' line shares no $x$-coordinate with the `Cats' line). However, the diagram on the right has fewer line segments representing the sets Cats, Lions and Dogs: $\mathit{l}'$ uses three line segments, whereas $\mathit{l}$ uses four; the Cats line is broken into two parts. It is known that possessing more line breaks is cognitively less effective for extracting some kinds of information from linear diagrams~\cite{rodgers:vswld}. This relationship can be encoded in \begin{samepage}$\ispacen''$:
\begin{center}
  \begin{tikzpicture}[construction]
	\node[termIrep] (t5) {$\top$};
	\node[constructorIS = {\texttt{lessEffective}}, below = of t5,yshift=0.3cm,xshift = 0.0cm] (u) {};
	\node[termrep,below left = of u, yshift=1.2cm, xshift=-0.2cm] (t4) {\scalebox{0.5}{\linearone}};
	\node[termrep,below right = of u,yshift=1.2cm,xshift=0.2cm] (r5) {\scalebox{0.5}{\lineartwo}};

	\path[->]
	(u) edge (t5)
	(t4) edge[bend right = -5] node[index label]{1} (u)
    (r5) edge[bend right = 5] node[index label]{2} (u);
	\end{tikzpicture}
\end{center}\end{samepage}
In the context of structure transfer, the goal may be to take a construction, $(\cgraphn,\mathit{l})$, and find a construction, $(\cgraphhn,\mathit{l}'')$ such that $\mathit{l}''$ is cognitively more effective than $l$; one possible outcome of such a transfer would be a construction of $\mathit{l}'$ in particular. Given this goal, a construction is sought where $\mathit{l}$ and $\mathit{l}''$ are instantiations of the vertices $a$ and $b$, respectively, of the following pattern, $\ppair$,  for $\ispacen''$, where $\irencodingn=(\rsystemn_{\mathit{LD}}, \rsystemn_{\mathit{LD}},\ispacen'')$\footnote{For an intra-system structural transformation (i.e. when our inter-representational-system encoding is $(\rsystemn,\rsystemn, \ispacen'')$ and we want to transform a construction in $\rsystemn$ into another construction in $\rsystemn$) it is reasonable to take $\ispacen''=\ispacen$, in which case $\ipispacen=\ispacen$.}:
\begin{center}
  \begin{tikzpicture}[construction,yscale=0.9]
	\node[typeIW={\typetrue}] (t5) {$v$};
	\node[constructorIS = {\texttt{lessEffective}}, below = 0.4cm of t5] (u) {};
	\node[typeS={\texttt{linearDiagram}},below left = of u, yshift=1cm, xshift=-0.5cm] (t4) {$\acolour{a}$};
	\node[typeS={\texttt{linearDiagram}},below right = of u,yshift=1.1cm,xshift=0.5cm] (r5) {$\bcolour{b}$};

	\path[->]
	(u) edge (t5)
	(t4) edge[bend right = -5] node[index label]{1} (u)
    (r5) edge[bend right = 5] node[index label]{2} (u);
	\end{tikzpicture}
\end{center}
In a practical setting, one could compute the number of set-line segments in each of $\mathit{l}$ and the construct, $\mathit{l}''$, of a \textit{potential} structural transformation. The number of line segments possessed by each linear diagram can be encoded using $\ispacen$, the identification space for $\rsystemn_{\mathit{LD}}$, by exploiting a \texttt{numLineSegements} constructor. In turn, one linear diagram possessing more line segments than another can be encoded by constructions in $\ispacen''$ that embed into $\ppaird$:
\begin{center}
  \begin{tikzpicture}[construction,yscale=0.9]
	\node[typeIW={\typetrue}] (t5) {$v'$};
	\node[constructorINE = {\texttt{greaterThan}}, below = 0.45cm of t5] (u) {};
	\node[typeNW={\texttt{numeral}},below left = 0.1cm and 1cm of u] (t4) {};
	\node[typeNE={\texttt{numeral}},below right = 0.1cm and 1cm of u] (r5) {};
    \node[constructorINW = {\texttt{numLineSegments}}, below = 0.45cm of t4] (u1) {};
	\node[typeS={\texttt{linearDiagram}},below left = 0.2cm and 0.6cm of u1] (t41) {$\acolour{a}$};
    \node[constructorINE = {\texttt{numLineSegments}}, below = 0.45cm of r5] (u2) {};
	\node[typeS={\texttt{linearDiagram}},below right = 0.2cm and 0.6cm of u2] (t42) {$\bcolour{b}$};
	\path[->]
	(u) edge (t5)
	(t4) edge[bend right = -10] node[index label]{1} (u)
    (r5) edge[bend right = 10] node[index label]{2} (u)
    (u1) edge[bend right = 0] (t4)
	(t41) edge[bend right = -10] node[index label]{1} (u1)
    (u2) edge[bend right = 00] (r5)
	(t42)[bend right = 10] edge node[index label]{1} (u2);
	\end{tikzpicture}
\end{center}

What we have witnessed so far in this example is how to encode the desire to transform one linear diagram into another in such a way that cognitive effectiveness is improved. However, the encoding used does not ensure that the resulting linear diagram, $l''$, represents the same information as $l$. To ensure that cognitive effectiveness is improved \textit{and} semantics are preserved, a more complex graph comprising two patterns, $\ppairp{1}$ and $\ppairp{2}$, can be employed for the encoding of the required relations between constructs:
%
%\begin{center}
%  \begin{tikzpicture}[construction]
%	%\node[typeIW={\typetrue}] (t5) {$v''$};
%	%\node[constructorIS = {\constructorAnd}, below = 0.4 of t5] (u) {};
%	\node[typeINW={\typetrue}] (t4) {};
%	\node[typeINE={\typetrue},right = of u,xshift=2cm] (r5) {};
%    %
%    \node[constructorINW = {\texttt{lessEffective}}, below = 0.45cm of t4] (u1) {};
%	\node[typeS={\texttt{linearDiagram}},below  = of u1, yshift=0.5cm, xshift=0cm] (t411) {$\acolour{a}$};
%    %
%    \node[constructorINE = {\texttt{representedBy}}, below = 0.45cm of r5] (u2) {};
%	\node[typeS={\texttt{linearDiagram}},below  = of u2, yshift=0.5cm, xshift=0cm] (t421) {$\bcolour{b}$};
%	\path[->]
%	%(u) edge (t5)
%	%(t4) edge[bend right = -10] node[index label]{1} (u)
%    %(r5) edge[bend right = 10] node[index label]{2} (u)
%    %
%    (u1) edge (t4)
%	(t411) edge node[index label]{1} (u1)
%    (t421) edge[bend right = 10]  node[index label]{2} (u1)
%    (u2) edge (r5)
%	(t421) edge node[index label]{2} (u2)
%    (t411) edge[bend right = 10] node[index label]{1} (u2);
%	\end{tikzpicture}
%\end{center}
\begin{center}
	\begin{tikzpicture}[construction]
	%\node[typeIW={\typetrue}] (t5) {$v''$};
	%\node[constructorIS = {\constructorAnd}, below = 0.4 of t5] (u) {};
	\node[typeINW={\typetrue}] (t4) {$v_1$};
	\node[typeINE={\typetrue},right = 2cm of t4] (r5) {$v_2$};
	\node[constructorINW = {\texttt{lessEffective}}, below = 0.45cm of t4] (u1) {};
	\node[typeS={\texttt{linearDiagram}},below = 0.6cm of u1] (t411) {$\acolour{a}$};
	\node[constructorINE = {\texttt{representedBy}}, below = 0.45cm of r5] (u2) {};
	\node[typeS={\texttt{linearDiagram}},below = 0.6cm of u2] (t421) {$\bcolour{b}$};
	\path[->]
	%(u) edge (t5)
	%(t4) edge[bend right = -10] node[index label]{1} (u)
	%(r5) edge[bend right = 10] node[index label]{2} (u)
	%
	(u1) edge (t4)
	(t411) edge[bend left = 10] node[index label]{1} (u1)
	(t421) edge[bend right = 10]  node[index label,pos=0.54]{2} (u1)
	(u2) edge (r5)
	(t421) edge[bend right = 10] node[index label]{2} (u2)
	(t411) edge[bend right = 10] node[index label,pos=0.47]{1} (u2);
	\coordinate[below left = of t411,yshift =0.4cm,xshift=0.2cm] (a1);
	\coordinate[above left = of t411,yshift =0.4cm,xshift=-0.8cm] (a2);
	\coordinate[above left = of t4,yshift =-0.4cm,xshift=0.3cm] (a3);
	\coordinate[above right = of t4,yshift =-0.4cm,xshift=-1.2cm] (a4);
	\coordinate[above right = of t421,xshift=-0.2cm,yshift=-0.9cm] (a5);
	\coordinate[below right = of t421,yshift =0.4cm,xshift=-0.2cm] (a6);
	\draw[rounded corners=7,thick, draw opacity = 0.30, text opacity = 0.6] (a1)--(a2)--node[yshift=0.2cm,xshift=-0.5cm]{$(\gamma_1,v_1)$}(a3)--(a4)--(a5)--(a6)--cycle;
	\coordinate[below left = of t411,yshift =0.3cm,xshift=0.2cm] (b1);
	\coordinate[above left = of t411,yshift =-0.9cm,xshift =0.2cm] (b2);
	\coordinate[above left = of r5,yshift =-0.4cm,xshift=1.1cm] (b3);
	\coordinate[above right = of r5,yshift =-0.4cm,xshift=-0.2cm] (b4);
	\coordinate[above right = of t421,xshift=0.8cm,yshift=0.4cm] (b5);
	\coordinate[below right = of t421,yshift =0.3cm,xshift=-0.2cm] (b6);
	\draw[rounded corners=7,thick, draw opacity = 0.30, text opacity = 0.6] (b1)--(b2)--(b3)--(b4)--node[yshift=0.2cm,xshift=0.5cm]{$(\gamma_2,v_2)$}(b5)--(b6)--cycle;
	\end{tikzpicture}
\end{center}
Given a set comprising two constructions, $\setofconstructionsirse=\{\cpairiwp{1},\cpairiwp{2}\}$, in $\ispacen''$ that embed into $\setofpatternstc=\{\ppairp{1},\ppairp{2}\}$, each with foundation token-sequence $[l,l'']$, the linear diagrams $l$ and $l''$ are related by the encoded construct relations: $l$ is less effective and semantically equivalent to $l''$.

In practice, however, it is unlikely that we will have access to a fully defined $\ipispacen$ construction space. Thus, it is important to consider what information can be exploited to deduce relations that, theoretically, $\ipispacen$ encodes. In the case of semantic equivalence, we can exploit constructions for this purpose. Notably, each linear diagram comprises blocks of line segments, separated by grey vertical lines; in our example, $l$ and $l'$ each comprise three blocks. Swapping the order of these blocks does not alter the semantics. Thus, given a pair of constructions, $\cpair$ and $\cpairhs$, that match the isomorphic patterns, $\ppaira$ and $\ppairb$,
\begin{center}
  \begin{tikzpicture}[construction,yscale=0.9]
	\node[typeW={\texttt{linearDiagram}}] (t5) {$a$};
	\node[constructorGNE = {\texttt{appendBlock}}, below = 0.45cm of t5] (u) {};
	\node[typeNW={\texttt{linearDiagram}},below left = of u, yshift=1cm, xshift=-0.2cm] (t4) {};
	\node[typeS={\texttt{block}},below right = of u,yshift=0.6cm,xshift=0.0cm] (r5) {$\acolour{a_3}$};
    \node[constructorGNW = {\texttt{appendBlock}}, below = 0.45cm of t4] (u1) {};
	\node[typeSE={\texttt{linearDiagram}},below left = 0.3cm and 1cm of u1] (t411) {$\acolour{a_1}$};
    \node[typeSE={\texttt{block}},below  right = 0.3cm and 1cm of u1] (t421) {$\acolour{a_2}$};
	\path[->]
	(u) edge (t5)
	(t4) edge[bend right = -10] node[index label]{1} (u)
    (r5) edge[bend right = 10] node[index label]{2} (u)
    (u1) edge (t4)
	(t411) edge[bend right = -10] node[index label]{1} (u1)
    (t421) edge[bend right = 10]  node[index label]{2} (u1);
	\end{tikzpicture}
\quad
\begin{tikzpicture}
\node at (0, 2)   (a) {\LARGE$\cong$};
\node at (0, 0)   () {};
\end{tikzpicture}
\quad
\begin{tikzpicture}[construction,yscale=0.9]
	\node[typeW={\texttt{linearDiagram}}] (t5) {$b$};
	\node[constructorGNE = {\texttt{appendBlock}}, below = 0.45cm of t5] (u) {};
	\node[typeNW={\texttt{linearDiagram}},below left = of u, yshift=1cm, xshift=-0.2cm] (t4) {};
	\node[typeS={\texttt{block}},below right = of u,yshift=0.6cm,xshift=0.0cm] (r5) {$\acolour{a_2}$};
    \node[constructorGNW = {\texttt{appendBlock}}, below = 0.45cm of t4] (u1) {};
	\node[typeSE={\texttt{linearDiagram}},below left = 0.3cm and 1cm of u1] (t411) {$\acolour{a_1}$};
    \node[typeSE={\texttt{block}},below right = 0.3cm and 1cm of u1] (t421) {$\acolour{a_3}$};
	\path[->]
	(u) edge (t5)
	(t4) edge[bend right = -10] node[index label]{1} (u)
    (r5) edge[bend right = 10] node[index label]{2} (u)
    (u1) edge (t4)
	(t411) edge[bend right = -10] node[index label]{1} (u1)
    (t421) edge[bend right = 10]  node[index label]{2} (u1);
	\end{tikzpicture}
\end{center}
if $\foundationsto{\cgraphn,t}=[l,x,y]$ and $\foundationsto{\cgraphhn,s}=[l,y,x]$ (i.e. the order of the blocks $x$ and $y$ is swapped) then we can infer that $\ipispacen$ encodes the fact that $t$ is represented by (i.e. is semantically equivalent to) $s$. This `swapped order of blocks' transformation is encoded using the triple $\tRelSpecn=(\ppaira, \ppairb, \emptyset)$. That is, the desired relations between $\cpair$ and $\cpairhs$  are entirely captured by the two patterns in $\tRelSpecn$. However, if we also require that $t$ is less effective than $s$ then even if $\cpair$ and $\cpairhs$ match $\ppaira$ and $\ppairb$ in $\tRelSpecn$, we could not \textit{infer} that $s$ and $t$ are in the required relationship: swapping the order of blocks does not guarantee a reduction in the number of line segments used.
\end{example}

\begin{example}[Alternative Theorems and Proofs]
In \textsc{Propositional Logic}, there are many ways to prove a given theorem from a particular set of axioms. One application of structure transfer is to find alternative proofs, such as to substitute a rule, or set of rules, with a different set of rules. A specific example is given here, where the two constructions have (essentially) the same foundation sequence and same construct, but are different proofs in the entailment space of the \textsc{Propositional Logic} representational system, $\rsystemn_{P}$:
\begin{center}
\adjustbox{}{%
  \begin{tikzpicture}[construction,yscale=0.9]\small
	\node[termrep] (r) {$R$};
	\node[constructorENE = {\texttt{disSyllogism}}, below = 0.45cm of r] (u) {};
	\node[termrep,below left = 0.2cm and 0.6cm of u] (dis) {$\neg P\vee R$};
	\node[termrep,below right = 0.2cm and 0.6cm of u] (negneg) {$\neg \neg P$};
    \node[constructorENW ={\texttt{condDisjuntion}}, below = 0.45cm of dis] (ul) {};
	\node[termrep,below left = 0.2cm and 0.6cm of ul] (imp) {$P\Rightarrow R$};
	\node[constructorENE = {\texttt{doubNegation}}, below = 0.45cm of negneg] (ur) {};
	\node[termrep,below right = 0.2cm and 0.6cm of ur] (p) {$P$};
	
	\path[->]
	(u) edge (r)
	(dis) edge[bend right = -10] node[index label]{1} (u)
    (negneg) edge[bend right = 10] node[index label]{2} (u)
    (ul) edge[bend right = 0] (dis)
	(imp) edge[bend right = -10] node[index label]{1} (ul)
    (ur) edge[bend right = 00] (negneg)
	(p)[bend right = 10] edge node[index label]{1} (ur);
	\end{tikzpicture}}
\hspace{1cm}
\adjustbox{raise = 0.6cm}{%
\begin{tikzpicture}[construction,yscale=0.9]\small
	\node[termrep] (r) {$R$};
	\node[constructorENE = {\texttt{modusPonens}}, below = 0.45cm of r] (u) {};
	\node[termrep,below left = 0.2cm and 0.6cm of u] (dis) {$P\Rightarrow R$};
	\node[termrep,below right = 0.2cm and 0.8cm of u] (negneg) {$P$};	
	\path[->]
	(u) edge (r)
	(dis) edge[bend right = -10] node[index label]{1} (u)
    (negneg) edge[bend right = 10] node[index label]{2} (u);
	\end{tikzpicture}}
 %end of scale box
\end{center}

It is evident that, irrespective of whether individual tokens, such as $P \Rightarrow R$, are the same in the two constructions, their foundation type-sequences are \textit{identical}: $[\texttt{P}\Rightarrow \texttt{R},\texttt{P}]$. Moreover, due to the inference rules used, the constructs are both of type $\texttt{R}$. The fact that the foundation type-sequences are the same supports the application of this structural transformation: whenever the foundation token-sequences give rise to the same foundation type-sequences and the inference rules are used as in the proofs above,{\pagebreak} we can replace constructions matching the pattern, $\ppaira$, below left, with constructions matching the pattern, $\ppairb$, below \begin{samepage}right:
\begin{center}
\adjustbox{}{%
  \begin{tikzpicture}[construction,yscale=0.9]\small
	\node[typeW={\texttt{prop}}] (r) {$\acolour{a}$};
	\node[constructorENE = {\texttt{disSyllogism}}, below = 0.45cm of r] (u) {};
	\node[typeNW={\texttt{prop}},below left = 0.2cm and 0.6cm of u] (dis) {};
	\node[typeNE={\texttt{prop}},below right = 0.2cm and 0.6cm of u] (negneg) {};
    \node[constructorENW ={\texttt{condDisjuntion}}, below = 0.45cm of dis] (ul) {};
	\node[typeW={\texttt{prop}},below left = 0.2cm and 0.6cm of ul] (imp) {$\acolour{a_1}$};
	\node[constructorENE = {\texttt{doubNegation}}, below = 0.45cm of negneg] (ur) {};
	\node[typeE={\texttt{prop}},below right = 0.2cm and 0.6cm of ur] (p) {$\acolour{a_2}$};
	
	\path[->]
	(u) edge (r)
	(dis) edge[bend right = -10] node[index label]{1} (u)
    (negneg) edge[bend right = 10] node[index label]{2} (u)
    (ul) edge[bend right = 0] (dis)
	(imp) edge[bend right = -10] node[index label]{1} (ul)
    (ur) edge[bend right = 00] (negneg)
	(p)[bend right = 10] edge node[index label]{1} (ur);
	\end{tikzpicture}}
\hspace{1cm}
%\begin{minipage}{0.1\textwidth}\centering
%	\vspace{-1.5cm}\longrightarrowX{\Phi}
%\end{minipage}
%
\adjustbox{raise = 0.6cm}{%
\begin{tikzpicture}[construction,yscale=0.9]\small
	\node[typeW={\texttt{prop}}] (r) {$\bcolour{b}$};
	\node[constructorENE = {\texttt{modusPonens}}, below = 0.45cm of r] (u) {};
	\node[typeW={\texttt{prop}},below left = 0.2cm and 0.6cm of u] (dis) {$\bcolour{b_1}$};
	\node[typeE={\texttt{prop}},below right = 0.2cm and 0.6cm of u] (negneg) {$\bcolour{b_2}$};	
	\path[->]
	(u) edge (r)
	(dis) edge[bend right = -10] node[index label]{1} (u)
    (negneg) edge[bend right = 10] node[index label]{2} (u);
	\end{tikzpicture}}
\end{center}\end{samepage}
Permissible rule substitutions, such as that above, can be encoded by a triple
\begin{displaymath}
\tRelSpecn=(\ppaira,\ppairb,\{\ppairp{1},\ppairp{2}\}),
\end{displaymath}
where $\ppaira$ and $\ppairb$ are as just given. The patterns $\ppairp{1}$ and $\ppairp{2}$ encode the desired relations between the foundation token-sequences:
\begin{center}
	\begin{tikzpicture}[construction,yscale=0.9]\small
	%\node[typeIW={\typetrue}] (m) {$v$};
	%\node[constructorIS = {\texttt{and}}, below = 0.4cm of r] (u) {};
	\node[typeINW={\typetrue}] (b1) {$v_1$};
	\node[typeINE={\typetrue}, right = 3.6cm of b1] (b2) {$v_2$};
	\node[constructorINW = {\texttt{sameAssignedType}}, below = 0.45cm of b1] (t1) {};
	\node[typeW={\texttt{prop}},below left = 0.25cm and 0.5cm of t1] (t1l) {$\acolour{a_1}$};
	\node[typeE={\texttt{prop}},below right = 0.2cm and 0.5cm of t1] (t1r) {$\bcolour{b_1}$};
	\node[constructorINE = {\texttt{sameAssignedType}}, below = 0.45cm of b2] (t2) {};
	\node[typeW={\texttt{prop}},below left = 0.25cm and 0.5cm of t2] (t2l) {$\acolour{a_2}$};
	\node[typeE={\texttt{prop}},below right = 0.2cm and 0.5cm of t2] (t2r) {$\bcolour{b_2}$};
	\path[->]
	%(u) edge (m)
	%(b1) edge[bend right = -10] node[index label]{1} (u)
	%(b2) edge[bend right = 10] node[index label]{2} (u)
	(t1) edge[bend right = 0] (b1)
	(t1l) edge[bend right = -10] node[index label]{1} (t1)
	(t1r) edge[bend right = 10] node[index label]{2} (t1)
	(t2) edge[bend right =0] (b2)
	(t2l) edge[bend right = -10] node[index label]{1} (t2)
	(t2r) edge[bend right = 10] node[index label]{2} (t2)
	;
	\end{tikzpicture}
\end{center}
%\begin{center}
%\begin{tikzpicture}[construction,yscale=0.9]\small
%	%\node[typeIW={\typetrue}] (m) {$v$};
%	%\node[constructorIS = {\texttt{and}}, below = 0.4cm of r] (u) {};
%	\node[typeINW={\typetrue}] (b1) {};
%	\node[typeINE={\typetrue}, right = 3.6cm of u] (b2) {};
%    %
%    \node[constructorINW = {\texttt{sameAssignedType}}, below = 0.45cm of b1] (t1) {};
%	\node[typeW={\texttt{type}},below left = 0.25cm and 0.5cm of t1] (t1l) {$\acolour{a_1}$};
%	\node[typeE={\texttt{type}},below right = 0.2cm and 0.5cm of t1] (t1r) {$\bcolour{b_1}$};
%    %
%    %
%    \node[constructorINE = {\texttt{sameAssignedType}}, below = 0.45cm of b2] (t2) {};
%	\node[typeW={\texttt{type}},below left = 0.25cm and 0.5cm of t2] (t2l) {$\acolour{a_2}$};
%	\node[typeE={\texttt{type}},below right = 0.2cm and 0.5cm of t2] (t2r) {$\bcolour{b_2}$};
%    %
%    %%
%	\path[->]
%	%(u) edge (m)
%	%(b1) edge[bend right = -10] node[index label]{1} (u)
%    %(b2) edge[bend right = 10] node[index label]{2} (u)
%    (t1) edge[bend right = 0] (b1)
%    (t1l) edge[bend right = -10] node[index label]{1} (t1)
%    (t1r) edge[bend right = 10] node[index label]{2} (t1)
%    (t2) edge[bend right =0] (b2)
%    (t2l) edge[bend right = -10] node[index label]{1} (t2)
%    (t2r) edge[bend right = 10] node[index label]{2} (t2)
%    ;
%	\end{tikzpicture}
%\end{center}
%
In particular, given constructions $\cpair$ and $\cpairhs$ that
\begin{enumerate}
\item match $\ppaira$ and $\ppairb$, respectively, and
\item have foundation token-sequences $\foundationsto{\cgraphn,t}=[t_1,t_2]$ and $\foundationsto{\cgraphhn,s}=[s_1,s_2]$, so $t_1$, $t_2$, $s_1$ and $s_2$ `instantiate' $a_1$, $a_2$, $b_1$ and $b_2$,
\end{enumerate}
if there exists a pair of constructions, $\cpairiwp{1}$ and $\cpairiwp{2}$, in $\ispacen''$ such that
\begin{enumerate}
\item $\cpairiwp{1}$ and $\cpairiwp{2}$ embed into $\ppairp{1}$ and $\ppairp{2}$, respectively, yielding an isomorphism from $\cgraphin_1\cup \cgraphin_2$ and $\pgraphn_1\cup \pgraphn_2$, and
\item $\cgraphin_1$ and $\cgraphin_2$ include the vertices $t_1$ and $s_1$ and, resp., $t_2$ and $s_2$, and the embedding maps them to $a_1$, $b_2$, $s_2$ and $b_2$, resp.,
\end{enumerate}
then we can transform $\cpair$ into $\cpairhs$. Given any pair of such constructions, it is necessarily the case that their constructs are the same (strictly, have the same assigned type). That is, such a structural transformation provides an alternative proof of a theorem given a set of axioms, falling into Class 2 (see page \pageref{class2} for a description of Class 2 problems).
\end{example}

In summary, these examples illustrate the variety of ways in which finding structural transformations can be of practical use, even though we have only illustrated simple cases due to the space-hungry nature of including more complexity. They also show how one can exploit knowledge of some part of a structure graph (e.g. in example~\ref{ex:stme:linearDiagramsConstructRelations}, we may know that two blocks, $x$ and $y$, are swapped) to infer parts of other structure graphs (e.g. a linear diagram obtained by swapping the order of a pair of blocks is semantically equivalent to the original diagram). As we proceed, we use the first-order arithmetic and dot diagrams representational systems to further illuminate the formal definitions that underpin the theory of structural transformations.

{\pagebreak}
\subsection{First-Order Arithmetic and Dot Diagrams}\label{sec:st:foaAndDotDiags}
Using the \textsc{First-Order Arithmetic} system, $\rsystemn_{\FOA}$, and the \textsc{Dot Diagrams} system, $\rsystemn_{D}$, Sections~\ref{sec:st:tconsAndTransfer} and~\ref{sec:st:soundness}  will demonstrate how to exploit structure transfer to establish that $1+2+3=\frac{3(3+1)}{2}$ is true \begin{samepage}by:
\begin{enumerate}
\item[-] taking a construction of $1+2+3$, and transforming it into a construction of a dot diagram, in particular \hspace{-0.05cm}\scalebox{0.74}{\tikz[baseline=0.07cm]{\oT}}\hspace{0.15cm},
\item[-] taking a construction of $\frac{3(3+1)}{2}$, and transforming it into a construction of a dot diagram, in particular \hspace{-0.05cm}\scalebox{0.74}{\tikz[baseline=0.07cm]{\oT}}\hspace{0.15cm},
\item[-] using the fact that \hspace{-0.05cm}\scalebox{0.74}{\tikz[baseline=0.07cm]{\oT}}\hspace{0.15cm} is (essentially\footnote{Noting that these two dot diagrams could occupy different coordinates in the plane, one is a translation of the other: they are isomorphic.}) the same as \hspace{-0.05cm}\scalebox{0.74}{\tikz[baseline=0.07cm]{\oT}}\hspace{0.15cm}, deduce that $1+2+3=\frac{3(3+1)}{2}$ is true.
\end{enumerate}\end{samepage}

As introduced in Section~\ref{sec:representationalSystems:ME}, $\rsystemn_{\FOA}$ does not include tokens that support the construction of fractions. We, therefore, extend $\rsystemn_{\FOA}$ with a new type, \FOAdiv, whose tokens are horizontal lines that appear in fractions, such as the line in $\frac{3(3+1)}{2}$.  Also included are three new constructors: \FOAcfrac\ and \FOAcimplicitMult, in the grammatical space, and \FOAcisValid\ in the identification space. These constructors have the following signatures:
\begin{eqnarray*}
\spec(\FOAcfrac) & = & ([\FOAnumexp,\FOAdiv,\FOAnumexp],\FOAnumexp)\\
\spec(\FOAcimplicitMult) & = & ([\FOAnumexp,\FOAnumexp],\FOAnumexp)\\
\spec(\FOAcisValid) & = & ([\FOAform],\typebool).
\end{eqnarray*}
The running example seeks to establish which token should be assigned to the vertex $x$, given the illustrated constructions of $1+2+3$ and $\frac{3(3+1)}{2}$:
\begin{center}
\adjustbox{scale=0.8}{%
\begin{tikzpicture}[construction,yscale=0.92]\small
\foavp
\foavone
\foau
\foavtwo
\foavthree
\foavfour
\foauone
\foavfive
\foavsix
\foavseven
\foautwo
\foaveight
\foavnine
\foavten
\foaufour
\foaveleven
\foavtwelve
\foavthirteen
\foaufive
\foavfourteen
\foavfifteen
\foaueight
\foavsixteen
\foavseventeen
\foaveighteen
\foaufifteen

\foavnineteen
\foavtwenty
\foavtwentyone
\foauseventeen
%
%\foarecu
%\foarecuthree
%\foarecufour
%\foarecufive
%\foarecufifteen
%\foarecufifteen
\coordinate[above left = of v2, xshift = 0.8cm, yshift = -0.6cm](a1);
\coordinate[above right = of v2, xshift = -0.8cm, yshift = -0.6cm](a2);
\coordinate[right = of v7, xshift = -0.85cm, yshift = -0.3cm](a3);
\coordinate[below right = of v13, xshift = -0.8cm, yshift = 0.3cm](a4);
\coordinate[below left = of v11, xshift = 0.8cm, yshift = 0.3cm](a5);
\coordinate[left = of v5, xshift = 0.2cm, yshift = 0cm](a6);
\draw[rounded corners=9,thick, draw opacity = 0.2, text opacity = 0.6] (a1) -- (a2) -- (a3) -- (a4) -- (a5) -- (a6) -- cycle;
\coordinate[above left = of v4, xshift = 0.8cm, yshift = -0.6cm](b1);
\coordinate[above right = of v4,xshift = -0.8cm, yshift = -0.6cm](b2);
\coordinate[right = of v10, xshift = -0.8cm, yshift = 0.1cm](b3);
\coordinate[below right = of v21, xshift = -0.7cm, yshift = 0.2cm](b4);
\coordinate[below left = of v19, xshift = 0.6cm, yshift = 0.2cm](b5);
\coordinate[left = of v8, xshift = -0.1cm, yshift = -0.4cm](b6);
\draw[rounded corners=9,thick, draw opacity = 0.2, text opacity = 0.6] (b1) -- (b2) -- (b3) -- (b4) -- (b5) -- (b6) -- cycle;
\end{tikzpicture}}
\end{center}
Note that in the above graph, the labels associated with the tokens, such as $t_{6}$ assigned to the token 1 (bottom left) are \textit{meta-level names} that allow us to refer to specific vertices and are \textit{not} types.

{\pagebreak}
Regarding the \textsc{Dot Diagrams} system, we introduce a fragment of $\rsystemn_{D}$ that is sufficient for our purposes. These types are of \begin{samepage}interest:
\begin{enumerate}
\item[-]  \DDsingleDot:  dot diagrams that are formed from a single dot,
\item[-]  \DDhorDots:  dot diagrams that are formed from a horizontal row of dots,
\item[-]  \DDverDots:  dot diagrams that are formed from a vertical row of dots,
\item[-]  \DDstack:  dot diagrams that are formed from vertically-stacked horizontal rows of dots without gaps between the rows,
\item[-]  \DDrectangulation:  dot diagrams that are stacks in the shape of a rectangle,
\item[-]  \DDtriangulation:  dot diagrams that are stacks in the shape of a triangle, and
\item[-]  \DDdotDiagram:  any dot diagram.
\end{enumerate}\end{samepage}
The types satisfy the order specified by the directed graph below:
\begin{center}
\scalebox{0.8}{
	\begin{tikzpicture}[node distance = 0.2cm and 0.4cm]\footnotesize
	%\draw[rounded corners, black!30, thick] (-6,-4.5) rectangle node[xshift = 3.2cm,yshift = 2cm,black!50] {$\tsystem$} (5,0.5);
	\node[] (dotDiagram) {$\DDdotDiagram$};
	\node[below left = of dotDiagram,xshift=0.5cm] (stack) {$\DDstack$};
	\node[below left = of stack, xshift = 0.5cm] (rect) {$\DDrectangulation$};
	\node[below right = of stack, xshift = -0.5cm] (tri) {$\DDtriangulation$};
	\node[below left = of rect, xshift = 0.7cm] (hor) {$\DDhorDots$};
	\node[below right = of rect, xshift = -0.7cm] (ver) {$\DDverDots$};
	\node[below right = 0.3cm and -0.5cm of ver] (sing) {$\DDsingleDot$};
	%	\node[below left = of sing, xshift = 0.6cm] (sing1) {$\{(0,1)\}$};
	%	\node[below  = of sing, xshift = 0.0cm] (sing2) {$\{(0,2)\}$};
	%	\node[below right = of sing, xshift = -0.6cm] (sing3) {$\cdots$};
	\draw[->] (sing) edge (hor);
	\draw[->] (sing) edge (ver);
	\draw[->] (sing) edge (tri);
	\draw[->] (ver) edge (rect);
	\draw[->] (hor) edge (rect);
	\draw[->] (rect) edge (stack);
	\draw[->] (tri) edge (stack);
	\draw[->] (stack) edge (dotDiagram);
	\draw[draw opacity=0.2, rounded corners] (-4.65,-2.95) rectangle node[above,xshift=-2.2cm,yshift=1.1cm, text opacity = 0.4] {$(\types,\leq)$} (0.95,0.27) ;
	\end{tikzpicture}
}
\end{center}
There are also (infinitely many) minimal types, such as $\{(\texttt{0},\texttt{0}),(\texttt{0},\texttt{1})\}$, where the natural order relation is taken to hold with respect to the types given above. Constructors are specified as follows:
\begin{enumerate}
\item[-]  \DDcjoin: forms the union of two dot diagrams, and was seen in example~\ref{ex:STMErerepresentingTokenSE}. We define  $\spec(\DDcjoin) = ([\DDdotDiagram,\DDdotDiagram],\DDdotDiagram)$.
\item[-]  \DDcremove: removes any dots from the first input that also occur in the second input; given a configuration of \DDcremove\ with foundation token-sequence $\big[\hspace{-0.1cm}\scalebox{0.74}{\tikz[baseline=0.07cm]{\oT}}\hspace{0.17cm},\hspace{-0.1cm}\scalebox{0.74}{\tikz[baseline=-0.15cm]{\ooo}}\hspace{0.15cm}\big]$, where the dots in the second input are those that form the bottom row of the first input, the output token is $\hspace{-0.05cm}\scalebox{0.74}{\tikz[baseline=-0.05cm]{\oL}}\hspace{0.15cm}$. We define $\spec(\DDcremove) =  ([\DDdotDiagram,\DDdotDiagram],\DDdotDiagram)$. %%%
\item[-]  \DDcappend:  places the second input at the end of the first; given a configuration of \DDcappend\ with foundation token-sequence $[\hspace{-0.05cm}\scalebox{0.74}{\tikz[baseline=-0.1cm]{\crcD}}\hspace{0.1cm},\hspace{-0.1cm}\scalebox{0.74}{\tikz[baseline=-0.1cm]{\oo}}\hspace{0.15cm}]$, the output token is $\hspace{-0.05cm}\scalebox{0.74}{\tikz[baseline=-0.15cm]{\ooo}}\hspace{0.15cm}$. We define $\spec(\DDcappend)  = ([\DDhorDots,\DDhorDots],\DDhorDots)$. %%%
\item[-]  \DDcstackLeft: takes two stacks of dots and places the second stack on top of the first stack, with their left-most dots aligned; given a configuration of \DDcstackLeft\ with foundation token-sequence $[\hspace{-0.1cm}\scalebox{0.74}{\tikz[baseline=-0.15cm]{\ooo}}\hspace{0.17cm}, \hspace{-0.1cm}\scalebox{0.74}{\tikz[baseline=-0.05cm]{\oL}}\hspace{0.16cm}]$, the output token is $\hspace{-0.1cm}\scalebox{0.74}{\tikz[baseline=0.07cm]{\oT}}\hspace{0.15cm}$. We define $\spec(\DDcstackLeft)  =  ([\DDstack,\DDstack],\DDstack)$. %%%
\item[-]  \DDcrectangulate: makes a rectangle whose lefthand column derived from the first input and whose bottom row is derived from the second input. This constructor requires that the bottom dot of the first input and the bottom dot of the second input are the same; given a configuration of \DDcrectangulate\ with foundation token-sequence $\big[\hspace{-0.1cm}\scalebox{0.74}{\tikz[baseline=0.07cm]{\oV}}\hspace{0.2cm}, \hspace{-0.1cm}\scalebox{0.74}{\tikz[baseline=-0.1cm]{\oH}}\hspace{0.15cm}\big]$, the output token is $\hspace{-0.05cm}\scalebox{0.74}{\tikz[baseline=0.07cm]{\oR}}\hspace{0.15cm}$. We define $\spec(\DDcrectangulate)  = ([\DDverDots,\DDhorDots], \DDrectangulation)$.  %%%
\item[-]  \DDcrotateTri: produces a new triangulation from a right-angled triangulation, reflecting along the hypotenuse; given a configuration of \DDcrotateTri\ with foundation token-sequence $\big[\hspace{-0.1cm}\scalebox{0.74}{\tikz[baseline=0.07cm]{\oT}}\hspace{0.17cm}\big]$, the output token is $\hspace{-0.1cm}\scalebox{0.74}{\tikz[baseline=0.07cm]{\oTrot}}\hspace{0.15cm}$, which shares three dots with the input token (i.e. those on the hypotenuse). We define $\spec(\DDcrotateTri)  = ([\DDtriangulation], \DDtriangulation)$. %%% note this is reflect not rotate!!!!

\end{enumerate}
{\pagebreak}
Relating this back to the running example's goal, to identify the token of type \typebool\ assigned to $x$  in the $\rsystemn_{\FOA}$ construction given above, this \textsc{Dot Diagrams} construction is of particular interest, as will become clear as the discussion proceeds throughout \begin{samepage}Section~\ref{sec:structuralTransformations}:
\begin{center}
\adjustbox{scale=0.9}{%
\begin{tikzpicture}[construction,yscale=0.9]\small
    \dotsv
    \dotsvone
    \dotsvtwo
    \dotsu
    \dotsvthree
    \dotsvfour
    \dotsuone
	\dotsvsix
    \dotsvseven
    \dotsuthree
    \dotsvfive
    \dotsutwo
    \dotsvtwonew
    \dotsveight
    \dotsufive
    \dotsvnine
    \dotsvten
    \dotsueight

    \dotsveleven
    \dotsvtwelve
    \dotsuten
\end{tikzpicture}}
\end{center}\end{samepage}
The identification space, $\ispacen_{D}$, of $\rsystemn_{D}$ includes the meta-types \typebool, \typetrue, and \typefalse, where $\typetrue \oleq \typebool$ and $\typefalse \oleq \typebool$. Alongside logical constructors, such as \constructorAnd, further constructors occur in the identification space for $\rsystemn_{D}$:
\begin{enumerate}
\item[-] \DDcisTranslationOf: $\spec(\DDcisTranslationOf)=([\DDdotDiagram,\DDdotDiagram],\typebool)$ has output $\top$ whenever the first dot diagram is a translation, in the plane, of the second dot diagram.
\item[-] \DDcisJoinOf: $\spec(\DDcisJoinOf)=([\DDdotDiagram,\DDdotDiagram,\DDdotDiagram],\typebool)$ has output $\top$ whenever there exists a configuration of $\DDcjoin$ in the grammatical space with the first dot diagram as its output and the other two dot diagrams as its input.
\item[-] \DDcsameNumOfDots: $\spec(\DDcsameNumOfDots)=([\DDdotDiagram,\DDdotDiagram],\typebool)$ has output $\top$ whenever the dot diagrams contain the same number of dots.
\end{enumerate}
As is suggested by the constructors above, $\ispacen_{D}$ includes meta-types \typebool, \typetrue\ and \typefalse. Finally, constructors are needed for the inter-representational-system encoding, $\ispacen''$:
\begin{enumerate}
\item[-] \constructorRepresentedBy: $\spec(\constructorRepresentedBy)=([\FOAnumexp,\DDdotDiagram],\typebool)$ has output $\top$ whenever the numerical expression evaluates to the number of dots in the dot diagram
\end{enumerate}
alongside standard logical constructors. Equipped with these partial definitions of the type systems and constructors for $\rsystemn_{\FOA}$, $\rsystemn_{D}$, and $\ispacen''$, we have set the groundwork for using first-order arithmetic and dot diagrams as a running example as we introduce the theory of structural transformations.

\subsection{Transformation Constraints and Structural Transformations}\label{sec:st:tconsAndTransfer}
To reiterate, the theory of structural transformations relies on having an inter-system encoding, say $\irencodingn=\irencoding$, and uses information provided by $\ipispacen$ to transform a construction in $\rsystemn$ into a construction in $\rsystemn'$. Our initial focus is on a simple case of structural transformations, where the goal is to replace one construction, $\cpair$, in $\rsystemn$ with another, $\cpairhs$, in $\rsystemn'$ such that some relations hold between their tokens, captured by some set of constructions $\setofconstructionsirse$, called a \textit{property identifier}, found in $\ipispacen$.

\begin{definition}\label{defn:propertyIDandEn}
Let $\irencodingn=\irencoding$ be an inter-representational-system encoding with description $\descriptionirse$. A set of constructions, $\setofconstructionsirse$, in $\ipispacen$, is a \textit{property identifier} for $\irencodingn$. A set of patterns, $\setofpatternstc$, that is a subset of $\setofpatterns''$ is called a \textit{property enforcer} for $\irencodingn$.
\end{definition}

\begin{example}[Satisfying Transformation Constraints] \label{ex:FOAandDtokenRelSpec}
In first-order arithmetic, the construction, which we call $\cpairp{2}$, on the left shows one way of building the expression 1+2+3, which evaluates to 6 under the standard semantics of addition. This construction is in some sense equivalent to the construction $\cpairhsp{2}$, on the right, drawn from $\rsystemn_{D}$.
\begin{center}
\begin{minipage}[b]{0.4\textwidth}\centering
\begin{tikzpicture}[construction,yscale=0.9]
\foavtwo
\foavfive
\foavsix
\foavseven
\foautwo
\end{tikzpicture}
\end{minipage}
\qquad
\begin{minipage}[b]{0.4\textwidth}\centering
\begin{tikzpicture}[construction,yscale=0.9]
\dotsvone
\dotsvthree
\dotsvfour
\dotsuone
\end{tikzpicture}
\end{minipage}
\end{center}
This notion of equivalence arises since the six-ness of $1+2+3$ is embodied by six dots.  The following two-construction property identifier, $\setofconstructionsirse=\{(\cgraphin,\top),(\cgraphin',\top)\}$, for $\irencodingn$ captures the facts that $1+2$ is represented by three dots, here shown as a left-aligned stack, and $3$ is also represented by three dots, shown as a horizontal arrangement:
\begin{center}
\begin{tikzpicture}[construction,yscale=0.8]
%\node[termIrep] at (0,-0.3) (v) {$\top$};
%
%\node[constructorINE = {\constructorAnd}] at (0,-1.2) (u) {};
%
\node[termIrep] at (-2,-1.6) (v1) {$\top$};
\node[termIrep] at (2,-1.6) (v2) {$\top$};
\node[constructorINW = {\constructorRepresentedBy}] at (-2,-2.5) (u1) {};
\node[termW = {$t_{3}$}] at (-3,-3.5) (v3) {$1+2$};
\node[termE = {$s_3$},inner sep = 1pt] at (-1,-3.5) (v4) {$\oL$};
\node[constructorINE = {\constructorRepresentedBy}] at (2,-2.5) (u2) {};
\node[termW = {$t_{5}$}] at (1,-3.5) (v5) {$3$};
\node[termE = {$s_4$},inner sep = 1pt] at (3,-3.5) (v6) {$\ooo$};
    \path[->]
%(v1) edge[bend right = -5] node[ index label] {1} (u)
%	(v2) edge[bend right = 5] node[ index label] {2} (u)
%	(u) edge (v)
    %
    (v3) edge[bend right = -10] node[ index label] {1} (u1)
	(v4) edge[bend right = 10] node[ index label] {2} (u1)
	(u1) edge (v1)
    (v5) edge[bend right = -10] node[ index label] {1} (u2)
	(v6) edge[bend right = 10] node[ index label] {2} (u2)
	(u2) edge (v2);

\end{tikzpicture}
\end{center}
It is assumed that tokens which appear identical across the three graphs just given are indeed the same tokens (e.g. the occurrence of the token 1+2 in $\cpairp{2}$ is the same token 1+2 in the constructions of $\setofconstructionsirse$; both are token $t_3$), hence the identical names assigned to the vertices. In fact, whenever we have constructions like $\cpairp{2}$ and $\cpairhsp{2}$ where the relationship between their tokens exists as identified by $\setofconstructionsirse$, we can transform the former into the latter. This general view of the transformation can be specified using a triple, $\tRelSpecn_1=(\ppairap{2}, \ppairbp{2}, \setofpatternstc)$ called a \textit{transformation constraint}:
\begin{center}
\adjustbox{scale=0.9}{%
\begin{tikzpicture}[construction]\small
\foavtwop
\foavfivep
\foavsixp
\foavsevenp
\foautwo
\end{tikzpicture}
}
\hspace{0.5cm}
\adjustbox{scale=0.9}{%
\begin{tikzpicture}[construction]\small
\dotsvonep
\dotsvthreep
\dotsvfourp
\dotsuone
\end{tikzpicture}
}
\hspace{0.5cm}
\adjustbox{scale=0.9,raise=0.2cm}{%
\begin{tikzpicture}[construction,yscale=0.85,xscale=0.85]\small
%\node[typeIE = {\typetrue}] at (0,-0.3) (v) {$v_1$};
%
%\node[constructorINE = {\constructorAnd}] at (0,-1.2) (u) {};
%
\node[typeIW = {\typetrue}] at (-1.8,-1.6) (v1) {$v_1$};
\node[typeIE = {\typetrue}] at (1.3,-1.55) (v2) {$v_1'$};
\node[constructorINW = {\constructorRepresentedBy}] at (-1.8,-2.6) (u1) {};
\node[typeS = {\FOAnumexp}] at (-2.7,-3.5) (v3) {$\acolour{a_{3}}$};
\node[typeS = {\DDstack}] at (-0.9,-3.5) (v4) {$\bcolour{b_3}$};
\node[constructorINE = {\constructorRepresentedBy}] at (1.3,-2.6) (u2) {};
\node[typeS = {\FOAnumexp}] at (0.4,-3.5) (v5) {$\acolour{a_5}$};
\node[typeS = {\DDstack}] at (2.2,-3.5) (v6) {$\bcolour{b_4}$};
    \path[->]
%	(v1) edge[bend right = -10] node[ index label] {1} (u)
%	(v2) edge[bend right = 10] node[ index label] {2} (u)
%	(u) edge (v)
    %
    (v3) edge[bend right = -10] node[ index label] {1} (u1)
	(v4) edge[bend right = 10] node[ index label] {2} (u1)
	(u1) edge (v1)
    (v5) edge[bend right = -10] node[ index label] {1} (u2)
	(v6) edge[bend right = 10] node[ index label] {2} (u2)
	(u2) edge (v2);
\end{tikzpicture}
}
\end{center}
Here, $\setofpatternstc$ is a two-pattern property enforcer. In the case of $\cpairp{2}$, $\cpairhsp{2}$, $(\cgraphin,\top)$ and $(\cgraphin',\top)$, their embeddings into $\ppairap{2}$, $\ppairbp{2}$, $(\pgraphmn_1,v_1)$ and $(\pgraphn_1',v_1')$ induce an isomorphism from $\cgraphn_2\cup \cgraphhn_2\cup \cgraphin\cup \cgraphin'$ to $\pgraphan_2\cup \pgraphbn_2\cup \pgraphmn\cup \pgraphmn'$. It is these conditions -- matching the two patterns and property enforcer and the existence of an induced isomorphism -- that capture the fact that $\cpairp{2}$ can be transformed into $\cpairhsp{2}$. Of further note is that given any constructions, $\cpaird$ in $\rsystemn_{\FOA}$, $\cpairhsd$ in $\rsystemn_{D}$, along with a property identifier $\setofconstructionsirse$ for $\irencodingn$, that match the above four patterns in the obvious way it is assured that this particular construction also exists in $\ipispacen$:
\begin{center}
\begin{tikzpicture}[construction,yscale=0.8]
\node[termIE = {\typetrue}] at (0,0) (v) {$\top$};
\node[constructorINW = {\constructorRepresentedBy}] at (0,-0.9) (u) {};
\node[termW = {\FOAnumexp}] at (-1,-1.8) (v1) {$t'$};
\node[termE = {\DDstack}] at (1,-1.8) (v2) {$s'$};
    \path[->]
	(v1) edge[bend right = -10] node[ index label] {1} (u)
	(v2) edge[bend right = 10] node[ index label] {2} (u)
	(u) edge (v);
\end{tikzpicture}
\end{center}
That is, the evaluation of the numerical expression $t'$, as an integer, is necessarily the same as the number of dots in the diagram $s'$.
\end{example}

Definition~\ref{defn:st:tokenRelSpec} formalizes the notion of a \textit{transformation constraint}, which will be used later on to capture the desired relationship between $\cpair$ and $\cpairhs$.

\begin{definition}\label{defn:st:tokenRelSpec}
Let $\irencodingn=\irencoding$ be an inter-representational-system encoding with description $\descriptionirse$. A \textit{transformation constraint} for $\irencodingn$ is a triple,  $\tRelSpecn=(\ppaira,\ppairb,\setofpatternstc)$, where $\ppaira\in \setofpatterns(\descriptionn)$, $\ppairb\in \setofpatterns(\descriptionn')$, and $\setofpatternstc$ is a property enforcer for $\irencodingn$.
 \end{definition}

In order to formalise what it means for $(\cpair, \cpairhs, \setofconstructionsirse)$ (i.e. a pair of constructions together with a set of property identifiers) to \textit{satisfy} a transformation constraint -- and, more generally, what it means for one transformation constraint to satisfy another -- we first generalise the concept of \textit{an embedding of one pattern into another} to \textit{sets of patterns into another set}, which requires us to turn a set of patterns into a single graph.

\begin{definition}\label{defn:patternGraphOfPatterns}
Let $\setofpatterns$ be a set of patterns. The \textit{structure graph of $\setofpatterns$}, denoted $\patterngraphn(\setofpatterns)$, is defined to be $\patterngraphn(\setofpatterns)=\bigcup\{\pgraphn\colon \ppair\in \setofpatterns\}$.
\end{definition}

\begin{definition}\label{defn:respectfulEmbedding}
Let $\setofpatterns$ and $\setofpatterns'$ be two sets of patterns. We say that $\setofpatterns$ \textit{respectfully embeds} into $\setofpatterns'$ provided there exists an embedding, $f\colon \patterngraphn(\setofpatterns) \to \patterngraphn(\setofpatterns')$,
such that for all $\ppair\in \setofpatterns$, the restriction of $f$ to $\pgraphn$ gives rise to an embedding of $\ppair$ into some pattern, $\ppaird$, in $\setofpatterns'$. Such an embedding, $f$, is called a \textit{respectful embedding} of $\setofpatterns$ into $\setofpatterns'$. %We write $f\colon \setofpatternstc\to \setofpatternstc'$ when $f$ is a respectful embedding.
\end{definition}

\begin{definition}\label{defn:satisfies}\label{defn:st:RTransform}
Let $\tRelSpecn=(\ppaira,\ppairb,\setofpatternstc)$ and $\tRelSpecn'=(\ppairad,\ppairbd,\setofpatternstc')$ be transformation constraints. Then $\tRelSpecn$ \textit{satisfies} $\tRelSpecn'$, denoted $\tRelSpecn\models \tRelSpecn'$, provided there exists embeddings $f_\alpha\colon \ppaira\to \ppairad$ and $f_\beta\colon \ppairb\to \ppairbd$ and a respectful embedding $f_\setofpatternstc\colon \patterngraphn(\setofpatternstc)\to \patterngraphn(\setofpatternstc')$  such that
\begin{displaymath}
f_\alpha\cup f_\beta\cup f_\setofpatternstc\colon \pgraphan\cup \pgraphbn\cup \patterngraphn(\setofpatternstc) \to \pgraphan'\cup \pgraphbn' \cup \patterngraphn(\setofpatternstc')
\end{displaymath}
is an isomorphism.
\end{definition}

%The existence of the three embeddings in definition~\ref{defn:satisfies} ensures that $\cpairp{1}$, $\cpairp{2}$, and $\cpairp{3}$ match $\ppairp{1}$, $\ppairp{2}$, and $\ppairp{3}$.
Clearly, two different transformation constraints, $\tRelSpecn$ and $\tRelSpecn'$, can be satisfied by precisely the same triples. Definition~\ref{defn:equivalentTokenRelSpecs} embodies such a notion of equivalence. Lemma~\ref{lem:equivTokenRelSpecs} captures the fact that equivalent transformation constraints are satisfied by the same triples.

\begin{definition}\label{defn:equivalentTokenRelSpecs}
Let $\tRelSpecn=(\ppaira,\ppairb,\setofpatternstc)$ and $\tRelSpecn'=(\ppairad,\ppairbd,\setofpatternstc')$ be transformation constraints. Then $\tRelSpecn$ and $\tRelSpecn'$ are \textit{equivalent}, denoted $\tRelSpecn\cong \tRelSpecn'$, provided $\tRelSpecn\models \tRelSpecn'$ and $\tRelSpecn'\models \tRelSpecn$.
\end{definition}

It can trivially be shown that $\cong$ is an equivalence relation, exploiting function inverses and compositions.

\begin{lemma}\label{lem:equivTokenRelSpecs}
Let $\tRelSpecn=(\ppaira,\ppairb,\setofpatternstc)$ and $\tRelSpecn'=(\ppairad,\ppairbd,\setofpatternstc')$ be equivalent transformation constraints. Let $\tRelSpecn''=(\ppairadd,\ppairbdd,\setofpatternstc'')$ be a transformation constraint. Then $\tRelSpecn''$ satisfies $\tRelSpecn$ iff $\tRelSpecn''$ satisfies $\tRelSpecn'$.
\end{lemma}

The proof of lemma~\ref{lem:equivTokenRelSpecs} can be found in Appendix~\ref{sec:app:structure}, see lemma~\ref{lema:equivTokenRelSpecs}. The theory is now extended to show how structural transformations exploit transformation constraints and decompositions. We return to the running example, in order to exemplify the definition of a \textit{constraint assignment} -- which assigns transformation constraints to pairs of vertices in pattern decompositions -- and definition~\ref{defn:structuralTransformationConstruction} which formalises structural transformations.

\begin{example}[Constraint Assignments and Structural Transformations]\label{ex:FOAandDotsStructuralLinks}
This example will exploit five transformation constraints for $\rsystemn_{\FOA}$ and $\rsystemn_{D}$, and will also use constructors found in the identification space of the inter-representational-system encoding.{\pagebreak} The first transformation constraint that we define is needed for the first two structure transfers that our running example seeks to \begin{samepage}produce:
\begin{enumerate}
\item[-]  take a construction, $(\cgraphn,1+2+3)$, of token $1+2+3$ and produce a construction of a dot diagram that represents $1+2+3$ (i.e., a diagram with six dots).
\item[-] take a construction, $\big(\cgraphn',\frac{3(3+1)}{2}\big)$, of token $\frac{3(3+1)}{2}$ and produce a construction of a dot diagram that represents $\frac{3(3+1)}{2}$.
\end{enumerate}\end{samepage}
The requirement that the constructed numerical expression is \textit{represented by} the constructed dot diagram is embodied by $\tRelSpecn_0=((\pgraphan',a_2),(\pgraphbn',b_2),\{\pgraphn_0,v_0\})$, where the constructions --  which are trivial -- and property enforcer are as follows:
\begin{center}
\adjustbox{scale=\myscale,raise=0.15cm}{%
\begin{tikzpicture}[construction]
\node[typeE = {\FOAnumexp}] at (0,0) (v) {$\acolour{a_2}$};
\end{tikzpicture}
}
\hspace{2cm}
\adjustbox{scale=\myscale,raise=0.1cm}{%
\begin{tikzpicture}[construction]
\node[typeE = {\DDstack}] at (0,0) (v) {$\bcolour{b_2}$};
\end{tikzpicture}
}
\hspace{2cm}
\adjustbox{scale=\myscale}{%
\begin{tikzpicture}[construction,yscale=0.8]
\node[typeIE = {\typetrue}] at (0,0) (v) {$v_0$};
\node[constructorINE = {\constructorRepresentedBy}] at (0,-1) (u) {};
\node[typeW = {\FOAnumexp}] at (-1,-1.8) (v1) {$\acolour{a_2}$};
\node[typeE = {\DDstack}] at (1,-1.8) (v2) {$\bcolour{b_2}$};
    \path[->]
	(v1) edge[bend right = -10] node[ index label] {1} (u)
	(v2) edge[bend right = 10] node[ index label] {2} (u)
	(u) edge (v);
\end{tikzpicture}
}
\end{center}
For the remaining five transformation constraints, let us first consider a construction, $(\cgraphn,1+2+3)$, along with a pattern that describes it:
\begin{center}
\adjustbox{scale=\myscale}{
\begin{tikzpicture}[construction,yscale=0.9]
\foavtwo
\foavfive
\foavsix
\foavseven
\foautwo
\foaveleven
\foavtwelve
\foavthirteen
\foaufive
\end{tikzpicture}
\hspace{2cm}
\begin{tikzpicture}[construction,yscale=0.9]
\foavtwop
\foavfivep
\foavsixp
\foavsevenp
\foautwo
\foavelevenp
\foavtwelvep
\foavthirteenp
\foaufive
\end{tikzpicture}
}
\end{center}

From the above construction of $1+2+3$, our structural transformation will produce the following dot diagram construction, $\big(\cgraphhn,\hspace{-0.1cm}\scalebox{0.74}{\tikz[baseline=0.07cm]{\oT}}\hspace{0.15cm}\big)$, shown here with a describing pattern:
\begin{center}
\adjustbox{scale=\myscale}{
\begin{tikzpicture}[construction,yscale=0.9]
\dotsvone
\dotsvthree
\dotsvfour
\dotsvsix
\dotsvseven
\dotsuone
\dotsuthree
\end{tikzpicture}
\hspace{2.5cm}
\begin{tikzpicture}[construction,yscale=0.9]
\dotsvonep
\dotsvthreep
\dotsvfourp
\dotsvsixp
\dotsvsevenp
\dotsuone
\dotsuthree
\end{tikzpicture}
}
\end{center}
The transformation constraint $\tRelSpecn_1=(\ppairap{2},\ppairbp{2},\setofpatternstc)$, given in example~\ref{ex:FOAandDtokenRelSpec}, is sufficient to support this transformation, along with $\tRelSpecn_0$. The constraint $\tRelSpecn_1$ arises since the addition of numbers, captured by $\ppairap{2}$ using the \FOAcinfixop\ constructor, `aligns' with the \DDcstackLeft\ operation on dot diagrams; \FOAcinfixop\ aligns with another dot diagram constructor, \DDcappend, and we return to this shortly. First, we illustrate the process of structurally transforming $(\cgraphn,1+2+3)$ into $\big(\cgraphhn,\hspace{-0.1cm}\scalebox{0.74}{\tikz[baseline=0.07cm]{\oT}}\hspace{0.2cm}\big)$.{\pagebreak} The decompositions above give rise to the following decompositions, $\decompositionn$ (top left), $\decompositionn'$ (top right), $\pdecompositionn$ (bottom left) and $\pdecompositionn'$ \begin{samepage}(bottom right):
\begin{center}\label{ex:decompositionOfPatternsForLater}
\adjustbox{scale=0.9}{%
\begin{tikzpicture}[decomposition,yscale=0.9]
        \node[draw=darkgreen, rounded corners, label=180:{\Large$d_1$}] (dv1) at (0,-0.1) {\begin{tikzpicture}\foavtwo\end{tikzpicture}};
        \node[draw=darkblue, rounded corners, label=180:{\Large$d_2$}] (dv2) at (0,-3) {\begin{tikzpicture}\foavtwo \foavfive \foavsix \foavseven \foautwo \end{tikzpicture}};
        \node[draw=darkblue, rounded corners, label=180:{\Large$d_3$}] (dv3) at (-1.7,-7) {\begin{tikzpicture}\foavfive \foaveleven \foavtwelve \foavthirteen \foaufive\end{tikzpicture}};
        \node[draw=darkred, rounded corners, label=270:{\Large$d_4$}] (dv4) at (0.6,-7) {\begin{tikzpicture}\foavsix\end{tikzpicture}};
        \node[draw=darkred, rounded corners, label=270:{\Large$d_5$}] (dv5) at (1.7,-7) {\begin{tikzpicture}\foavseven\end{tikzpicture}};
	\path[->]
    ([xshift=0.6cm]dv2.north) edge[out = 90, in = -90] node[decomp arrow label] {1} ([xshift=0.1cm]dv1.south)
    (dv3) edge[out = 90, in = -120] node[decomp arrow label] {1} ([xshift = -1.1cm]dv2.south)
    (dv4) edge[out = 90, in = -95] node[decomp arrow label] {2} ([xshift = 0.2cm]dv2.south)
    (dv5) edge[out = 90, in = -80] node[decomp arrow label] {3} ([xshift = 1.4cm]dv2.south)
;
	\end{tikzpicture}
}\hspace{1.7cm}
\adjustbox{scale=0.9}{%
\begin{tikzpicture}[decomposition,yscale=0.9]
        \node[draw=darkgreen, rounded corners, label=180:{\Large$d_1'$}] (dv1) at (0,0) {\begin{tikzpicture}\dotsvone\end{tikzpicture}};
        \node[draw=darkblue, rounded corners, label=180:{\Large$d_2'$}] (dv2) at (0,-3) {\begin{tikzpicture}\dotsvone \dotsvthree \dotsvfour \dotsuone \end{tikzpicture}};
        \node[draw=darkblue, rounded corners, label=180:{\Large$d_3'$}] (dv3) at (-1.3,-7) {\begin{tikzpicture}\dotsvthree \dotsvsix \dotsvseven \dotsuthree\end{tikzpicture}};
        \node[draw=darkred, rounded corners, label=270:{\Large$d_4'$}] (dv4) at (1.2,-7) {\begin{tikzpicture}\dotsvfour\end{tikzpicture}};
	\path[->]
    ([xshift=0.1cm]dv2.north) edge[out = 90, in = -90] node[decomp arrow label] {1} ([xshift=0.1cm]dv1.south)
    (dv3) edge[out = 90, in = -110] node[decomp arrow label] {1} ([xshift=-0.9cm]dv2.south)
    (dv4) edge[out = 90, in = -80] node[decomp arrow label] {2} ([xshift=1.1cm]dv2.south)
;
	\end{tikzpicture}
}
\bigskip

\adjustbox{scale=0.9}{%
\begin{tikzpicture}[decomposition,yscale=0.9]
        \node[draw=darkgreen, rounded corners, label=180:{\Large$\delta_1$}] (dv1) at (0,-0.2) {\begin{tikzpicture}\foavtwop\end{tikzpicture}};
        \node[draw=darkblue, rounded corners, label=180:{\Large$\delta_2$}] (dv2) at (0,-3) {\begin{tikzpicture}\foavtwop \foavfivep \foavsixp \foavsevenp \foautwo \end{tikzpicture}};
        \node[draw=darkblue, rounded corners, label=180:{\Large$\delta_3$}] (dv3) at (-1.8,-7) {\begin{tikzpicture}\foavfivep \foavelevenp \foavtwelvep \foavthirteenp \foaufive\end{tikzpicture}};
        \node[draw=darkred, rounded corners, label=270:{\Large$\delta_4$}] (dv4) at (0.6,-7) {\begin{tikzpicture}\foavsixp\end{tikzpicture}};
        \node[draw=darkred, rounded corners, label=270:{\Large$\delta_5$}] (dv5) at (2,-7) {\begin{tikzpicture}\foavsevenp\end{tikzpicture}};
	\path[->]
    ([xshift=0.7cm]dv2.north) edge[out = 90, in = -90] node[decomp arrow label] {1} ([xshift=0.45cm]dv1.south)
    (dv3) edge[out = 90, in = -120] node[decomp arrow label] {1} ([xshift=-1.2cm]dv2.south)
    (dv4) edge[out = 90, in = -90] node[decomp arrow label] {2} ([xshift=0.4cm]dv2.south)
    (dv5) edge[out = 90, in = -70] node[decomp arrow label] {3} ([xshift=1.7cm]dv2.south)
;
	\end{tikzpicture}
}\hspace{1.7cm}
\adjustbox{scale=0.9}{%
\begin{tikzpicture}[decomposition,yscale=0.9]
        \node[draw=darkgreen, rounded corners, label=180:{\Large$\delta_1'$}] (dv1) at (0,-0.2) {\begin{tikzpicture}\dotsvonep\end{tikzpicture}};
        \node[draw=darkblue, rounded corners, label=180:{\Large$\delta_2'$}] (dv2) at (0,-3) {\begin{tikzpicture}\dotsvonep \dotsvthreep \dotsvfourp \dotsuone \end{tikzpicture}};
        \node[draw=darkblue, rounded corners, label=180:{\Large$\delta_3'$}] (dv3) at (-1.3,-7) {\begin{tikzpicture}\dotsvthreep \dotsvsixp \dotsvsevenp \dotsuthree\end{tikzpicture}};
        \node[draw=darkred, rounded corners, label=270:{\Large$\delta_4'$}] (dv4) at (1.2,-7) {\begin{tikzpicture}\dotsvfourp\end{tikzpicture}};
	\path[->]
    ([xshift=0.26cm]dv2.north) edge[out = 90, in = -90] node[decomp arrow label] {1} ([xshift=0.26cm]dv1.south)
    (dv3) edge[out = 90, in = -110] node[decomp arrow label] {1} ([xshift=-0.8cm]dv2.south)
    (dv4) edge[out = 90, in = -80] node[decomp arrow label] {2} ([xshift=1.2cm]dv2.south)
;
	\end{tikzpicture}\medskip
}
\end{center}\end{samepage}
Given $\setOfTokenRels=\{\tRelSpecn_0,\tRelSpecn_1\}$, we assign constraints to pairs of vertices drawn from $\pdecompositionn$ and $\pdecompositionn'$: $\delta$ and $\delta'$ \textit{can be} assigned $\tRelSpecn=(\ppaira, \ppairb, \setofpatternstc)$,  provided the patterns, $\ppairad$ and $\ppairbd$, that label $\delta$ and $\delta'$, respectively, match $\ppaira$ and $\ppairb$, that is $\ppairad\matches \ppaira$ and $\ppairbd\matches \ppairb$. This linking process allows us to choose a \textit{partial function}, $\vertexRelations\colon V(\pdecompositionn)\times V(\pdecompositionn')\to \setOfTokenRels$ where for all $\delta_i$ and $\delta_j'$ for which $\vertexRelations$ is defined:
\begin{displaymath}
\vertexRelations(\delta_i,\delta_j') \in \{(\ppaira,\ppairb,\setofpatternstc)\in \setOfTokenRels \colon  \labp(\delta_i)\matches \ppaira \wedge \labp'(\delta_j')\matches \ppairb\}.
\end{displaymath}
We \textit{choose} $\vertexRelations$ as follows:\label{ex:aSuitableL}
\begin{displaymath}
 \vertexRelations(\delta_1,\delta_1')  =  \tRelSpecn_0 \qquad
\vertexRelations(\delta_2,\delta_2')  =  \tRelSpecn_1 \qquad
\vertexRelations(\delta_3,\delta_3')  =  \tRelSpecn_1 \qquad
\vertexRelations(\delta_5,\delta_4') = \tRelSpecn_0
\end{displaymath}
with $\vertexRelations$ undefined for all other vertex pairs.{\pagebreak} The chosen mapping restricts the possible structural transformations, leading to an selection of the resulting construction. Of particular note is that, for the goal of this example we chose $\vertexRelations(\delta_1,\delta_1') = \tRelSpecn_0$ because $\tRelSpecn_0$ ensures that the constructed tokens are in the desired relationship; in a more complex example, it could have been that other choices -- beyond just being undefined -- were possible for $\vertexRelations(\delta_1,\delta_1')$. We see \begin{samepage}that:
\begin{enumerate}
\item[-] given the pair of constructions, $(\scalebox{0.74}{\tikz[baseline=-4.12cm]{\foavtwo}},1+2+3)$ and $\big(\scalebox{0.74}{\tikz[baseline=-2.03cm]{\dotsvone}}\hspace{0.02cm},\hspace{-0.1cm}\scalebox{0.7}{\tikz[baseline=0.07cm]{\oT}}\hspace{0.12cm}\big)$, that label $d_1$ and $d_1'$, there is a suitable construction in $\ispacen''$ such that the resulting triple satisfies $\tRelSpecn_0$; we defined $\vertexRelations(\delta_1,\delta_1')  =  \tRelSpecn_0$.
\item[-] the pair of constructions that label $d_2$ and $d_2'$ similarly `satisfy' $\tRelSpecn_1$; we defined $\vertexRelations(\delta_2,\delta_2')  =  \tRelSpecn_1$.
\item[-] the pair of constructions that label $d_3$ and $d_3'$ `satisfy' $\tRelSpecn_1$; we defined $\vertexRelations(\delta_3,\delta_3')  = \tRelSpecn_1$.
\end{enumerate}\end{samepage}
That is, the constraints imposed by the choice of $\vertexRelations$ are met and we have it that $\constructionofD{\decompositionn'}=\big(\cgraphhn, \hspace{-0.1cm}\scalebox{0.74}{\tikz[baseline=0.07cm]{\oT}}\hspace{0.2cm}\big)$ is a structural transformation of $\constructionofD{\decompositionn}=(\cgraphn,1+2+3)$.

We now give more examples of constraints, enlarging the set $\setOfTokenRels$, in order to produce the second structural transformation. In particular, we seek to transform the construction $\big(\cgraphn',\frac{3(3+1)}{2}\big)$, on the left, into the construction $\big(\cgraphhn',\hspace{-0.1cm}\scalebox{0.74}{\tikz[baseline=0.07cm]{\oT}}\hspace{0.2cm}\big)$, on the right:
\begin{center}
\adjustbox{scale=1}{%
\begin{tikzpicture}[construction,yscale=0.9]\small
\foavfour
\foaveight
\foavnine
\foavten
\foavfourteen
\foavfifteen
\foavsixteen
\foavseventeen
\foaveighteen
\foavnineteen
\foavtwenty
\foavtwentyone
\foaufour
\foaueight
\foaufifteen
\foauseventeen
\end{tikzpicture}
}\hspace{1.5cm}
\adjustbox{scale=1}{%
\begin{tikzpicture}[construction,yscale=0.9]\small
    \dotsvtwo

    \dotsvfive
    \dotsutwo
    \dotsvtwonew
    \dotsveight
    \dotsufive
    \dotsvnine
    \dotsvten
    \dotsueight

    \dotsveleven
    \dotsvtwelve
    \dotsuten
\end{tikzpicture}
}
\end{center}
{\pagebreak}

To do so, we exploit decompositions that have the following descriptions at the pattern \begin{samepage}level:
\begin{center}
\adjustbox{scale=0.9}{%
\begin{tikzpicture}[decomposition,yscale=0.92]
        \node[draw=darkgreen, rounded corners, label=180:{\Large$\delta_6$}] (dv6) at (0,0) {\begin{tikzpicture}\foavfourp\end{tikzpicture}};
        \node[draw=darkblue, rounded corners, label=180:{\Large$\delta_7$}] (dv7) at (0,-2.65) {\begin{tikzpicture}\foavfourp \foaveightp \foavninep \foavtenp \foaufour \end{tikzpicture}};
        \node[draw=darkblue, rounded corners, label=180:{\Large$\delta_8$}] (dv8) at (-1.9,-6.6) {\begin{tikzpicture}\foaveightp \foavfourteenp \foavfifteenp \foaueight\end{tikzpicture}};
        \node[draw=darkred, rounded corners, label=270:{\Large$\delta_9$}] (dv9) at (1,-6.5) {\begin{tikzpicture}\foavninep\end{tikzpicture}};
        \node[draw=darkred, rounded corners, label=270:{\Large$\delta_{10}$}] (dv10) at (2.5,-6.5) {\begin{tikzpicture}\foavtenp\end{tikzpicture}};
        \node[draw=darkred, rounded corners, label=270:{\Large$\delta_{11}$}] (dv11) at (-4,-10) {\begin{tikzpicture}\foavelevenp\end{tikzpicture}};
        \node[draw=darkblue, rounded corners, label=0:{\Large$\delta_{12}$}] (dv12) at (-0.5,-10.7) {\begin{tikzpicture}\foavfifteenp \foavsixteenp \foavseventeenp \foaveighteenp \foaufifteen\end{tikzpicture}};

         \node[draw=darkred, rounded corners, label=270:{\Large$\delta_{13}$}] (dv13) at (-4,-14) {\begin{tikzpicture}\foavsixteenp\end{tikzpicture}};
        \node[draw=darkblue, rounded corners, label=-20:{\Large$\delta_{14}$}] (dv14) at (-0.75,-14.75) {\begin{tikzpicture}\foavseventeenp \foavnineteenp \foavtwentyp \foavtwentyonep \foauseventeen\end{tikzpicture}};
        \node[draw=darkred, rounded corners, label=270:{\Large$\delta_{15}$}] (dv15) at (2.5,-14) {\begin{tikzpicture}\foaveighteenp\end{tikzpicture}};
	\path[->]
    (dv7) edge[out = 90, in = -90] node[decomp arrow label] {1} ([xshift=-0.4cm]dv6.south)
    (dv8) edge[out = 90, in = -110] node[decomp arrow label] {1} ([xshift=-1.5cm]dv7.south)
    (dv9) edge[out = 90, in = -90] node[decomp arrow label] {2} ([xshift=1.1cm]dv7.south)
    (dv10) edge[out = 90, in = -90] node[decomp arrow label] {3} ([xshift=2.1cm]dv7.south)
    (dv11) edge[out = 90, in = -100] node[decomp arrow label] {1} ([xshift=-1.3cm]dv8.south)
    (dv12) edge[out = 90, in = -80] node[decomp arrow label] {2} ([xshift=0.9cm]dv8.south)
    ([xshift=0.2cm]dv13.north) edge[out = 90, in = -110] node[decomp arrow label] {1} ([xshift=-1.8cm]dv12.south)
    (dv14) edge[out = 90, in = -90] node[decomp arrow label] {2} ([xshift=0.0cm]dv12.south)
    ([xshift=-0.2cm]dv15.north) edge[out = 90, in = -70] node[decomp arrow label] {3} ([xshift=1.8cm]dv12.south)
;
	\end{tikzpicture}
}\hfill
\adjustbox{scale=\myscale}{%
\begin{tikzpicture}[decomposition,yscale=0.92]
        \node[draw=darkgreen, rounded corners, label=180:{\Large$\delta_5'$}] (dv1) at (0,0) {\begin{tikzpicture}\dotsvtwop\end{tikzpicture}};
        \node[draw=darkblue, rounded corners, label=180:{\Large$\delta_6'$}] (dv2) at (0,-4) {\begin{tikzpicture}[yscale=0.85]\dotsvtwop  \dotsvfivep
    \dotsutwo
    \dotsvtwonewp
    \dotsveightp
    \dotsufivep
         \end{tikzpicture}};
        \node[draw=darkblue, rounded corners, label=180:{\Large$\delta_7'$}] (dv3) at (-1.5,-9.75) {\begin{tikzpicture}\dotsveightp %
    \dotsvninep
    \dotsvtenp
    \dotsueightp
        \end{tikzpicture}};
        \node[draw=darkred, rounded corners, label=270:{\Large$\delta_8'$}] (dv4) at (2,-9) {\begin{tikzpicture}\dotsvtwonewp\end{tikzpicture}};
        \node[draw=darkred, rounded corners, label=270:{\Large$\delta_9'$}] (dv5) at (-3,-13) {\begin{tikzpicture}\dotsvninep\end{tikzpicture}};
        \node[draw=darkblue, rounded corners, label=-160:{\Large$\delta_{10}'$}] (dv6) at (0.5,-14) {\begin{tikzpicture}\dotsvtenp \dotsvelevenp \dotsvtwelvep \dotsutenp\end{tikzpicture}};
	\path[->]
    (dv2) edge[out = 90, in = -90] node[decomp arrow label] {1} ([xshift=-0.3cm]dv1.south)
    ([xshift=-0.6cm]dv3.north) edge[out = 90, in = -90] node[decomp arrow label] {1} ([xshift=-1.4cm]dv2.south)
    ([xshift=-0.2cm]dv4.north) edge[out = 90, in = -90] node[decomp arrow label] {2} ([xshift=1.7cm]dv2.south)
    (dv5) edge[out = 90, in = -90] node[decomp arrow label] {1} ([xshift=-1.4cm]dv3.south)
    (dv6) edge[out = 90, in = -70] node[decomp arrow label] {2} ([xshift=1.2cm]dv3.south)
;
	\end{tikzpicture}
}
\end{center}
\end{samepage}

The next transformation constraint that we define, noting that $\tRelSpecn_0$ provides a constraint for $(\delta_6,\delta_5')$, focuses on $(\delta_7,\delta_6')$. Here, division by two, under some circumstances, aligns with removing some dots from a rectangle. These circumstances are embodied in the definition of $\tRelSpecn_2=(\ppairap{10},\ppairbp{7}, \setofpatternstc_2)$, where $\setofpatternstc_2=\{(\pgraphmn_2,v_2),(\pgraphmn_2',v_2'),(\pgraphmn_2'',v_2'')\}$:
\begin{center}
	\adjustbox{scale=\myscale,raise=0.15cm}{
		\begin{tikzpicture}[construction,yscale=0.95,xscale=0.9]
		\foavfourp
		\foaveightp
		\foavninep
		\foavtenp
		\foaufour
		\end{tikzpicture}
	}\hspace{-0.0cm}
	\adjustbox{scale=\myscale}{
		\begin{tikzpicture}[construction,yscale=0.8,xscale=0.9]
		\dotsvtwop  \dotsvfivep
		\dotsutwo
		\dotsvtwonewp
		\dotsveightp
		\dotsufivep
		\end{tikzpicture}
	}\hspace{-0.2cm}
	\adjustbox{scale=\myscale,raise=0.2cm}{
		\begin{tikzpicture}[construction,yscale=0.8,xscale=0.9]
		%\node[typeIE = {\typetrue}] at (-1.6,-0.9) (v) {$v_2$};
		%\node[constructorINE = {\constructorAnd}] at (-1.6,-1.9) (u) {};
		%\node[typeIW = {\typebool}] at (-3,-2.4) (v1) {};
		\node[typeIE = {\typetrue}] at (1.2,-3.9) (v2) {$v_2''$};
		%
		%\node[constructorINW = {\constructorAnd}] at (-3,-3.3) (u1) {};
		\node[typeIW = {\typetrue}] at (-3.6,-3.9) (v3) {$v_2$};
		\node[typeIE = {\typetrue}] at (-1.5,-3.7) (v4) {$v_2'$};
		\node[constructorINW = {\constructorRepresentedBy}] at (-3.6,-4.8) (u2) {};
		\node[typeS = {\FOAnumexp}] at (-4.9,-6) (v5) {$\acolour{a_{11}}$};
		\node[typeS = {\DDrectangulation}] at (-2.9,-6) (v6) {$\bcolour{b_9}$};
		\node[constructorINE = {\DDcisJoinOf}] at (-1.5,-4.7) (u3) {};
		\node[typeS = {\DDstack}] at (-1.2,-6) (v6') {$\bcolour{b_{8}}$};
		\node[typeS = {\DDstack}] at (0.4,-6) (v7) {$\bcolour{b_{10}}$};
		\node[constructorIpos = {\DDcsameNumOfDots}{13}{1.1cm}] at (1.2,-4.9) (u4) {};
		\node[typeS = {\DDstack}] at (2.5,-6) (v8) {$\bcolour{b_7}$};
		\path[->]
		%	(v1) edge[bend right = -10] node[ index label] {1} (u)
		%   (v2) edge[bend right = 10] node[ index label] {2} (u)
		%	(u) edge (v)
		%
		%	(v3) edge[bend right = -10] node[ index label] {1} (u1)
		%    (v4) edge[bend right = 10] node[ index label] {2} (u1)
		%	(u1) edge (v1)
		%
		(v5) edge[bend right = -10] node[ index label] {1} (u2)
		(v6) edge[bend right = 10] node[ index label] {2} (u2)
		(u2) edge (v3)
		(v6) edge[bend right = -10] node[ index label] {1} (u3)
		(v6') edge[bend right = 10] node[ index label] {2} (u3)
		(v7) edge[bend right = 10] node[ index label] {3} (u3)
		(u3) edge (v4)
		(v7) edge[bend right = -10] node[ index label] {1} (u4)
		(v8) edge[bend right = 10] node[ index label] {2} (u4)
		(u4) edge (v2)
		;
		\coordinate[below left = of v5,yshift =0.35cm,xshift=0.7cm] (a1);
		\coordinate[above left = of v5,yshift =0.4cm,xshift=0.3cm] (a2);
		\coordinate[above left = of v3,yshift =-0.55cm,xshift=0.3cm] (a3);
		\coordinate[above right = of v3,yshift =-0.55cm,xshift=-1cm] (a4);
		\coordinate[above right = of v6,xshift=-0.75cm,yshift=-0.5cm] (a5);
		\coordinate[below right = of v6,yshift =0.4cm,xshift=-0.45cm] (a6);
		\draw[rounded corners=7,thick, draw opacity = 0.30, text opacity = 0.6] (a1)--(a2)--node[yshift=0.45cm,xshift=-0.15cm]{$(\gamma_2,v_2)$}(a3)--(a4)--(a5)--(a6)--cycle;
		\coordinate[below left = of v6,yshift =0.3cm,xshift=0.5cm] (b1);
		\coordinate[above left = of v6,yshift =-0.6cm,xshift=0.7cm] (b2);
		\coordinate[above left = of v4,yshift =-0.6cm,xshift=0.9cm] (b3);
		\coordinate[above right = of v4,yshift =-0.6cm,xshift=-0.5cm] (b4);
		\coordinate[above right = of v7,xshift=-1.0cm] (b5);
		\coordinate[below right = of v7,yshift =0.3cm,xshift=-0.7cm] (b6);
		\draw[rounded corners=7,thick, draw opacity = 0.30, text opacity = 0.6] (b1)--(b2)--(b3)--node[yshift=0.25cm,xshift=-0.02cm]{$(\gamma_2',v_2')$}(b4)--(b5)--(b6)--cycle;
		\coordinate[below left = of v7,yshift =0.4cm,xshift=0.8cm] (g1);
		\coordinate[above left = of v7,yshift =-0.5cm,xshift=0.8cm] (g2);
		\coordinate[above left = of v2,yshift =-0.55cm,xshift=0.9cm] (g3);
		\coordinate[above right = of v2,yshift =-0.55cm,xshift=-0.1cm] (g4);
		\coordinate[above right = of v8,xshift=0.15cm,yshift=0.4cm] (g5);
		\coordinate[below right = of v8,yshift =0.4cm,xshift=-0.4cm] (g6);
		\draw[rounded corners=7,thick, draw opacity = 0.30, text opacity = 0.6] (g1)--(g2)--(g3)--(g4)--node[yshift=0.4cm,xshift=0.25cm]{$(\gamma_2'',v_2'')$}(g5)--(g6)--cycle;
		\end{tikzpicture}
	}
\end{center}
Noting that implicit multiplication aligns with rectangulation,{\pagebreak} and that addition aligns with appending dot diagrams, we further define $\tRelSpecn_3=(\ppairap{9},\ppairbp{9},\setofpatternstc_3)$ and $\tRelSpecn_4 = (\ppairap{17}, \linebreak\ \ppairbp{12},\setofpatternstc_4)$ using the following two triples of graphs where, in each case, the property enforcer comprises two \begin{samepage}patterns:
\begin{center}
\adjustbox{scale=0.8}{
\begin{tikzpicture}[construction,yscale=0.9,xscale=0.9]
\foaveightp
\foavfourteenp
\foavfifteenp
\foaueight
\end{tikzpicture}
}\hspace{0.7cm}
\adjustbox{scale=0.8}{
\begin{tikzpicture}[construction,yscale=0.9,xscale=0.9]
\dotsveightp
\dotsvninep
\dotsvtenp
\dotsueightp
\end{tikzpicture}
}\hspace{0.3cm}
\adjustbox{scale=0.8}{
\begin{tikzpicture}[construction,yscale=0.9,xscale=0.75]
%\node[typeIE = {\typetrue}] at (0,-0.3) (v) {$v_3$};
%
%\node[constructorINE = {\constructorAnd}] at (0,-1.3) (u) {};
%
\node[typeIW = {\typetrue}] at (-2,-2) (v1) {$v_3$};
\node[typeIE = {\typetrue}] at (2,-2) (v2) {$v_3'$};
\node[constructorINW = {\constructorRepresentedBy}] at (-2,-3) (u1) {};
\node[typeW = {\FOAnumexp}] at (-3,-4) (v3) {$\acolour{a_{14}}$};
\node[typeS = {\DDdotDiagram}] at (-1,-4) (v4) {$\bcolour{b_{11}}$};
\node[constructorINE = {\constructorRepresentedBy}] at (2,-3) (u2) {};
\node[typeS = {\FOAnumexp}] at (1,-4) (v5) {$\acolour{a_{15}}$};
\node[typeE = {\DDdotDiagram}] at (3,-4) (v6) {$\bcolour{b_{12}}$};
    \path[->]
	%(v1) edge[bend right = -10] node[ index label] {1} (u)
	%(v2) edge[bend right = 10] node[ index label] {2} (u)
	%(u) edge (v)
    %
    (v3) edge[bend right = -10] node[ index label] {1} (u1)
	(v4) edge[bend right = 10] node[ index label] {2} (u1)
	(u1) edge (v1)
    (v5) edge[bend right = -10] node[ index label] {1} (u2)
	(v6) edge[bend right = 10] node[ index label] {2} (u2)
	(u2) edge (v2);
\end{tikzpicture}
}
\\[0.2cm]
\adjustbox{scale=0.8}{
\begin{tikzpicture}[construction,yscale=0.9,xscale=0.9]
\foavnineteenp
\foavtwentyp
\foavtwentyonep
\foavseventeenp
\foauseventeen
\end{tikzpicture}
}\hspace{0.8cm}
\adjustbox{scale=0.8}{
\begin{tikzpicture}[construction,yscale=0.9,xscale=0.9]
\dotsvelevenp
\dotsvtwelvep
\dotsvtenp
\dotsutenp
\end{tikzpicture}
}\hspace{1cm}
\adjustbox{scale=0.8}{
\begin{tikzpicture}[construction,yscale=0.8,xscale=0.8]
%\node[typeIE = {\typetrue}] at (0,-0.3) (v) {$v_4$};
%
%\node[constructorINE = {\constructorAnd}] at (0,-1.3) (u) {};
%
\node[typeIW = {\typetrue}] at (-2,-2) (v1) {$v_4$};
\node[typeIE = {\typetrue}] at (2,-2) (v2) {$v_4'$};
\node[constructorINW = {\constructorRepresentedBy}] at (-2,-3) (u1) {};
\node[typeW = {\FOAnumexp}] at (-3,-4) (v3) {$\acolour{a_{19}}$};
\node[typeS = {\DDdotDiagram}] at (-1,-4) (v4) {$\bcolour{b_{13}}$};
\node[constructorINE = {\constructorRepresentedBy}] at (2,-3) (u2) {};
\node[typeS = {\FOAnumexp}] at (1,-4) (v5) {$\acolour{a_{21}}$};
\node[typeE = {\DDdotDiagram}] at (3,-4) (v6) {$\bcolour{b_{14}}$};
    \path[->]
	%(v1) edge[bend right = -10] node[ index label] {1} (u)
	%(v2) edge[bend right = 10] node[ index label] {2} (u)
	%(u) edge (v)
    %
    (v3) edge[bend right = -10] node[ index label] {1} (u1)
	(v4) edge[bend right = 10] node[ index label] {2} (u1)
	(u1) edge (v1)
    (v5) edge[bend right = -10] node[ index label] {1} (u2)
	(v6) edge[bend right = 10] node[ index label] {2} (u2)
	(u2) edge (v2);
\end{tikzpicture}
}
\end{center}\end{samepage}
Given $\tRelSpecn_0$, $\tRelSpecn_1$, $\tRelSpecn_2$, $\tRelSpecn_3$, and $\tRelSpecn_4$, we define $\vertexRelations$ as follows:
\begin{displaymath}
\vertexRelations(\delta_6,\delta_5') =  \tRelSpecn_0 \hspace{0.6cm}
\vertexRelations(\delta_7,\delta_6')  =  \tRelSpecn_2 \hspace{0.6cm}
\vertexRelations(\delta_8,\delta_7')  =  \tRelSpecn_3 \hspace{0.6cm}
\vertexRelations(\delta_{11},\delta_9')  =  \tRelSpecn_0 \hspace{0.6cm}
\vertexRelations(\delta_{14},\delta_{10}')  =  \tRelSpecn_4.
\end{displaymath}
As with the previous structural transformation, the decompositions of $\big(\cgraphn',\frac{3(3+1)}{2}\big)$ and $\big(\cgraphhn',\hspace{-0.1cm}\scalebox{0.74}{\tikz[baseline=0.07cm]{\oT}}\hspace{0.2cm}\big)$ are such that the constraints imposed by the choice of $\vertexRelations$ are met and the latter is, therefore, a structural transformation of the former.
\end{example}

\begin{definition}\label{defn:constraintAssignment}\label{defn:transformationSpecification}
Let $\irencodingn=\irencoding$ be an inter-representational-system encoding with description $\descriptionirse$. Let  $\pdecompositionn$ and $\pdecompositionn'$ be pattern decompositions in $\descriptionn$ and $\descriptionn'$, respectively.
A \textit{constraint assignment} for $\pdecompositionn$ and $\pdecompositionn'$ is a partial function, $\vertexRelations\colon V(\pdecompositionn)\times V(\pdecompositionn') \to \setOfTokenRels$, where
\begin{enumerate}
\item $\setOfTokenRels$ is a set of transformation constraints for $\irencodingn$, and
\item for all $(\delta,\delta')\in V(\pdecompositionn)\times V(\pdecompositionn')$
\begin{displaymath}
  \vertexRelations(\delta,\delta')\in\{(\ppaira,\ppairb,\setofpatternstc)\in \setOfTokenRels \colon \labp(\delta)\matches \ppaira \wedge \labp'(\delta')\matches \ppairb \}
\end{displaymath}
where $\labp(\delta)$ and $\labp'(\delta')$ are the patterns that  label $\delta$ and $\delta'$.
\end{enumerate}
\end{definition}

\begin{definition}\label{defn:structuralTransformationDecomposition}
Let $\irencodingn=\irencoding$ be an inter-representational-system encoding with description $\descriptionirse$. Let $\pdecompositionn$ and $\pdecompositionn'$ be pattern decompositions in $\descriptionn$ and $\descriptionn'$, respectively, with constraint assignment $\vertexRelations\colon V(\pdecompositionn)\times V(\pdecompositionn') \to \setOfTokenRels$. Let $\decompositionn$ and $\decompositionn'$ be decompositions in $\rsystemn$ and $\rsystemn'$ that match $\pdecompositionn$ and $\pdecompositionn'$, with embeddings $f\colon \decompositionn\to \pdecompositionn$ and $f'\colon \decompositionn'\to \pdecompositionn'$. Then $\vertexRelations$ is \textit{satisfied} by the pair $(\decompositionn,\decompositionn')$ provided for all $(d,d')\in V(\decompositionn)\times V(\decompositionn')$,
if $\vertexRelations(f(d),f'(d'))$ is defined then there exists a property identifier, $\setofconstructionsirse $, for $\irencodingn$ such that
\begin{displaymath}
(\labc(d),\labc'(d'),\setofconstructionsirse )\models \vertexRelations(f(d),f'(d'))
\end{displaymath}
where $\labc(d)$ and $\labc'(d')$ are the constructions that label $d$ and $d'$.
\end{definition}

Finally, we define structural transformations at the level of constructions. Here, we make use of canonical decompositions: recall, given a pattern decomposition, $\pdecompositionn$, the $\pdecompositionn$-canonical decomposition of a matching construction, $\cpair$, is denoted $\deltacanonical{\cpair,\pdecompositionn}$; see definition~\ref{defn:deltaCanonical}.

\begin{definition}\label{defn:structuralTransformationConstruction}
Let $\irencodingn=\irencoding$ be an inter-representational-system encoding with description $\descriptionirse$. Let $\pdecompositionn$ and $\pdecompositionn'$ be pattern decompositions in $\descriptionn$ and $\descriptionn'$, respectively, with
constraint assignment $\vertexRelations\colon V(\pdecompositionn)\times V(\pdecompositionn') \to \setOfTokenRels$. Let $\cpair$ and $\cpaird$ be constructions in $\rsystemn$ and $\rsystemn'$ that match $\pdecompositionn$ and $\pdecompositionn'$. Then $\cpaird$ is a \textit{structural $\vertexRelations$-transformation} of $\cpair$ provided $\vertexRelations$ is satisfied by $(\deltacanonical{\cpair,\pdecompositionn},\deltacanonical{\cpaird,\pdecompositionn'})$.
\end{definition}

%\gnote{Some of the key messages in this paragraph may get moved to the conclusion later on, when it gets written!}

Thus, we are able to identify relationships between constructions, via their decompositions, that allow representation selection through imposed transformation constraints. These constraints can be varied in nature and can include (identifying just a few cases): enforcing semantic equivalence, specifying improved cognitive effectiveness, and even -- in an information visualisation context -- encode a change of representation when underlying data changes dynamically. The general nature of the theory we have devised means that universally-applicable algorithms can be developed to invoke structural transformations between a wide range of representational systems. Whilst such algorithms would require system-specific inputs (such as constructions, patterns and constraint assignments), our research eliminates a key barrier seen in prior work: there is now no need to develop \textit{inter-system-specific} transformation algorithms, which necessarily exist when different systems adopt different formalisation approaches. As such, we claim that Representational Systems Theory provides the first \textit{general framework} that enables \textit{representation selection} via \textit{structural transformations}.

\subsection{Partial Transformations}\label{sec:st:partial}

So far, we have devised theory that supports \textit{complete} structural transformations: given $\cpair$, identify $\cpaird$ given a constraint assignment. However, in practice, we are likely to not have full knowledge of the construction spaces associated with $\irencodingn=\irencoding$. Thus, it is important for us to explore \textit{partial transformations} that could be the output of a structure transfer algorithm when a complete transfer cannot be produced. The result of a partial transformation will be a pattern, $\ppaird$, perhaps with some vertices instantiated tokens, that can be provided to an end-user as an indication of how the constructors of $\rsystemn'$ can be composed to produce a suitable $\cpaird$, should it exist\footnote{In future work, we intended to develop a method for scoring each partial transformation, $\ppaird$, to indicate how likely it is that a construction, $\cpaird$, exists in $\rsystemn'$ that matches $\ppaird$. In terms of practical utility, these scores will allow us to rank partial transformations, potentially from a wide variety of representational systems. These scores, in turn, will give an indication as to whether it is possible to transform $\cpair$ into some $\cpaird$ subject to the required transformation constraints holding.}.

Consider, then, a pattern decomposition, $\pdecompositionn'$, that describes the class of constructions we wish to produce from $\cpair$. The goal of partial structure transfer is to replace \textit{some} of the patterns that label the vertices in $\pdecompositionn'$ with constructions or `partly completed' constructions, by instantiating pattern-vertices using tokens in $\rsystemn'$.

\begin{example}[Instantiating Vertices in a Pattern]\label{ex:instantiatingPatterns} Suppose one of the vertices, $\delta'$ say, in $\pdecompositionn'$ is labelled by pattern (a), below left. Then we could replace (a)  with the alternative pattern, (b), in the middle. Here, some of the original pattern's vertices have been replaced by tokens. However, the types assigned to the vertices were unchanged. In order to get `closer' to a construction in $\rsystemn_{\FOA}$, we must also assign more specific types to the tokens, as shown in pattern (c), below right.

\begin{center}
\adjustbox{scale=0.8}{
\begin{tikzpicture}[construction,yscale=0.9,xscale=0.65]\small
\node[] at (-3,-0.2) {\normalsize(a)};
\foavp
\foavonep
\foau
\foavtwop
\foavthreep
\foavfourp
\foauone
\end{tikzpicture}
}
\hspace{0.2cm}
\adjustbox{scale=0.8}{
\begin{tikzpicture}[construction,yscale=0.9,xscale=0.65]\small
\node[] at (-3,-0.2) {\normalsize(b)};
\node[termIW = {\typebool}] (v) at (0,-0.2) {$\top$};
\foavonep
\foau
\node[termW = {\FOAnumexp}] at (-2.3,-4) (v2) {$1+2+3$};
\node[termS = {\FOAeq}] at (-0.5,-4.2) (v3)  {\vphantom{\scalebox{0.83}{i}}$=$};
\foavfourp
\foauone
\end{tikzpicture}
}
\hspace{0.2cm}
\adjustbox{scale=0.8}{
\begin{tikzpicture}[construction,yscale=0.9,xscale=0.65]\small
\node[] at (-3,-0.2) {\normalsize(c)};
\node[termIW = {\typetrue}] (v) at (0,-0.2) {$\top$};
\foavonep
\foau
\node[termW = {\texttt{1+2+3}}] at (-2.3,-4) (v2) {$1+2+3$};
\node[termS = {\FOAeq}] at (-0.5,-4.2) (v3)  {\vphantom{\scalebox{0.83}{i}}$=$};
\foavfourp
\foauone
\end{tikzpicture}
}
\end{center}
\end{example}

In the context of partial transformations, the act of replacing vertices with tokens, or types with subtypes, in a pattern $\ppaird$, creates a new pattern, $(\hat{\pgraphn}',\hat{v}')$, that matches $\ppaird$. That is, there is an embedding from $(\hat{\pgraphn}',\hat{v}')$ to $\ppaird$. Focusing on what happens at the $\pdecompositionn'$-level, replacing vertices or types in patterns that label vertices in $\pdecompositionn'$ creates a new pattern decomposition, $\hat{\pdecompositionn}'$. The replacement must be done in a consistent way: if two patterns, $\ppaird$ and $(\pgraphn'',v'')$, in $\pdecompositionn'$ share a vertex, $v$ say, then
\begin{enumerate}
\item[-] if $v$ is replaced by $t$ in $\ppaird$ then $v$ must also be replaced by $t$ in $(\pgraphn'',v'')$, and
\item[-] if the type, $\tau$, assigned to $v$ is replaced by a subtype, $\tau'$, in $\ppaird$ then the type assigned to $v$ in $(\pgraphn'',v'')$ must also change to $\tau'$.
\end{enumerate}
These observations are embodied by the requirement that the new pattern decomposition, $\hat{\pdecompositionn}'$, can be embedded into $\pdecompositionn'$. Definition~\ref{defn:partialLDeltaTransform} exploits this relationship between $\pdecompositionn'$ and $\hat{\pdecompositionn}'$, on route to generalising the concept of satisfaction given in definition~\ref{defn:structuralTransformationDecomposition} -- for decompositions of constructions -- to encompass decompositions of patterns.

\begin{definition}\label{defn:DeltaPartialSatisfication}
Let $\irencodingn=\irencoding$ be an inter-representational-system encoding with description $\descriptionirse$. Let $\pdecompositionn$ and $\pdecompositionn'$ be pattern decompositions in $\descriptionn$ and $\descriptionn'$, respectively, with
constraint assignment $\vertexRelations\colon V(\pdecompositionn)\times V(\pdecompositionn') \to \setOfTokenRels$.  Let $\decompositionn$ be a decomposition in $\rsystemn$ that matches $\pdecompositionn$, with embedding $f\colon \decompositionn\to \pdecompositionn$.  Let $\hat{\pdecompositionn}'$ be a decomposition of some pattern for $\rsystemn'$ that matches $\pdecompositionn'$, with embedding $f'\colon \hat{\pdecompositionn}' \to \pdecompositionn'$. Then $\vertexRelations$ is \textit{satisfied} by the pair $(\decompositionn,\hat{\pdecompositionn}')$ provided for all $(d,\hat{\delta}')\in V(\decompositionn)\times V(\hat{\pdecompositionn}')$ if $\vertexRelations(f(d),f'(\hat{\delta}'))$ is defined and $\widehat{\labp}(\hat{\delta}')$ is a construction in $\rsystemn'$ then there exists a property identifier, $\setofconstructionsirse$, in $\irencoding$ such that
\begin{displaymath}
  (\labc(d),\widehat{\labp}(\hat{\delta}'),\setofconstructionsirse) \models \vertexRelations(f(d),f'(\hat{\delta}')).
\end{displaymath}
\end{definition}

\begin{definition}\label{defn:partialLDeltaTransform}
Let $\irencodingn=\irencoding$ be an inter-representational-system encoding with description $\descriptionirse$. Let $\pdecompositionn$ and $\pdecompositionn'$ be pattern decompositions in $\descriptionn$ and $\descriptionn'$, respectively, with
constraint assignment $\vertexRelations\colon V(\pdecompositionn)\times V(\pdecompositionn') \to \setOfTokenRels$.  Let $\cpair$ be a construction in $\rsystemn$ that matches $\pdecompositionn$. Let $(\hat{\pgraphn}',\hat{v}')$ be a pattern for $\rsystemn'$ that matches $\pdecompositionn'$. Then $(\hat{\pgraphn}',\hat{v}')$ is a \textit{partial $\vertexRelations$-transformation} of $\cpair$ provided $\vertexRelations$ is satisfied by $(\deltacanonical{\cpair,\pdecompositionn},\deltacanonical{(\hat{\pgraphn}',\hat{v}'),\pdecompositionn'})$.
\end{definition}

\begin{theorem}\label{thm:completePartialIsTransformation}
Let $\irencodingn=\irencoding$ be an inter-representational-system encoding with description $\descriptionirse$. Let $\pdecompositionn$ and $\pdecompositionn'$ be pattern decompositions in $\descriptionn$ and $\descriptionn'$, respectively, with
constraint assignment $\vertexRelations\colon V(\pdecompositionn)\times V(\pdecompositionn') \to \setOfTokenRels$.  Let $\cpair$ be a construction in $\rsystemn$ that matches $\pdecompositionn$. Let $(\hat{\pgraphn},\hat{v})$ be a pattern for $\rsystemn'$ that matches $\pdecompositionn'$. If $(\hat{\pgraphn}',\hat{v}')$ is a construction in $\rsystemn'$ and a partial $\vertexRelations$-transformation of $\cpair$ then $(\hat{\pgraphn}',\hat{v}')$ is a structural $\vertexRelations$-transformation of $\cpair$.
\end{theorem}

The proof of theorem~\ref{thm:completePartialIsTransformation} can be found in Appendix~\ref{sec:app:structure}, see theorem~\ref{athm:completePartialIsTransformation}. An important consequence of supporting partial transformations  is related to the termination of algorithms. Obviously there does not exist a universally applicable algorithm that can transformation an arbitrary representation into another (or, in our terminology, one construction into another), and be guaranteed to terminate: such transformation problems are not decidable in general. However, Representational Systems Theory enables the implementation of algorithms that \textit{readily produce}  a partial transformation of $\cpair$:
\begin{enumerate}
\item[-] check whether $\cpair$ matches $\pdecompositionn$, for which a deterministic, terminating algorithm can be readily devised, so long as the type system satisfies the ascending chain condition, and
\item[-] in each pattern, $\ppairp{i}$, that labels some vertex, $\delta_i$, in $\pdecompositionn'$, replace each vertex, $v_i'$, in $\pgraphn_i$ that is a token in $\rsystemn'$ with a fresh, non-token, vertex; after all such replacements, the result is $\hat{\pdecompositionn}'$.
\end{enumerate}
We are assured that producing such a $\constructionofD{\hat{\pdecompositionn}'}$ is a partial $\vertexRelations$-transformation of $\cpair$. Subsequently it is possible to attempt to instantiate some of the vertices of $\constructionofD{\pdecompositionn'}$, in a terminating way (e.g. by restricting to a search through known tokens, stored in a knowledge base, that are in $\rsystemn'$). In conclusion, Representational Systems Theory \textit{necessarily permits} the derivation of partial transformations -- given an appropriate type system -- even though there does not exist, in general, terminating algorithms that are guaranteed to produce complete transformations. This is particularly significant for practical applications, where it is not possible to derive complete structural transformations. Notably, where the set of tokens in a representational system is infinite, an AI system will not have explicit knowledge of the token-set. Therefore, producing a partial transformation may be a likely output in such situations. A human end-user may be readily able to instantiate the tokens in order to produce a complete transformation.

\subsection{Soundness}\label{sec:st:soundness}

Just as the theory of partial transformations is of practical use, so is the notion \textit{soundness}. This concept relates to transformation constraints and how if `part of a transformation constraint is satisfied' then the entire constraint is also satisfied. We begin by illustrating validity and soundness with an example.

\begin{example}[A Valid Constraint Assignment]
In example~\ref{ex:FOAandDotsStructuralLinks}, we saw how to structurally transform the construction $(\cgraphn,1+2+3)$ into the construction $\big(\cgraphhn, \hspace{-0.1cm}\scalebox{0.74}{\tikz[baseline=0.07cm]{\oT}}\hspace{0.2cm}\big)$. This example uses the same pattern decompositions, $\pdecompositionn$ and $\pdecompositionn'$, repeated here for ease of readability:
\begin{center}
\adjustbox{scale=\myscale}{%
\begin{tikzpicture}[decomposition,yscale = 0.9]
        \node[draw=darkgreen, rounded corners, label=180:{\Large$\delta_1$}] (dv1) at (0,-0.2) {\begin{tikzpicture}[construction,yscale=0.9]\foavtwop\end{tikzpicture}};
        \node[draw=darkblue, rounded corners, label=180:{\Large$\delta_2$}] (dv2) at (0,-3) {\begin{tikzpicture}[construction,yscale=0.9]\foavtwop \foavfivep \foavsixp \foavsevenp \foautwo \end{tikzpicture}};
        \node[draw=darkblue, rounded corners, label=180:{\Large$\delta_3$}] (dv3) at (-2,-7) {\begin{tikzpicture}[construction,yscale=0.9]\foavfivep \foavelevenp \foavtwelvep \foavthirteenp \foaufive\end{tikzpicture}};
        \node[draw=darkred, rounded corners, label=270:{\Large$\delta_4$}] (dv4) at (0.7,-7) {\begin{tikzpicture}[construction,yscale=0.9]\foavsixp\end{tikzpicture}};
        \node[draw=darkred, rounded corners, label=270:{\Large$\delta_5$}] (dv5) at (2.2,-7) {\begin{tikzpicture}[construction,yscale=0.9]\foavsevenp\end{tikzpicture}};
	\path[->]
    ([xshift=0.7cm]dv2.north) edge[out = 90, in = -90] node[decomp arrow label] {1} ([xshift=0.4cm]dv1.south)
    (dv3) edge[out = 90, in = -120] node[decomp arrow label] {1} ([xshift=-1.3cm]dv2.south)
    (dv4) edge[out = 90, in = -90] node[decomp arrow label] {2} ([xshift=0.4cm]dv2.south)
    (dv5) edge[out = 90, in = -80] node[decomp arrow label] {3} ([xshift=1.7cm]dv2.south)
;
	\end{tikzpicture}
}\hspace{1.5cm}
\adjustbox{scale=\myscale}{%
\begin{tikzpicture}[decomposition,yscale = 0.9]
        \node[draw=darkgreen, rounded corners, label=180:{\Large$\delta_1'$}] (dv1) at (0,-0.2) {\begin{tikzpicture}[construction,yscale=0.9]\dotsvonep\end{tikzpicture}};
        \node[draw=darkblue, rounded corners, label=180:{\Large$\delta_2'$}] (dv2) at (0,-3) {\begin{tikzpicture}[construction,yscale=0.9]\dotsvonep \dotsvthreep \dotsvfourp \dotsuone \end{tikzpicture}};
        \node[draw=darkblue, rounded corners, label=180:{\Large$\delta_3'$}] (dv3) at (-1.25,-7) {\begin{tikzpicture}[construction,yscale=0.9]\dotsvthreep \dotsvsixp \dotsvsevenp \dotsuthree\end{tikzpicture}};
        \node[draw=darkred, rounded corners, label=270:{\Large$\delta_4'$}] (dv4) at (1.25,-7) {\begin{tikzpicture}[construction,yscale=0.9]\dotsvfourp\end{tikzpicture}};
	\path[->]
    ([xshift=0.2cm]dv2.north) edge[out = 90, in = -90] node[decomp arrow label] {1} ([xshift=0.3cm]dv1.south)
    (dv3) edge[out = 90, in = -110] node[decomp arrow label] {1} ([xshift=-1cm]dv2.south)
    (dv4) edge[out = 90, in = -80] node[decomp arrow label] {2} ([xshift=1.2cm]dv2.south)
;
	\end{tikzpicture}
}
\medskip
\end{center}
{\pagebreak}
We continue to use $\vertexRelations$ as defined on page~\pageref{ex:aSuitableL}. Consider the following decomposition\footnote{Note here the use of $\delta$ and $\delta'$ in vertex names, since our formalisation in definition~\ref{defn:validForVertices} uses canonical decompositions, where the vertices in a decomposition of a construction are the same as those in a decomposition of a matched pattern.} of a construction of $2+2+4$ in $\rsystemn_{\FOA}$ and decomposition of a pattern for \begin{samepage}$\rsystemn_{\mathit{D}}$:
\begin{center}\label{ex:decompositionOfPatternsForLater}
\adjustbox{scale=\myscale}{%
\begin{tikzpicture}[decomposition,yscale=0.9]
        \node[draw=darkgreen, rounded corners, label=180:{\Large$\delta_1$}] (dv1) at (0,-0.1) {\begin{tikzpicture}[construction,yscale=0.9]\foavtwon\end{tikzpicture}};
        \node[draw=darkblue, rounded corners, label=180:{\Large$\delta_2$}] (dv2) at (0,-3) {\begin{tikzpicture}[construction,yscale=0.9]\foavtwon \foavfiven \foavsixn \foavsevenn \foautwo \end{tikzpicture}};
        \node[draw=darkblue, rounded corners, label=180:{\Large$\delta_3$}] (dv3) at (-1.75,-7) {\begin{tikzpicture}[construction,yscale=0.9]\foavfiven \foavelevenn \foavtwelven \foavthirteenn \foaufive\end{tikzpicture}};
        \node[draw=darkred, rounded corners, label=270:{\Large$\delta_4$}] (dv4) at (0.75,-7) {\begin{tikzpicture}[construction,yscale=0.9]\foavsixn\end{tikzpicture}};
        \node[draw=darkred, rounded corners, label=270:{\Large$\delta_5$}] (dv5) at (2,-7) {\begin{tikzpicture}[construction,yscale=0.9]\foavsevenn\end{tikzpicture}};
	\path[->]
	([xshift=0.7cm]dv2.north) edge[out = 90, in = -90] node[decomp arrow label] {1} ([xshift=0.4cm]dv1.south)
	(dv3) edge[out = 90, in = -120] node[decomp arrow label] {1} ([xshift=-1.3cm]dv2.south)
	(dv4) edge[out = 90, in = -90] node[decomp arrow label] {2} ([xshift=0.3cm]dv2.south)
	(dv5) edge[out = 90, in = -80] node[decomp arrow label] {3} ([xshift=1.5cm]dv2.south)
;
	\end{tikzpicture}
}\hspace{1.5cm}
\adjustbox{scale=\myscale}{%
\begin{tikzpicture}[decomposition,yscale=0.9]
        \node[draw=darkgreen, rounded corners, label=180:{\Large$\delta_1'$}] (dv1) at (0,0) {\begin{tikzpicture}[construction,yscale=0.9]\dotsvonep\end{tikzpicture}};
        \node[draw=darkblue, rounded corners, label=180:{\Large$\delta_2'$}] (dv2) at (0,-2.8) {\begin{tikzpicture}[construction,yscale=0.9]\dotsvonep \dotsvthreen \dotsvfourn \dotsuone \end{tikzpicture}};
        \node[draw=darkblue, rounded corners, label=180:{\Large$\delta_3'$}] (dv3) at (-1.25,-7) {\begin{tikzpicture}[construction,yscale=0.9]\dotsvthreen \dotsvsixn \dotsvsevenn \dotsuthree\end{tikzpicture}};
        \node[draw=darkred, rounded corners, label=270:{\Large$\delta_4'$}] (dv4) at (1.25,-7) {\begin{tikzpicture}[construction,yscale=0.9]\dotsvfourn\end{tikzpicture}};
	\path[->]
	([xshift=0cm]dv2.north) edge[out = 90, in = -90] node[decomp arrow label] {1} ([xshift=0.3cm]dv1.south)
	(dv3) edge[out = 90, in = -110] node[decomp arrow label] {1} ([xshift=-1cm]dv2.south)
	(dv4) edge[out = 90, in = -80] node[decomp arrow label] {2} ([xshift=1cm]dv2.south);
	\end{tikzpicture}
}
\medskip
\end{center}\end{samepage}
In this case, we have
\begin{enumerate}
\item[-] $\vertexRelations(\delta_3,\delta_3')=\tRelSpecn_1$, and (the construction labelling) $\delta_3'$ is a structural $\vertexRelations$-transformation of $\delta_3$, and
\item[-] $\vertexRelations(\delta_5,\delta_4')=\tRelSpecn_0$, and (the construction labelling) $\delta_4'$ is a structural $\vertexRelations$-transformation of $\delta_5$.
\end{enumerate}
\begin{samepage} This is an example of \textit{validity} and \textit{ancestor-validity}:
\begin{enumerate}
\item for any pair of \textit{leaves} for which $\vertexRelations$ is defined, namely $(\delta_3,\delta_3')$ and $(\delta_5,\delta_4')$, the construction that labels $\delta_i'$ is a structural $\vertexRelations$-transformation\footnote{Recall that, when $\vertexRelations(\delta_j,\delta_i')$ is not defined, we always have a structural transformation. This is because there is no assigned transformation constraint that must be satisfied.} of $\delta_j$, which means we have \textit{validity} at the leaves, and
\item for the $\delta_2$- and $\delta_2'$-induced decomposition trees, we have ancestor-validity: for any pair of $\delta_2$ and $\delta_2'$ ancestors (i.e. in this example, any pair of leaves) for which $\vertexRelations$ is defined we have validity, and
\item ancestor-validity for $\delta_2$ and $\delta_2'$, allows us to \textit{deduce} that any construction in $\rsystemn_{\mathit{D}}$ that matches $\delta_2'$ %-- provided we have a construction in $\rsystemn_{D}$ --
    will ensure validity at $\delta_2$ and $\delta_2'$: the construction of the $\delta_2'$-induced decomposition is a structural $\vertexRelations$-transformation of the construction of the $\delta_2$-induced decomposition; in this particular case, we see that ancestor-validity implies validity.
\end{enumerate}
\end{samepage}

Whilst the example just given focuses on constructions that match the patterns that label ancestors of $\delta_2'$ in $\pdecompositionn'$, in such a way that we have structural $\vertexRelations$-transformations, we can actually make a more general claim: for any construction, $\cpair$, that matches $\pdecompositionn$ and for any pattern, $\ppaird$, that matches $\pdecompositionn'$, and for any $(\delta,\delta')\in V(\pdecompositionn)\times V(\pdecompositionn')$, given
\begin{enumerate}
\item the $\delta$-induced decomposition, $\decompositionn$, in the $\pdecompositionn$-canonical decomposition of $\cpair$, and %$\de

\item the $\delta'$-induced decomposition, $\decompositionn'$, in the $\pdecompositionn'$-canonical decomposition of $\ppaird$,
\end{enumerate}
if the construction $\constructionofD{\decompositionn'}$ is actually \textit{a construction in} $\rsystemn'$ -- so, this `part' of $\ppaird$ is a construction -- then, \textit{provided} we have ancestor-validity for $(\delta,\delta')$,  $\constructionofD{\decompositionn'}$ is, roughly speaking, guaranteed to be a structural $\vertexRelations$-transformation of $\constructionofD{\decompositionn}$. Technically, though, the transformation requires a restriction of $\vertexRelations$ to the $\delta$- and $\delta'$-induced decompositions  of $\pdecompositionn$ and $\pdecompositionn'$. Given this general claim is true, we say that $\vertexRelations$ is \textit{sound}. From a practical perspective, soundness means that if
\begin{quote}
 all of the patterns that label leaves of $\pdecompositionn'$ are replaced by matching constructions in such a way that the associated transformation constraints are satisfied
\end{quote}
then we are assured that
\begin{quote}
if we continue to replace patterns in $\pdecompositionn'$with matching constructions, yielding a construction, $\cpaird$, in $\rsystemn'$, then $\cpaird$ is a structural $\vertexRelations$-transformation of $\cpair$.
\end{quote}
Evidently, soundness is a powerful property that can be exploited in practical applications.
\end{example}

 We now formalise what it means for $\vertexRelations$ to be sound, in definition~\ref{defn:soundL}, which builds on definition~\ref{defn:validForVertices} (validity). This requires us to produce \textit{restrictions} of constraint assignments.

\begin{definition}\label{defn:restrictionOfS}
Let $\irencodingn=\irencoding$ be an inter-representational-system encoding with description $\descriptionirse$.
Let $\pdecompositionn$ and $\pdecompositionn'$ be pattern decompositions in $\descriptionn$ and $\descriptionn'$, respectively, with
constraint assignment $\vertexRelations\colon V(\pdecompositionn)\times V(\pdecompositionn') \to \setOfTokenRels$. Let $(\delta,\delta')\in  V(\pdecompositionn)\times V(\pdecompositionn')$.
The \textit{$(\delta,\delta')$-restriction} of $\vertexRelations$, denoted $\vertexRelations_{(\delta,\delta')}$ is the restriction of $\vertexRelations$ to $V(\videctree{\delta,\pdecompositionn})\times V(\videctree{\delta',\pdecompositionn'})$.
\end{definition}

\begin{definition}\label{defn:validForVertices}
Let $\irencodingn=\irencoding$ be an inter-representational-system encoding with description $\descriptionirse$. Let $\pdecompositionn$ and $\pdecompositionn'$ be pattern decompositions in $\descriptionn$ and $\descriptionn'$, respectively, with
constraint assignment $\vertexRelations\colon V(\pdecompositionn)\times V(\pdecompositionn') \to \setOfTokenRels$. Let $\cpair$ be a construction in $\rsystemn$ that matches $\pdecompositionn$. Let $\ppaird$ be a pattern for $\rsystemn'$ that matches $\pdecompositionn'$. Let $(\delta,\delta')\in  V(\pdecompositionn)\times V(\pdecompositionn')$. Let $\decompositionn$ and $\decompositionn'$ be such that
\begin{enumerate}
\item  $\decompositionn$ is the $\delta$-induced decomposition in the $\pdecompositionn$-canonical decomposition of $\cpair$, and %$\deltacanonical{\cpair,\pdecompositionn}$, and $\deltacanonical{\ppaird,\pdecompositionn'}$, and

\item $\decompositionn'$, is the $\delta'$-induced decomposition in the $\pdecompositionn'$-canonical decomposition of $\ppaird$.
\end{enumerate}
Then $\vertexRelations$ is \textit{valid} for $\cpair$ and $\ppaird$ at $(\delta,\delta')$ provided whenever $\constructionofD{\decompositionn'}$ is a construction in $\rsystemn'$ it is also the case that $\constructionofD{\decompositionn'}$ is a structural $\vertexRelations_{(\delta,\delta')}$-transformation of $\constructionofD{\decompositionn}$.  We say that $\vertexRelations$ is \textit{ancestor-valid} for $\cpair$ and $\ppaird$ at $(\delta,\delta')$ provided for all $(\varepsilon,\varepsilon')\in  \ancestors{\delta,\pdecompositionn}\times \ancestors{\delta',\pdecompositionn'}$, $\vertexRelations$ is valid for $\cpair$ and $\ppaird$ at $(\varepsilon,\varepsilon')$.
\end{definition}

\begin{definition}\label{defn:soundL}
Let $\irencodingn=\irencoding$ be an inter-representational-system encoding with description $\descriptionirse$ . Let $\pdecompositionn$ and $\pdecompositionn'$ be pattern decompositions in $\descriptionn$ and $\descriptionn'$, respectively, with
constraint assignment $\vertexRelations\colon V(\pdecompositionn)\times V(\pdecompositionn') \to \setOfTokenRels$. Then $\vertexRelations$ is \textit{sound} provided for all constructions, $\cpair$ in $\rsystemn$, that match $\pdecompositionn$, for all patterns, $\ppair$, for $\rsystemn'$ that match $\pdecompositionn'$, and for all $(\delta,\delta')\in V(\pdecompositionn)\times V(\pdecompositionn')$, whenever $\vertexRelations$ is ancestor-valid for $\cpair$ and $\ppair$ at $(\delta,\delta')$ it is also the case that $\vertexRelations$ is valid for $\cpair$ and $\ppair$ at $(\delta,\delta')$.
\end{definition}

\begin{example}[Using Sound Constraint Assignments to Solve Problems]
Recall that we set out to exploit structural transformations to establish that $1+2+3=\frac{3(3+1)}{2}$ is true. Examples~\ref{ex:FOAandDtokenRelSpec} and~\ref{ex:FOAandDotsStructuralLinks} exemplified transformations of constructions of $1+2+3$ and $\frac{3(3+1)}{2}$, into dot diagrams where the respective $\rsystemn_{D}$ constructs -- each of which was a triangular arrangement of six dots -- represented the numerical expressions.{\pagebreak} Suppose that it is known, in $\rsystemn_{D}$, that the triangular arrangement of dots constructed by $\big(\cgraphhn,\hspace{-0.1cm}\scalebox{0.74}{\tikz[baseline=0.07cm]{\oT}}\hspace{0.2cm}\big)$ is a translation of the triangular arrangement\footnote{Recall that dots in dot diagrams occupy integer coordinates in the plane. The two triangular dot diagrams here look identical, since we have not made explicit the positions occupied by the dots, yet they are positioned in different locations. The additional white space in the vertex-contour for the rightmost dot diagram is intended to signify that this token takes a different location in the plane.}
 constructed by \begin{samepage}$\big(\cgraphhn',\hspace{-0.1cm}\scalebox{0.74}{\tikz[baseline=0.07cm]{\oT}}\hspace{0.2cm}\big)$:
\begin{center}
\begin{tikzpicture}[construction,yscale=0.9]
    \dotsv
    \dotsvone
    \dotsvtwo
    \dotsu
\end{tikzpicture}
\end{center}\end{samepage}
Given the transformation constraints $\tRelSpecn_5=((\pgraphbn_1',b_1), (\pgraphan_1',x), \{(\pgraphmn_5,v_5)\})$, whose graphs are
\begin{center}
\adjustbox{scale=0.8, raise = 0.2cm}{%
\begin{tikzpicture}[construction,yscale=0.9]\small
    \dotsvp
    %\dotsvonep
%    \dotsvtwop
%    \dotsu
\end{tikzpicture}
}\hspace{2cm}
\adjustbox{scale=0.8}{%
\begin{tikzpicture}[construction,yscale=0.8]\small
\foavp
\foavonep
\foau
%
%\foavtwop
%\foavthreep
%\foavfourp
%%
%\foauone
\end{tikzpicture}
}\hspace{2cm}
\adjustbox{scale=0.8, raise = 0.1cm}{%
\begin{tikzpicture}[construction,yscale=0.8,xscale=0.8]\small
\node[typeIE = {\typetrue}] at (0,-0.3) (v) {$v_5$};
\node[constructorINE = {\constructorAnd}] at (0,-1.3) (u) {};
\node[typeIW = {\typetrue}] at (-2,-2) (v1) {$\bcolour{b_1}$};
\node[typeIE = {\typebool}] at (2,-2) (v2) {$\acolour{x}$};
    \path[->]
	(v1) edge[bend right = -10] node[ index label] {1} (u)
	(v2) edge[bend right = 10] node[ index label] {2} (u)
	(u) edge (v);
\end{tikzpicture}
}
\end{center}
and $\tRelSpecn_6=((\pgraphbn_2',b_1), (\pgraphan_2',a_1), \{(\pgraphmn_6,v_6),(\pgraphmn_6',v_6')\})$ where
\begin{center}
	\adjustbox{scale=0.8, raise = 0.2cm}{%
\begin{tikzpicture}[construction,xscale=0.7,yscale=0.9]\small
    \dotsvp
    \dotsvonep
    \dotsvtwop
    \dotsu
\end{tikzpicture}
}\hspace{0.2cm}
\adjustbox{scale=0.8, raise = 0.2cm}{%
\begin{tikzpicture}[construction,xscale=0.7,yscale=0.9]\small
%\foavp
\foavonep
%
%\foau
%
\foavtwop
\foavthreep
\foavfourp
\foauone
\end{tikzpicture}
}\hspace{0.2cm}
\adjustbox{scale=0.8}{%
\begin{tikzpicture}[construction,yscale=0.7,xscale=0.8]\small
%\node[typeIE = {\typetrue}] at (0,-0.4) (v) {$v'$};
%
%\node[constructorINE = {\constructorAnd}] at (0,-1.4) (u) {};
%
\node[typeIW = {\typetrue}] at (-2,-2) (v1) {$v_6$};
\node[typeIE = {\typetrue}] at (2,-2) (v2) {$v_6'$};
\node[constructorINW = {\constructorRepresentedBy}] at (-2,-3) (u1) {};
\node[typeW = {\FOAnumexp}] at (-3,-4) (v3) {$\acolour{a_{2}}$};
\node[typeS = {\DDdotDiagram}] at (-1,-4) (v4) {$\bcolour{b_{2}}$};
\node[constructorINE = {\constructorRepresentedBy}] at (2,-3) (u2) {};
\node[typeS = {\FOAnumexp}] at (1,-4) (v5) {$\acolour{a_{10}}$};
\node[typeE = {\DDdotDiagram}] at (3,-4) (v6) {$\bcolour{b_{7}}$};
    \path[->]
	%(v1) edge[bend right = -10] node[ index label] {1} (u)
	%(v2) edge[bend right = 10] node[ index label] {2} (u)
	%(u) edge (v)
    %
    (v3) edge[bend right = -10] node[ index label] {1} (u1)
	(v4) edge[bend right = 10] node[ index label] {2} (u1)
	(u1) edge (v1)
    (v5) edge[bend right = -10] node[ index label] {1} (u2)
	(v6) edge[bend right = 10] node[ index label] {2} (u2)
	(u2) edge (v2);
\end{tikzpicture}
}
\end{center}
\begin{samepage}along with $\pdecompositionn$ (left) and $\pdecompositionn'$ (right)
\begin{center}\label{ex:decompositionOfPatternsForLater}
\adjustbox{scale=0.8,raise=0.25cm}{%
\begin{tikzpicture}[decomposition]
        \node[draw=darkgreen, rounded corners, label=180:{\Large$\delta_1$}] (dv1) at (0,0) {\begin{tikzpicture}[construction,yscale=0.9,xscale=0.8]
    \dotsvp
    %\dotsvonep
%    \dotsvtwop
%    \dotsu
\end{tikzpicture}};
\node[draw=darkred, rounded corners, label=180:{\Large$\delta_2$}] (dv2) at (-0.5,-3.5) {\begin{tikzpicture}[construction,yscale=1,xscale=0.8]
    \dotsvp
    \dotsvonep
    \dotsvtwop
    \dotsu
\end{tikzpicture}};
\path[->]
    ([xshift=-0.0cm]dv2.north) edge[out = 90, in = -90] node[decomp arrow label] {1} ([xshift=-0.2cm]dv1.south);
	\end{tikzpicture}
}\hspace{1.5cm}
\adjustbox{scale=0.8}{%
\begin{tikzpicture}[decomposition]
        \node[draw=darkgreen, rounded corners, label=180:{\Large$\delta_1'$}] (dv1) at (0,0) {\begin{tikzpicture}[construction,yscale=0.92,xscale=0.8]
\foavp
\foavonep
\foau
%
%\foavtwop
%\foavthreep
%\foavfourp
%%
%\foauone
\end{tikzpicture}};
\node[draw=darkred, rounded corners, label=180:{\Large$\delta_2'$}] (dv4) at (0.5,-3.5) {\begin{tikzpicture}[construction,yscale=0.92,xscale=0.8]
%\foavp
\foavonep
%
%\foau
%
\foavtwop
\foavthreep
\foavfourp
\foauone
\end{tikzpicture}};
	\path[->]
    ([xshift=-0.3cm]dv4.north) edge[out = 90, in = -90] node[decomp arrow label] {1} ([xshift=-0.2cm]dv1.south)
;
	\end{tikzpicture}
}
\end{center}\end{samepage}
we define $\vertexRelations(\delta_1,\delta_1')=\tRelSpecn_5$ and $\vertexRelations(\delta_2,\delta_2')=\tRelSpecn_6$. Given
\begin{enumerate}
\item any construction in $\rsystemn_{\FOA}$ that matches $\pdecompositionn$, and
\item any construction in $\rsystemn_{D}$, that matches the pattern that labels $\delta_2'$ and ensures $\tRelSpecn_6$ is satisfied,
\end{enumerate}
we are assured that, when we complete the instantiation of $\delta_1'$, $x$ will be replaced by a token of type $\typetrue$. This is because $\vertexRelations$ is sound.{\pagebreak} In the context of our running example, we have the following \begin{samepage}constructions:
\begin{center}\label{ex:decompositionOfPatternsForLater}
\adjustbox{scale=0.8,raise=0.3cm}{%
\begin{tikzpicture}[decomposition]
        \node[draw=darkgreen, rounded corners, label=180:{\Large$\delta_1$}] (dv1) at (0,0) {\begin{tikzpicture}[construction,yscale=0.9]
    \dotsv
    %\dotsvonep
%    \dotsvtwop
%    \dotsu
\end{tikzpicture}};
\node[draw=darkred, rounded corners, label=180:{\Large$\delta_2$}] (dv2) at (-0.5,-3.5) {\begin{tikzpicture}[construction,yscale=1]
    \dotsv
    \dotsvone
    \dotsvtwo
    \dotsu
\end{tikzpicture}};
\path[->]
([xshift=-0.0cm]dv2.north) edge[out = 90, in = -90] node[decomp arrow label] {1} ([xshift=-0.1cm]dv1.south);
	\end{tikzpicture}
}\hspace{1.5cm}
\adjustbox{scale=0.8}{%
\begin{tikzpicture}[decomposition]
        \node[draw=darkgreen, rounded corners, label=180:{\Large$\delta_1'$}] (dv1) at (0,0) {\begin{tikzpicture}[construction,yscale=0.92]
\foavp
\foavone
\foau
%
%\foavtwop
%\foavthreep
%\foavfourp
%%
%\foauone
\end{tikzpicture}};
\node[draw=darkred, rounded corners, label=180:{\Large$\delta_2'$}] (dv4) at (0.5,-3.5) {\begin{tikzpicture}[construction,yscale=0.92]
%\foavp
\foavone
%
%\foau
%
\foavtwo
\foavthree
\foavfour
\foauone
\end{tikzpicture}};
\path[->]
([xshift=-0.3cm]dv4.north) edge[out = 100, in = -95] node[decomp arrow label] {1} ([xshift=-0.0cm]dv1.south);
	\end{tikzpicture}
}
\end{center}\end{samepage}
where $\delta_2'$ is a structural $\vertexRelations(\delta_2,\delta_2')$-transformation of $\delta_2$. By the soundness of $\vertexRelations$, we can deduce, therefore, that an instantiation of $x$, as a token in $\rsystemn_{\FOA}$, must ensure that $\tRelSpecn_5$ is satisfied. The only possibility, therefore, is that $x$ is $\top$. Using knowledge from $\rsystemn_{D}$ -- in particular that the $s_2$ triangulation is a translation of the $s_7$ triangulation -- we can infer that $1+2+3=\frac{3(3+1)}{2}$ is true.
\end{example}

We now prove, in theorem~\ref{thm:soundness}, a highly desirable property: if we have
\begin{enumerate}
\item a sound constraint assignment, $\vertexRelations\colon V(\pdecompositionn)\times V(\pdecompositionn')\to \setOfTokenRels$,
\item a construction, $\cpair$, that matches $\pdecompositionn$ and
\item  a pattern, $\ppairhatd$, whose $\pdecompositionn'$-canonical decomposition has constructions at the leaves that ensure we have `leaf validity', given $\cpair$,
\end{enumerate}
then no matter how we continue to replace vertices with tokens in $\ppairhatd$ we are assured that any resulting construction in $\rsystemn'$ is a structural $\vertexRelations$-transformation of $\cpair$.  Firstly, we require a definition concerning the leaves of a pattern decomposition.

\begin{definition}\label{defn:leafInstantiates}
Let $\rsystemn$ be a representational system, let $\pdecompositionn$ be a decomposition, and let $\ppair$ be a pattern for $\rsystemn$ that matches $\pdecompositionn$. Then $\ppair$ \textit{leaf-instantiates} $\pdecompositionn$ provided all leaves of the $\pdecompositionn$-canonical decomposition of $\ppair$ are labelled by constructions in $\rsystemn$. Given a $\ppair$ that leaf-instantiates $\pdecompositionn$, a \textit{complete extension} of $\ppair$ is a construction, $\cpair$, in $\rsystemn$ where  there is an embedding, $f\colon \cpair \to \ppair$, such that for all tokens, $t_i$, that occur as a vertex in some pattern, $\ppairp{i}$, that labels a leaf of the $\pdecompositionn$-canonical decomposition of $\ppair$, it is the case that $f^{-1}(t_i)=t_i$.
%%
%\begin{enumerate}
%\item there is an embedding, $f\colon \hat{\pgraphn} \to \pgraphn$, such that for all tokens, $t_i$, that occur as a vertex in some pattern, $\ppairp{i}$, that labels a leaf of the $\pdecompositionn$-canonical decomposition of $\ppair$, it is the case that $f^{-1}(t_i)=t_i$, and
%%
%\item $(\hat{\pgraphn},\hat{v})$ is a construction .
%\end{enumerate}
\end{definition}

\begin{definition}\label{defn:validAtLeaves}
Let $\irencodingn=\irencoding$ be an inter-representational-system encoding with description $\descriptionirse$.
Let $\pdecompositionn$ and $\pdecompositionn'$ be pattern decompositions in $\descriptionn$ and $\descriptionn'$, respectively, with
constraint assignment $\vertexRelations\colon V(\pdecompositionn)\times V(\pdecompositionn') \to \setOfTokenRels$. Let $\cpair$ be a construction in $\rsystemn$ that matches $\pdecompositionn$. Let $\ppaird$ be a pattern for $\rsystemn'$ that leaf-instantiates $\pdecompositionn'$. Then $\cpair$ and $\ppaird$ are \textit{valid at the leaves} of $\pdecompositionn$ and $\pdecompositionn'$ provided for all leaves, $l$ and $l'$, in $\pdecompositionn$ and $\pdecompositionn'$, $\vertexRelations$ is valid for $\cpair$ and $\ppaird$ at $(l,l')$.
\end{definition}

\begin{theorem}\label{thm:soundness}
Let $\irencodingn=\irencoding$ be an inter-representational-system encoding with description $\descriptionirse$. Let $\pdecompositionn$ and $\pdecompositionn'$ be pattern decompositions in $\descriptionn$ and $\descriptionn'$, respectively, with sound
constraint assignment $\vertexRelations\colon V(\pdecompositionn)\times V(\pdecompositionn') \to \setOfTokenRels$. Then for all constructions, $\cpair$, in $\rsystemn$ that match $\pdecompositionn$, and for all patterns, $\ppairhatd$, for $\rsystemn'$ that leaf-instantiate $\pdecompositionn'$, if $\cpair$ and $\ppairhatd$ are valid at the leaves of $\pdecompositionn$ and $\pdecompositionn'$ then any complete extension, $(\hat{\cgraphn}',\hat{t}')$ of $\ppairhatd$ is a structural $\vertexRelations$-transformation of $\cpair$.
\end{theorem}

\begin{proof}
Suppose that $\cpair$ and $\ppairhatd$ are valid at the leaves of $\pdecompositionn$ and $\pdecompositionn'$. Let $(\hat{\cgraphn}',\hat{t}')$ be a complete extension of $\ppairhatd$. We must establish that $(\hat{\cgraphn}',\hat{t}')$ is a structural $\vertexRelations$-transformation of $\cpair$. The proof strategy has two main parts:
\begin{enumerate}
\item show that $(\hat{\cgraphn}',\hat{t}')$ is a construction in $\rsystemn'$ that matches $\pdecompositionn'$, and
\item show that $\vertexRelations$ is valid for $\cpair$ and $(\hat{\cgraphn}',\hat{t}')$ at the roots, $\delta_r$ and $\delta_r'$ of $\pdecompositionn$ and $\pdecompositionn'$ respectively.
\end{enumerate}
For the first part of the proof, we note that there are embeddings, $f\colon (\hat{\cgraphn}',\hat{t}')\to (\hat{\pgraphn}',\hat{v}')$ and $f'\colon (\hat{\pgraphn}',\hat{v}')\to (\pgraphn',v')$, where $\ppaird$ is the construction of $\pdecompositionn'$. It can readily be shown that $(\hat{\cgraphn}',\hat{t}')$ matches $\pdecompositionn'$, since the $\pdecompositionn'$-canonical decomposition of $\ppairhatd$ can be used to produce an isomorphic $\pdecompositionn'$-canonical decomposition of $(\hat{\cgraphn}',\hat{t}')$. Since $(\hat{\cgraphn}',\hat{t}')$ is a complete extension of $\ppair$, we know that $(\hat{\cgraphn}',\hat{t}')$ is a construction in $\rsystemn'$. Hence the first part holds.

The second part of the proof proceeds by strong double induction on the length of trails, in $\pdecompositionn$ and $\pdecompositionn'$, sourced on leaves. For the base case, we consider trails of length $0$. It is given that $\cpair$ and $\ppairhatd$ are valid at the leaves of $\pdecompositionn$ and $\pdecompositionn'$. Since $(\hat{\cgraphn}',\hat{t}')$ is a complete extension of $\ppairhatd$, it trivially holds that $\cpair$ and $(\hat{\cgraphn}',\hat{t}')$ are also valid at the leaves of $\pdecompositionn$ and $\pdecompositionn'$. Assume, for any pair of vertices, $(\delta,\delta')$, in $V(\pdecompositionn)\times V(\pdecompositionn')$, for which the longest trails sourced on leaves of $\pdecompositionn$ and $\pdecompositionn'$ that have targets $\delta$ and $\delta'$ are of lengths $j$ and $k$, that $\vertexRelations$ is valid for $\cpair$ and $(\hat{\cgraphn}',\hat{t}')$ at $(\delta,\delta')$. There are two inductive steps, where we increment either $j$ or $k$ by $1$. In the first instance, suppose that $(\delta,\delta')$ is a pair of vertices in $V(\pdecompositionn)\times V(\pdecompositionn')$, for which the longest trails sourced on leaves of $\pdecompositionn$ and $\pdecompositionn'$ that have targets $\delta$ and $\delta'$ are of lengths $j+1$ and $k$, respectively. Then, for any pair of ancestors, $(\varepsilon,\varepsilon')\in \ancestors{\delta,\pdecompositionn}\times \ancestors{\delta',\pdecompositionn'}$, the lengths of the longest trails sourced on leaves that have targets $\varepsilon$ and $\varepsilon'$ are at most $j$ and $k-1$. By the inductive assumption, $\vertexRelations$ is valid for $\cpair$ and $(\hat{\cgraphn}',\hat{t}')$ at $(\varepsilon,\varepsilon')$. Hence, $\vertexRelations$ is ancestor-valid for $\cpair$ and $(\hat{\cgraphn}',\hat{t}')$ at $(\delta,\delta')$. By the soundness of $\vertexRelations$, ancestor-validity implies validity, that is $\vertexRelations$ is valid for $\cpair$ and $(\hat{\cgraphn}',\hat{t}')$ at $(\delta,\delta')$, as required. The case where we increment $k$ by $1$ is similar. Hence, $\vertexRelations$ is valid for $\cpair$ and $(\hat{\cgraphn}',\hat{t}')$ at any pair of vertices, $(\delta,\delta')$, in $V(\pdecompositionn)\times V(\pdecompositionn')$. In particular, $\vertexRelations$ is valid for $\cpair$ and $(\hat{\cgraphn}',\hat{t}')$ at the roots, $\delta_r$ and $\delta_r'$.  By the definition of validity, since $(\hat{\cgraphn}',\hat{t}')$ is a construction in $\rsystemn'$, it follows that $(\hat{\cgraphn}',\hat{t}')$ is a structural $\vertexRelations$-transformation of $\cpair$, as required.
\end{proof}

In summary, validity and soundness are highly desirable properties of constraint assignments. In terms of practical applications, soundness means that we only have to find structural transformations of leaves in decompositions provided we have representational-system-specific knowledge about how the constructors operate on tokens in $\rsystemn'$. In particular, when we have a functional and total construction space, once the leaves are instantiated (and, hence all foundation tokens are known) then, provided $\ppaird$ describes a construction in $\rsystemn'$, we know it is possible to instantiate the remaining vertices in $\ppaird$ in a unique way. When we do not have functional constructors, we are free to choose any instantiation of tokens provided the result is a construction in $\rsystemn'$. In the absence of totality, it may be that alternative choices of foundation tokens are needed in order to complete the structural transformation. But, in any of these cases, soundness reduces the need to seek a property identifier, $\setofconstructionsirse$, for $\irencodingn$ when checking for the satisfaction of transformation constraints. In essence, we only require partial knowledge of $\ipispacen$ when exploiting structure transfer algorithms for representation selection.

%% file: conclusion.tex
It is self-evident that the study of representations and, more generally, representational systems is of fundamental importance to understanding how to effectively communicate and reason about information. The novel contribution of this paper is \textit{Representational Systems Theory}, which delivers several major advances over the existing state-of-the-art:
\begin{enumerate}
\item \textit{Contribution 1:} Representational Systems Theory provides \textit{the first general framework} that facilitates the study of representational systems. A key aspect of our contribution is the introduction of construction spaces, which provide a single, unifying paradigm that allows us to encode representational systems at the syntactic levels of their grammatical rules and entailment relations, as well as the meta-level properties of their tokens.
\item \textit{Contribution 2:} Our theory facilitates the development and implementation of \textit{repre\-sentational-system-agnostic transformation algorithms}, removing the current need to devise system-specific algorithms, because of its unified nature. The approach taken to defining transformations is designed to support \textit{representation selection} based on desired structural relations and meta-level properties holding between representations.
\item \textit{Contribution 3:} Our theory \textit{permits the derivation of partial transformations} when no terminating algorithm can produce a full transformation.
\end{enumerate}
The abstract, mathematical nature of Representational Systems Theory leads it to occupy a fundamental position in the study of representations. In this sense, one can liken our contribution to the development of category theory for the study of mathematical structures: just as category theory provides foundational tools for mathematicians, our theoretical results provide foundational tools for examining representational systems.

An exciting next step is to develop structure transfer algorithms, further demonstrating the potential for Representational Systems Theory to substantially improve our understanding of representational systems as well as significantly enhancing their practical use. We have already produced a prototype implementation of a structure transfer algorithm. When providing our tool with constructions of \textsc{First-Order Arithmetic} representations, along with transformation constraints similar to those in this paper, it automatically produces constructions of semantically equivalent dot diagrams, or partially instantiated patterns (depending on the knowledge base it has available). We plan to extend the knowledge base to other representational systems, such as \textsc{Bayesian Probability} and \textsc{Area Diagrams}, to more fully exemplify the practical utility of our theory. To take our implementation beyond transformations geared towards preserving semantics, we are particularly keen to devise  general models of cognition that can be encoded using identification spaces. This will drive forwards our ability to automatically enable the selection of representations that are cognitively more effective, which is paramount to improving representation use. In summary, Representational Systems Theory is the first general framework that supports representation selection, via structural transformations, and has the potential for widespread practical application.

%% file: appendixNotation.tex
The paper introduces a large number of new concepts and exploits a lot of notation. This appendix provides a summary of the novel notation, in the order that it appears in the paper. For pre-existing concepts, we refer to Section~\ref{sec:prelims}.

\setlength{\columnsep}{0.25cm}
\renewcommand*{\arraystretch}{1.4}
\begin{longtable}{ p{.17\textwidth}  p{.80\textwidth} }
$\tsystemn=\tsystem$ &
 a type system, where $\types$ is a set of types and $\leq$ is a partial order over $\types$; definition~\ref{defn:typeSystem}, page~\pageref{defn:typeSystem}\\
$\cspecificationn=(\constructors, \sig)$ &  a constructor specification, where $\constructors$ is a set of constructors and $\sig\colon \constructors \to \sequence(\types)\times \types$ maps constructors to signatures; definition~\ref{defn:constructionSpecification}, page~\pageref{defn:constructionSpecification}\\
$\inputsty{c}$ & the input type-sequence for constructor $c$; definition~\ref{defn:constructionSpecification}, page~\pageref{defn:constructionSpecification} \\
$\outputsty{c}$ & the output type for constructor $c$; definition~\ref{defn:constructionSpecification}, page~\pageref{defn:constructionSpecification}\\
$\graphn$ & a bipartite graph, $\graphn=\graph$, which may be a configuration or a structure graph; definition~\ref{defn:configuration}, page~\pageref{defn:configuration} (configuration) and definition~\ref{defn:structureGraph}, page~\pageref{defn:structureGraph}\\
$\tokens$ & a set of vertices, called tokens, in a bipartite graph; definition~\ref{defn:configuration}, page~\pageref{defn:configuration}\\
$\pb$ & a set of vertices, called configurators, in a bipartite graph; definition~\ref{defn:configuration}, page~\pageref{defn:configuration}\\
$\arrows$ & a set of arrows in a bipartite graph; definition~\ref{defn:configuration}, page~\pageref{defn:configuration}\\
$\incVert$ & a function that maps each arrow to its incident source and target vertices; definition~\ref{defn:configuration}, page~\pageref{defn:configuration}\\
$\arrowl$ & a function that labels the arrows in a graph with indices; definition~\ref{defn:configuration}, page~\pageref{defn:configuration}\\
$\tokenl$ & a function that labels tokens in a graph with types; definition~\ref{defn:configuration}, page~\pageref{defn:configuration}\\
$\consl$ & a function that labels configurators in a graph with constructors; definition~\ref{defn:configuration}, page~\pageref{defn:configuration}\\
%
%$\pb$ & a set of vertices in a bipartite graph whose elements are called configurators; definition~\ref{defn:configuration}, page~\pageref{defn:configuration}\\
%
$\outputsto{u}$ & the output token that is the target of the arrow sourced on configurator $u$ in a configuration; definition~\ref{defn:configuration}, page~\pageref{defn:configuration}\\
$\outputsty{u}$ & the type assigned to the token that is the output of configurator $u$ in a configuration; definition~\ref{defn:configuration}, page~\pageref{defn:configuration}\\
$\inputsA{u}$ & the input arrow-sequence of configurator $u$ in a configuration; definition~\ref{defn:configuration}, page~\pageref{defn:configuration}\\
$\inputsto{u}$ & the input token-sequence of configurator $u$ in a configuration; definition~\ref{defn:configuration}, page~\pageref{defn:configuration}\\
$\inputsty{u}$ & the input type-sequence of configurator $u$ in a configuration; definition~\ref{defn:configuration}, page~\pageref{defn:configuration}\\
$\cspacen=\cspace$ & a construction space where $\tsystemn=\tsystem$ is a type system, $\cspecificationn=\cspecification$ is a constructor specification over $\tsystemn$, and $\graphn=\graph$ is a structure graph for $\cspecificationn$; definition~\ref{defn:constructionSpace}, page~\pageref{defn:constructionSpace}\\
$\subtypeClose(\types', \tsystemn)$ & the subtype closure of $\types'$ given a type system, $\tsystemn$; definition~\ref{defn:subtypeClosure}, page~\pageref{defn:subtypeClosure}\\
$\subtypeSets(\tsystemn)$ & the set of all subtype closures given a type system, $\tsystemn$; definition~\ref{defn:subtypeClosure}, page~\pageref{defn:subtypeClosure}\\
$\ispacen$ & an identification space, $\ispacen=\ispace$, where $\tsystemn=\tsystem$ and $\mpsystemn=\mpsystem$ are compatible type systems; definition~\ref{defn:identificationSpace}, page~\pageref{defn:identificationSpace}\\
$\rsystemn= \rsystem$ & a representational system formed over type system, $\tsystemn$, and meta-type system, $\mpsystemn$, where $\gspacen$, $\espacen$, and $\ispacen$ are the grammatical, entailment and identification spaces; definition~\ref{defn:representationalSystem}, page~\pageref{defn:representationalSystem}\\
$\graphn_{\gspacen}$ & the structure graph associated with the grammatical space, $\gspacen$; definition~\ref{defn:representationalSystem}, page~\pageref{defn:representationalSystem}\\
$\graphn_{\espacen}$ & the structure graph associated with the entailment space, $\espacen$; definition~\ref{defn:representationalSystem}, page~\pageref{defn:representationalSystem}\\
$\graphn_{\ispacen}$ & the structure graph associated with the identification space, $\ispacen$; definition~\ref{defn:representationalSystem}, page~\pageref{defn:representationalSystem}\\
$\uspace{\rsystemn}$ &  the universal space of a representational system, $\rsystemn$; definition~\ref{defn:universalSpace}, page~\pageref{defn:universalSpace}\\
$\ugraph{\rsystemn}$ & the \textit{universal structure graph} for a representational system, $\rsystemn$;  definition~\ref{defn:universalSpace}, page~\pageref{defn:universalSpace}\\
$\irencodingn$ & an inter-representational-system encoding, $\irencodingn=\irencoding$, where $\rsystemn$ and $\rsystemn'$ are compatible representational systems, and $\ispacen''$ is an identification space that satisfies various condititions;  definition~\ref{defn:interRSEncoding}, page~\pageref{defn:interRSEncoding}\\
$\ipispacen$ & an inter-property identification space, where $\irencodingn$ is an inter-representational-system encoding; definition~\ref{defn:interRSEncoding}, page~\pageref{defn:interRSEncoding}\\
$\cpair$ & a construction, where $\cgraphn$ is a finite, uni-structured, structure graph, and $t$ is a token such that each vertex in $\cgraphn$ is the source of a trail that targets $t$; definition~\ref{defn:construction}, page~\pageref{defn:construction}\\
$\allconstructions{\rsystemn}$ & the set of constructions in a representational system, $\rsystemn$; definition~\ref{defn:setOfConstructions}, page~\pageref{defn:setOfConstructions}\\
$\tesequence{\trail}$ & the trail extension sequence of trail $\trail$; definition~\ref{defn:tes}, page~\pageref{defn:tes}\\
$\ctsequence{\cgraphn,v}$ & the complete trail sequence of construction $\cpair$; definition~\ref{defn:CTS}, page~\pageref{defn:CTS}\\
$\ssequence{\trailSeq}$ & the source sequence of a sequence, $\trailSeq$, of trails; definition~\ref{defn:sourceSequence}, page~\pageref{defn:sourceSequence}\\
$\foundationsto{\cgraphn,t}$ & the foundation token-sequence of construction $\cpair$; definition~\ref{defn:foundations}, page~\pageref{defn:foundations}\\
$\foundationsty{\cgraphn,t}$ & the foundation type-sequence of construction $\cpair$; definition~\ref{defn:foundations}, page~\pageref{defn:foundations}\\
$\trailToGraph{\trailSeq}$ & the graph formed from the trails in the trail-sequence $\trailSeq$; definition~\ref{defn:tr-graph}, page~\pageref{defn:tr-graph}\\
$\icgp{\trail}$ & the graph formed from the trails in $\tesequence{\trail}$; definition~\ref{defn:inducedConstruction}, page~\pageref{defn:inducedConstruction}\\
$\icp{\trail}$ & the construction $\icp{\trail}=(\icgp{\trail},\source{\trail})$; definition~\ref{defn:inducedConstruction}, page~\pageref{defn:inducedConstruction}\\
$\icsequence{\trailSeq}$ & the induced construction sequence, $\icsequence{\trailSeq}=[\icp{\trail_1},\ldots,\icp{\trail_n}]$, where $\trailSeq=[\trail_1,...,\trail_n]$;  definition~\ref{defn:inducedConstruction}, page~\pageref{defn:inducedConstruction}\\
$\csplit$ & a split of a construction, $\cpair$, where $(\cgraphn',t)$ is a generator and $\ics$ is the induced construction sequence of $\ctsequence{\cgraphn',t}$ in $\cpair$; definition~\ref{defn:split}, page~\pageref{defn:split}\\
%
%$\splitarrow$ & used to denote that $\csplit$ is a split of $\cpair$: $\cpair\splitarrow \csplit$; definition~\ref{defn:split}, page~\pageref{defn:split}\\
%
$\decompositionn$ & a decomposition, $\decompositionn=\decomposition$, which is type of a directed, rooted in-tree; definition~\ref{defn:decompositionTree}, page~\pageref{defn:decompositionTree} (directed, rooted in-tree), and definition~\ref{defn:decomposition}, page~\pageref{defn:decomposition} (decomposition)\\
$\treeroot{\decompositionn}$ & the root of decomposition $\decompositionn$; definition~\ref{defn:decTreeSequence}, page~\pageref{defn:decTreeSequence}\\
$\videctree{v}$ & the $v$-induced decomposition tree, which is the maximal sub-tree of $\decompositionn$ with root $v$; definition~\ref{defn:decTreeSequence}, page~\pageref{defn:decTreeSequence}\\
$\dsequence{v}$ & the decomposition-tree sequence, $\dsequence{v}=[\videctree{v_1},\ldots,\videctree{v_n}]$, of $v$; definition~\ref{defn:decTreeSequence}, page~\pageref{defn:decTreeSequence}\\
$\ancestors{v}$ & the ancestors of vertex $v$; definition~\ref{defn:decTreeSequence}, page~\pageref{defn:decTreeSequence}\\
 $\ancestorsandvertex{v}$ & the ancestors of vertex $v$ together with $v$; definition~\ref{defn:decTreeSequence}, page~\pageref{defn:decTreeSequence}\\
$\labc$ & the vertex-labelling function of a decomposition, $\decompositionn=\decomposition$, where each vertex is labelled by a construction; definition~\ref{defn:decomposition}, page~\pageref{defn:decomposition}\\
$\lcsequence{\decompositionn}$ & the leaf construction sequence of decomposition $\decompositionn$; definition~\ref{defn:leafConstructionSequence}, page~\pageref{defn:leafConstructionSequence}\\
$\constructionofD{\decompositionn}$ & the construction with decomposition $\decompositionn$;
definition~\ref{defn:constructionOfD}, page~\pageref{defn:constructionOfD}\\
$\cpair \matches \ppair$ & constuction $\cpair$ is a specialisation of, equivalently matches, pattern $\ppair$; definition~\ref{defn:embeddingPattern}, page~\pageref{defn:embeddingPattern}\\
$\pspacen=\pspace$ & a pattern space, where $\tsystemn$ is a type system, $\cspecificationn$ is a constructor specification, and $\patterngraphn$ is a structure graph; definition~\ref{defn:patternSpace}, page~\pageref{defn:patternSpace}\\
$\ppair$ & a construction that arises in some pattern space; definition~\ref{defn:matchAndDescribe}, page~\pageref{defn:matchAndDescribe}\\
$\labp$ & the vertex labelling function for a pattern decomposition, $\pdecompositionn$; first usege, definition~\ref{defn:decompositionMatchesPattern}, page~\pageref{defn:decompositionMatchesPattern}\\
$\decompositionn\matches \pdecompositionn$ & the decomposition $\decompositionn$, of a construction, matches the decomposition $\pdecompositionn$, of a pattern;
definition~\ref{defn:decompositionMatchesPattern}, page~\pageref{defn:decompositionMatchesPattern}\\
$\descriptionn$ & a description of a construction space; definition~\ref{defn:descriptionOfCSpace}, page~\pageref{defn:descriptionOfCSpace}\\
$V(\descriptionn)$  & the set of vertices of a description, $\descriptionn$; definition~\ref{defn:verticesOfDecomposition}, page~\pageref{defn:verticesOfDecomposition}\\
$\descriptionirse$ & a description of an inter-representational-system encoding; definition~\ref{defn:irsedescription}, page~\pageref{defn:irsedescription}\\
$\setofpatterns(\descriptionn)$ & the set of patterns arising in the description $\descriptionn$; definition~\ref{defn:descriptionPatterns}, page~\pageref{defn:descriptionPatterns}\\
$\patternequivclasses{\descriptionn}$ & the set of equivalence classes of patterns arising in the description $\descriptionn$; definition~\ref{defn:descriptionPatterns}, page~\pageref{defn:descriptionPatterns}\\
$\deltacanonical{\cpair,\pdecompositionn}$ & the $\pdecompositionn$-canonical decomposition of $\cpair$; definition~\ref{defn:deltaCanonical}, page~\pageref{defn:deltaCanonical}\\
$\setofconstructionsirse$ & a set of constructions, called a property identifier, drawn from $\irencoding$; definition~\ref{defn:propertyIDandEn}, page~\pageref{defn:propertyIDandEn}\\
$\setofpatternstc$ & a set of patterns, called a property enforcer, drawn from $\setofpatterns''$; definition~\ref{defn:propertyIDandEn}, page~\pageref{defn:propertyIDandEn}\\
$\tRelSpecn$ &  a transformation constraint, which is a triple comprising two patterns, $\ppaira$ and $\ppairb$, and a property enforcer $\setofpatternstc$; definition~\ref{defn:st:tokenRelSpec}, page~\pageref{defn:st:tokenRelSpec}\\
$\patterngraphn(\setofpatternstc)$ & the structure graph formed from a set, $\setofpatternstc$, of patterns; definition~\ref{defn:patternGraphOfPatterns}, page~\pageref{defn:patternGraphOfPatterns}\\
$\tRelSpecn\models \tRelSpecn'$ & denotes satisfaction of transformation constraint $\tRelSpecn'$ by $\tRelSpecn$; definition~\ref{defn:satisfies}, page~\pageref{defn:satisfies}\\
$\tRelSpecn\cong \tRelSpecn'$ & equivalent transformation constraints; definition~\ref{defn:equivalentTokenRelSpecs}, page~\pageref{defn:equivalentTokenRelSpecs}\\
$\vertexRelations$ & a constraint assignment that maps pairs of vertices to transformation constraints;  definition~\ref{defn:transformationSpecification}, page~\pageref{defn:transformationSpecification}\\
$\setOfTokenRels$ & a set of transformation constraints;  definition~\ref{defn:transformationSpecification}, page~\pageref{defn:transformationSpecification}\\
$\vertexRelations(p,p')$ & the $(p,p')$-restriction of constraint assignment $\vertexRelations$; definition~\ref{defn:restrictionOfS}, page~\pageref{defn:restrictionOfS}\\
\end{longtable}

%% file: appendixTerminology.tex
Complementing Appendix~\ref{sec:defnsAndNotation}, this appendix provides a summary of the novel terminology in the order it appears in the paper; for pre-existing concepts, we refer to Section~\ref{sec:prelims}. For brevity, the phrasing used in this appendix is loose and is, thus, intended only to provide informal insight to aid recollection; we encourage revisiting the precisely phrased definitions earlier in the paper in any case where an unambiguous understanding is required.

\renewcommand*{\arraystretch}{1.4}
\begin{longtable}{ p{.19\textwidth} p{0.01\textwidth} p{.738\textwidth} }
type system & &a pair comprising a set of types, $\types$, and a partial order, $\leq$, over $\types$; definition~\ref{defn:typeSystem}, page~\pageref{defn:typeSystem}\\
subtype  & & $\tau_1$ is a subtype of $\tau_2$ provided $\tau_1\leq \tau_2$; definition~\ref{defn:typeSystem}, page~\pageref{defn:typeSystem}\\
supertype  & & $\tau_2$ is a supertype of $\tau_1$ provided $\tau_1\leq \tau_2$; definition~\ref{defn:typeSystem}, page~\pageref{defn:typeSystem}\\
constructor \par specification  & & a pair comprising a set, $\constructors$, of \textit{constructors} and a function, $\spec$,  that returns the signature of each constructor; definition~\ref{defn:constructionSpecification}, page~\pageref{defn:constructionSpecification}\\
constructor & & elements of the set, $\constructors$, in a constructor specification; definition~\ref{defn:constructionSpecification}, page~\pageref{defn:constructionSpecification}\\
signature of a \par constructor  & & a pair comprising a non-empty sequence of types and a type; definition~\ref{defn:constructionSpecification}, page~\pageref{defn:constructionSpecification}\\
input \par type-sequence \par of a constructor & & given a constructor, $c$, $\inputsty{c}=[\tau_1,\ldots,\tau_n]$, where  $\spec(c)=([\tau_1,\ldots,\tau_n],\tau)$; definition~\ref{defn:constructionSpecification}, page~\pageref{defn:constructionSpecification}\\
output type \par of a constructor & & given a constructor $c$, $\outputsty{c}=\tau$, where $\spec(c)=([\tau_1,\ldots,\tau_n],\tau)$; definition~\ref{defn:constructionSpecification}, page~\pageref{defn:constructionSpecification}\\
specialisation & &  given a sequence, $\T=[\tau_1,\ldots,\tau_n]$,  of types, a specialisation is a sequence, $\T'$, obtained by replacing each type in $\T$ with a subtype; definition~\ref{defn:specialization}, page~\pageref{defn:specialization}\\
generalisation & & given a sequence, $\T=[\tau_1,\ldots,\tau_n]$, of types, a generalisation is a sequence, $\T'$, obtained by replacing each type in $\T$ with a supertype; definition~\ref{defn:specialization}, page~\pageref{defn:specialization}\\
configuration & & a directed bipartite graph that contains one vertex, $u$, labelled by a constructor, together with vertices (called tokens) that are inputs to $u$ and one vertex (a token) that is the output of $u$; definition~\ref{defn:configuration}, page~\pageref{defn:configuration}\\
$c$-configurator & &  a vertex in a configuration that is labelled by the constructor, $c$; definition~\ref{defn:configuration}, page~\pageref{defn:configuration}\\
token & & the input or output of a configurator; definition~\ref{defn:configuration}, page~\pageref{defn:configuration}\\
output token & & the vertex that is the output of a configurator; definition~\ref{defn:configuration}, page~\pageref{defn:configuration}\\
output type & & the type assigned to the vertex that is the output of a configurator; definition~\ref{defn:configuration}, page~\pageref{defn:configuration}\\
input \par arrow-sequence\par of a configurator &  & the sequence of arrows, ordered by their indices, that are incoming to a configurator; definition~\ref{defn:configuration}, page~\pageref{defn:configuration}\\
input \par token-sequence \par of a configurator & & the sequence of tokens, ordered by their incident arrows' indices, that are inputs to a configurator; definition~\ref{defn:configuration}, page~\pageref{defn:configuration}\\
input \par type-sequence  \par of a configurator& & the type sequence obtained from an input token-sequence by replacing tokens with their assigned types; definition~\ref{defn:configuration}, page~\pageref{defn:configuration}\\
structure graph & & a graph formed from a set of configurations; definition~\ref{defn:structureGraph}, page~\pageref{defn:structureGraph}\\
configurator & &  a vertex in a structure graph that is labelled by a constructor; definition~\ref{defn:structureGraph}, page~\pageref{defn:structureGraph}\\
construction space formed over $\tsystemn$ & & a triple comprising a type system, $\tsystemn$, a constructor specification, $\cspecification$, and a structure graph, $\graphn$; definition~\ref{defn:constructionSpace}, page~\pageref{defn:constructionSpace}\\
compatible \par type systems & & two type systems whose union is a type system; definition~\ref{defn:compatible}, page~\pageref{defn:compatible}\\
compatible \par construction space & &  two construction spaces whose union is a construction space; definition~\ref{defn:compatibleCS}, page~\pageref{defn:compatibleCS}\\
token-functional & &  for any input token-sequence to a constructor, there is a unique output token; definition~\ref{defn:deterministic}, page~\pageref{defn:deterministic}\\
type-functional & & for any input type-sequence to a constructor, there is a unique output type; definition~\ref{defn:deterministic}, page~\pageref{defn:deterministic}\\
functional & & a construction space that is both token- and type-functional is functional; definition~\ref{defn:deterministic}, page~\pageref{defn:deterministic}\\
instantiated & & a type, $\tau$, (resp. sequence of types, $[\tau_1,\ldots, \tau_n]$) is instantiated provided some token, $t$, (resp. sequence of tokens, $[t_1,\ldots, t_n]$) ensures $\type(t)\leq \tau$ (resp. for each $i$, $\type(t_i)\leq \tau_i$); definition~\ref{defn:instantiated}, page~\pageref{defn:instantiated}\\
instantiation & & a token, $t$, (resp. sequence of tokens, $[t_1,\ldots, t_n]$) is an instantation of type $\tau$ (resp. sequence of types, $[\tau_1,...,\tau_n]$) provided $\type(t)\leq \tau$ (resp. for each $i$, $\type(t_1)\leq \tau_1$); definition~\ref{defn:instantiated}, page~\pageref{defn:instantiated}\\
direct \par instantiation & & an instantiation, $t$, of type $\tau$, where $\type(t)=\tau$; definition~\ref{defn:instantiated}, page~\pageref{defn:instantiated}\\
token-total & & for any instantiation of the input type-sequence of a constructor, there is an output token; definition~\ref{defn:total}, page~\pageref{defn:total}\\
type-total & & for any specialisation, $\T$, of the input type-sequence of a constructor, there is an input token-sequence that specialises $\T$; definition~\ref{defn:total}, page~\pageref{defn:total}\\
total & & a construction space that is both token- and type-total is total; definition~\ref{defn:total}, page~\pageref{defn:total}\\
subtype closure & & given a set of types, $\types$, its closure under the subtype relation; definition~\ref{defn:subtypeClosure}, page~\pageref{defn:subtypeClosure}\\
upper bound \par extension & & an extension of a type system, $\tsystemn$, such that the subtype closure of each set of types has a fresh upper bound; definition~\ref{defn:intertionofLUB}, page~\pageref{defn:intertionofLUB}\\
defining \par upper bound & & a fresh type that is an upper bound of a subtype-closed set of types; definition~\ref{defn:intertionofLUB}, page~\pageref{defn:intertionofLUB}\\
identification space formed over $\tsystemn$ and $\mpsystemn$ & & a construction space formed over two compatible type systems; definition~\ref{defn:identificationSpace}, page~\pageref{defn:identificationSpace}\\
meta-type system & & one of the type systems over which an identification space is formed; definition~\ref{defn:identificationSpace}, page~\pageref{defn:identificationSpace}\\
representational system & & a triple comprising grammatical space, entailment space, and identification space that satisfy certain properties; definition~\ref{defn:representationalSystem}, page~\pageref{defn:representationalSystem}\\
grammatical space & & a functional construction space; definition~\ref{defn:representationalSystem}, page~\pageref{defn:representationalSystem}\\
entailment space & & a construction space that only has tokens occuring in a grammatical space; definition~\ref{defn:representationalSystem}, page~\pageref{defn:representationalSystem}\\
identification space & & in a representational system, a construction space that identifies propertes of tokens; definition~\ref{defn:representationalSystem}, page~\pageref{defn:representationalSystem}\\
meta-tokens & & tokens that occur in an identification space but not a grammatical space; definition~\ref{defn:representationalSystem}, page~\pageref{defn:representationalSystem}\\
universal space & & the space that is the union of the grammatical, entailment and identification spaces of a representational system; definition~\ref{defn:universalSpace}, page~\pageref{defn:universalSpace}\\
universal structure graph & & the structure graph of the universal space of a representational system;  definition~\ref{defn:universalSpace}, page~\pageref{defn:universalSpace}\\
compatible \par representational systems& & two representational systems whose union is a representational system; definition~\ref{defn:compatibleRSs}, page~\pageref{defn:compatibleRSs}\\
inter-representational-system encoding & & a triple comprising two representational systems along with an identification space that identifies properties of their tokens; definition~\ref{defn:interRSEncoding}, page~\pageref{defn:interRSEncoding}\\
inter-property identification space & & the space formed by unioning the identification spaces of two representational systems along with the identification space of an inter-representational-system encoding; definition~\ref{defn:interRSEncoding}, page~\pageref{defn:interRSEncoding}\\
uni-structured & & a property of a structure graph that asserts every token is the target of at most one arrow; definition~\ref{defn:uniStructured}, page~\pageref{defn:uniStructured}\\
construction & & a finite uni-structured structure graph, with an identified vertex, $t$, such that each vertex lies on a trail to $t$; definition~\ref{defn:construction}, page~\pageref{defn:construction}\\
constructs & & a construction constructs the identified vertex; definition~\ref{defn:construction}, page~\pageref{defn:construction}\\
construct & & the identified vertex of a construction; definition~\ref{defn:construction}, page~\pageref{defn:construction}\\
trivial \par construction & & a construction containing exactly one vertex; definition~\ref{defn:construction}, page~\pageref{defn:construction}\\
basic construction  & & a construction containining exactly one configurator; definition~\ref{defn:construction}, page~\pageref{defn:construction}\\
construction in $\rsystemn$ & & a construction that arises in the structure graph of representational system $\rsystemn$; definition~\ref{defn:setOfConstructions}, page~\pageref{defn:setOfConstructions}\\
well-formed trail & & a trail whose source and target are both tokens; definition~\ref{defn:wellformedTrail}, page~\pageref{defn:wellformedTrail}\\
trail extension \par sequence & & given a trail, $\trail$, in a construction, $\cpair$, its trail extension sequence comprises all trails in $\cgraphn$ that maximally extend $\trail$, ordered so that they respect the arrow indices; definition~\ref{defn:tes}, page~\pageref{defn:tes}\\
complete trail \par sequence & & the trail extension sequence of $[]_t$ (the empty trail that targets $t$), given a construction $\cpair$; definition~\ref{defn:CTS}, page~\pageref{defn:CTS}\\
source sequence & & given a sequence of trails, the induced sequence of sources of these trails; definition~\ref{defn:sourceSequence}, page~\pageref{defn:sourceSequence}\\
foundation \par token-sequence & & the source sequence of a construction's complete trail sequence; definition~\ref{defn:foundations}, page~\pageref{defn:foundations}\\
foundation \par type-sequence & & the sequence obtained by replacing tokens, in the foundation token-sequence, with their types; definition~\ref{defn:foundations}, page~\pageref{defn:foundations}\\
arity of a \par construction & & the length of the foundation token-sequence; definition~\ref{defn:foundations}, page~\pageref{defn:foundations}\\
generator & & given a construction, $\cpair$, another construction, $\genpair$, where $\cgraphn'$ is a subgraph of $\cgraphn$; definition~\ref{defn:generator}, page~\pageref{defn:generator}\\
$\trailSeq$-graph & & the graph containing precisely the trails in the trail sequence $\trailSeq$; definition~\ref{defn:tr-graph}, page~\pageref{defn:tr-graph}\\
induced construction graph of a trail & & the graph obtained from the trail extension sequence of the trail; definition~\ref{defn:inducedConstruction}, page~\pageref{defn:inducedConstruction}\\
induced construction of a trail & & the construction formed from the induced construction graph of the trail, with the source of the trail as its construct; definition~\ref{defn:inducedConstruction}, page~\pageref{defn:inducedConstruction}\\
induced construction sequence & &  the sequence of constructions induced by the trails in a sequence of trails; definition~\ref{defn:inducedConstruction}, page~\pageref{defn:inducedConstruction}\\
split of a \par construction & & a generator together with the induced construction sequence that arises from the generator's complete trail sequence; definition~\ref{defn:split}, page~\pageref{defn:split}\\
decomposition tree &  & a directed, arrow-labelled in-tree, where for each vertex, $v$, the incoming arrows are numbered from 1 to $|\inA{v}|$; definition~\ref{defn:decompositionTree}, page~\pageref{defn:decompositionTree}\\
$v$-induced \par decomposition tree &  & the maximal sub-tree of a decompostion tree with root $v$; definition~\ref{defn:decTreeSequence}, page~\pageref{defn:decTreeSequence}\\
decomposition-tree sequence of $v$ &  & the sequence of induced decomposition trees arising from the sources of $v$'s incoming arrows; definition~\ref{defn:decTreeSequence}, page~\pageref{defn:decTreeSequence}\\
$v$-ancestors &  & the vertices that appear in the trees occuring in the decompostion-tree sequence of $v$; definition~\ref{defn:decTreeSequence}, page~\pageref{defn:decTreeSequence}\\
decomposition of a construction, $\cpair$ &  & a decomposition tree whose root, $\treeroot{\decompositionn}$, is labelled by a generator, $\genpair$, of $\cpair$ and the $\treeroot{\decompositionn}$-induced decomposition trees are labelled by generators of the induced constructions in the $\genpair$-split of $\cpair$, and so forth; definition~\ref{defn:decomposition}, page~\pageref{defn:decomposition}\\
leaf construction-sequence of \par decomposition $\decompositionn$ &  & the sequence of leaves of $\decompositionn$, ordered using the arrow indices; definition~\ref{defn:leafConstructionSequence}, page~\pageref{defn:leafConstructionSequence}\\
construction of $\decompositionn$ &  & the construction whose decomposition is $\decompositionn$; definition~\ref{defn:constructionOfD}, page~\pageref{defn:constructionOfD}\\
$\cgraphn$ is a specialisation of $\pgraphn$ &  & $\cgraphn$ is label-isomorphic, up to the tokens in $\cgraphn$, to $\pgraphn$ and the label of each token in $\cgraphn$ is a subtype of that used by its identified vertex in $\pgraphn$; definition~\ref{defn:embeddingPattern}, page~\pageref{defn:embeddingPattern}\\
embedding of $\cgraphn$ \par into $\pgraphn$ &  & an isomorphism that establishes $\cgraphn$ is a specialisation of $\pgraphn$; definition~\ref{defn:embeddingPattern}, page~\pageref{defn:embeddingPattern}\\
pattern space, $\pspacen$, \par for a construction space, $\cspacen$ &  & $\pspacen$ is formed over the same type system and has the same constructor specification as $\cspacen$, and is required to ensure that every construction in $\cspacen$ is a specialisation of a construction in $\pspacen$; definition~\ref{defn:patternSpace}, page~\pageref{defn:patternSpace}\\
pattern for a construction space $\cspacen$ &  & a construction in a pattern space for $\cspacen$; definition~\ref{defn:matchAndDescribe}, page~\pageref{defn:matchAndDescribe}\\
$\cpair$ matches \par $\ppair$ &  & $\cpair$ specialises $\ppair$; definition~\ref{defn:matchAndDescribe}, definition~\ref{defn:matchAndDescribe}, page~\pageref{defn:matchAndDescribe}\\
$\ppair$ describes \par $\cpair$ &  & $\cpair$ matches $\ppair$; definition~\ref{defn:matchAndDescribe}, page~\pageref{defn:matchAndDescribe}\\
$\decompositionn$ matches $\pdecompositionn$ &  &  there is an isomorphism from $\decompositionn$ to $\pdecompositionn$ that ensures the arrow labels match and that for each vertex in $\decompositionn$, its label, $\cpair$, matches the label of its identified vertex in $\pdecompositionn$; definition~\ref{defn:decompositionMatchesPattern}, page~\pageref{defn:decompositionMatchesPattern}\\
embedding of $\decompositionn$ \par into $\pdecompositionn$ &  & an isomorphism that establishes $\decompositionn$ matches $\pdecompositionn$; definition~\ref{defn:decompositionMatchesPattern}, page~\pageref{defn:decompositionMatchesPattern}\\
$\pdecompositionn$ describes $\decompositionn$ &  & $\decompositionn$ matches $\pdecompositionn$; definition~\ref{defn:decompositionMatchesPattern}, page~\pageref{defn:decompositionMatchesPattern}\\
description of a \par construction space, $\cspacen$ &  & a set of decompositions of patterns for $\cspacen$; definition~\ref{defn:descriptionOfCSpace}, page~\pageref{defn:descriptionOfCSpace}\\
vertices of $\descriptionn$ &  & the set of vertices used in decompositions in description $\descriptionn$; definition~\ref{defn:verticesOfDecomposition}, page~\pageref{defn:verticesOfDecomposition}\\
description of \par representational system, $\rsystemn$ &  & a description of the universal construction space, $\uspace{\rsystemn}$; definition~\ref{defn:rsdescription}, page~\pageref{defn:rsdescription}\\
description of an \par inter-representational-system encoding &  & a triple comprising descriptions of each representational system and a set of patterns for the inter-property idenfitication space; definition~\ref{defn:irsedescription}, page~\pageref{defn:irsedescription}\\
complete \par description of a \par construction space, $\cspacen$&  & a description that ensures each construction in $\cspacen$ has a decomposition that matches some pattern in the description; definition~\ref{defn:complete}, page~\pageref{defn:complete}\\
$\descriptionn$-patterns &  & the set of patterns used in a description, $\descriptionn$; definition~\ref{defn:descriptionPatterns}, page~\pageref{defn:descriptionPatterns}\\
compact \par description of a \par construction space, $\cspacen$ &  & a description of $\cspacen$ that exploits a finite number of patterns, up to label-preserving isomorphism; definition~\ref{defn:compact}, page~\pageref{defn:compact}\\
decoupling &  & a decomposition where every construction is trival or basic; definition~\ref{defn:decoupling}, page~\pageref{defn:decoupling}\\
$\pdecompositionn$-canonical \par decomposition of \par $\cpair$&  & the decomposition of $\cpair$ that matches $\pdecompositionn$ and is obtained by replacing patterns that label vertices of $\pdecompositionn$ with constructions arising from splitting $\cpair$; definition~\ref{defn:deltaCanonical}, page~\pageref{defn:deltaCanonical}\\
property identifier &  & a set of constructions arising in an inter-property identification space; definition~\ref{defn:propertyIDandEn}, page~\pageref{defn:propertyIDandEn}\\
property enforcer &  & a set of patterns for an inter-property identification space; definition~\ref{defn:propertyIDandEn}, page~\pageref{defn:propertyIDandEn}\\
transformation constraint &  & a triple comprising two patterns and a property enforcer; definition~\ref{defn:st:tokenRelSpec}, page~\pageref{defn:st:tokenRelSpec}\\
structure graph of $\setofpatterns$ &  & the graph obtained by unioning the graphs in the patterns included in the set $\setofpatterns$; definition~\ref{defn:patternGraphOfPatterns}, page~\pageref{defn:patternGraphOfPatterns}\\
$\setofpatterns$ respectfully \par embeds into $\setofpatterns'$ &  & an embedding from the structure graph of $\setofpatterns$ into the structure graph of $\setofpatterns'$ such that each pattern in $\setofpatterns$  embeds into a pattern in $\setofpatterns'$; definition~\ref{defn:respectfulEmbedding}, page~\pageref{defn:respectfulEmbedding}\\
respectful \par embedding &  & a function that respecfully embeds $\setofpatterns$ into $\setofpatterns'$; definition~\ref{defn:respectfulEmbedding}, page~\pageref{defn:respectfulEmbedding}\\
satisfies &  & a relation between transformation constraints, $\tRelSpecn=(\ppaira,\ppairb,\setofpatternstc)$ and $\tRelSpecn'=(\ppairad,\ppairbd,\setofpatternstc')$ that arises when $\ppaira$ and $\ppairb$ embed into $\ppairad$ and $\ppairbd$, $\setofpatternstc$ embeds into $\setofpatternstc'$ and the union of the embeddings is an isomorphism; definition~\ref{defn:satisfies}, page~\pageref{defn:satisfies}\\

equivalent \par transformation \par constraints &  & a pair of transformation constraints that satisfy each other; definition~\ref{defn:equivalentTokenRelSpecs}, page~\pageref{defn:equivalentTokenRelSpecs}\\
constraint \par assignment &  & a partial function that maps pairs of vertices in pattern decompositions to transformation constraints in such a way that the patterns assigned to the vertices match those in the transformation constraint; definition~\ref{defn:constraintAssignment}, page~\pageref{defn:constraintAssignment}\\
satisfied by a \par pair of \par decompositions &  & a constraint assignment is satisfied by $(\decompositionn,\decompositionn')$ provided there exist property identifiers that ensure the transformation constraints assigned to pairs of vertices are satisfied; definition~\ref{defn:structuralTransformationDecomposition}, page~\pageref{defn:structuralTransformationDecomposition}\\
structural \par $\vertexRelations$-transformation of $\cpair$ &  & a construction, $\cpaird$, where there are (canonical) decompositions of $\cpair$ and $\cpaird$ that satisfy $\vertexRelations$; definition~\ref{defn:structuralTransformationConstruction}, page~\pageref{defn:structuralTransformationConstruction}\\
satisfied by a decomposition and a pattern decomposition &  & a constraint assignment is satisfied by $(\decompositionn,\pdecompositionn')$ provided there exist property identifiers that ensure the transformation constraints assigned to pairs of vertices are satisfied; definition~\ref{defn:DeltaPartialSatisfication}, page~\pageref{defn:DeltaPartialSatisfication}\\
partial \par $\vertexRelations$-transformation of $\cpair$ &  & a pattern, $\ppaird$ where there are (canonical) decompositions of $\cpair$ and $\ppaird$ that satisfy $\vertexRelations$; definition~\ref{defn:partialLDeltaTransform}, page~\pageref{defn:partialLDeltaTransform}\\
$(\delta,\delta')$-restriction of $\vertexRelations$ &  & the function obtained by restricting the domain of $\vertexRelations$ to the vertices of the $\delta$- and $\delta'$-induced decomposition trees; definition~\ref{defn:restrictionOfS}, page~\pageref{defn:restrictionOfS}\\
$\vertexRelations$ is valid at $(\delta,\delta'$) &  & whenever the $\delta'$-induced decomposition of $\ppaird$ is a construction in representational system $\rsystemn'$ it is also a structural $\vertexRelations_{(\delta,\delta')}$-transformation of the construction obtained from the $\delta$-induced decomposition of $\cpair$; definition~\ref{defn:validForVertices}, page~\pageref{defn:validForVertices}\\
$\vertexRelations$ is ancestor-valid at $(\delta,\delta'$) & & $\vertexRelations$ is valid for all ancestors of $\delta$ and $\delta'$; definition~\ref{defn:validForVertices}, page~\pageref{defn:validForVertices}\\
sound constraint \par assignment, $\vertexRelations$ &  & whenever $\vertexRelations$ is ancestor-valid, it is also valid; definition~\ref{defn:soundL}, page~\pageref{defn:soundL}\\
$\ppair$ leaf-\par instantiates $\pdecompositionn$ &  & all of the leaves of $\pdecompositionn$ are labelled by constructions in representational system $\rsystemn$; definition~\ref{defn:leafInstantiates}, page~\pageref{defn:leafInstantiates}\\
complete \par extension of $\ppair$ &  & a construction in representational system $\rsystemn$ that is obtained from $\ppair$ by replacing vertices in $\ppair$ with tokens in $\rsystemn$, whilst ensuring that patterns at the leaves of the pattern decomposition $\pdecompositionn$ remain unchanged; definition~\ref{defn:leafInstantiates}, page~\pageref{defn:leafInstantiates}\\
$\cpair$ and $\ppaird$ are valid at the leaves of $\pdecompositionn$ and $\pdecompositionn'$ &  & for all leaves, $l$ and $l'$, of $\pdecompositionn$ and $\pdecompositionn'$, resp., the constraint assignment $\vertexRelations$, is valid for $\cpair$ and $\ppaird$ at $(l,l')$; definition~\ref{defn:validAtLeaves}, page~\pageref{defn:validAtLeaves}\\
orderly trail \par sequence for $\cpair$ & & a sequence of trails, each of which targets $t$, ordered in such a way that the arrow indices are respected, such that for each configurator, $u$, if any trail in the sequence includes an arrow that targets $u$ then all arrows that target $u$ are in some trail in the orderly trail sequence; a definition~\ref{defn:orderlyTrailSequence}, page~\pageref{defn:orderlyTrailSequence}\\
generalised split & & a pair, built using an orderly trail sequence, $\otsequence$, comprising a generator arising from the graph of $\otsequence$, and the sequence of induced constructions obtained by extending the trails in $\otsequence$; definition~\ref{defn:genSplit}, page~\pageref{defn:genSplit}\\
$\otsequence$ identifies \par a generalised split  & & the generalised split to which $\otsequence$ gives rise is identified by $\otsequence$; definition~\ref{defn:genSplit}, page~\pageref{defn:genSplit}\\
\end{longtable}

%% file: appendixRepresentationalSystems.tex
\begin{lemma}\label{alem:pairwiseCompatibleTTSs}
Let $\tsystemn_1=\tsystemp{1}$, $\tsystemn_2=\tsystemp{2}$ and $\tsystemn_3=\tsystemp{3}$ be pairwise compatible term-type systems. Then $\tsystemn_1\cup \tsystemn_2$ is compatible with $\tsystemn_3$.
\end{lemma}

\begin{proof}
We know that $\tsystemn_1\cup \tsystemn_2$ is a type system, since $\tsystemn_1$ and $\tsystemn_2$ are compatible. For the compatibility of $\tsystemn_1\cup \tsystemn_2$ with $\tsystemn_3$, we show that the three conditions of definition~\ref{defn:typeSystem} hold. For notational convenience, we define %
\begin{displaymath}
\tsystem=(\tsystemp{1}\cup \tsystemp{2})\cup \tsystemp{3}.
\end{displaymath}
We see immediately that the first condition holds: $\types$ is a set. For the second condition, we must show that $\leq$ is a partial order on $\types$. Obviously, the union of three binary relations, $\leq_1$, $\leq_2$, and $\leq_3$, is a binary relation, namely $\leq$. Given that $\tsystemn_i$ and $\tsystemn_j$ are pairwise compatible,  $\leq_i\!\cup\! \leq_j$ is a partial order for any $1\leq i,j\leq 3$. Let $\tau$, $\tau'$ and $\tau''$ be types in $\types$. For reflexivity, we show that $\tau\leq \tau$. Well, $\tau\in \types_i$, for some $i$, so $\tau\leq_i\tau$ which implies that $\tau\leq \tau$. Hence, $\leq$ is reflexive. For anti-symmetry, assume that $\tau\leq \tau'$ and $\tau'\leq \tau$. Then $\tau\leq_i\tau'$ and $\tau'\leq_j\tau$ for some $1\leq i,j\leq 3$. Since we have $\tau\leq_i\tau'$ and $\tau'\leq_j\tau$ we know $\tau\leq_i\!\cup\! \leq_j\tau'$ and $\tau'\leq_i\!\cup\! \leq_j\tau$. Using the fact that $\leq_i\!\cup\! \leq_j$ is a partial order, and thus anti-symmetric, we deduce that $\tau=\tau'$, as required. Hence $\leq$ is anti-symmetric. Lastly, we must show that $\leq$ is transitive. Suppose that $\tau\leq \tau'$ and $\tau'\leq \tau''$. Then $\tau\leq_i\tau'$ and $\tau'\leq_j\tau''$ for some $1\leq i,j\leq 3$. Then we know that $\tau\leq_i\!\cup\! \leq_j\tau'$ and $\tau'\leq_i\!\cup\! \leq_j\tau''$. Using the fact that $\leq_i\!\cup\! \leq_j$ is a partial order, and thus transitive, we deduce that $\tau\leq_i\!\cup\! \leq_j\tau''$. Since $\leq_i\!\cup\! \leq_j\;\subseteq\; \leq$, it follows that $\tau\leq \tau''$. Hence $\leq$ is transitive. Therefore, $\leq$ is a partial order and we have established that $(\types,\leq)$ is a type system. Hence condition 2 of definition~\ref{defn:typeSystem} holds. Therefore, $(\tsystemn_1\cup \tsystemn_2)\cup \tsystemn_3$ is a type system.
\end{proof}

\begin{lemma}\label{alem:pairwiseCompatibleCSformCS}
Let\, $\cspacen_{1}$, $\cspacen_2$ and\, $\cspacen_3$ be pairwise com\-pa\-ti\-ble construction spaces. Then\, $\cspacen_1\cup \cspacen_2$ is com\-pa\-ti\-ble with $\cspacen_3$.
\end{lemma}

\begin{proof}
Firstly, assume that $\cspacen_i=\cspacep{i}$, is a construction space for type system $\tsystemn_i=\tsystemp{i}$, given $\cspecificationn_i=\cspecificationp{i}$, and $\graphn_i=\graphp{i}$, for $1\leq i \leq 3$. For notational convenience, we define %
\begin{displaymath}
\tsystem=(\tsystemp{1}\cup \tsystemp{2})\cup \tsystemp{3}.
\end{displaymath}

Similarly, $\constructors=(\constructors_1\cup \constructors_2)\cup \constructors_3$, $\spec=(\spec_1\cup \spec_2)\cup \spec_3$, and $\graphn=(\graphn_1\cup \graphn_2)\cup \graphn_3$. The proof has three main parts, reflecting definition~\ref{defn:constructionSpace}:
\begin{enumerate}
\item[(a)] show, given $\tsystemn_1$, $\tsystemn_2$, and $\tsystemn_3$ are pairwise compatible, that $\tsystemn=(\tsystemn_1\cup \tsystemn_2)\cup \tsystemn_3$ is a type system,
\item[(b)] show that $\cspecificationn=(\cspecificationn_1\cup \cspecificationn_2)\cup \cspecificationn_3$ is a constructor specification for $\tsystemn$, and
\item[(c)] show that $\graphn=(\graphn_1\cup \graphn_2)\cup \graphn_3$ is a structure graph for $\tsystemn$ given $\cspecificationn$.
\end{enumerate}
We proceed with each case in turn.
\begin{enumerate}
\item[(a)] Part (a) is established by lemma~\ref{lem:pairwiseCompatibleTTSs}.
\item[(b)] \textit{Show that $\cspecificationn=(\cspecificationn_1\cup \cspecificationn_2)\cup \cspecificationn_3$ is a constructor specification.} We know that $\cspecificationn_1\cup \cspecificationn_2$ is a cons\-truc\-tor specification, since $\cspacen_1$ and $\cspacen_2$ are compatible. Our task is to show that $\cspecificationn=(\constructors, \spec)$ is a constructor specification for $\tsystemn$. With reference to definition~\ref{defn:constructionSpecification}, we must show that $\constructors$ is disjoint from $\types$. Let $c\in \constructors$. Then immediately we have it that $c\in \constructors_i$, for some $1\leq i \leq 3$ and, since $\cspecificationn_i$ is a construction space, $c\not \in \types_i$. In particular, this tells us that for any $1\leq i\leq 3$,  $c\not \in \types_i$. Thus, $c\not\in \types$ and we deduce that $\constructors$ is disjoint from $\types$. Next, we show that $\spec$ is a function. Since we know that each $\spec_i$ is a function, our task is to show that for any constructor, $c$, if $c\in \constructors_i$ and $c\in \constructors_j$ then $\spec_i(c)=\spec_j(c)$. This follows immediately from the pairwise compatibility of the three construction spaces. Hence $\cspecificationn=(\constructors,\spec)$ is a constructor specification for $\tsystemn$.
\item[(c)] \textit{Show that $\graphn$ is a structure graph for $\tsystemn$ given $\cspecificationn$.}  We know that $\graphn_1\cup \graphn_2$ is a structure graph for $\tsystemn_1\cup \tsystemn_2$, since $\cspacen_1$ and $\cspacen_2$ are compatible. Our task is to show that $\graphn=(\graphn_1\cup \graphn_2)\cup \graphn_3$ is a structure graph for $\tsystemn$ given $\cspecificationn$. Firstly, we show that $\graphn$ is a directed labelled bipartite graph.
    \begin{enumerate}
        \item[-] Setting $\pa=(\pa_1\cup \pa_2)\cup \pa_3$, $\pb=(\pb_1\cup \pb_2)\cup \pb_3$, $\arrows=(\arrows_1\cup \arrows_2)\cup \arrows_3$ and $\incVert = \incVert_1\cup \incVert_2\cup \incVert_3$, we show that $(\pa,\pb,\arrows,\incVert)$ is a directed bipartite graph. It is trivial that $\pa\cap \pb=\emptyset$, since the three construction spaces are pairwise compatible. We now show that $\incVert$ is a function with signature  $\incVert\colon \arrows\to (\pa\times \pb) \cup (\pb\times \pa)$.
        To do so, we must show that if an arrow, $a$, is in $A_i$ and $A_j$, where $1\leq i,j \leq 3$, then $\incVert_i(a)=\incVert_j(a)$. This immediately follows since the construction spaces are pairwise compatible, since $\incVert_i\cup \incVert_j$ is a function. Hence $(\pa,\pb,\arrows,\incVert)$ is a directed bipartite graph.
        \item[-] Setting $\arrowl=(\arrowl_1\cup \arrowl_2)\cup \arrowl_3$, it can similarly be shown that $\arrowl$ is a func\-tion that labels arrows with labels in $\mathbb{N}$, again due to pairwise compatibility.
        \item[-] Setting $\tokenl=(\tokenl_1\cup \tokenl_2)\cup \tokenl_3$, it can similarly be shown that $\tokenl$ is a function that labels tokens with types in $\types$, again due to pairwise compatibility.
        \item[-] Setting $\consl=(\consl_1\cup \consl_2)\cup \consl_3$, the argument is similar.
        \end{enumerate}
        Therefore, $\graphn$ is a  directed labelled bipartite graph. Our next task is to show that definition~\ref{defn:structureGraph} holds: for all $u\in \pb$, $\consl(u)$ is in $\constructors$ and $\neigh{u,\graphn}$ is a configuration of $\consl(u)$ given $\spec$.
        \begin{enumerate}
        \item[-] Well, $u\in \pb= (\pb_1\cup \pb_2)\cup \pb_3$, which implies that $\consl(u)\in (\constructors_1\cup \constructors_2)\cup \constructors_3=\constructors$, as required.
        \item[-] Consider now the graph $\neigh{u,\graphn}$. If $u$ occurs in only one of $\graphn_1$, $\graphn_2$ and $\graphn_3$ then it trivially holds that $\neigh{u,\graphn}$ is a configuration of $\consl(u)$ given $\spec$. Otherwise, $u$ occurs in $\graphn_i$ and $\graphn_j$, where $1\leq i,j\leq 3$. Without loss of generality, set $i=1$ and $j=2$. We know that $\neigh{u,\graphn_1}$ and $\neigh{u,\graphn_2}$ are configurations of $\consl_1(u)$ and $\consl_2(u)$ in $\graphn_1$ and $\graphn_2$ respectively.
            By pairwise compatibility, $\consl_1(u)=\consl_2(u)=\consl(u)$. We show that $\neigh{u,\graphn_1}=\neigh{u,\graphn_2}$. Now,
            \begin{equation}\label{eq:nhG1G2}
            \neigh{u,\graphn_1}\cup \neigh{u,\graphn_2} = \neigh{u,\graphn_1\cup\graphn_2}
            \end{equation}
            and, since $\graphn_1\cup\graphn_2$ is a structure graph, $\neigh{u,\graphn_1\cup\graphn_2}$ is a configuration of $\consl(u)$ given $\spec_1\cup \spec_2$. Since $\cspacen_1$ and $\cspacen_2$ are pairwise compatible, we know that $\spec_1(\consl(u))=\spec_2(\consl(u))$. Since $\neigh{u,\graphn_1\cup\graphn_2}$ is a configuration and $(\spec_1\cup \spec_2)(\consl(u))=\spec_1(\consl(u))=\spec_2(\consl(u))$, it follows that \ref{eq:nhG1G2} holds. If $u$ occurs in the third structure graph, $\graphn_3$, then it can similarly be shown that $\neigh{u,\graphn_2}=\neigh{u,\graphn_3}$, otherwise $\neigh{u,\graphn_3}=\emptyset$. Therefore,
            \begin{eqnarray*}
             \neigh{u,\graphn} & = &  \neigh{u,\graphn_1}\cup \neigh{u,\graphn_2} \cup \neigh{u,\graphn_3} \\
             &  = & \neigh{u,\graphn_1}.
            \end{eqnarray*}
            Hence $\neigh{u,\graphn}$ is a configuration of $\consl(u)$ given $\spec$, as required.
         \end{enumerate}
Thus, definition~\ref{defn:structureGraph} holds. Therefore, $\graphn$ is a structure graph.
\end{enumerate}
We deduce that $\cspace$ is a construction space for $\tsystem$, as required. Hence, $\cspacen_1\cup \cspacen_2$ is compatible with $\cspacen_3$.
\end{proof}

\begin{lemma}\label{alem:leastUpperBoundExtensionExists}
Every type system, $\tsystemn=\tsystem$, has an upper bound extension.
%Let $\tsystemn=\tsystem$ be a type system with meta-property system $\mpsystemn$. Then $\tsystemn$ given $\mpsystemn$ has a least upper bound extension.
\end{lemma}

\begin{proof}
Trivially, we can always choose a set of fresh types, $\types'$, for which the bijection $f\colon \subtypeSets(\tsystemn)\to \types'$ exists. All that remains is to show that $\leq'$, as in definition~\ref{defn:intertionofLUB}, is a partial order. Since it is given that $\leq'$ is reflexive, we show that $\leq'$ is anti-symmetric and transitive, focusing on the latter first. Let $\tau_1$, $\tau_2$ and $\tau_3$ be types in $\types\cup \types'$ such that $\tau_1\leq'\tau_2$ and $\tau_2\leq' \tau_3$. If all of $\tau_1$, $\tau_2$ and $\tau_3$ are in $\types$ then it is trivial that $\tau_1\leq'\tau_3$, since $\leq$ is a subset of $\leq'$ and $\leq$ is transitive. Noting that, for all $\tau'$ in $\types'$, if $\tau'\leq' \tau_i$ then $\tau_i=\tau'$, the only remaining non-trivial case is when $\tau_3\in \types'$ and $\tau_1,\tau_2\in \types$. Consider, now, $f^{-1}(\tau_3)$.  Since $\tau_2\leq'\tau_3$ and $\tau_3$ is a fresh type, it must be that $\tau_2\in f^{-1}(\tau_3)$. Now, $\tau_1\leq \tau_2$ implies that $\tau_1\in f^{-1}(\tau_3)$, because $f^{-1}(\tau_3)$ is the subtype closure of some set of types. Therefore, $\tau_1\leq'\tau_3$.  Hence $\leq'$ is transitive.

For antisymmetry, let $\tau_1$ and $\tau_2$ be types in $\types\cup \types'$ such that $\tau_1\leq' \tau_2$ and $\tau_2\leq' \tau_1$.  If  $\tau_1$ or $\tau_2$ are in $\types'$ then it is trivial that $\tau_1=\tau_2$. The only remaining case is when $\tau_1$ and $\tau_2$ are both in $\types$. Assume that $\tau_1\neq \tau_2$.  Since $\leq$ is antisymmetric, we know that $\tau_1\leq \tau_2$ and $\tau_2\leq \tau_1$ implies that $\tau_1=\tau_2$. Then it must be that either $\tau_1\leq' \tau_2$ and $\tau_1\not\leq \tau_2$ or $\tau_2\leq' \tau_1$ and $\tau_2\not\leq \tau_1$. But all of the relations, $\tau\leq'\tau'$, that are in $\leq'\backslash \leq$ ensure that $\tau'$ is in $\types'$. In turn, this implies that either $\tau_1\in \types'$ or $\tau_2\in \types'$. Since $\types'$ is a set of fresh types, either $\tau_1$ or $\tau_2$ is not in $\types$ which is a contradiction, because in this case we asserted that both $\tau_1$ or $\tau_2$ are in $\types$. Thus, $\leq'$ is antisymmetric and, hence, a partial order. Therefore, $\tsystemn'=(\types\cup \types', \leq')$ is a type system. Hence, $\tsystemn$ has an upper bound extension.
\end{proof} 

%% file: appendixGeneralisedSplits.tex
Our approach to proving theorem~\ref{thm:splitPreservesCompleteTrailSequence} -- that a split, $\csplit$, of construction, $\cpair$, allows the ready recreation of $\cpair$'s complete trail sequence -- is to use an induction argument. At first thought, one might expect this induction to be over the number of configurators in the split's generator. However, such an approach is not straightforward. This is because each trail in $\ctsequence{\cgraphn',t}$ gives rise to an induced construction but adding the neighbourhood, $\neigh{u}$, of a configurator, $u$, to $\genpair$ -- for an inductive step -- can have a non-trivial impact on the complete trail sequence. Therefore, we exploit a more general approach for our proof strategy, that relies on \textit{orderly trail sequences}, which we will illustrate shortly by example. Importantly, unlike complete trail sequences, we can extend a \textit{single} trail in an orderly trail sequence, $\otsequence$, using $\neigh{u}$ and obtain another orderly trail sequence, $\otsequence'$. Ultimately, this will allow us to define \textit{generalised splits}. Subsequently, we will be able to prove the desired result concerning the derivability of $\ctsequence{\cgraphn,t}$ from the induced constructions:  theorem~\ref{thm:splitPreservesCompleteTrailSequence} follows from theorem~\ref{athm:GenSplitPreservesCompleteTrailSequence}, which focuses on generalised splits.

\subsection{Motivating Example: Adding a Configurator to a Generator}\label{sec:app:ME:constructions}

We now set out to illustrate the non-trivial impact of adding a configurator, and its neighbourhood, to a generator on complete trail sequences. We appeal to the construction, $\cpair$, given below left, and one of its generators, $\genpair$, below middle:
\begin{center}
\adjustbox{scale=\myscale}{
\begin{tikzpicture}[construction]
\node[term={}] (v) at (3.2,4.2) {$t$};
\node[above left = -0.2cm and 0.8cm of v] () {$(g,t)$};
\node[term={}] (v1) at (2.4,2.2) {$t_1$};
\node[term={}] (v2) at (4,2.2) {$t_2$};
\node[constructor={$c$}] (u) at (3.2,3.2) {$u$};
\node[constructor={$c_1$}] (u1) at (2.4,1.2) {$u_1$};
\node[term={}] (v3) at (1.6,0.2) {$t_3$};
\node[constructor={$c_3$}] (u3) at (1.6,-0.8) {$u_3$};
\node[term={},\termC,\termN] (v5) at (1.6,-1.8) {$t_5$};
\node[term={}] (v4) at (3.2,0.2) {$t_4$};
\node[constructor={$c_4$}] (u4) at (3.2,-0.8) {$u_4$};
\node[constructor={$c_2$}] (u2) at (4,1.2) {$u_2$};
\node[term={},\termC,\termN] (v6) at (3.2,-1.8) {$t_6$};
\path[->]
(u) edge[bend right = 0] (v)
(v1) edge[bend right = -10] node[index label] {1} (u)
(v2) edge[bend right = 10] node[index label] {2} (u)
(u1) edge[bend right = 0] (v1)
(v3) edge[bend right = -10] node[index label] {1} (u1)
(v4) edge[bend right = 10] node[index label] {2} (u1)
(u2) edge[bend right = 0] (v2)
(u4) edge[bend right = 0] (v4)
(v4) edge[bend right = -10] node[index label] {1} (u2)
(v6) edge[bend right = 0] node[index label] {1} (u4)
(v5) edge[bend right = 0] node[index label] {1} (u3)
(u3) edge[bend right = 0] (v3);
\end{tikzpicture}
}
\hspace{1.2cm}
\adjustbox{scale=\myscale}{
\begin{tikzpicture}[construction]
\node[term={}] (v) at (3.2,4.2) {$t$};
\node[above left = -0.2cm and 0.8cm of v] () {$(g',t)$};
\node[term={},\termC,\termN] (v1) at (2.4,2.2) {$t_1$};
\node[term={}] (v2) at (4,2.2) {$t_2$};
\node[constructor={$c$}] (u) at (3.2,3.2) {$u$};
%\node[constructor={$c_1$}] (u1) at (2.4,1.2) {$u_1$};
%\node[term={}] (v3) at (1.6,0.2) {$t_3$};
%\node[constructor={$c_3$}] (u3) at (1.6,-0.8) {$u_3$};
%\node[term={},\termC,\termN] (v5) at (1.6,-1.8) {$t_5$};
\node[term={}] (v4) at (3.2,0.2) {$t_4$};
\node[constructor={$c_4$}] (u4) at (3.2,-0.8) {$u_4$};
\node[constructor={$c_2$}] (u2) at (4,1.2) {$u_2$};
\node[term={},\termC,\termN] (v6) at (3.2,-1.8) {$t_6$};
\path[->]
(u) edge[bend right = 0] (v)
(v1) edge[bend right = -10] node[index label] {1} (u)
(v2) edge[bend right = 10] node[index label] {2} (u)
%(u1) edge[bend right = 0] (v1)
%(v3) edge[bend right = -10] node[index label] {1} (u1)
%(v4) edge[bend right = 10] node[index label] {2} (u1)
(u2) edge[bend right = 0] (v2)
(u4) edge[bend right = 0] (v4)
(v4) edge[bend right = -10] node[index label] {1} (u2)
(v6) edge[bend right = 0] node[index label] {1} (u4)
%(v5) edge[bend right = 0] node[index label] {1} (u3)
%(u3) edge[bend right = 0] (v3)
;
\end{tikzpicture}
}
\hspace{0.8cm}
\adjustbox{scale=\myscale}{
\begin{tikzpicture}[construction]
\node[term={}] (v) at (3.2,4.2) {$t$};
\node[above left = -0.2cm and 0.2cm of v] () {$(g'\cup\neigh{u_1},t)$};
\node[term={}] (v1) at (2.4,2.2) {$t_1$};
\node[term={}] (v2) at (4,2.2) {$t_2$};
\node[constructor={$c$}] (u) at (3.2,3.2) {$u$};
\node[constructor={$c_1$}] (u1) at (2.4,1.2) {$u_1$};
\node[term={},\termC,\termN] (v3) at (1.6,0.2) {$t_3$};
%\node[constructor={$c_3$}] (u3) at (1.6,-0.8) {$u_3$};
%\node[term={},\termC,\termN] (v5) at (1.6,-1.8) {$t_5$};
\node[term={}] (v4) at (3.2,0.2) {$t_4$};
\node[constructor={$c_4$}] (u4) at (3.2,-0.8) {$u_4$};
\node[constructor={$c_2$}] (u2) at (4,1.2) {$u_2$};
\node[term={},\termC,\termN] (v6) at (3.2,-1.8) {$t_6$};
\path[->]
(u) edge[bend right = 0] (v)
(v1) edge[bend right = -10] node[index label] {1} (u)
(v2) edge[bend right = 10] node[index label] {2} (u)
(u1) edge[bend right = 0] (v1)
(v3) edge[bend right = -10] node[index label] {1} (u1)
(v4) edge[bend right = 10] node[index label] {2} (u1)
(u2) edge[bend right = 0] (v2)
(u4) edge[bend right = 0] (v4)
(v4) edge[bend right = -10] node[index label] {1} (u2)
(v6) edge[bend right = 0] node[index label] {1} (u4)
%(v5) edge[bend right = 0] node[index label] {1} (u3)
%(u3) edge[bend right = 0] (v3)
;
\end{tikzpicture}
}
\end{center}
The complete trail sequence of $\genpair$ is
\begin{eqnarray*}
\ctsequence{\cgraphn',t} & = & [t_1\arrow[1] u \arrow t,\\
                        & &  t_6\arrow[1] u_4\arrow t_4\arrow[1] u_2\arrow t_2\arrow[2] u \arrow t].
\end{eqnarray*}
If we add the neighbourhood of the constructor $u_1$ to $\genpair$ then we create a new generator, $(\cgraphn'\cup \neigh{u_1},t)$, above right. How does $\ctsequence{\cgraphn',t}$ compare to $\ctsequence{\cgraphn'\cup \neigh{u_1},t}$? Well, we have
\begin{eqnarray*}
\ctsequence{\cgraphn'\cup \neigh{u_1},t} & = & [t_3\arrow[1] u_1\arrow t_1\arrow[1] u \arrow t, \\
     & & t_6\arrow[1] u_4\arrow t_4\arrow[2] u_1\arrow t_1\arrow[1] u \arrow t,\\
& & t_6\arrow[1] u_4\arrow t_4\arrow[1] u_2\arrow t_2\arrow[2] u \arrow t].
\end{eqnarray*}
What is important to note here is that the output of $u_1$ is $t_1$, yet we cannot take the trail $t_1\arrow[1] u \arrow t$, sourced on $t_1$, in $\ctsequence{\cgraphn',t}$ and replace it with the pair of trails
\begin{displaymath}
t_3\arrow[1] u_1\arrow t_1\arrow[1] u \arrow t \qquad \textup{and} \qquad t_4\arrow[2] u_1\arrow t_1\arrow[1] u \arrow t,
\end{displaymath}
reflecting the inputs $t_3$ and $t_4$ to $u_1$, to yield $\ctsequence{\cgraphn'\cup \neigh{u_1},t}$: looking at the generator $\genpair$ and, in particular, how it changes \textit{locally} when $\neigh{u_1}$ is added to it \textit{does not} allow the new complete trail sequence to be derived. This is because $t_4\arrow[2] u_1\arrow t_1\arrow[1] u \arrow t$ is extendable in  $(\cgraphn'\cup \neigh{u_1},t)$ and, so, not in $\ctsequence{\cgraphn'\cup \neigh{u_1},t}$. One must consider not only the neighbourhood of the new configurator, but the \textit{entire} structure of the newly created generator, $(\cgraphn'\cup \neigh{u_1},t)$.  This motivates why we take a more general approach to proving theorem~\ref{thm:splitPreservesCompleteTrailSequence} which asserts that $\ctsequence{\cgraphn,t}$ is directly obtainable from the generator and the complete trail sequences of the induced constructions in any split.

This more general approach, as indicated above, uses orderly trail sequences where we \textit{are able to consider local changes only}. The sequence $\ctsequence{\cgraphn',t}$ is an orderly trail sequence for $\cpair$, as is $\otsequence =[\trail_1,\trail_2,\trail_3]$ where
\begin{eqnarray*}
\trail_1 & = & t_3\arrow[1] u_1\arrow t_1\arrow[1] u \arrow t, \\
\trail_2 & = & t_4\arrow[2] u_1\arrow t_1\arrow[1] u \arrow t, \\
\trail_3 & = & t_6\arrow[1] u_4\arrow t_4\arrow[1] u_2\arrow t_2\arrow[2] u \arrow t;
\end{eqnarray*}
notice that $\otsequence$ is obtained from  $\ctsequence{\cgraphn',t}$  by replacing $t_1\arrow[1] u \arrow t$ with
$$t_3\arrow[1] u_1\arrow t_1\arrow[1] u \arrow t \qquad \textup{and} \qquad t_4\arrow[2] u_1\arrow t_1\arrow[1] u \arrow t.$$
The three trails in $\otsequence$ induce constructions as follows:
\begin{center}
\adjustbox{scale=\myscale,valign=t}{%
\begin{tikzpicture}[construction]
\node[termrep] (v3) at (1.6,0.2) {$t_3$};
\node[above left = -0.2cm and 0.2cm of v3] {$\ic(\trail_1)$};
\node[constructor={$c_3$}] (u3) at (1.6,-0.8) {$u_3$};
\node[termrep,\termC,\termN] (v5) at (1.6,-1.8) {$t_5$};
\path[->]
(v5) edge[bend right = 0] node[index label] {1} (u3)
(u3) edge[bend right = 0] (v3);
\end{tikzpicture}}
\hspace{1.5cm}
\adjustbox{scale=\myscale,valign=t}{%
\begin{tikzpicture}[construction]
\node[termrep] (v4) at (3.2,0.2) {$t_4$};
\node[above left = -0.2cm and 0.2cm of v4] {$\ic(\trail_2)$};
\node[constructor={$c_4$}] (u4) at (3.2,-0.8) {$u_4$};
\node[termrep,\termC,\termN] (v6) at (3.2,-1.8) {$t_6$};
\path[->]
(u4) edge[bend right = 0] (v4)
(v6) edge[bend right = 0] node[index label] {1} (u4)
;
\end{tikzpicture}}
\hspace{1.5cm}
\adjustbox{scale=\myscale,valign=t}{%
\begin{tikzpicture}[construction]
\node[termrep,\termC,\termN] (v6) at (3.2,-1.8) {$t_6$};
\node[above left = -0.2cm and 0.2cm of v6] {$\ic(\trail_3)$};
\end{tikzpicture}}
\end{center}

Effectively, what we have just witnessed is that extending the trail $t_1\arrow[1] u \arrow t$ using $\neigh{u_1}$ gives rise to two trails, since $u_1$ has two inputs, $t_3$ and $t_4$:
\begin{multline*}
[t_3\arrow[1] u_1,t_4\arrow[2] u_1] \triangleleft (u_1\arrow t_1\oplus t_1\arrow[1] u \arrow t)= \\ [t_3\arrow[1] u_1\arrow t_1\arrow[1] u \arrow t, t_4\arrow[2] u_1\arrow t_1\arrow[1] u \arrow t].
\end{multline*}
Correspondingly,{\pagebreak} the induced construction (below left) arising from  $t_1\arrow[1] u \arrow t$ bifurcates into two induced constructions (below middle and right) by simply deleting $u_1$ and its output, \begin{samepage}$t_1$:

\centerline{
\adjustbox{scale=\myscale}{%
\begin{tikzpicture}[construction]
\node[term={}] (v1) at (2.4,2.2) {$t_1$};
\node[constructor={$c_1$}] (u1) at (2.4,1.2) {$u_1$};
\node[term={}] (v3) at (1.6,0.2) {$t_3$};
\node[constructor={$c_3$}] (u3) at (1.6,-0.8) {$u_3$};
\node[term={},\termC,\termN] (v5) at (1.6,-1.8) {$t_5$};
\node[term={}] (v4) at (3.2,0.2) {$t_4$};
\node[constructor={$c_4$}] (u4) at (3.2,-0.8) {$u_4$};
\node[term={},\termC,\termN] (v6) at (3.2,-1.8) {$t_6$};
\path[->]
(u1) edge[bend right = 0] (v1)
(v3) edge[bend right = -10] node[index label] {1} (u1)
(v4) edge[bend right = 10] node[index label] {2} (u1)
(u4) edge[bend right = 0] (v4)
(v6) edge[bend right = 0] node[index label] {1} (u4)
(v5) edge[bend right = 0] node[index label] {1} (u3)
(u3) edge[bend right = 0] (v3);
\end{tikzpicture}
}
\hspace{50pt}
\adjustbox{scale=\myscale}{
\begin{tikzpicture}[construction]
\node[term={}] (v3) at (1.6,0.2) {$t_3$};
\node[constructor={$c_3$}] (u3) at (1.6,-0.8) {$u_3$};
\node[term={},\termC,\termN] (v5) at (1.6,-1.8) {$t_5$};
\path[->]
(v5) edge[bend right = 0] node[index label] {1} (u3)
(u3) edge[bend right = 0] (v3);
\end{tikzpicture}
}
\hspace{50pt}
\adjustbox{scale=\myscale}{
\begin{tikzpicture}[construction]
\node[term={}] (v4) at (3.2,0.2) {$t_4$};
\node[constructor={$c_4$}] (u4) at (3.2,-0.8) {$u_4$};
\node[term={},\termC,\termN] (v6) at (3.2,-1.8) {$t_6$};
\path[->]
(u4) edge[bend right = 0] (v4)
(v6) edge[bend right = 0] node[index label] {1} (u4);
\end{tikzpicture}
}
}\end{samepage}
This kind of well-behaved `move'
\begin{enumerate}
\item[-] extending a trail, $\trail$, in an orderly trail sequence using $\neigh{u}$, to produce a new orderly trail sequence, and
\item[-] deleting $u$ and its output from $\ic(\trail)$, to produce a sequence of induced constructions
\end{enumerate}
 gives us the opportunity to exploit a proof by induction approach, using orderly trail sequences, on route to a proof of theorem~\ref{thm:splitPreservesCompleteTrailSequence}.

\subsection{Orderly Trail Sequences and Generalised Splits}

We now formally define an orderly trail sequence.

\begin{definition}\label{defn:orderlyTrailSequence}
Let $\cpair$ be a construction. An \textit{orderly trail sequence} for $\cpair$ is a sequence of trails, $\otsequence$, that satisfies the following:
\begin{enumerate}
\item  $\otsequence=[[]_{t}]$, or
\item there exists an orderly trail sequence, $\otsequence'=[\trail_1,...,\trail_m]$, for $\cpair$ containing an extendable trail, $\trail_i$, where, given the arrow $a$ that targets $\source{\trail_i}$ and $\inputsA{\sor{a}}=[a_1,...,a_n]$,
        \begin{displaymath}
          \otsequence=[\trail_1,...,\trail_{i-1}] \oplus \big([[a_1],...,[a_n]]\triangleleft ([a] \oplus \trail_{i})\big) \oplus [\trail_{i+1}  ,...,\trail_m].
        \end{displaymath}

\end{enumerate}
\end{definition}

One property of orderly trail sequences is that they only contain well-formed trails. This is important since we wish to use these trails to form induced constructions.

\begin{lemma}\label{lem:orderlyTrailSequenceIsWellFormed}
Let $\cpair$ be a construction with orderly trail sequence $\otsequence$. Then every trail, $\trail$, in $\otsequence$, is well-formed.
\end{lemma}

\begin{proof}[Proof sketch]
The proof is by induction. The base case holds trivially, since $[]_t$ is obviously well-formed. The trail, $\trail_i$, in $\otsequence$ that is extended, in part (2) of definition~\ref{defn:orderlyTrailSequence}, is done so in such a way that resulting trails are well-formed, given that $\trail_i$ is well-formed.
\end{proof}

Now, given \textit{any} orderly trail sequence, $\otsequence$, for $\cpair$, we can convert it into a graph, $\trailToGraph{\otsequence}$, following definition~\ref{defn:tr-graph}. Notably, this graph together with $t$ forms a construction. We will take $(\trailToGraph{\otsequence},t)$ as a generator for a \textit{generalised split}, which we define momentarily. Referring to the example in section~\ref{sec:app:ME:constructions}, the orderly trail sequence
\begin{eqnarray*}
\otsequence & = & [\trail_1,\trail_2,\trail_3] \\
            & = & [t_3\arrow[1] u_1\arrow t_1\arrow[1] u \arrow t,\\
            & & t_4\arrow[2] u_1\arrow t_1\arrow[1] u \arrow t, \\
& & t_6\arrow[1] u_4\arrow t_4\arrow[1] u_2\arrow t_2\arrow[2] u \arrow t],
\end{eqnarray*}
that can be derived from $\cpair$, gives rise to the graph $\cgraphn'\cup \neigh{u_1}$. The construction $(\cgraphn'\cup \neigh{u_1},t)$ together with the induced construction sequence $[\ic(\trail_1),\ic(\trail_2), \ic(\trail_3)]$ form a generalised split. In this case, the generalised split is \textit{not} a split since
$$\otsequence\neq \ctsequence{\cgraphn'\cup \neigh{u_1},t}.$$
This example demonstrates the key difference between splits and generalised splits: in a split, the induced construction sequence is determined by $\cpair$ and the generator; in a generalised split, the induced construction sequence is determined by $\cpair$ and the specified orderly trail sequence, $\otsequence$. We now prove that the $\otsequence$-graph, namely $\trailToGraph{\otsequence}$, gives rise to a construction of $t$.

\begin{lemma}
Let $\cpair$ be a construction with orderly trail sequence $\otsequence$. Then $(\trailToGraph{\otsequence},t)$ is a construction.
\end{lemma}

\begin{proof}
It is given that $t$ is a token, since $\cpair$ is a construction. Clearly, $\trailToGraph{\otsequence}$ is finite and uni-structured since it is a subgraph of $\cgraphn$. Let $v$ be a vertex in $\trailToGraph{\otsequence}$. Then either $v=t$, in which case $[]_t$ is a trail in $\trailToGraph{\otsequence}$ from $v$ to $t$, or $v\neq t$. In the latter case, there is a trail, $\trail$, in $\otsequence$ that includes some arrow, $a$, where $\sor{a}=v$. By the definition of an orderly trail sequence for $\cpair$ the target of $\trail$ is $t$. Therefore, there is a trail in $\trailToGraph{\otsequence}$ from $v$ to $t$. Hence $(\trailToGraph{\otsequence},t)$ is a construction.
\end{proof}

We now use orderly trail sequences to define generalised splits. Recall that a generator's complete trail sequence is used to create the induced constructions in a split. Here, we  follow a similar process: we use an orderly trail sequence, $\otsequence$, to form induced constructions, as  illustrated above.

\begin{definition}\label{defn:genSplit}
Let $\cpair$ be a construction. A \textit{generalised split} of $\cpair$ is a pair, $\csplit$, such that there exists an orderly trail sequence, $\otsequence$, for $\cpair$ where
\begin{enumerate}
\item $\genpair=(\trailToGraph{\otsequence},t)$, and
\item  $\ics$ is the induced construction sequence obtained by extending the trails in $\otsequence$ in $\cpair$: $\ics=\icsequence{\otsequence,\cpair}$.
\end{enumerate}
We say that such an orderly trail sequence, $\otsequence$, \textit{identifies} the generalised split $\csplit$.
\end{definition}

\noindent An obvious question arises: are all splits also generalised splits? We answer this affirmatively in the next section.

\subsection{Splits are Generalised Splits}

In order to exploit generalised splits to prove our desired results on complete trail sequences (theorem~\ref{thm:splitPreservesCompleteTrailSequence}), foundation sequences (theorem~\ref{thm:splitPreservesfoundations}) and that the original construction can be obtained from the generator and the induced constructions (theorem~\ref{thm:SplitsCoverConstruction}), this section focuses on proving that every split is a generalised split. To begin, we will establish that, given any generator, $\genpair$ of construction $\cpair$,
\begin{enumerate}
\item[(a)] the complete trail sequence, $\ctsequence{\cgraphn',t}$, is an orderly trail sequence for $\cpair$, and
\item[(b)] the graph $\cgraphn'$ is the $\ctsequence{\cgraphn',t}$-graph.
\end{enumerate}
Lemma~\ref{lem:topDTSisOrderlyForConstruction} addresses (a) and corollary~\ref{cor:DTSforTopGeneratesTop} addresses (b). The proof strategy for lemma~\ref{lem:topDTSisOrderlyForConstruction} is to take $\ctsequence{\cgraphn',t}$ and reduce the trails in it until we obtain $[[]_t]$. This process can then be reversed to show that $\ctsequence{\cgraphn',t}$ is an orderly trail sequence. For example, taking the generator $\genpair$ for $\cpair$, as given in section~\ref{sec:app:ME:constructions}, we have
$$\ctsequence{\cgraphn',t}= [t_1\arrow[1] u \arrow t, t_6\arrow[1] u_4\arrow t_4\arrow[1] u_2\arrow t_2\arrow[2] u \arrow t].$$
We can iteratively reduce the length of the trails using neighbourhoods of configurators. For instance, we can remove the neighbourhood of $u_4$, specifically the arrows from $t_6$ to $u_4$ and from $u_4$ to $t_4$, from the second trail. This is the first reduction we make below, the second reduction is similar, removing the arrows in the neighbourhood of $u_2$. The last step reduces two trails, using the neighbourhood of $u$, which has two incoming arrows. We have:
\begin{eqnarray*}
\ctsequence{\cgraphn',v}& = & [t_1\arrow[1] u \arrow t, t_6\arrow[1] u_4\arrow t_4\arrow[1] u_2\arrow t_2\arrow[2] u \arrow t]  \\
                        & \mapsto & [t_1\arrow[1] u \arrow t, t_4\arrow[1] u_2\arrow t_2\arrow[2] u \arrow t] \\
                        & \mapsto & [t_1\arrow[1] u \arrow t,t_2\arrow[2] u \arrow t] \\
                        & \mapsto & [[]_t].
\end{eqnarray*}
We can reverse this process to show that $\ctsequence{\cgraphn',t}$ is an orderly trail sequence.

\begin{lemma}\label{lem:topDTSisOrderlyForConstruction}
Let $\cpair$ be a construction with generator $\genpair$. Then $\ctsequence{\cgraphn',t}$ is an orderly trail sequence for $\cpair$.
\end{lemma}

\begin{proof}
The complete trail sequence $\ctsequence{\cgraphn',t}$ is constructed by a sequence of recursive steps, each of which extends trails to make one or more new trails. Consider any given sequence of trail-construction steps used to form $\ctsequence{\cgraphn',t}$. Removing a recursive leaf from this sequence of construction steps strictly reduces the sum of the trail-lengths and replaces $[a_1,a]\oplus \trail, ..., [a_n,a]\oplus \trail$, where $\inputsA{a}=[a_1,...,a_n]$, with $\trail$. Notably, $\inputsA{a}$ is the same in $\cgraphn'$ and $\cgraphn$. Repeatedly applying this shortening of trails in $\ctsequence{\cgraphn',t}$ eventually yields $[[]_t]$. Reversing the process yields an orderly trail sequence for $\cpair$. Hence $\ctsequence{\cgraphn',t}$ is an orderly trail sequence for $\cpair$.
\end{proof}

Corollary~\ref{cor:DTSforTopGeneratesTop} establishes that for any construction, $\cpair$, with generator $\genpair$ the graph $\cgraphn'$ is the same as the $\ctsequence{\cgraphn',t}$-graph. Intuitively, this is because every arrow of $\cgraphn'$ must occur in some non-extendable trail, also in $\cgraphn'$, that targets $t$. Once we know all of the arrows in $\cgraphn'$ occur in some trail in $\ctsequence{\cgraphn',t}$, the result follows. Lemma~\ref{lem:constructionIsTES-graph} establishes this result in a more general setting: given any construction, $\cpair$, the graph $\cgraphn$ is the $\ctsequence{\cgraphn,t}$-graph.

\begin{lemma}\label{lem:constructionIsTES-graph}
Let $\cpair$ be a construction. Then $\cgraphn$ is the $\ctsequence{\cgraphn,t}$-graph, that is $\cgraphn=\trailToGraph{\ctsequence{\cgraphn,t}}$.
\end{lemma}

\begin{proof}
By the definition of $\trailToGraph{\ctsequence{\cgraphn,t}}$, we know that $\trailToGraph{\ctsequence{\cgraphn,t}}\subseteq \cgraphn$.
Given any arrow, $a$, that occurs in $\cgraphn$, there is a trail, $\trail$, from its source, $\source{\trail}$, to $t$. Therefore, $a$ occurs in some well-formed, non-extendable trail, $\trail'\oplus \trail$, that targets $t$. By lemma~\ref{lem:CESisComplete}, we deduce that $a$ occurs in $\tesequence{[]_t}=\ctsequence{\cgraphn,t}$. Thus, every arrow in $\cgraphn$ occurs in some trail in $\ctsequence{\cgraphn,t}$. It follows trivially that $\cgraphn\subseteq \trailToGraph{\ctsequence{\cgraphn,t}}$. Hence $\cgraphn=\trailToGraph{\ctsequence{\cgraphn,t}}.$
\end{proof}

\begin{corollary}\label{cor:DTSforTopGeneratesTop}
Let $\cpair$ be a construction with generator $(\cgraphn',t)$. Then $\cgraphn' = \trailToGraph{\ctsequence{\cgraphn',t}}$.
\end{corollary}

Summarising lemma~\ref{lem:topDTSisOrderlyForConstruction} and corollary~\ref{cor:DTSforTopGeneratesTop}, we now know that the complete trail sequence, $\ctsequence{\cgraphn',t}$, of a generator, $\genpair$, is an orderly trail sequence for $\cpair$ and that $\cgraphn'$ is the graph obtained from $\ctsequence{\cgraphn',t}$. These facts tell us that any generator, from which we can derive a split of $\cpair$, forms part of a generalised split. We now use this insight to prove that any split is a generalised split.

\begin{theorem}\label{thm:splitsAreGenSplits}
Let $\cpair$ be a construction with split $\csplit$. Then $\csplit$ is a generalised split identified by $\ctsequence{\cgraphn',t}$.
\end{theorem}

\begin{proof}
By lemma~\ref{lem:topDTSisOrderlyForConstruction}, $\ctsequence{\cgraphn',t}$ is an orderly trail sequence for $\cpair$. By corollary~\ref{cor:DTSforTopGeneratesTop}, we know that $\cgraphn'=\trailToGraph{\ctsequence{\cgraphn',t}}$. Therefore $\genpair=(\trailToGraph{\ctsequence{\cgraphn',t}},t)$,  which implies that $\genpair$ satisfies part 1 of definition~\ref{defn:genSplit}.  For part 2 of definition~\ref{defn:genSplit}, by the definition of a split, $\mathit{ics}=\icsequence{\ctsequence{\genpair},\cpair}$, which is exactly what we need $\mathit{ics}$ to be when forming a generalised split from the orderly trail sequence $\ctsequence{\genpair}$. Hence $\csplit$ is a generalised split identified by $\ctsequence{\cgraphn',t}$.
\end{proof}
Therefore, the results proved in the rest of section~\ref{sec:constructions} for generalised splits also apply to splits.

\subsection{Deriving Complete Trail Sequences from Splits}

Recall that we are aiming to show, given any \textit{split} of $\cpair$, that we can derive the complete trail sequence of $\cpair$ from the generator and the induced construction sequence, embodied in theorem~\ref{thm:splitPreservesCompleteTrailSequence}. We prove this result holds for \textit{generalised splits} in theorem~\ref{athm:GenSplitPreservesCompleteTrailSequence}. A key part of the proof strategy  establishes how we can derive $\tesequence{[]_t}$ from trail extension sequences that arise from trails in an orderly trail sequence, $\otsequence$, captured by theorem~\ref{thm:trailExtensionSequencesEssentiallyMatch}. The proof of theorem~\ref{thm:trailExtensionSequencesEssentiallyMatch} exploits our ability to readily use an induction argument over orderly trail sequences, extending trails using the neighbourhood of some configurator.

\begin{theorem}\label{thm:trailExtensionSequencesEssentiallyMatch}
Let $\cpair$ be a construction with orderly trail sequence $\otsequence=[\trail_1,...,\trail_m]$. Then
\begin{displaymath}
\tesequence{[]_t}= \big(\tesequence{\trail_1}\triangleleft \trail_1\big) \oplus \cdots \oplus \big(\tesequence{\trail_m}\triangleleft \trail_m\big).
\end{displaymath}
\end{theorem}

\begin{proof}
The proof is by induction on the depth of $\otsequence$ in the inductive construction of orderly trail sequences. For the base case, $\otsequence=[[]_t]$ and the result trivially holds. Assume that the result holds for all orderly trail sequences at depth $k$. Suppose that $\otsequence$ is at depth $k+1$. Then there exists an orderly trail sequence,
\begin{displaymath}
\otsequence'=[\trail_1,...,\trail_m],
\end{displaymath}
at depth $k$, where $\otsequence'$ contains a trail, $\trail_i$, that is extendable by arrow $a$ in $\cgraphn$ such that $\otsequence$ is obtained from $\otsequence'$, $\trail_i$, and $a$ as in definition~\ref{defn:orderlyTrailSequence}. Then, given $\inputsA{\sor{a}}=[a_1,...,a_n]$, we have
        \begin{displaymath}
        \otsequence=[\trail_1,...,\trail_{i-1}] \oplus \Big([[a_1],...,[a_n]]\triangleleft ([a] \oplus \trail_{i})\Big) \oplus [\trail_{i+1}  ,...,\trail_m].
        \end{displaymath}
       Thus, given that
        \begin{displaymath}
        [[a_1],...,[a_n]]\triangleleft \big([a] \oplus \trail_{i}\big)= [[a_1,a]\oplus \trail_i,...,[a_n,a]\oplus \trail_i],
        \end{displaymath}
        by our inductive assumption it is sufficient to show that
        \begin{multline*}
        \tesequence{\trail_i}\triangleleft\trail_i=\Big(\tesequence{[a_1,a]\oplus \trail_i}\triangleleft \big([a_1,a]\oplus \trail_i\big)\Big)\oplus \cdots \\ \oplus \Big(\tesequence{[a_n,a]\oplus \trail_i}\triangleleft\big([a_n,a]\oplus \trail_i\big)\Big) \qquad \textup{(*)}.
        \end{multline*}
        By definition~\ref{defn:DTS},
        \begin{displaymath}
        \tesequence{\trail_i}= \Big(\tesequence{[a_1,a]\oplus \trail_i}\triangleleft [a_1,a] \Big) \oplus \cdots \oplus \Big(\tesequence{[a_n,a]\oplus \trail_i}\triangleleft [a_n,a]\Big).
        \end{displaymath}
        Therefore,
        \begin{displaymath}
        \tesequence{\trail_i}\triangleleft \trail_i= \Big(\big(\tesequence{[a_1,a]\oplus \trail_i}\triangleleft [a_1,a] \big) \oplus \cdots \oplus \big(\tesequence{[an,a]\oplus \trail_i}\triangleleft [a_n,a]\big)\Big) \triangleleft\trail_i.
        \end{displaymath}
        Noting, given sequences of sequences, $\textup{S}_1$,..., $\textup{S}_n$, and a sequence, $S$, that
        \begin{displaymath}
        (\textup{S}_1\oplus \cdots \oplus \textup{S}_n)\triangleleft S = (\textup{S}_1\triangleleft S)\oplus \cdots \oplus (\textup{S}_n\triangleleft S),
        \end{displaymath}
        we see that the RHS becomes
        \begin{multline*}
        \tesequence{\trail_i}\triangleleft \trail_i= \Big(\big(\tesequence{[a_1,a]\oplus \trail_i}\triangleleft [a_1,a] \big)\triangleleft\trail_i\Big) \oplus \cdots \\
        \oplus \Big(\big(\tesequence{[a_n,a]\oplus \trail_i}\triangleleft [a_n,a]\big) \triangleleft\trail_i\Big).
        \end{multline*}
        Now, using the fact that $(\textup{S}_j\triangleleft S)\triangleleft S' = \textup{S}_j \triangleleft (S\oplus S')$, we deduce that
        \begin{multline*}
        \tesequence{\trail_i}\triangleleft \trail_i= \Big(\tesequence{[a_1,a]\oplus \trail_i}\triangleleft \big([a_1,a] \oplus\trail_i\big)\Big) \oplus \cdots \\ \oplus \Big(\tesequence{[a_n,a]\oplus \trail_i}\triangleleft \big([a_n,a]\oplus \trail_i\big)\Big).
        \end{multline*}
        But this is exactly (*) above, which is what we needed to show so we are done. Thus, for all orderly trail sequence $\otsequence=[\trail_1,...,\trail_m]$, we have shown that
\begin{displaymath}
\tesequence{[]_t}= \big(\tesequence{\trail_1}\triangleleft \trail_1\big) \oplus \cdots \oplus \big(\tesequence{\trail_m}\triangleleft \trail_m\big),
\end{displaymath}
by induction.
\end{proof}

\begin{corollary}\label{cor:trailExtensionSequencesEssentiallyMatch}
Let $\cpair$ be a construction with generalised split $\csplit$ identified by orderly trail sequence $\otsequence=[\trail_1,...,\trail_m]$. Then
\begin{displaymath}
\tesequence{[]_t,\cpair}= \big(\tesequence{\trail_1,\cpair}\triangleleft \trail_1\big) \oplus \cdots \oplus \big(\tesequence{\trail_m,\cpair}\triangleleft \trail_m\big).
\end{displaymath}
\end{corollary}

Our next task is to prove that the trail extension sequence of a trail, $\trail$, in $\cpair$ is the same as the trail extension sequence of $[]_{t'}$ in the induced construction $\icp{\trail}$, where $t'$ is the source of $\trail$. Since $\tesequence{[]_{t'},\icp{\trail}}=\ctsequence{\icp{\trail}}$, this result provides an important link between trail extension sequences and complete trail sequences.

\begin{lemma}\label{lem:DTSofTRisDTSofBot}
Let $\cpair$ be a construction containing a well-formed trail, $\trail$, with source $t'$ and target $t$. It is the case that
\begin{displaymath}
\tesequence{\trail,\cpair}=\tesequence{[]_{t'},\icp{\trail}}.
\end{displaymath}
\end{lemma}

\begin{proof} We have the following facts:
\begin{enumerate}
\item both $\tesequence{\trail,\cpair}$ and $\tesequence{[]_{t'},\icp{\trail}}$  contain trails that target $t'$, and
\item for each configurator, $u$, that appears in both $\cpair$ and $\icp{\trail}$, the neighbourhoods $\neigh{u,\cgraphn}$ and $\neigh{u,\icp{\trail}}$ are identical.
\end{enumerate}
Therefore, following definition~\ref{defn:DTS}, we can apply identical recursive steps to create $\tesequence{\trail,\cpair}$ and $\tesequence{[]_{t'},\icp{\trail}}$ until we finish constructing both sequences, in which case $\tesequence{\trail,\cpair}=\tesequence{[]_{t'},\icp{\trail}}$, or we reach a step where there is a trail extendable by an arrow, $a$, in exactly one of $\cpair$ and $\icp{\trail}$. Since $\icgp{\trail}$ is a subgraph of $\cgraphn$, such an arrow $a$ must be in $\cpair$ but not $\icp{\trail}$. Suppose that such an arrow $a$ exists. Then there are trails, $\trail_a$ and $\trail_a'$, in $\cpair$ where the trail $\trail_a'\oplus [a]\oplus \trail_a$ is in $\tesequence{\trail,\cpair}$. In turn, this implies that $a$ is in $\trailToGraph{\tesequence{\trail,\cpair}}$. But this graph is precisely $\icgp{\trail}$, by definition~\ref{defn:inducedConstruction}. In turn, this means that $a$ is in $\icp{\trail}$, which is a contradiction. Therefore, we can apply identical steps to create $\tesequence{\trail,\cpair}$ and $\tesequence{[]_{t'},\icp{\trail}}$ and we deduce that $\tesequence{\trail,\cpair}=\tesequence{[]_{t'},\icp{\trail}}$, as required.
\end{proof}

\begin{theorem}\label{athm:GenSplitPreservesCompleteTrailSequence}
Let $\cpair$ be a construction with  generalised split $\csplit$ identified by orderly trail sequence $\otsequence=[\trail_1,\ldots,\trail_m]$. Then
\begin{displaymath}
\ctsequence{\cgraphn,t} = \big(\ctsequence{\icp{\trail_1}}\triangleleft \trail_1\big) \oplus\cdots \oplus \big(\ctsequence{\icp{\trail_n}}\triangleleft\trail_m\big).
\end{displaymath}
\end{theorem}

\begin{proof}[Proof]
There are three main steps:
\begin{enumerate}
\item The trail extension sequence for $[]_t$ in $\cpair$ can be readily obtained from the trail extension sequences arising from the trails in $\ctsequence{\cgraphn',t}=[\trail_1,\ldots,\trail_m]$:
\begin{displaymath}
\tesequence{[]_t,\cpair} = \big(\tesequence{\trail_1,\cpair}\triangleleft \trail_1\big) \oplus\cdots \oplus \big(\tesequence{\trail_m,\cpair}\triangleleft\trail_m\big),
\end{displaymath}
by corollary~\ref{cor:trailExtensionSequencesEssentiallyMatch}.
\item For each trail, $\trail_i$, with source $t_i$, in $\ctsequence{\cgraphn',t}=[\trail_1,\ldots,\trail_m]$ the trail extension sequence $\tesequence{\trail_i,\cpair}$ is the same as $\tesequence{[]_{t_i},\icp{\trail_i}}$, by lemma~\ref{lem:DTSofTRisDTSofBot}:
\begin{eqnarray*}
  \tesequence{\trail_i,\cpair} & = & \tesequence{[]_{t_i},\icp{\trail_i}}\\
                                & = & \ctsequence{\icp{\trail_i}} \qquad \textup{by definition~\ref{defn:CTS}}.
\end{eqnarray*}
\item Now, using (2) and the fact that $\tesequence{[]_t,\cpair}=\ctsequence{\cgraphn,t}$, substituting into (1), we obtain:
\begin{displaymath}
\ctsequence{\cgraphn,t} = \big(\ctsequence{\icp{\trail_1}}\triangleleft \trail_1\big) \oplus\cdots \oplus \big(\ctsequence{\icp{\trail_m}}\triangleleft\trail_m\big)
\end{displaymath}
as required.
\end{enumerate}
\end{proof}

Theorem~\ref{thm:splitPreservesCompleteTrailSequence} follows from theorem~\ref{athm:GenSplitPreservesCompleteTrailSequence} because we know that all splits are generalised splits.

%% file: appendixPatterns.tex
\begin{lemma}\label{lema:embeddingRels}
Let $\cspacen$ be a construction space. Let $\cpair$ be a construction in $\cspacen$ that matches some pattern, $\ppair$. Given an embedding, $f\colon \cpair \to \ppair$,  the following hold:
\begin{enumerate}
\item if \begin{displaymath}
  \ctsequence{\cgraphn,t}= [[a_{1,1},\ldots,a_{j,1}],\ldots,[a_{1,n},\ldots, a_{k,n}]]
\end{displaymath}
then
\begin{displaymath}
  \ctsequence{\pgraphn,v}= [[f(a_{1,1}),\ldots,f(a_{j,1})],\ldots,[f(a_{1,n}),\ldots,f(a_{k,n})]].
\end{displaymath}
\item if $\foundationsto{\cgraphn,t}=[t_1,\ldots,t_n]$ then $\foundationsto{\pgraphn,v}=[f(t_1),\ldots,f(t_n)]$, and
\item $\foundationsty{\cgraphn,t}$ is a specialisation of  $\foundationsty{\pgraphn,v}$.
\end{enumerate}
\end{lemma}

\begin{proof}[Proof Sketch]
Since the embedding $f\colon \cpair \to \ppair$ is an isomorphism that maps $t$ to $v$ and ensures that arrow indices match, it follows that if
%
%\begin{displaymath}
$  \ctsequence{\cgraphn,t}= [[a_{1,1},\ldots,a_{j,1}],\ldots,[a_{1,n},\ldots,a_{k,n}]]$
%\end{displaymath}
%
then
%\begin{displaymath}
 $ \ctsequence{\pgraphn,v}= [[f(a_{1,1}),\ldots,f(a_{j,1})],\ldots,[f(a_{1,n}),\ldots,f(a_{k,n})]].$
%\end{displaymath}
Therefore (1) holds. In which case, since $f(t)=v$ and the foundation token-sequences are obtained from the sources of trails in complete sequences, we know that if
%
%\begin{displaymath}
$  \foundationsto{\cgraphn,t}= [t_{1},\ldots,t_{n}]$
%\end{displaymath}
%
then
%
%\begin{displaymath}
$  \foundationsto{\pgraphn,v}= [f(t_{1}),\ldots,f(t_{n})].$
%\end{displaymath}
%
Therefore, (2) holds. For part (3), given the vertex labelling functions, $\tokenl$ and $\tokenl'$,  of $\cgraphn$ and $\pgraphn$ respectively, we take
%
%\begin{displaymath}
$  \foundationsty{\cgraphn,t}= [\tokenl(t_{1}),\ldots,\tokenl(t_{n})]$
%\end{displaymath}
%
and
%\begin{displaymath}
 $ \foundationsty{\pgraphn,v}= [\tokenl'(f(t_{1})),\ldots,\tokenl'(f(t_{n}))],$
%\end{displaymath}
%
Definition~\ref{defn:embeddingPattern} gives us that, for each $i$, $\tokenl(t_{1,i})\leq \tokenl'(f(t_{1,i}))$. By definition~\ref{defn:specialization}, it follows that $\foundationsty{\cgraphn,t}$ is a specialisation of $\foundationsty{\pgraphn,v}$. Hence (3) holds.
\end{proof}

\begin{lemma}\label{lema:decTrivialOrBasic}
Let $\cspacen=\cspace$ be a construction space containing construction $\cpair$. Then there exists a decomposition, $\decompositionn$, that is a decoupling of $\cpair$.
\end{lemma}

\begin{proof}[Proof Sketch]
The proof uses a straightforward induction argument on the number of configurators. For the two base cases, $\cpair$ is trivial or basic (i.e $\cpair$ contains $0$ or $1$ configurators). Take $\decompositionn$ to be a single-vertex decomposition, where the vertex is labelled by $\cpair$. For the inductive step, $\cpair$ contains at least two configurators. In this case, take $u$ to be the unique configurator in $\cpair$ whose outgoing arrow targets $t$. Produce the split $((\neigh{u},t),[\ic_1,\ldots,\ic_n])$ of $\cpair$. By the inductive assumption, each induced construction, $\ic_i$, has a decoupling decomposition, $\decompositionn_i$. These decompositions -- or label-isomorphic copies of them -- together with a single-vertex, $v_r$, labelled by $(\neigh{u},t)$ can be used to form a decoupling decomposition, $\decompositionn$ of $\cpair$: $\decompositionn$ has root $v_r$ and $\dsequence{v_r}=[\decompositionn_1,\ldots,\decompositionn_n]$. By construction, $\decompositionn$ has vertices labelled only by trivial or basic constructions, as required. Hence $\decompositionn$ is a decoupling of $\cpair$.
\end{proof}

%% file: appendixStructuralTransformations.tex
\begin{lemma}\label{lema:equivTokenRelSpecs}
Let $\tRelSpecn=(\ppaira,\ppairb,\setofpatternstc)$ and $\tRelSpecn'=(\ppairad,\ppairbd,\setofpatternstc')$ be equivalent transformation constraints. Let $\tRelSpecn''=(\ppairadd,\ppairbdd,\setofpatternstc'')$ be a transformation constraint. Then $\tRelSpecn''$ satisfies $\tRelSpecn$ iff $\tRelSpecn''$ satisfies $\tRelSpecn'$.
\end{lemma}

\begin{proof}
It is sufficient to show that if $(\ppairadd,\ppairbdd,\setofpatternstc'')$ satisfies $\tRelSpecn$ then $(\ppairadd,\ppairbdd,\setofpatternstc'')$ satisfies $\tRelSpecn'$, since the other direction is similar. Suppose that $(\ppairadd,\ppairbdd,\setofpatternstc'')$ satisfies $\tRelSpecn$. Then, by definition~\ref{defn:satisfies}, there exists embeddings $f_{\pgraphan''}\colon \ppairadd\to \ppaira$ and $f_{\pgraphbn''}\colon \ppairbdd\to \ppairb$, and a respectful embedding $f_{\setofpatternstc''}\colon \patterngraphn(\setofpatternstc'')\to \patterngraphn(\setofpatternstc)$ such that
\begin{displaymath}
f_{\pgraphan''}\cup f_{\pgraphbn''}\cup f_{\setofpatternstc''}\colon \pgraphan''\cup \pgraphbn''\cup  \patterngraphn(\setofpatternstc'') \to \pgraphan\cup \pgraphbn \cup \patterngraphn(\setofpatternstc)
\end{displaymath}
is an isomorphism. Now, since $\tRelSpecn=(\ppaira,\ppairb,\setofpatternstc)$ and $\tRelSpecn'=(\ppairad,\ppairbd,\setofpatternstc')$ are equivalent, by definition~\ref{defn:equivalentTokenRelSpecs} we know that $\tRelSpecn\models \tRelSpecn'$. By definition~\ref{defn:satisfies}, there exists label-preserving embeddings, $f_{\pgraphan}\colon \ppaira \to \ppairad$ and $f_\pgraphbn\colon \ppairb\to \ppairbd$, and a respectful embedding $f_\setofpatternstc\colon \patterngraphn(\setofpatternstc)\to \patterngraphn(\setofpatternstc')$, such that
\begin{displaymath}
f_\pgraphan\cup f_\pgraphbn\cup f_\setofpatternstc\colon \pgraphan\cup \pgraphbn\cup \patterngraphn(\setofpatternstc) \to \pgraphan'\cup \pgraphbn'\cup \patterngraphn(\setofpatternstc')
\end{displaymath}
is an isomorphism.
To show that $(\ppairadd,\ppairbdd,\setofpatternstc'')$ satisfies $\tRelSpecn'$ we must show that there exist suitable embeddings whose union is an isomorphism, as specified by definition~\ref{defn:satisfies}. We show that $f_\pgraphan\circ f_{\pgraphan''}\colon \ppairadd \to \ppairad$, $f_\pgraphbn\circ f_{\pgraphbn''}\colon \ppairbdd \to \ppairbd$, and $f_\setofpatternstc\circ f_{\setofpatternstc''}\colon  \patterngraphn(\setofpatternstc'')  \to  \patterngraphn(\setofpatternstc') $ are such embeddings. Taking $f_\pgraphan\circ f_{\pgraphan''}\colon \ppairadd \to \ppairad$, we show that the conditions in definition~\ref{defn:embeddingPattern}, of an embedding, are satisfied and that $a''$ maps to $a'$. Firstly, it is trivial that the composition of isomorphisms is itself an isomorphism. In particular, $f_\pgraphan\circ f_{\pgraphan''}\colon \ppairadd \to \ppairad$ is an isomorphism since $f_{\pgraphan''}\colon \ppairadd \to \ppaira$ and $f_\pgraphan\colon \ppaira \to \ppairad$ are isomorphisms. Four conditions remain:
\begin{enumerate}
\item It is required that $a''$ maps to $a'$. Since $f_{\pgraphan''}$ and $f_\pgraphan$ are embeddings, $f_{\pgraphan''}(a'')=a$ and $f_\pgraphan(a)=a'$, so we have it that $f_\pgraphan\circ f_{\pgraphan''}(a'')=a'$.
\item The subtype relation must be respected by $f_\pgraphan\circ f_{\pgraphan''}$. We know that, for all tokens, $a'''$, in $\pgraphan''$, the vertex-labelling functions ensure that $\tokenl(a''')\leq \tokenl(f_{\pgraphan''}(a'''))$ and, since $f_\pgraphan$ is label-preserving, $\tokenl(f_{\pgraphan''}(a'''))=\tokenl(f_\pgraphan(f_{\pgraphan''}(a''')))$. Thus, $\tokenl(a''')\leq \tokenl(f_\pgraphan(f_{\pgraphan''}(a''')))=\tokenl(f_\pgraphan\circ f_{\pgraphan''}(a'''))$.
\item The constructors assigned to configurators must match. Since $f_{\pgraphan''}$ and $f_\pgraphan$ ensure that the constructors assigned to configurators match, it is trivial that for all configurators, $u$, in $\pgraphan''$, $\consl(u)=\consl(f_\pgraphan\circ f_{\pgraphan''}(u))$.
\item Lastly, since $f_{\pgraphan''}$ and $f_\pgraphan$ ensure that the arrow labels match, it is trivial that $f_\pgraphan\circ f_{\pgraphan''}$ also ensures that the arrow labels match.
\end{enumerate}
Hence $f_\pgraphan\circ f_{\pgraphan''}\colon \ppairadd \to \ppairad$ is an embedding. The arguments that establish $f_\pgraphbn\circ f_{\pgraphbn''}\colon \ppairbdd \to \ppairbd$ is an embedding and $f_\setofpatternstc\circ f_{\setofpatternstc''}\colon  \patterngraphn(\setofpatternstc'')  \to  \patterngraphn(\setofpatternstc')$ is a respectful embedding are similar. The last part of the proof requires us to show that
\begin{displaymath}
  (f_\pgraphan\circ f_{\pgraphan''}) \cup (f_\pgraphbn\circ f_{\pgraphbn''}) \cup (f_\setofpatternstc\circ f_{\setofpatternstc''}) \colon \pgraphan'' \cup \pgraphbn'' \cup \patterngraphn(\setofpatternstc'') \to \pgraphan'\cup \pgraphbn' \cup \patterngraphn(\setofpatternstc')
\end{displaymath}
is an isomorphism. It is given above that
\begin{displaymath}
f_{\pgraphan''}\cup f_{\pgraphbn''}\cup f_{\setofpatternstc''}\colon \pgraphan''\cup \pgraphbn''\cup \patterngraphn(\setofpatternstc'') \to \pgraphan\cup \pgraphbn\cup \patterngraphn(\setofpatternstc) \quad \textup{and} \quad f_\pgraphan\cup f_\pgraphbn\cup f_\setofpatternstc\colon \pgraphan\cup \pgraphbn\cup \patterngraphn(\setofpatternstc) \to \pgraphan'\cup \pgraphbn'\cup \patterngraphn(\setofpatternstc') \quad \textup{(i)}
\end{displaymath}
are isomorphisms. Noting that the composition
\begin{displaymath}
  (f_\pgraphan\cup f_\pgraphbn\cup f_\setofpatternstc)\circ (f_{\pgraphan''}\cup f_{\pgraphbn''}\cup f_{\setofpatternstc''})\colon \pgraphan'' \cup \pgraphbn'' \cup \patterngraphn(\setofpatternstc'') \to \pgraphan'\cup \pgraphbn' \cup \patterngraphn(\setofpatternstc') \qquad \textup{(ii)}
\end{displaymath} is obviously an isomorphism, we show that
\begin{displaymath}
 (f_\pgraphan\cup f_\pgraphbn\cup f_\setofpatternstc)\circ (f_{\pgraphan''}\cup f_{\pgraphbn''}\cup f_{\setofpatternstc''})= (f_\pgraphan\circ f_{\pgraphan''}) \cup (f_\pgraphbn\circ f_{\pgraphbn''}) \cup (f_\setofpatternstc\circ f_{\setofpatternstc''}) \qquad \textup{(iii)}.
\end{displaymath}
We first show that $(f_\pgraphan\circ f_{\pgraphan''}) \cup (f_\pgraphbn\circ f_{\pgraphbn''}) \cup (f_\setofpatternstc\circ f_{\setofpatternstc''})$ is a function. Trivially, the embedding $f_\pgraphan\circ f_{\pgraphan''}$ can be obtained by restricting the domain of (ii) to $\pgraphan''$. Similarly, $f_\pgraphbn\circ f_{\pgraphbn''}$ and $f_\setofpatterns\circ f_{\setofpatternstc''}$ are functions obtained by restricting the domain of (ii) to $\pgraphan''$ and, respectively, $\patterngraphn(\setofpatternstc'')$. Hence, their their union is a sub-function of (ii) with the same domain, and it follows that (iii) holds. We deduce that $(f_\pgraphan\circ f_{\pgraphan''}) \cup (f_\pgraphan\circ f_{\pgraphbn''}) \cup (f_\setofpatternstc\circ f_{\setofpatterns''})$ is an isomorphism. Therefore, we have it that $(\ppairadd,\ppairbdd,\setofpatternstc'')$ satisfies $\tRelSpecn'$, as required.
\end{proof}

\begin{theorem}\label{athm:completePartialIsTransformation}
Let $\irencodingn=\irencoding$ be an inter-representational-system encoding with description $\descriptionirse$. Let $\pdecompositionn$ and $\pdecompositionn'$ be pattern decompositions in $\descriptionn$ and $\descriptionn'$, respectively, with
constraint assignment $\vertexRelations\colon V(\pdecompositionn)\times V(\pdecompositionn') \to \setOfTokenRels$.  Let $\cpair$ be a construction in $\rsystemn$ that matches $\pdecompositionn$. Let $(\hat{\pgraphn},\hat{v})$ be a pattern for $\rsystemn'$ that matches $\pdecompositionn'$. If $(\hat{\pgraphn}',\hat{v}')$ is a construction in $\rsystemn'$ and a partial $\vertexRelations$-transformation of $\cpair$ then $(\hat{\pgraphn}',\hat{v}')$ is a structural $\vertexRelations$-transformation of $\cpair$.
\end{theorem}

\begin{proof}
Suppose that $(\hat{\pgraphn}',\hat{v}')$ is a construction in $\rsystemn'$ and a partial $\vertexRelations$-transformation of $\cpair$. Then, by definition~\ref{defn:partialLDeltaTransform}, $\vertexRelations$ is satisfied by $(\deltacanonical{\cpair,\pdecompositionn},\deltacanonical{(\hat{\pgraphn}',\hat{v}'),\pdecompositionn'})$. For notational convenience, we set $\decompositionn=\deltacanonical{\cpair,\pdecompositionn}$ and $\hat{\pdecompositionn}'=\deltacanonical{(\hat{\pgraphn}',\hat{v}'),\pdecompositionn'}$. By definition~\ref{defn:DeltaPartialSatisfication}, we know that for all $(d,\hat{\delta}')\in V(\decompositionn)\times V(\hat{\pdecompositionn}')$ if $\vertexRelations(f(d),f'(\hat{\delta}'))$ is defined and $\widehat{\labp}(\hat{\delta}')$ is a construction in $\rsystemn'$ then there exists a property identifier, $\setofconstructionsirse$, for $\irencodingn$ such that
%
%\begin{displaymath}
$  (\labc(d),\widehat{\labp}(\hat{\delta}'),\setofconstructionsirse) \models \vertexRelations(f(d),f'(\hat{\delta}')).$
%\end{displaymath}
where $f\colon \decompositionn\to \pdecompositionn$ and $f'\colon \hat{\pdecompositionn'} \to \pdecompositionn'$ are embeddings. Now, it is given that $(\hat{\pgraphn}',\hat{v}')$ is a construction in $\rsystemn'$. Therefore, it is also that case that for all $\hat{\delta}'\in V(\hat{\pdecompositionn}')$, $\widehat{\labp}(\hat{\delta}')$ is a construction in $\rsystemn'$. Thus, we know that
for all $(d,\hat{\delta}')\in V(\decompositionn)\times V(\hat{\pdecompositionn}')$ if $\vertexRelations(f(d),f'(\hat{\delta}'))$ is defined then there exists a property identifier, $\setofconstructionsirse$, for $\irencodingn$ such that
%
%\begin{displaymath}
$  (\labc(d),\widehat{\labp}(\hat{\delta}'),\setofconstructionsirse) \models \vertexRelations(f(d),f'(\hat{\delta}')).$
%\end{displaymath}
In turn, we deduce that $\vertexRelations$ is satisfied by the pair $(\decompositionn,\hat{\pdecompositionn}')$, as in definition~\ref{defn:structuralTransformationDecomposition}. Hence, $(\hat{\pgraphn}',\hat{v}')$ is a structural $\vertexRelations$-transformation of $\cpair$, by definition~\ref{defn:structuralTransformationConstruction}, as required.
\end{proof} 